\documentclass{article}

\PassOptionsToPackage{numbers,compress}{natbib}

\usepackage[final]{nips_2017}

\usepackage[utf8]{inputenc} 
\usepackage[T1]{fontenc}    
\usepackage{hyperref}       
\usepackage{url}            
\usepackage{booktabs}       
\usepackage{amsfonts}       
\usepackage{nicefrac}       
\usepackage{microtype}      

\usepackage[toc,page]{appendix}

\usepackage{bm}
\usepackage{amsmath}
\usepackage{amsfonts}
\usepackage{amssymb}
\usepackage{amsthm}
\usepackage{mathtools}
\usepackage{booktabs}
\usepackage{makeidx}
\usepackage{sectsty}

\newtheorem{theorem}{Theorem}
\newtheorem{definition}{Definition}
\newtheorem{corollary}[theorem]{Corollary}%
\newtheorem{lemma}[theorem]{Lemma}%
\newtheorem*{theorem*}{Theorem}%

\renewcommand\bibsection{} 

\newcommand\Ba{\bm{a}}
\newcommand\Bb{\bm{b}}

\newcommand\Bw{\bm{w}}
\newcommand\Bx{\bm{x}}
\newcommand\By{\bm{y}}
\newcommand\Bz{\bm{z}}
\newcommand\BA{\bm{A}}

\newcommand\BW{\bm{W}}

 \newcommand{\dR}{\mathbb{R}}

 \newcommand{\rN}{\mathrm{N}}

 \newcommand{\cN}{\mathcal{N}}

\newcommand\EXP{\mathbf{\mathrm{E}}}

\newcommand\VAR{\mathbf{\mathrm{Var}}}

\newcommand\nn{\mathrm{new}}

\newcommand\munn{{\tilde \mu}}
\newcommand\nunn{{\tilde \nu}}
\newcommand\xinn{{\tilde \xi}}

\renewcommand{\leq}{\leqslant}
\renewcommand{\geq}{\geqslant}

\DeclareMathOperator{\erfc}{erfc}
\DeclareMathOperator{\erf}{erf}
\DeclareMathOperator{\E}{E}
\DeclareMathOperator{\Var}{Var}

\DeclareMathOperator{\selu}{selu}

\allowdisplaybreaks

\makeindex

\title{Self-Normalizing Neural Networks}

\author{
 G\"{u}nter Klambauer
 \And 
 Thomas Unterthiner
 \And
 Andreas Mayr
 \And 
 Sepp Hochreiter \\ 
 LIT AI Lab \& Institute of Bioinformatics, \\
 Johannes Kepler University Linz\\
 A-4040 Linz, Austria\\
\texttt{\{klambauer,unterthiner,mayr,hochreit\}@bioinf.jku.at}
}

\begin{document}
\renewcommand\indexname{Brief index}
\maketitle

\begin{abstract}
Deep Learning has revolutionized vision via convolutional neural networks (CNNs) 
and natural language processing via recurrent neural networks (RNNs).
However, success stories of Deep Learning with 
standard feed-forward neural networks (FNNs) are rare. 
FNNs that perform well are typically shallow 
and, therefore cannot exploit many levels of abstract representations. 
We introduce self-normalizing neural networks (SNNs) to
enable high-level abstract representations.
While batch normalization requires explicit normalization, 
neuron activations of SNNs 
automatically converge towards zero mean and unit variance. 
The activation function of SNNs are ``scaled exponential linear units''
(SELUs), which induce self-normalizing properties.
Using the Banach fixed-point theorem, 
we prove that activations close to zero mean and unit variance 
that are propagated through many network layers will converge 
towards zero mean and unit variance ---
even under the presence of noise and perturbations.
This convergence property of SNNs allows to (1) train deep networks with many layers, (2) employ strong regularization 
schemes, and (3) to make learning highly robust. 
Furthermore, for activations not close to unit  
variance, we prove an upper and lower bound 
on the variance, thus, vanishing and exploding gradients are impossible.
We compared SNNs on (a) 121 tasks from the UCI machine learning repository, 
on (b) drug discovery benchmarks, and on (c) astronomy tasks
with standard FNNs, and other machine learning methods such as random forests and support vector machines.
For FNNs we considered  
(i) ReLU networks without normalization, 
(ii) batch normalization, (iii) layer normalization, (iv) weight normalization,
(v) highway networks, and (vi) residual networks. 
SNNs significantly outperformed all competing FNN methods at
121 UCI tasks, outperformed all competing methods 
at the Tox21 dataset, and set a new record at an astronomy data set. 
The winning SNN architectures are often very deep. 
Implementations are available at: \href{https://www.github.com/bioinf-jku/SNNs}{github.com/bioinf-jku/SNNs}.
\end{abstract}

Accepted for publication at NIPS 2017; please cite as: \\
{\tt Klambauer, G., Unterthiner, T., Mayr, A., \& Hochreiter, S. (2017). Self-Normalizing Neural Networks. In Advances in Neural Information Processing Systems (NIPS).}

\section*{Introduction}
\label{sec:introduction}
Deep Learning has set new records at different benchmarks and
led to various commercial applications \citep{bib:Lecun2015,bib:Schmidhuber2015}. 
Recurrent neural networks (RNNs) \citep{bib:Hochreiter1997} 
achieved new levels at speech and natural language processing, 
for example at the TIMIT benchmark \citep{bib:Graves2013} or at 
language translation \citep{bib:Sutskever2014}, and 
are already employed in mobile devices \citep{bib:Sak2015}. 
RNNs have won handwriting recognition challenges (Chinese and Arabic
handwriting) \cite{bib:Schmidhuber2015, bib:Graves2009,bib:Cirecsan2015}
and Kaggle challenges,
such as the ``Grasp-and Lift EEG'' competition. 
Their counterparts, convolutional neural networks (CNNs) \citep{bib:Lecun1995} excel
at vision and video tasks. 
CNNs are on par with human dermatologists at the 
visual detection of skin cancer \citep{bib:Esteva2017}. 
The visual processing for self-driving cars is based on CNNs \citep{bib:Huval2015},
as is the visual input to AlphaGo which has beaten
one of the best human GO players \citep{bib:Silver2016}.
At vision challenges, CNNs are constantly winning, for example at 
the large ImageNet competition \citep{bib:Krizhevsky2012, bib:He2015res}, but 
also almost all Kaggle vision challenges, such as  the ``Diabetic Retinopathy'' and 
the ``Right Whale'' challenges \citep{bib:Dugan2016,bib:Gulshan2016}. 

However, looking at Kaggle challenges that are not related to vision or sequential
tasks, gradient boosting, random forests, or support vector machines (SVMs) 
are winning most of the competitions. 
Deep Learning is notably absent, and for the few cases where FNNs won, 
they are shallow. For example, the HIGGS challenge,
the Merck Molecular Activity challenge, and the 
Tox21 Data challenge were all won by FNNs with at most four hidden layers.
Surprisingly, it is hard to find success stories with FNNs that
have many hidden layers, though they would allow for different levels
of abstract representations 
of the input \citep{bib:Bengio2013b}.

To robustly train very deep CNNs, batch normalization evolved into a standard to normalize
neuron activations to zero mean and unit variance \citep{bib:Ioffe2015}. 
Layer normalization \citep{bib:Ba2016} also ensures zero mean and unit
variance, while weight normalization \citep{bib:Salimans2016} ensures
zero mean and unit variance if in the previous layer the activations have
zero mean and unit variance.
However, training with normalization techniques is perturbed by
stochastic gradient descent (SGD), stochastic
regularization (like dropout), 
and the estimation of the normalization parameters.
Both RNNs and CNNs can stabilize learning via weight sharing, 
therefore they are less prone to these perturbations. 
In contrast, FNNs trained with normalization techniques suffer from
these perturbations and have high variance 
in the training error (see Figure~\ref{fig:perturb}). 
This high variance hinders learning and slows it down. 
Furthermore, strong regularization, such as dropout, 
is not possible as it would further 
increase the variance which in turn would lead to divergence of  the
learning process.
We believe that this sensitivity to perturbations 
is the reason that FNNs are less
successful than RNNs and CNNs.

Self-normalizing neural networks (SNNs) are robust to perturbations
and do not have high variance in their training errors (see Figure~\ref{fig:perturb}).
SNNs push neuron activations to zero mean and unit variance 
thereby leading to the same effect as batch normalization,
which enables to robustly learn many layers. 
SNNs are based on scaled exponential linear units ``SELUs''
which induce self-normalizing properties like variance stabilization
which in turn avoids exploding and vanishing gradients.

\section*{Self-normalizing Neural Networks (SNNs)}
\paragraph{Normalization and SNNs.}
\index{definitions}
For a neural network with activation function $f$, 
we consider two consecutive layers 
that are connected by a weight matrix $\BW$.
Since the input to a neural
network is a random variable, 
the activations $\Bx$ in the lower
layer, the network inputs $\Bz=\BW \Bx$, and the 
activations $\By=f(\Bz)$ in the higher layer are
random variables as well. 
We assume that all activations $x_i$ of the lower layer
have mean 
$\mu:=\EXP(x_i)$ and variance $\nu:=\VAR(x_i)$. 
An activation $y$ in the
higher layer has mean
$\munn:=\EXP(y)$ and variance
$\nunn:=\VAR(y)$. 
Here $\EXP(.)$ denotes the expectation and
$\VAR(.)$ the variance of a random variable.
A single activation $y=f(z)$ has net input 
$z=\Bw^T \Bx$. 
For $n$ units with activation 
$x_i, 1\leq i \leq n$ in the lower layer, we define $n$  
times the mean of the 
weight vector $\Bw \in \dR^n$ as $\omega:=\sum_{i=1}^n  w_i$ and $n$
times the second moment as $\tau:=\sum_{i=1}^n  w_i^2$.

We consider the mapping $g$ that maps mean and variance of the activations from one layer
to mean and variance of the activations in the next layer \index{mapping $g$}

\begin{align}
\begin{pmatrix}
\mu \\ \nu
\end{pmatrix}
\ &\mapsto \ 
\begin{pmatrix}
\munn \\ \nunn
\end{pmatrix} \ : \quad
\begin{pmatrix}
\munn \\ \nunn
\end{pmatrix}
  \ = \ g 
\begin{pmatrix}
\mu \\ \nu
\end{pmatrix}
 \ . 
\end{align}
Normalization techniques like batch, layer, or weight normalization 
ensure a mapping $g$ that keeps 
$(\mu,\nu)$ and $(\munn,\nunn)$
close to predefined values, typically $(0,1)$. 
\begin{definition}[Self-normalizing neural net] \index{self-normalizing neural networks}
\label{def:SNN}
A neural network is self-normalizing if it possesses a mapping 
$g:\Omega \mapsto \Omega$ for each activation $y$ that maps mean and variance from one layer to the next 
and has a stable and attracting fixed point depending on $(\omega,\tau)$ in $\Omega$.
Furthermore, the mean and the variance remain in the domain $\Omega$, that is $g(\Omega) \subseteq \Omega$, where
$\Omega= \{ (\mu, \nu)\ |\ \mu \in [\mu_{\min},\mu_{\max}], \nu \in [\nu_{\min},\nu_{\max}] \}$.
When iteratively applying the mapping $g$, each point within $\Omega$ converges to this fixed point.
\end{definition}
Therefore, we consider activations of a neural network to be normalized, 
if both their mean and their variance across samples are within predefined intervals. 
If mean and variance of $\Bx$ are already within these intervals, then also
mean and variance of $\By$ remain in these intervals, i.e., the
normalization is transitive across layers. Within these intervals, 
the mean and variance both converge to a fixed point if the mapping $g$ is applied
iteratively. 

\setlength{\belowcaptionskip}{0pt}
\begin{figure}
 \includegraphics[width=0.49\columnwidth]{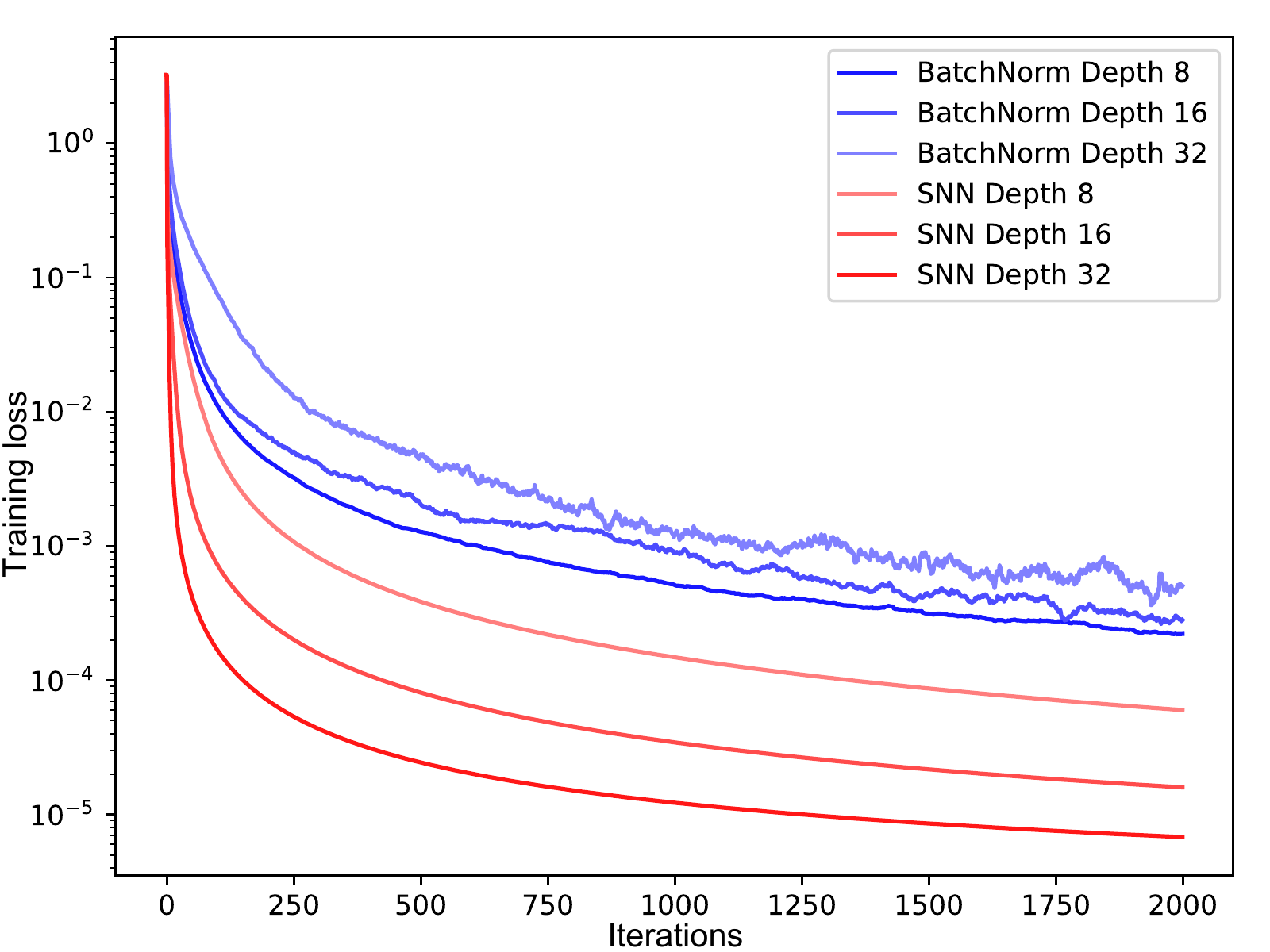}
 \includegraphics[width=0.49\columnwidth]{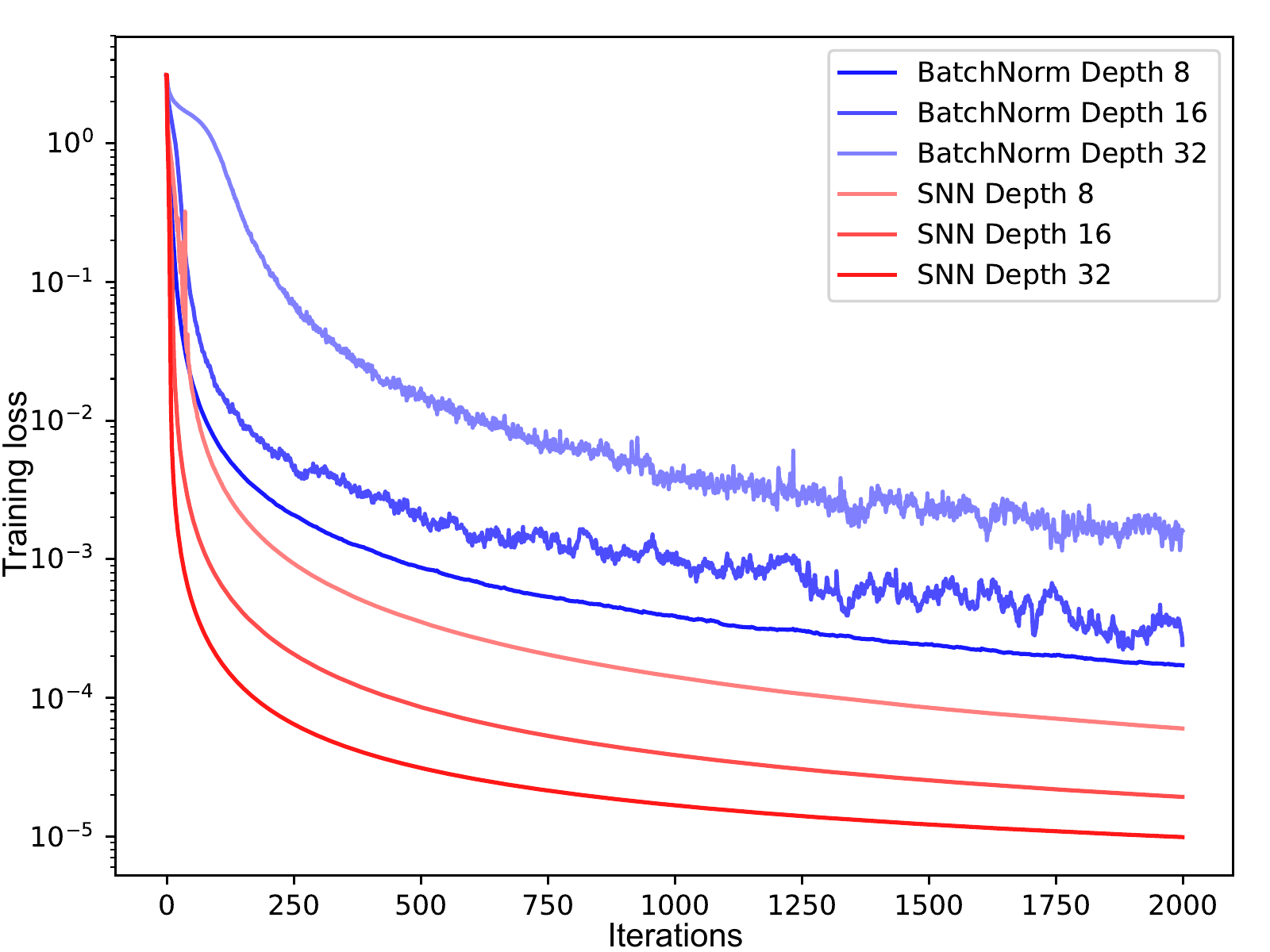}
 \caption[FNN and SNN trainin error curves]{The left panel and the right panel show the training error (y-axis) for feed-forward neural networks (FNNs) with batch
   normalization (BatchNorm) and self-normalizing networks (SNN) across update steps (x-axis)
   on the MNIST dataset the CIFAR10 dataset, respectively.  
   We tested networks with 8, 16, and 32 layers and learning rate 1e-5. FNNs 
   with batch normalization exhibit high variance due to perturbations. 
   In contrast, SNNs do not suffer from high variance as they are
   more robust to perturbations and learn faster.  \label{fig:perturb}}
\end{figure}

Therefore, SNNs
keep normalization of activations when propagating them 
through layers of the network. 
The normalization effect is observed across layers of a network: 
in each layer the activations are getting closer to the fixed point.
The normalization effect can also observed be for two fixed layers across 
learning steps: perturbations of lower layer activations or
weights are damped in the higher layer by drawing the activations
towards the fixed point.
If for all $y$ in the higher layer, $\omega$
and $\tau$ of the corresponding weight vector are the same, then
the fixed points are also the same. In this case we have a
unique fixed point for all activations $y$. 
Otherwise, in the more general case, $\omega$
and $\tau$ differ for different $y$ but the mean activations are
drawn into $[\mu_{\min},\mu_{\max}]$ and the variances
are drawn into $[\nu_{\min},\nu_{\max}]$.

\paragraph{Constructing Self-Normalizing Neural Networks.}
We aim at constructing self-normalizing
neural networks by adjusting the properties of the function $g$.
Only two design choices are available for the
function $g$: 
(1) the activation function and 
(2) the initialization of the weights.

For the activation function,
we propose ``scaled exponential linear units'' (SELUs) to render a FNN
as self-normalizing. The SELU activation function is given by \index{SELU!definition}
\begin{align}
 \selu(x) \ &= \ \lambda 
\ \begin{cases}
  x & \text{if } x > 0 \\
 \alpha e^{x}-\alpha & \text{if } x \leq 0
 \end{cases}  \ .
\end{align}
SELUs allow to construct a mapping $g$ with properties that lead to SNNs.
SNNs cannot be derived with (scaled)
rectified linear units (ReLUs), sigmoid units, $\tanh$ units, and leaky
ReLUs.
The activation function is required to have 
(1) negative and positive values for controlling the mean,
(2) saturation regions (derivatives approaching zero) to dampen the variance if
it is too large in the lower layer, 
(3) a slope larger than one to increase the variance if
it is too small in the lower layer,
(4) a continuous curve. 
The latter ensures a fixed point, where variance damping is equalized by variance increasing.
We met these properties of the activation function by multiplying the exponential linear
unit (ELU) \citep{bib:Clevert2015} with $\lambda>1$ to ensure a slope larger than one
for positive net inputs. 

For the weight initialization, we 
propose $\omega=0$ and $\tau=1$ for all units in the higher layer.
The next paragraphs will show the advantages of this initialization. 
Of course, during learning these assumptions on the weight vector will
be violated. However, we can prove the self-normalizing property
even for weight vectors that are not normalized, therefore, the 
self-normalizing property can be kept during learning and weight changes.


\paragraph{Deriving the Mean and Variance Mapping Function $g$.}
We assume that the $x_i$ are independent from each other but share
the same mean $\mu$ and variance $\nu$. 
Of course, the independence assumptions is 
not fulfilled in general. We will elaborate on the independence
assumption below.  
The network input $z$ in the higher layer 
is $z=\Bw^T \Bx$ for which we can infer the following moments 
$\EXP(z) \ = \ \sum_{i=1}^n  \ w_i \ \EXP(x_i) \ = \ \mu \ \omega$
and 
$\VAR(z) \ = \ \VAR(\sum_{i=1}^n w_i \ x_i)  \ = \ \nu \ \tau$,
where we used the independence of the $x_i$.
The net input $z$ is a weighted sum of independent, 
but not necessarily identically distributed variables $x_i$,
for which the central limit theorem (CLT) states that $z$ approaches a normal distribution:
$z \sim \cN (\mu \omega  ,  \sqrt{\nu \tau})$ 
with density $p_{\rN}(z  ;  \mu \omega  , \sqrt{\nu \tau})$. 
According to the CLT, the larger $n$, the closer is $z$ to a normal distribution.
For Deep Learning, broad layers with hundreds of neurons $x_i$ are common. 
Therefore the assumption that $z$ is normally distributed is met well for most currently used
neural networks (see Figure~\ref{fig:clt}).
%
The function $g$ maps the mean and variance of activations in the lower layer to the mean
$\munn=\EXP(y)$ and variance $\nunn=\VAR(y)$ of the activations $y$ in the next layer:
\begin{align}
\label{eq:mappingG}
g:\begin{pmatrix}
\mu \\ \nu 
\end{pmatrix} \mapsto \begin{pmatrix}
\munn \\ \nunn 
\end{pmatrix}: \ \ \quad
&\munn(\mu,\omega, \nu, \tau)   \ = \
\int_{-\infty}^{\infty }\selu(z) \ 
p_{{\rN}}(z;\mu \omega,\sqrt{\nu \tau})\ \mathrm{d}z  \\ \nonumber
& \nunn(\mu,\omega, \nu, \tau)   \ = \ \int_{-\infty}^{\infty} \selu(z)^2 \ 
p_{{\rN}}(z;\mu \omega,\sqrt{\nu \tau})\ \mathrm{d}z  \ - \ (\munn)^2 \ . 
\end{align}
These integrals can be analytically computed and lead to following \index{mapping $g$}
mappings of the moments: 
\begin{align}
\label{eq:mappingMean}
\munn\
  &= \ \frac{1}{2} \lambda  \left((\mu \omega)
    \erf  \left(\frac{\mu \omega}{\sqrt{2} \sqrt{\nu \tau}}\right)
    +\right. \\ \nonumber 
    &\left. \alpha \ e^{\mu \omega + \frac{\nu \tau}{2}} \erfc \left(\frac{\mu \omega + \nu \tau}{\sqrt{2} \sqrt{\nu \tau}}\right) - 
    \alpha  \erfc \left(\frac{\mu \omega}{\sqrt{2} \sqrt{\nu \tau }}\right) + \sqrt{\frac{2}{\pi }} 
    \sqrt{\nu \tau} e^{-\frac{(\mu \omega)^2}{2 (\nu \tau)}}+ \mu \omega \right) \\
\label{eq:mappingVar}
\nunn \ &= \ \frac{1}{2} \lambda ^2 \left(\left((\mu \omega)^2+\nu \tau \right)   \left(2 - \erfc  \left(\frac{\mu \omega}{\sqrt{2}
  \sqrt{\nu \tau}}\right) \right)+  \alpha ^2 \left(-2 e^{\mu \omega+\frac{\nu \tau}{2}} \erfc \left(\frac{\mu \omega+\nu \tau}{\sqrt{2} \sqrt{\nu \tau}}  \right) \right. \right.\\ \nonumber
 &\left. \left. + e ^{2 (\mu \omega+\nu \tau)} \erfc \left(\frac{\mu \omega+2 \nu \tau}{\sqrt{2} \sqrt{\nu \tau}}\right)+ \erfc \left(\frac{\mu \omega}{\sqrt{2} \sqrt{\nu \tau}}\right)\right)+\sqrt{\frac{2}{\pi }} (\mu \omega) \sqrt{\nu \tau} 
 e^{-\frac{(\mu \omega)^2}{2 (\nu \tau)}}\right)-  \left(\munn \right)^2 \\ \nonumber
\end{align}

\paragraph{Stable and Attracting Fixed Point $\bm{(0,1)}$ for Normalized Weights.}
\label{sec:perfect}

We assume a normalized weight
vector $\Bw$ with $\omega=0$ and $\tau=1$. 
Given a fixed point $(\mu,\nu)$,
we can solve equations Eq.~\eqref{eq:mappingMean} and Eq.~\eqref{eq:mappingVar} for  $\alpha$ and
$\lambda$. 
We chose the fixed point $(\mu,\nu)=(0,1)$,
which is typical for activation normalization.
We obtain the fixed point equations $\munn=\mu=0$ and $\nunn=\nu=1$ 
that we solve  for  $\alpha$ and $\lambda$ and obtain the solutions \index{SELU!parameters} $\alpha_{\mathrm{01}}\approx \ 1.6733$ and $\lambda_{\mathrm{01}}\approx \ 1.0507$,
where the subscript ${\mathrm{01}}$ indicates that these are the parameters for fixed point $(0,1)$.
The analytical expressions for $\alpha_{\mathrm{01}}$ and $\lambda_{\mathrm{01}}$ are given in Eq.~\eqref{eq:alphalambda}.  
We are interested whether the fixed point $(\mu,\nu)=(0,1)$ is stable
and attracting. If the Jacobian of $g$ has a norm smaller than 1 at the
fixed point, then $g$ is a contraction mapping and the fixed point is stable.
The (2x2)-Jacobian $\mathcal J(\mu,\nu)$ of $g:(\mu,\nu) \mapsto (\munn,\nunn)$ evaluated at the fixed point $(0,1)$ with $\alpha_{\mathrm{01}}$ and 
$\lambda_{\mathrm{01}}$ is 
\begin{align}
\mathcal J(\mu,\nu) \ &= \ 
\begin{pmatrix}
 \partial \frac{\mu^{\nn}(\mu,\nu)}{\partial \mu}  & \partial
 \frac{\mu^{\nn}(\mu,\nu)}{\partial \nu} \\
~ & ~ \\
 \partial \frac{\nu^{\nn}(\mu,\nu)}{\partial \mu} & \partial \frac{\nu^{\nn}(\mu,\nu)}{\partial \nu}
 \end{pmatrix},\
  &\mathcal J(0,1) \ = \
\begin{pmatrix}
 0.0  & 0.088834 \\
 0.0  &  0.782648 
\end{pmatrix} \ .
\end{align}

The spectral norm of $\mathcal J(0,1)$ (its largest
singular value) is $0.7877<1$. That means $g$ is a contraction
mapping around the fixed point $(0,1)$ (the mapping is depicted in Figure~\ref{fig:arrows}).
Therefore, $(0,1)$ is a stable fixed point of
the mapping $g$.

\paragraph{Stable and Attracting Fixed Points for Unnormalized Weights.}
A normalized weight vector $\Bw$ cannot be ensured during learning.
For SELU parameters 
$\alpha = \alpha_{\rm 01}$ and $\lambda=\lambda_{\rm 01}$,
we show in the next theorem that 
if $(\omega,\tau)$ is close to $(0,1)$, then $g$ still has an
attracting and stable fixed point
that is close to $(0,1)$.
Thus, in the general case there still exists a stable fixed point
which, however, depends on $(\omega,\tau)$.
If we restrict $(\mu,\nu,\omega,\tau)$ to certain intervals, then we 
can show that $(\mu,\nu)$ is mapped to the respective intervals.
Next we present the central theorem of this paper,
from which follows that SELU
networks are self-normalizing under mild conditions on the weights.
\begin{theorem}[Stable and Attracting Fixed Points]
\label{lem:fixedPoint} \index{Theorem 1}
We assume $\alpha = \alpha_{\rm 01}$ and $\lambda=\lambda_{\rm 01}$.
We restrict the range of the variables to the following intervals
$\mu \in [-0.1,0.1]$,
$\omega \in [-0.1,0.1]$,
$\nu \in [0.8,1.5]$, and
$\tau \in [0.95,1.1]$, 
that define the functions' domain $\Omega$.
For $\omega=0$ and $\tau=1$, the mapping  Eq.~\eqref{eq:mappingG}
 has the stable
fixed point $(\mu,\nu)=(0,1)$, whereas for other $\omega$ and $\tau$ the mapping  Eq.~\eqref{eq:mappingG}
 has a stable and
attracting fixed point depending on $(\omega,\tau)$ in the 
$(\mu,\nu)$-domain: $\mu \in [-0.03106, 0.06773]$ and 
$\nu \in [0.80009,1.48617]$.
All points within the $(\mu,\nu)$-domain converge when
iteratively applying the mapping  Eq.~\eqref{eq:mappingG} to this fixed point.
\end{theorem}

\begin{figure}
 \includegraphics[width=\columnwidth]{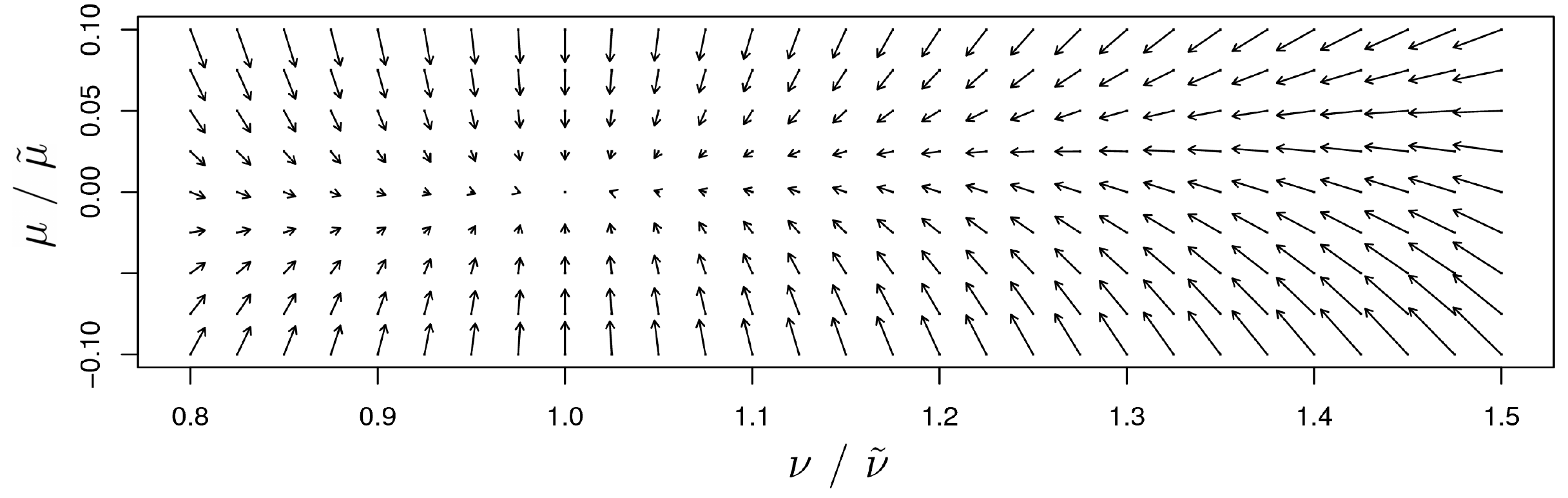}
 \caption[Visualization of the mapping $g$]{For $\omega=0$ and $\tau=1$, 
  the mapping $g$ of mean $\mu$ ($x$-axis) and variance $\nu$ ($y$-axis)
  to the next layer's mean $\munn$ and variance $\nunn$ 
  is depicted.
  Arrows show in which direction $(\mu,\nu)$ is mapped by $g:(\mu,\nu) \mapsto (\munn, \nunn)$.
  The fixed point of the mapping $g$ is $(0,1)$.
 \label{fig:arrows}}
\end{figure}

\begin{proof} \index{Theorem 1!proof sketch}
We provide a proof sketch (see detailed proof in Appendix~Section~\ref{sec:proofs}).
With the Banach fixed point theorem we show that there exists
a unique attracting and stable fixed point. 
To this end, we have to prove that  a) $g$ is a contraction mapping and b) that 
the mapping stays in the domain, that is, $g(\Omega)  \subseteq \Omega$. 
The spectral norm of the Jacobian of $g$ can be obtained via 
an explicit formula for the largest singular value for a $2\times2$ matrix. 
$g$ is a contraction mapping if its spectral norm is smaller than $1$.
We perform a computer-assisted proof to 
evaluate the largest singular value on a fine
grid and ensure the precision of the computer 
evaluation by an error propagation analysis of the implemented
algorithms on the according hardware.
Singular values between grid points are upper bounded by the 
mean value theorem. To this end, we bound the derivatives
of the formula for the largest singular value with respect to
$\omega,\tau,\mu,\nu$. 
Then we apply the mean value theorem to pairs of points, where one is on the grid and the
other is off the grid. This shows that for all values of
$\omega,\tau,\mu,\nu$ in  the domain $\Omega$, the spectral norm of
$g$ is smaller than one. 
Therefore, $g$ is a contraction mapping on the domain $\Omega$.
Finally, we show that the mapping $g$ stays in the domain $\Omega$ by
deriving bounds on $\munn$ and $\nunn$.
Hence, the Banach fixed-point theorem holds and there exists a unique
fixed point in $\Omega$ that is attained.
\end{proof}

Consequently, feed-forward neural networks with many units in each layer 
and with the SELU activation function are self-normalizing (see definition~\ref{def:SNN}), which 
readily follows from Theorem~\ref{lem:fixedPoint}.
To give an intuition, the main property of SELUs is that they damp the variance for negative
net inputs and increase the variance for positive net inputs. 
The variance damping is stronger if net inputs are further away from zero while
the variance increase is stronger if net inputs are close to zero.
Thus, for large variance of the activations in the lower
layer the damping effect is dominant and the variance decreases in the
higher layer.
Vice versa, for small variance the 
variance increase is dominant and the variance increases in the higher layer.

However, we cannot guarantee that mean and variance remain in the domain $\Omega$.
Therefore, we next treat the case where $(\mu,\nu)$ are outside $\Omega$.
It is especially crucial to consider $\nu$ because this variable has much stronger 
influence than $\mu$. Mapping $\nu$ across layers to a high value corresponds to an 
exploding gradient, since the  Jacobian of the activation of high layers with respect to activations
in lower layers has large singular values. 
Analogously, mapping $\nu$ across layers to a low value corresponds to an 
vanishing gradient. Bounding the mapping of $\nu$ from above and below would avoid 
both exploding and vanishing gradients.
Theorem~\ref{th:varDecrease} states that the variance of neuron activations of SNNs
is bounded from above, and therefore ensures that SNNs learn robustly and do not  
suffer from exploding gradients.

\begin{theorem}[Decreasing $\nu$]
\label{th:varDecrease} \index{Theorem 2}
For $\lambda=\lambda_{\rm 01}$, $\alpha=\alpha_{\rm 01}$
and the domain $\Omega^+$:
$-1 \leq \mu \leq 1$, 
$-0.1 \leq \omega \leq 0.1$,
$3 \leq \nu \leq 16$, and 
$0.8 \leq \tau \leq 1.25$, 
we have for the mapping of the variance
$\nunn(\mu,\omega,\nu,\tau,\lambda, \alpha )$  given in Eq.~\eqref{eq:mappingVar}:
$\nunn(\mu,\omega,\nu,\tau,\lambda_{\rm 01},\alpha_{\rm 01})  < \nu$.
\end{theorem}
The proof can be found in the Appendix~Section~\ref{sec:proofs}.
Thus, when mapped across many layers, the variance in the interval $[3,16]$ is mapped to a value below $3$. Consequently, all fixed
points $(\mu,\nu)$ of the mapping $g$ (Eq.~\eqref{eq:mappingG}) have $\nu<3$.
Analogously, Theorem~\ref{th:s2Increase} states that the variance of neuron activations of SNNs
is bounded from below, and therefore ensures that SNNs do not suffer from vanishing gradients.
\begin{theorem}[Increasing $\nu$]
\label{th:s2Increase} \index{Theorem 3}
We consider $\lambda=\lambda_{\rm 01}$, $\alpha=\alpha_{\rm 01}$
and the domain $\Omega^-$: 
$-0.1 \leq \mu \leq 0.1$, and
$-0.1 \leq \omega \leq 0.1$.
For the domain 
$0.02 \leq \nu \leq 0.16$
and $0.8 \leq \tau \leq 1.25$ as well as for the domain
$0.02 \leq \nu \leq 0.24$
and $0.9 \leq \tau \leq 1.25$,
the mapping of the variance
$\nunn(\mu,\omega,\nu,\tau,\lambda,\alpha )$  given in Eq.~\eqref{eq:mappingVar} 
increases:
$\nunn(\mu,\omega,\nu,\tau,\lambda_{\rm 01},\alpha_{\rm 01})  >  \nu$.
\end{theorem}
The proof can be found in the Appendix~Section~\ref{sec:proofs}.
All fixed
points $(\mu,\nu)$ of the mapping $g$ (Eq.~\eqref{eq:mappingG}) ensure for $0.8 \leq \tau$ that
$\nunn>0.16$ 
and for $0.9 \leq \tau$ that $\nunn>0.24$.
Consequently, the variance mapping Eq.~\eqref{eq:mappingVar} ensures a lower bound on the variance $\nu$. 
Therefore SELU networks control the variance of the activations and
push it into an interval, whereafter the mean and variance move toward
the fixed point. 
Thus, SELU networks are steadily normalizing the variance and
subsequently normalizing the mean, too. 
In all experiments, we observed that 
self-normalizing neural networks push the mean and variance of activations into the domain $\Omega$ .

\paragraph{Initialization.}
\label{sec:init} \index{initialization}
Since SNNs have a fixed point at zero mean and unit variance 
for normalized weights $\omega=\sum_{i=1}^n w_i=0$ and
$\tau=\sum_{i=1}^n w_i^2=1$ (see above), 
we initialize SNNs such that these
constraints are fulfilled in expectation.
We draw the weights from a Gaussian distribution 
with $\EXP(w_i)=0$ and variance $\VAR(w_i)=1/n$.
Uniform and truncated Gaussian distributions with these moments 
led to networks with similar behavior. 
The ``MSRA initialization''  is similar since 
it uses zero mean and variance $2/n$ to initialize the weights \citep{bib:He2015init}.
The additional factor $2$ counters the effect of rectified 
linear units.

\paragraph{New Dropout Technique.}
\label{sec:dropout} \index{dropout}
Standard dropout randomly sets an activation $x$ to zero with probability $1-q$ for $0 < q \leq 1$. 
In order to preserve the mean, the activations are scaled by $1/q$ during training. 
If $x$ has mean $\EXP(x)=\mu$ and variance
$\VAR(x)=\nu$, and the dropout variable $d$ follows
a binomial distribution $B(1,q)$, then the mean $\EXP(1/q d x)= \mu$ is kept.
Dropout fits well to rectified linear units, since
zero is in the low variance region and corresponds
to the default value.
For scaled exponential linear units, the default and low variance
value is $\lim_{x \to -\infty} \selu(x)=-\lambda \alpha=\alpha'$. 
Therefore, we propose ``alpha dropout'', 
that randomly sets inputs to $\alpha'$. 
The new mean and new variance is
$\EXP(x d + \alpha' (1-d)) = q \mu + (1-q) \alpha'$, and 
$\VAR(x d + \alpha' (1-d)) = q ((1-q)(\alpha'-\mu)^2+\nu )$.
We aim at keeping mean and variance to their original values after ``alpha
dropout'', in order to ensure the self-normalizing property even for ``alpha dropout''.
The affine transformation $a(x d + \alpha'(1-d))+b$ allows to
determine parameters $a$ and $b$ such that mean and variance are kept to their values: 
$\EXP(a(x d + \alpha'(1-d))+b)=\mu \ \ \text{and} \ \ \VAR(a(x d + \alpha'(1-d))+b)=\nu \ .$
In contrast to dropout, $a$ and $b$ will depend on $\mu$ and $\nu$,
however our SNNs converge to activations with
zero mean and unit variance.  
With $\mu=0$ and $\nu=1$, we obtain $a=\left(q+\alpha'^2q(1-q)\right)^{-1/2}$ and $b=-\left(q+\alpha'^2q(1-q)\right)^{-1/2} \left((1-q)\alpha'\right)$.
The parameters $a$ and $b$ only depend on the dropout rate $1-q$ 
and the most negative activation $\alpha'$. 
Empirically, we found that dropout rates $1-q=0.05$ or $0.10$ lead to models with good performance.
``Alpha-dropout'' fits well to scaled exponential linear units by randomly setting 
activations to the negative saturation value.

\paragraph{Applicability of the central limit theorem and independence assumption.}
\label{sec:clt} \index{central limit theorem}
In the derivative of the mapping (Eq.~\eqref{eq:mappingG}), we used the central limit theorem (CLT) 
to approximate the network inputs $z=\sum_{i=1}^n w_i x_i$ with a normal distribution.
We justified normality because network inputs represent a weighted sum of the inputs $x_i$, where for Deep Learning $n$ is typically large.
The Berry-Esseen theorem states that the convergence rate to normality is $n^{-1/2}$ \citep{bib:Korolev2012}. 
In the classical version of the CLT, the random variables have to be independent and identically 
distributed, which typically does not hold for neural networks.
However, the Lyapunov CLT does not require the variable to be identically distributed anymore. Furthermore,
even under weak dependence, sums of random variables converge in distribution to a Gaussian distribution \cite{bib:Bradley1981}.

\section*{Experiments}
\index{experiments}
We compare SNNs to other deep networks at different
benchmarks. 
Hyperparameters such as
number of layers (blocks), neurons per layer, learning rate, and dropout rate,
are adjusted by grid-search for each dataset on a separate validation set
(see Section~\ref{sec:experiments}). 
We compare the following FNN methods: \index{experiments!methods compared}
\begin{itemize}
\item {\bf ``MSRAinit'':} FNNs without normalization and 
with ReLU activations and ``Microsoft weight initialization'' \citep{bib:He2015init}.
\item {\bf ``BatchNorm'':} FNNs with batch normalization \citep{bib:Ioffe2015}. 
\item {\bf ``LayerNorm'':} FNNs with layer normalization \citep{bib:Ba2016}. 
\item {\bf ``WeightNorm'':} FNNs with weight normalization \citep{bib:Salimans2016}. 
\item {\bf ``Highway'':} Highway networks \citep{bib:Srivastava2015}.
\item {\bf ``ResNet'':} Residual networks \citep{bib:He2015res} adapted to FNNs  
using residual blocks with 2 or 3 layers with rectangular or diavolo shape. 
\item {\bf ``SNNs'':} Self normalizing networks with SELUs with $\alpha=\alpha_{\mathrm{01}}$ and $\lambda=\lambda_{\mathrm{01}}$ and 
the proposed dropout technique and initialization strategy. 
\end{itemize}

%

%

\paragraph{121 UCI Machine Learning Repository datasets.} \index{experiments!UCI}
The benchmark comprises 121 classification datasets from the UCI Machine Learning repository \cite{bib:Fernandez2014} 
from diverse application areas, such as physics, geology, or biology.
The size of the datasets ranges between $10$ and $130,000$ data points and the 
number of features from $4$ to $250$. 
In abovementioned work \citep{bib:Fernandez2014},
there were methodological mistakes \citep{bib:Wainberg2016} which we avoided here. 
Each compared FNN method
was optimized with respect to its architecture and hyperparameters on a validation set that was then 
removed from the subsequent analysis. 
The selected hyperparameters served to evaluate the methods in terms of accuracy on 
the pre-defined test sets (details on the hyperparameter selection are given in Section~\ref{sec:experiments}).
The accuracies are reported in the Table~\ref{tab:UCIfull}. 
We ranked the methods by their accuracy for each
prediction task and compared their average ranks.
SNNs significantly outperform all competing networks in pairwise comparisons (paired
Wilcoxon test across datasets) as reported in Table~\ref{tab:uci} (left panel).


\begin{table}[htp]
\caption[Comparison of seven FNNs on 121 UCI tasks]{{\bf Left:} Comparison of seven FNNs on 121 UCI tasks. 
We consider the average rank difference to rank $4$, which is
the average rank of seven methods with random predictions. 
The first column gives the method, the second 
the average rank difference, and the last the $p$-value 
of a paired Wilcoxon test whether the difference to the best performing 
method is significant.
SNNs significantly outperform all other methods.
{\bf Right:} Comparison of 24 machine learning methods (ML) on the UCI datasets
with more than 1000 data points. 
The first column gives the method, the second 
the average rank difference to rank $12.5$, and the last the $p$-value 
of a paired Wilcoxon test whether the difference to the best performing 
method is significant. Methods that were significantly worse than
the best method are marked with ``*''.
The full tables can be found in Table~\ref{tab:UCIfull}, Table~\ref{tab:uciS1} and Table~\ref{tab:uciS2}.
SNNs outperform all competing methods. 
\label{tab:uci} \label{tab:uci2}}
\centering
\begin{tabular}{lcclcc}
  \toprule
\multicolumn{3}{c}{FNN method comparison}   &  \multicolumn{3}{c}{ML method comparison} \\
 Method      & avg. rank diff. & $p$-value  & Method  &  avg. rank diff. & $p$-value \\ 
    \midrule
SNN         & -0.756$\ \ $ &  &  SNN &  -6.7$\ \ $  &  \\ 
MSRAinit    & -0.240* &    { 2.7e-02}  & SVM  &  -6.4$\ \ $  &  5.8e-01 \\ 
LayerNorm   & -0.198*  &    { 1.5e-02} &  RandomForest &  -5.9$\ \ $  &  2.1e-01 \\ 
Highway     & $\  $0.021*  &    { 1.9e-03} &  MSRAinit &  -5.4* &  { 4.5e-03} \\ 
ResNet      & $\  $0.273* &    { 5.4e-04} &  LayerNorm &  -5.3$\ \ $  &  7.1e-02 \\ 
WeightNorm  & $\  $0.397* &    { 7.8e-07} &  Highway &  -4.6* &  { 1.7e-03} \\ 
BatchNorm   & $\  $0.504* &    { 3.5e-06} &  $\ldots$ &   $\ldots$ &  $\ldots$ \\ 
   \bottomrule
\end{tabular}
\end{table}

We further included 17 machine learning methods representing diverse method groups \citep{bib:Fernandez2014} 
in the comparison and 
the grouped the data sets into ``small'' and ``large'' data sets (for details see Section~\ref{sec:experiments}).
On 75 small datasets with less than 1000 data points, random forests and SVMs outperform SNNs and other FNNs. 
On 46 larger datasets with at least 1000 data points, 
SNNs show the highest performance followed by SVMs and random forests (see right panel of Table~\ref{tab:uci2},
for complete results see Tables~\ref{tab:uciS1}~and~\ref{tab:uciS1}).
Overall, SNNs have outperformed state of the art machine learning methods on UCI datasets
with more than 1,000 data points.


Typically, hyperparameter selection chose SNN architectures that were
much deeper than the selected architectures of other FNNs, with an average depth of 10.8 layers, 
compared to average depths of 6.0 for BatchNorm, 3.8 WeightNorm, 7.0 LayerNorm, 5.9 Highway,  
and 7.1 for MSRAinit networks. For ResNet, the average number of blocks was 6.35. 
SNNs with many more than 4 layers often provide the best predictive accuracies across all neural networks.

\paragraph{Drug discovery: The Tox21 challenge dataset.} \index{experiments!Tox21}
The Tox21 challenge dataset comprises about 12,000 chemical compounds
whose twelve toxic effects have to be predicted based on their chemical structure. 
We used the validation sets 
of the challenge winners for hyperparameter selection (see Section~\ref{sec:experiments}) and 
the challenge test set for performance comparison. 
We repeated the whole evaluation procedure 5 times 
to obtain error bars. 
The results in terms of average AUC are given in Table~\ref{tab:tox21}.
In 2015, the challenge organized by the US NIH 
was won by an ensemble of shallow ReLU FNNs which achieved an AUC of 0.846 \citep{bib:Mayr2016}.
Besides FNNs, this ensemble also contained random forests and SVMs.
Single SNNs came close with an AUC of 0.845$\pm$0.003.
The best performing SNNs have 8 layers, compared to the runner-ups ReLU networks with layer normalization with 2 and 3 layers. 
Also batchnorm and weightnorm networks, typically perform best with shallow
networks of 2 to 4 layers (Table~\ref{tab:tox21}). The deeper the networks, the 
larger the difference in performance between SNNs and other methods (see columns 5--8 of Table~\ref{tab:tox21}).
The best performing method is an SNN with 8 layers.

\index{experiments!Tox21!hyperparameters}
\begin{table}[ht]
\caption[Comparison of FNNs at the Tox21 challenge dataset]{Comparison of FNNs at the Tox21 challenge dataset
in terms of AUC. The rows  represent different methods and the columns 
different network depth and for ResNets  the number of residual blocks 
(``na'': 32 blocks were omitted due to computational constraints).
The deeper the networks, the more prominent is the advantage of SNNs.
The best networks are SNNs with 8 layers. 
 \label{tab:tox21}}
\centering
\begin{tabular}{lccccccc}
  \toprule
  \multicolumn{8}{c}{\#layers / \#blocks} \\
  method & 2 & 3 & 4 & 6 & 8 & 16 & 32 \\ 
  \midrule 
  SNN        & {\em 83.7} \tiny $\pm$ 0.3 & {\bf 84.4} \tiny $\pm$ 0.5 & {\bf 84.2} \tiny $\pm$ 0.4 & {\bf 83.9} \tiny $\pm$ 0.5 & {\bf 84.5} \tiny $\pm$ 0.2 & {\bf 83.5} \tiny $\pm$ 0.5 & {\bf 82.5} \tiny $\pm$ 0.7 \\ 
  Batchnorm  & 80.0 \tiny $\pm$ 0.5       & 79.8 \tiny $\pm$ 1.6       & 77.2 \tiny $\pm$ 1.1       & 77.0 \tiny $\pm$ 1.7       & 75.0 \tiny $\pm$ 0.9       & 73.7 \tiny $\pm$ 2.0       & 76.0 \tiny $\pm$ 1.1 \\ 
  WeightNorm & {\em 83.7} \tiny $\pm$ 0.8 & 82.9 \tiny $\pm$ 0.8       & 82.2 \tiny $\pm$ 0.9       & {\em 82.5} \tiny $\pm$ 0.6 & {\em 81.9} \tiny $\pm$ 1.2 & 78.1 \tiny $\pm$ 1.3       & 56.6 \tiny $\pm$ 2.6 \\ 
  LayerNorm  & {\bf 84.3} \tiny $\pm$ 0.3 & {\em 84.3} \tiny $\pm$ 0.5 & {\em 84.0} \tiny $\pm$ 0.2 & {\em 82.5} \tiny $\pm$ 0.8 & 80.9 \tiny $\pm$ 1.8       & 78.7 \tiny $\pm$ 2.3       & 78.8 \tiny $\pm$ 0.8 \\ 
  Highway    & 83.3 \tiny $\pm$ 0.9       & 83.0 \tiny $\pm$ 0.5       & 82.6 \tiny $\pm$ 0.9       & 82.4 \tiny $\pm$ 0.8       & 80.3 \tiny $\pm$ 1.4       & 80.3 \tiny $\pm$ 2.4       & 79.6 \tiny $\pm$ 0.8 \\ 
  MSRAinit     & 82.7 \tiny $\pm$ 0.4       & 81.6 \tiny $\pm$ 0.9       & 81.1 \tiny $\pm$ 1.7       & 80.6 \tiny $\pm$ 0.6       & 80.9 \tiny $\pm$ 1.1       & 80.2 \tiny $\pm$ 1.1       & 80.4 \tiny $\pm$ 1.9 \\ 
  ResNet    & 82.2 \tiny $\pm$ 1.1       & 80.0 \tiny $\pm$ 2.0       & 80.5 \tiny $\pm$ 1.2       &   81.2 \tiny $\pm$  0.7                       & 81.8 \tiny $\pm$ 0.6       & {\em 81.2} \tiny $\pm$ 0.6 &   na \\ 
   \bottomrule
\end{tabular}
\end{table}

\paragraph{Astronomy: Prediction of pulsars in the HTRU2 dataset.} \index{experiments!HTRU2} \index{experiments!astronomy}

Since a decade, machine learning methods have been used to identify pulsars in radio wave signals \citep{bib:Lyon2016a}.
Recently, the High Time Resolution Universe Survey (HTRU2) dataset has been released with
1,639 real pulsars and 16,259 spurious signals.
Currently, the highest AUC value of a 10-fold cross-validation is 0.976
which has been achieved by Naive Bayes classifiers followed by decision tree C4.5 with 0.949 and SVMs with 0.929.
We used eight features constructed by the PulsarFeatureLab as used previously \citep{bib:Lyon2016a}.
We assessed the performance of FNNs using 10-fold nested cross-validation,
where the hyperparameters were selected in the inner loop on a validation set (for details on the hyperparameter selection see Section~\ref{sec:experiments}). 
Table~\ref{tab:HTRU2} reports the results
in terms of AUC. SNNs outperform all other methods and have pushed the state-of-the-art 
to an AUC of $0.98$.

\begin{table}[ht]
\caption[Comparison of FNNs and reference methods at HTRU2]{Comparison of FNNs and reference methods at HTRU2 
in terms of AUC. 
The first, fourth and seventh column give the method, 
the second, fifth and eight column the AUC averaged over 10 cross-validation folds, 
and the third and sixth column the $p$-value of a paired Wilcoxon test of the AUCs against 
the best performing method across the 10 folds. 
FNNs achieve better results than Naive Bayes (NB), C4.5, and SVM.
SNNs exhibit the best performance and set a new record. \label{tab:HTRU2}}
\centering
\begin{tabular}{lcclccllc}
  \toprule
  \multicolumn{3}{c}{FNN methods}   &  \multicolumn{3}{c}{FNN methods} & \multicolumn{2}{c}{ref. methods}  \\
   method &  AUC & $p$-value & method &  AUC & $p$-value & method &  AUC  \\ 
  \midrule
  SNN           &    0.9803$\ \ $ \tiny $\pm$  0.010 &    & & & & &                     \\ 
  MSRAinit      & 0.9791$\ \ $  \tiny $\pm$  0.010 & \tiny3.5e-01       &  LayerNorm  & 0.9762* \tiny $\pm$  0.011 & \tiny { 1.4e-02}     & NB &0.976      \\ 
  WeightNorm    & 0.9786*         \tiny $\pm$  0.010 & \tiny{ 2.4e-02} &  BatchNorm  & 0.9760$\ \ $ \tiny $\pm$  0.013 & \tiny 6.5e-02                & C4.5 & 0.946  \\ 
  Highway       & 0.9766*         \tiny $\pm$  0.009 & \tiny{ 9.8e-03} & ResNet & 0.9753* \tiny $\pm$ 0.010 & \tiny 6.8e-03       & SVM  & 0.929\\ 
   \bottomrule
\end{tabular}
\end{table}
%

\section*{Conclusion}
We have introduced self-normalizing neural networks for 
which we have proved that neuron activations are pushed towards zero mean and unit variance
when propagated through the network. 
Additionally, for activations not close to unit
variance, we have proved an upper and lower bound 
on the variance mapping. Consequently, SNNs do not face vanishing and exploding gradient 
problems.  Therefore, SNNs work well for architectures with many layers, allowed us to introduce a 
novel regularization scheme, and learn very robustly.
On 121 UCI benchmark datasets, SNNs have outperformed other FNNs with and without normalization techniques, 
such as batch, layer, and weight normalization, or specialized architectures, such as Highway or 
Residual networks. 
SNNs also yielded the best results on drug discovery and astronomy tasks.
The best performing SNN architectures are typically very deep in contrast to other FNNs. 


\section*{Acknowledgments}
This work was supported by IWT research grant IWT150865 (Exaptation), H2020 project
grant 671555 (ExCAPE), grant IWT135122 (ChemBioBridge), 
Zalando SE with Research Agreement 01/2016, 
Audi.JKU Deep Learning Center, Audi Electronic Venture GmbH, 
and the NVIDIA Corporation.

\section*{References}
The references are provided in Section~\ref{sec:references}.


\section*{Appendix}
\renewcommand{\thesection}{A\arabic{section}}
\renewcommand{\thefigure}{A\arabic{figure}}
\renewcommand{\thetable}{A\arabic{table}}

\sectionfont{\large}
\subsectionfont{\normalsize}
\subsubsectionfont{\normalsize}
\paragraphfont{\normalsize}


\setcounter{theorem}{0}
\tableofcontents


\vspace{1cm}
This appendix is organized as follows: the first section 
sets the background, definitions, and formulations.
The main theorems are presented in the next section.
The following section is devoted to the proofs of these theorems.
The next section reports additional results and details on the
performed computational experiments, such as hyperparameter selection. 
The last section shows that our theoretical bounds can be
confirmed by numerical methods as a sanity check.

The proof of theorem 1 is based on the Banach's fixed point theorem 
for which we require (1) a contraction mapping, which is proved in Subsection~\ref{sec:S}
and (2) that the mapping stays within its domain, which is proved in Subsection~\ref{sec:maptoregion}
For part (1), the proof relies on the main Lemma~12, which is a computer-assisted proof, and can be found 
in Subsection~\ref{sec:S}. The validity of the computer-assisted proof is shown in Subsection~\ref{sec:error} by 
error analysis and the precision of the functions' implementation. 
The last Subsection~\ref{sec:smallLemmata} compiles various lemmata with intermediate results that support 
the proofs of the main lemmata and theorems.

\section{Background}
\label{sec:fixedpointanalysis}
We consider a neural network with {\bf activation function} $f$ and 
two consecutive layers that are connected by {\bf weight matrix} $\BW$.
Since samples that serve as input to the neural network are chosen according to a distribution, 
the  {\bf activations $\Bx$ in the lower layer}, 
the {\bf network inputs} $\Bz=\BW \Bx$, and {\bf activations $\By=f(\Bz)$ in the
higher layer} are all random variables. We assume that all units $x_i$ in the lower layer
have {\bf mean activation} $\mu:=\E(x_i)$ and {\bf variance of the
activation}
$\nu:=\Var(x_i)$ and a unit $y$ in the
higher layer has mean activation $\munn:=\E(y)$ and variance
$\nunn:=\Var(y)$. Here $\E(.)$ denotes the expectation and
$\Var(.)$ the variance of a random variable.
For activation of unit $y$, we have net input  $z=\Bw^T \Bx$ and 
the {\bf scaled exponential linear unit (SELU)}
activation $y=\selu(z)$, with 
\begin{align}
 \selu(x) \ &= \ \lambda 
\ \begin{cases}
  x & \text{if } x > 0 \\
 \alpha e^{x}-\alpha & \text{if } x \leq 0
 \end{cases}  \ .
\end{align}
For $n$ units $x_i, 1\leq i \leq n$ in the lower layer and 
the {\bf weight vector} $\Bw \in \dR^n$, we define 
{\bf $n$ times the mean} by $\omega:=\sum_{i=1}^n  w_i$ 
and {\bf $n$ times the second moment} by $\tau:=\sum_{i=1}^n  w_i^2$.

We define a {\bf mapping $g$} from mean $\mu$ and
variance $\nu$ of one layer 
to the mean $\munn$ and  variance  $\nunn$ in the next layer:
\begin{align}
\label{eq:mapping}
g: (\mu,\nu) \mapsto (\munn,\nunn) \ .
\end{align}
For neural networks with scaled exponential linear units, 
the mean is of the activations in the next layer computed according to
\begin{align}
\munn \ &= \ \int_{-\infty}^{0} \lambda \alpha (\exp(z)- 1)
  p_{\mathrm{Gauss}}(z; \mu \omega ,\sqrt{\nu \tau}) dz \ + \ 
\int_{0}^{\infty} \lambda  z  p_{\mathrm{Gauss}}(z; \mu \omega ,\sqrt{\nu \tau}) dz \ ,
\end{align}
and the second moment of the activations in the next layer is computed according to
\begin{align}
\xinn \ &= \ \int_{-\infty}^{0} \lambda^2 \alpha^2 (\exp (z)- 1)^2
  p_{\mathrm{Gauss}}(z; \mu \omega ,\sqrt{\nu \tau}) dz \ + \ 
\int_{0}^{\infty} \lambda^2  z^2  p_{\mathrm{Gauss}}(z; \mu \omega ,\sqrt{\nu \tau}) dz \ .
\end{align}

Therefore, the expressions $\munn$ and $\nunn$ have the following form: \index{mapping $g$!definition}
\begin{align} 
&\munn(\mu,\omega,\nu,\tau, \lambda ,\alpha ) \
  = \frac{1}{2} \lambda  \left(-(\alpha +\mu  \omega ) \erfc \left(\frac{\mu  \omega }{\sqrt{2} \sqrt{\nu  \tau }}\right)+ \right. \\ \nonumber & \left.
    \alpha  e^{\mu  \omega +\frac{\nu  \tau }{2}} \erfc 
    \left(\frac{\mu  \omega +\nu  \tau }{\sqrt{2} \sqrt{\nu  \tau }}\right)+ 
    \sqrt{\frac{2}{\pi }} \sqrt{\nu  \tau } e^{-\frac{\mu ^2 \omega ^2}{2 \nu  \tau }}+2 \mu  \omega \right) \\  
&\nunn(\mu, \omega,\nu,\tau,\lambda ,\alpha )
  \ = \xinn(\mu, \omega,\nu,\tau,\lambda ,\alpha ) - \left( \munn(\mu,\omega,\nu,\tau, \lambda ,\alpha ) \right)^2 \\
\label{eq:mappingSecondMom}
&\xinn(\mu, \omega,\nu,\tau,\lambda ,\alpha )
  \ = \ \frac{1}{2} \lambda ^2 \left(\left((\mu \omega)^2+\nu \tau \right) 
  \left(\erf  \left(\frac{\mu \omega}{\sqrt{2}
  \sqrt{\nu \tau}}\right)+1\right)+\right. \\ \nonumber
 &\left. \alpha ^2 \left(-2 e^{\mu \omega+\frac{\nu \tau}{2}} \erfc \left(\frac{\mu \omega+\nu \tau}{\sqrt{2} \sqrt{\nu \tau}}\right)
 + e ^{2 (\mu \omega+\nu \tau)} \erfc \left(\frac{\mu \omega+2 \nu \tau}{\sqrt{2} \sqrt{\nu \tau}}\right)+\right.\right. \\ \nonumber
 &\left.\left. \erfc \left(\frac{\mu \omega}{\sqrt{2} \sqrt{\nu \tau}}\right)\right)+\sqrt{\frac{2}{\pi }} (\mu \omega) \sqrt{\nu \tau} 
 e^{-\frac{(\mu \omega)^2}{2 (\nu \tau)}}\right) 
\end{align}

We solve equations Eq.~\ref{eq:mappingMean} and
Eq.~\ref{eq:mappingVar} for fixed points $\munn= \mu$ and $\nunn= \nu$. 
For a normalized weight vector with $\omega=0$ and $\tau = 1$ and the
{\bf fixed point $(\mu,\nu)=(0,1)$},
 we can solve equations Eq.~\ref{eq:mappingMean} and 
Eq.~\ref{eq:mappingVar} for $\alpha$ and $\lambda$.
We denote the solutions to fixed point $(\mu,\nu)=(0,1)$
by  $\alpha_{\rm 01}$ and $\lambda_{\rm 01}$.
\begin{align}
\label{eq:alphalambda}
&\alpha_{\rm 01} =  -\frac{\sqrt{\frac{2}{\pi}}}{\erfc \left(\frac{1}{\sqrt{2}}\right) \exp \left(\frac{1}{2}\right)-1} \approx  1.67326\\ \nonumber
&\lambda_{\rm 01} =  \left(1-\erfc \left(\frac{1}{\sqrt{2}}\right) \sqrt{e}\right) \sqrt{2 \pi}  \\ \nonumber
&\left(2 \erfc \left(\sqrt{2}\right) e^2 +\pi
  \erfc \left(\frac{1}{\sqrt{2}}\right)^2 e -2 (2+\pi ) \erfc \left(\frac{1}{\sqrt{2}}\right) \sqrt{e}+\pi +2\right)^{-1/2}\\ \nonumber
&\lambda_{\rm 01} \approx  1.0507 \ .
\end{align}
The parameters $\alpha_{\rm 01}$ and $\lambda_{\rm 01}$ ensure \index{SELU!parameters}
\begin{align}
 &\munn(0, 0, 1, 1, \lambda_{\rm 01} ,\alpha_{\rm 01}) = 0 \nonumber \\
 &\nunn(0, 0, 1, 1, \lambda_{\rm 01} ,\alpha_{\rm 01}) = 1 \nonumber
\end{align}

Since we focus on the fixed point  $(\mu,\nu)=(0,1)$,
we assume throughout the analysis that $\alpha = \alpha_{\rm 01}$ and $\lambda=\lambda_{\rm 01}$.
We consider the functions $\munn(\mu, \omega,\nu,\tau,\lambda_{\rm 01} ,\alpha_{\rm 01})$, 
$\nunn(\mu, \omega,\nu,\tau,\lambda_{\rm 01} ,\alpha_{\rm 01})$, 
and $\xinn(\mu, \omega,\nu,\tau,\lambda_{\rm 01} ,\alpha_{\rm 01})$ 
on the {\bf domain} 
$\Omega = \{(\mu, \omega,\nu,\tau) \ | \  \mu \in [\mu_{\rm
  min},\mu_{\rm max}]= [-0.1,0.1], \omega \in [\omega_{\rm min},
\omega_{\rm max}]=[-0.1,0.1],
\nu \in [\nu_{\rm min},\nu_{\rm max}]=
[0.8,1.5], \tau \in [\tau_{\rm min},\tau_{\rm max}]=[0.95,1.1] \}$.

Figure~\ref{fig:arrows} visualizes 
the mapping $g$ for $\omega=0$ and $\tau=1$ and 
$\alpha_{\rm 01}$ and $\lambda_{\rm 01}$ at few pre-selected points.
It can be seen that $(0,1)$ is an attracting 
fixed point of the mapping $g$.



\section{Theorems}
\subsection{Theorem 1: Stable and Attracting Fixed Points Close to (0,1)}

\index{Theorem 1}
Theorem~\ref{lem:fixedPoint}
shows that the mapping $g$ defined by Eq.~\eqref{eq:mappingMean}
and Eq.~\eqref{eq:mappingVar} 
exhibits a stable and attracting fixed point close to zero mean and
unit variance. 
Theorem~\ref{lem:fixedPoint} establishes the self-normalizing property of self-normalizing
neural networks (SNNs). The stable and
attracting fixed point leads to robust learning through many layers.

\begin{theorem}[Stable and Attracting Fixed Points]
We assume $\alpha = \alpha_{\rm 01}$ and $\lambda=\lambda_{\rm 01}$.
We restrict the range of the variables to the domain \index{domain!Theorem 1}
$\mu \in [-0.1,0.1]$,
$\omega \in [-0.1,0.1]$,
$\nu \in [0.8,1.5]$, and
$\tau \in [0.95,1.1]$.
For $\omega=0$ and $\tau=1$, the mapping  Eq.~\eqref{eq:mappingMean}
and Eq.~\eqref{eq:mappingVar} has the stable
fixed point $(\mu,\nu)=(0,1)$.
For other $\omega$ and $\tau$ the mapping  Eq.~\eqref{eq:mappingMean}
and Eq.~\eqref{eq:mappingVar}  has a stable and
attracting fixed point depending on $(\omega,\tau)$ in the 
$(\mu,\nu)$-domain: $\mu \in [-0.03106, 0.06773]$ and 
$\nu \in [0.80009,1.48617]$.
All points within the $(\mu,\nu)$-domain converge when
iteratively applying the mapping  Eq.~\eqref{eq:mappingMean}
and Eq.~\eqref{eq:mappingVar} to this fixed point.
\end{theorem}

\subsection{Theorem 2: Decreasing Variance from Above}
The next Theorem~\ref{th:s2Decrease} states \index{Theorem 2}
that the variance of unit activations 
does not explode through
consecutive layers of self-normalizing networks.
Even more, a large variance of unit activations decreases when
propagated through the network. 
In particular this ensures that exploding gradients will never be
observed.
In contrast to the domain in previous subsection, 
in which $\nu \in [0.8, 1.5]$, we now consider a domain
in which the variance of the inputs is higher $\nu \in [3, 16]$ and even the 
range of the mean is increased $\mu \in[-1,1]$. We denote this new domain with 
the symbol $\Omega^{++}$ to indicate that the variance lies above the variance of the original domain $\Omega$.
In $\Omega^{++}$, we can show that the variance $\nunn$ in the next layer is always smaller 
then the original variance $\nu$.
Concretely, this theorem states that:

\begin{theorem}[Decreasing $\nu$]
\label{th:s2Decrease}
For $\lambda=\lambda_{\rm 01}$, $\alpha=\alpha_{\rm 01}$ 
and the domain $\Omega^{++}$: \index{domain!Theorem 2}
$-1 \leq \mu \leq 1$, 
$-0.1 \leq \omega \leq 0.1$,
$3 \leq \nu \leq 16$, and 
$0.8 \leq \tau \leq 1.25$ we have for 
the mapping of the variance
$\nunn(\mu,\omega,\nu,\tau,\lambda, \alpha )$  given in Eq.~\eqref{eq:mappingVar}
\begin{align}
\nunn(\mu,\omega,\nu,\tau,\lambda_{\rm 01},\alpha_{\rm 01}) \ &< \ \nu \ .
\end{align}
The variance decreases in $[3,16]$ and all fixed
points $(\mu,\nu)$ of mapping Eq.~\eqref{eq:mappingVar} and Eq.~\eqref{eq:mappingMean} have $\nu<3$.
\end{theorem}

\subsection{Theorem 3: Increasing Variance from Below}
The next Theorem~\ref{th:s2Increase} states \index{Theorem 3}
that the variance of unit activations 
does not vanish through
consecutive layers of self-normalizing networks.
Even more, a small variance of unit activations increases when
propagated through the network. 
In particular this ensures that vanishing gradients will never be
observed.
In contrast to the first domain, 
in which $\nu \in [0.8, 1.5]$, we now consider two domains $\Omega_1 ^-$ and
$\Omega_2 ^-$ in which the variance of the inputs is lower $0.05 \leq \nu \leq 0.16$ and $0.05 \leq \nu \leq 0.24$,
and even the parameter $\tau$ is different $0.9 \leq \tau \leq 1.25$ to the original $\Omega$. 
We denote this new domain with 
the symbol $\Omega^{-}_i$ to indicate that the variance lies below the variance of the original domain $\Omega$.
In $\Omega_1 ^-$ and $\Omega_2 ^-$, 
we can show that the variance $\nunn$ in the next layer is always larger 
then the original variance $\nu$, which means that the variance does not vanish through
consecutive layers of self-normalizing networks.
Concretely, this theorem states that: \index{Theorem 3}

\begin{theorem}[Increasing $\nu$]
We consider $\lambda=\lambda_{\rm 01}$, $\alpha=\alpha_{\rm 01}$
and the two domains 
$\Omega_1 ^-=\{(\mu,\omega,\nu,\tau) \ | \ -0.1 \leq \mu \leq 0.1, -0.1 \leq \omega \leq 0.1, 0.05 \leq \nu \leq 0.16, 0.8 \leq \tau \leq 1.25 \}$ 
 and
$\Omega_2 ^-=\{(\mu,\omega,\nu,\tau) \ | \ -0.1 \leq \mu \leq 0.1, -0.1 \leq \omega \leq 0.1, 0.05 \leq \nu \leq 0.24, 0.9 \leq \tau \leq 1.25 \}$. \index{domain!Theorem 3}

The mapping of the variance
$\nunn(\mu,\omega,\nu,\tau,\lambda,\alpha )$  given in Eq.~\eqref{eq:mappingVar} increases
\begin{align}
\nunn(\mu,\omega,\nu,\tau,\lambda_{\rm 01},\alpha_{\rm 01}) \ &> \ \nu 
\end{align}
in both $\Omega_1^-$ and $\Omega_2^-$.
All fixed
points $(\mu,\nu)$ of mapping Eq.~\eqref{eq:mappingVar} and
Eq.~\eqref{eq:mappingMean} ensure for $0.8 \leq \tau$ that
$\nunn>0.16$ 
and for $0.9 \leq \tau$ that $\nunn>0.24$.
Consequently, the variance mapping Eq.~\eqref{eq:mappingVar} and
Eq.~\eqref{eq:mappingMean} ensures a lower bound on the variance $\nu$. 
\end{theorem}


\section{Proofs of the Theorems}
\label{sec:proofs}

\subsection{Proof of Theorem 1} \index{Theorem 1!proof}

We have to show that the mapping $g$ defined by Eq.~\eqref{eq:mappingMean}
and Eq.~\eqref{eq:mappingVar} 
has a stable and attracting fixed point close to $(0,1)$.
To proof this statement and Theorem~\ref{lem:fixedPoint}, 
we apply the Banach fixed point theorem which 
requires (1) that $g$ is a contraction mapping and (2) 
that $g$ does not map outside the function's 
domain, concretely: 

\begin{theorem}[Banach Fixed Point Theorem]
\label{lem:Banach} \index{Banach Fixed Point Theorem}
Let $(X, d)$ be a non-empty complete metric space with a 
contraction mapping $f: X \to X$. Then $f$ has 
a unique fixed-point $x_f \in X$ with $f(x_f) = x_f$. 
Every sequence $x_n = f(x_{n-1})$
with starting element $x_0 \in X$ converges to the fixed point:
$x_n \xrightarrow[n \to \infty] \ x_f$.
\end{theorem}

Contraction mappings are functions that map two points such that their distance is decreasing:
\begin{definition}[Contraction mapping]
 A function $f:X \to X$ on a metric space $X$ with distance $d$ is a contraction mapping, if there
 is a $0 \leq \delta < 1$, such that for all points $\bm u$ and $\bm v$ in $X$:
 $d(f(\bm u),f(\bm v)) \leq \delta d(\bm u,\bm v)$.
\end{definition}

To show that $g$ is a contraction mapping in $\Omega$ with distance $\| . \|_2$, we use the Mean Value 
Theorem for $u,v \in \Omega$
\begin{align}
& \| g(\bm u) - g(\bm v) \|_2 \leq  M  \     \| \bm u - \bm v \|_2, 
\end{align}

in which $M$ is an upper bound on the spectral norm the Jacobian $\mathcal H$ of $g$.
The spectral norm is given by the largest singular value of the Jacobian of $g$.
If the largest singular value of the Jacobian is smaller than 1, 
the mapping  $g$ of the mean and variance to the mean and variance in the next layer is contracting.
We show that the largest singular value is smaller than 1 by
evaluating the function for the singular value
$S(\mu, \omega, \nu, \tau, \lambda, \alpha)$ on a grid.
Then we use the Mean Value Theorem to bound the deviation of the
function $S$ between grid points. 
To this end, we have to bound the gradient of $S$ with respect to
$(\mu, \omega, \nu, \tau)$. If all function values plus
gradient times the deltas (differences between grid points and evaluated
points) is still smaller than 1, then we have proofed that the
function is below 1 (Lemma~\ref{lem:sBound}). To show that the mapping does not map outside the function's domain, we 
derive bounds on the expressions for the mean and the variance (Lemma~\ref{lem:region}).
Section~\ref{sec:S} and Section~\ref{sec:maptoregion} are concerned with the contraction mapping and 
the image of the function domain of $g$, respectively.

With the results that the largest singular value of the Jacobian is smaller than 
one (Lemma~\ref{lem:sBound}) and that the mapping stays in the domain $\Omega$
(Lemma~\ref{lem:region}), we can prove Theorem~\ref{lem:fixedPoint}.
We first recall Theorem~\ref{lem:fixedPoint}:

\begin{theorem*}[Stable and Attracting Fixed Points]
We assume $\alpha = \alpha_{\rm 01}$ and $\lambda=\lambda_{\rm 01}$.
We restrict the range of the variables to the domain
$\mu \in [-0.1,0.1]$,
$\omega \in [-0.1,0.1]$,
$\nu \in [0.8,1.5]$, and
$\tau \in [0.95,1.1]$.
For $\omega=0$ and $\tau=1$, the mapping  Eq.~\eqref{eq:mappingMean}
and Eq.~\eqref{eq:mappingVar} has the stable
fixed point $(\mu,\nu)=(0,1)$.
For other $\omega$ and $\tau$ the mapping  Eq.~\eqref{eq:mappingMean}
and Eq.~\eqref{eq:mappingVar}  has a stable and
attracting fixed point depending on $(\omega,\tau)$ in the 
$(\mu,\nu)$-domain: $\mu \in [-0.03106, 0.06773]$ and 
$\nu \in [0.80009,1.48617]$.
All points within the $(\mu,\nu)$-domain converge when
iteratively applying the mapping  Eq.~\eqref{eq:mappingMean}
and Eq.~\eqref{eq:mappingVar} to this fixed point.
\end{theorem*}

\begin{proof} \index{Theorem 1!proof}
According to Lemma~\ref{lem:sBound} the mapping $g$ (Eq.~\eqref{eq:mappingMean} and Eq.~\eqref{eq:mappingVar})
is a contraction mapping in the given
domain, that is, it has a Lipschitz constant smaller than one.
We showed that $(\mu,\nu)=(0,1)$ is a fixed point of the
mapping for $(\omega,\tau)=(0,1)$. 

The domain is compact (bounded and closed), therefore it is a  
complete metric space.
We further have to make sure the  mapping $g$ does not map outside its domain $\Omega$.
According to Lemma~\ref{lem:region}, the mapping maps into the domain $\mu \in [-0.03106, 0.06773]$ and 
$\nu \in [0.80009,1.48617]$. 

Now we can apply the Banach fixed point theorem 
given in Theorem~\ref{lem:Banach} from which the statement of the 
theorem follows.
\end{proof}

\subsection{Proof of Theorem 2} \index{Theorem 2!proof}

First we recall Theorem~\ref{th:s2Decrease}:
\begin{theorem*}[Decreasing $\nu$]
For $\lambda=\lambda_{\rm 01}$, $\alpha=\alpha_{\rm 01}$ 
and the domain $\Omega^{++}$:
$-1 \leq \mu \leq 1$, 
$-0.1 \leq \omega \leq 0.1$,
$3 \leq \nu \leq 16$, and 
$0.8 \leq \tau \leq 1.25$ we have for 
the mapping of the variance
$\nunn(\mu,\omega,\nu,\tau,\lambda, \alpha )$  given in Eq.~\eqref{eq:mappingVar}
\begin{align}
\nunn(\mu,\omega,\nu,\tau,\lambda_{\rm 01},\alpha_{\rm 01}) \ &< \ \nu \ .
\end{align}
The variance decreases in $[3,16]$ and all fixed
points $(\mu,\nu)$ of mapping Eq.~\eqref{eq:mappingVar} and Eq.~\eqref{eq:mappingMean} have $\nu<3$.
\end{theorem*}

\begin{proof}
We start to consider an even larger domain
$-1 \leq \mu \leq 1$, 
$-0.1 \leq \omega \leq 0.1$,
$1.5 \leq \nu \leq 16$, and 
$0.8 \leq \tau \leq 1.25$.
We prove facts for this domain and later restrict to
$3 \leq \nu \leq 16$, i.e. $\Omega^{++}$.
We consider the function $g$ of the difference between the second moment $\xinn$ in the next layer
and the variance $\nu$ in the lower layer:
\begin{align}
g(\mu,\omega,\nu,\tau,\lambda_{\rm 01},\alpha_{\rm 01}) \ &= \
    \xinn(\mu,\omega,\nu,\tau,\lambda_{\rm 01},\alpha_{\rm 01})
    \ - \ \nu \ .
\end{align}
If we can show that $g(\mu,\omega,\nu,\tau,\lambda_{\rm 01},\alpha_{\rm 01}) < 0$ for 
all $(\mu,\omega,\nu,\tau) \in \Omega^{++}$, then
we would obtain our desired result $\nunn \leq \xinn < \nu$. 
The derivative with respect to $\nu$ is according to Theorem~\ref{th:s2Cont}:
\begin{align}
\frac{\partial}{\partial \nu}g(\mu,\omega,\nu,\tau,\lambda_{\rm 01},\alpha_{\rm 01}) \ &= \ 
\frac{\partial}{\partial \nu}\xinn(\mu,\omega,\nu,\tau,\lambda_{\rm 01},\alpha{\rm 01}) \ - \ 1 \ < \ 0 \ .
\end{align}
Therefore $g$ is strictly monotonically decreasing in $\nu$.
Since $\xinn$ is a function in $\nu \tau$
(these variables only appear as this product), we
have for $x=\nu \tau$
\begin{align}
\frac{\partial}{\partial
\nu}\xinn \ = \ \frac{\partial}{\partial
x}\xinn \ \frac{\partial x}{\partial
\nu} \ = \ \frac{\partial}{\partial
x}\xinn \ \tau
\end{align}
and 
\begin{align}
\frac{\partial}{\partial
\tau}\xinn \ = \ \frac{\partial}{\partial
x}\xinn \ \frac{\partial x}{\partial
\tau} \ = \ \frac{\partial}{\partial
x}\xinn \ \nu \ .
\end{align}
Therefore we have according to Theorem~\ref{th:s2Cont}:
\begin{align}
\frac{\partial}{\partial
\tau}\xinn(\mu,\omega,\nu,\tau,\lambda_{\rm 01},\alpha{\rm 01} ) \ &= \ \frac{\nu}{\tau} \ \frac{\partial}{\partial
\nu}\xinn(\mu,\omega,\nu,\tau,\lambda_{\rm 01},\alpha{\rm 01} ) \ > \ 0 \ .
\end{align}
Therefore
\begin{align}
\frac{\partial}{\partial
\tau}g(\mu,\omega,\nu,\tau,\lambda_{\rm 01},\alpha{\rm 01} ) \ &= \ 
\frac{\partial}{\partial
\tau}\xinn(\mu,\omega,\nu,\tau,\lambda_{\rm 01},\alpha{\rm 01})\ > \ 0 \ .
\end{align}
Consequently, $g$ is strictly monotonically increasing in $\tau$.
Now we consider the derivative with respect to $\mu$ and $\omega$. We start with $\frac{\partial }{\partial \mu }
\xinn(\mu,\omega,\nu,\tau,\lambda ,\alpha )$, 
which is
\begin{align}
\label{eq:J21New}
&\frac{\partial }{\partial \mu } \xinn(\mu,\omega,\nu,\tau,\lambda ,\alpha )\ = \\ \nonumber &\lambda ^2 \omega \left(\alpha ^2 \left(-e^{\mu
      \omega+\frac{\nu \tau}{2}}\right)
  \erfc\left(\frac{\mu \omega+\nu \tau}{\sqrt{2}
      \sqrt{\nu \tau}}\right)+\right. \\ \nonumber
&\left. \alpha ^2 e^{2 \mu \omega+2
    \nu \tau} \erfc\left(\frac{\mu \omega+2
      \nu \tau}{\sqrt{2} \sqrt{\nu
        \tau}}\right)+\mu \omega
  \left(2-\erfc\left(\frac{\mu \omega}{\sqrt{2} \sqrt{\nu
          \tau}}\right)\right)+\sqrt{\frac{2}{\pi }}
  \sqrt{\nu \tau} e^{-\frac{\mu^2 \omega^2}{2 \nu
      \tau}}\right)\ .
\end{align}

We consider the sub-function
\begin{align}
\sqrt{\frac{2}{\pi }} \sqrt{\nu \tau}-\alpha ^2
  \left(e^{\left(\frac{\mu \omega+\nu \tau}{\sqrt{2}
  \sqrt{\nu \tau}}\right)^2}
  \erfc\left(\frac{\mu \omega+\nu \tau}{\sqrt{2}
  \sqrt{\nu \tau}}\right)-e^{\left(\frac{\mu \omega+2
  \nu \tau}{\sqrt{2} \sqrt{\nu \tau}}\right)^2}
  \erfc\left(\frac{\mu \omega+2 \nu \tau}{\sqrt{2}
  \sqrt{\nu \tau}}\right)\right) \ .
\end{align}
We set $x=\nu\tau$ and $y=\mu \omega$ and obtain
\begin{align}
& \sqrt{\frac{2}{\pi }} \sqrt{x}-\alpha ^2
  \left(e^{\left(\frac{x+y}{\sqrt{2} \sqrt{x}}\right)^2}
  \erfc\left(\frac{x+y}{\sqrt{2}
  \sqrt{x}}\right)-e^{\left(\frac{2 x+y}{\sqrt{2} \sqrt{x}}\right)^2}
  \erfc\left(\frac{2 x+y}{\sqrt{2} \sqrt{x}}\right)\right) \ .
\end{align}

The derivative to this sub-function with respect to $y$ is
\begin{align}
&\frac{\alpha ^2 \left(e^{\frac{(2 x+y)^2}{2 x}} (2 x+y)
   \erfc\left(\frac{2 x+y}{\sqrt{2}
   \sqrt{x}}\right)-e^{\frac{(x+y)^2}{2 x}} (x+y)
   \erfc\left(\frac{x+y}{\sqrt{2} \sqrt{x}}\right)\right)}{x} \ = \\ \nonumber 
&\frac{\sqrt{2} \alpha ^2 \sqrt{x} \left(\frac{e^{\frac{(2 x+y)^2}{2 x}} (2 x+y) \erfc\left(\frac{2 x+y}{\sqrt{2} \sqrt{x}}\right)}{\sqrt{2} \sqrt{x}}
   -\frac{e^{\frac{(x+y)^2}{2 x}} (x+y) \erfc\left(\frac{x+y}{\sqrt{2} \sqrt{x}}\right)}{\sqrt{2} \sqrt{x}}\right)}{x}\ >\ 0 \ .
\end{align}

The inequality follows from Lemma~\ref{lem:xeErfc}, which states that 
$z e^{z^2} \erfc(z)$ is monotonically increasing in $z$.
Therefore the sub-function is increasing in $y$. The derivative to this sub-function with respect to $x$ is
\begin{align}
&\frac{1}{2 \sqrt{\pi } x^2} \sqrt{\pi } \alpha ^2 \left(e^{\frac{(2 x+y)^2}{2 x}} \left(4 x^2-y^2\right) \erfc\left(\frac{2 x+y}{\sqrt{2} \sqrt{x}}\right) \right. \\ \nonumber
& \left. - e^{\frac{(x+y)^2}{2 x}} (x-y) (x+y) \erfc\left(\frac{x+y}{\sqrt{2} \sqrt{x}}\right)\right)-\sqrt{2} \left(\alpha ^2-1\right) x^{3/2} .
\end{align}

The sub-function is increasing in $x$, since the
derivative is larger than zero:
\begin{align}
&\frac{\sqrt{\pi } \alpha ^2 \left(e^{\frac{(2 x+y)^2}{2 x}} \left(4 x^2-y^2\right) \erfc\left(\frac{2 x+y}{\sqrt{2} \sqrt{x}}\right)-e^{\frac{(x+y)^2}{2 x}} (x-y) (x+y) \erfc\left(\frac{x+y}{\sqrt{2} \sqrt{x}}\right)\right)-\sqrt{2} x^{3/2} \left(\alpha ^2-1\right)}{2 \sqrt{\pi } x^2}\ \geq\\ \nonumber 
&\frac{\sqrt{\pi } \alpha ^2 \left(\frac{(2 x-y) (2 x+y) 2}{\sqrt{\pi } \left(\frac{2 x+y}{\sqrt{2} \sqrt{x}}+\sqrt{\left(\frac{2 x+y}{\sqrt{2} \sqrt{x}}\right)^2+2}\right)}-\frac{(x-y) (x+y) 2}{\sqrt{\pi } \left(\frac{x+y}{\sqrt{2} \sqrt{x}}+\sqrt{\left(\frac{x+y}{\sqrt{2} \sqrt{x}}\right)^2+\frac{4}{\pi }}\right)}\right)-\sqrt{2} x^{3/2} \left(\alpha ^2-1\right)}{2 \sqrt{\pi } x^2}\ =
  \\ \nonumber 
&\frac{\sqrt{\pi } \alpha ^2 \left(\frac{(2 x-y) (2 x+y) 2 \left(\sqrt{2} \sqrt{x}\right)}{\sqrt{\pi } \left(2 x+y+\sqrt{(2 x+y)^2+4 x}\right)}-\frac{(x-y) (x+y) 2 \left(\sqrt{2} \sqrt{x}\right)}{\sqrt{\pi } \left(x+y+\sqrt{(x+y)^2+\frac{8 x}{\pi }}\right)}\right)-\sqrt{2} x^{3/2} \left(\alpha ^2-1\right)}{2 \sqrt{\pi } x^2}\ =
  \\ \nonumber 
&\frac{\sqrt{\pi } \alpha ^2 \left(\frac{(2 x-y) (2 x+y) 2}{\sqrt{\pi } \left(2 x+y+\sqrt{(2 x+y)^2+4 x}\right)}-\frac{(x-y) (x+y) 2}{\sqrt{\pi } \left(x+y+\sqrt{(x+y)^2+\frac{8 x}{\pi }}\right)}\right)-x \left(\alpha ^2-1\right)}{\sqrt{2} \sqrt{\pi } x^{3/2}}\ >
  \\ \nonumber 
&\frac{\sqrt{\pi } \alpha ^2 \left(\frac{(2 x-y) (2 x+y) 2}{\sqrt{\pi } \left(2 x+y+\sqrt{(2 x+y)^2+2 (2 x+y)+1}\right)}-\frac{(x-y) (x+y) 2}{\sqrt{\pi } \left(x+y+\sqrt{(x+y)^2+0.878 \cdot 2 (x+y)+0.878^2}\right)}\right)-x \left(\alpha ^2-1\right)}{\sqrt{2} \sqrt{\pi } x^{3/2}}\ =
  \\ \nonumber 
&\frac{\sqrt{\pi } \alpha ^2 \left(\frac{(2 x-y) (2 x+y) 2}{\sqrt{\pi } \left(2 x+y+\sqrt{(2 x+y+1)^2}\right)}-\frac{(x-y) (x+y) 2}{\sqrt{\pi } \left(x+y+\sqrt{(x+y+0.878)^2}\right)}\right)-x \left(\alpha ^2-1\right)}{\sqrt{2} \sqrt{\pi } x^{3/2}}\ =
  \\ \nonumber 
&\frac{\sqrt{\pi } \alpha ^2 \left(\frac{(2 x-y) (2 x+y) 2}{\sqrt{\pi } (2 (2 x+y)+1)}-\frac{(x-y) (x+y) 2}{\sqrt{\pi } (2 (x+y)+0.878)}\right)-x \left(\alpha ^2-1\right)}{\sqrt{2} \sqrt{\pi } x^{3/2}}\ =
  \\ \nonumber 
&\frac{\sqrt{\pi } \alpha ^2 \left(\frac{(2
                 (x+y)+0.878) (2 x-y) (2 x+y) 2}{\sqrt{\pi
                 }}-\frac{(x-y) (x+y) (2 (2 x+y)+1) 2}{\sqrt{\pi
                 }}\right)}{(2 (2 x+y)+1) (2 (x+y)+0.878) \sqrt{2}
                 \sqrt{\pi } x^{3/2}} \ + \\\nonumber &\frac{\sqrt{\pi } \alpha ^2 \left(-x \left(\alpha ^2-1\right) (2 (2 x+y)+1) (2 (x+y)+0.878)\right)}{(2 (2 x+y)+1) (2 (x+y)+0.878) \sqrt{2} \sqrt{\pi } x^{3/2}}\ =
  \\ \nonumber &\frac{8 x^3+12 x^2 y+4.14569 x^2+4 x y^2-6.76009 x y-1.58023 x+0.683154 y^2}{(2 (2 x+y)+1) (2 (x+y)+0.878) \sqrt{2} \sqrt{\pi } x^{3/2}}\ >
  \\ \nonumber 
& \frac{8 x^3-0.1 \cdot 12 x^2+4.14569 x^2+4 \cdot (0.0)^2 x-6.76009 \cdot 0.1 x-1.58023 x+0.683154 \cdot (0.0)^2}{(2 (2 x+y)+1) (2 (x+y)+0.878) \sqrt{2} \sqrt{\pi } x^{3/2}}\ =
  \\ \nonumber &\frac{8 x^2+2.94569 x-2.25624}{(2 (2 x+y)+1) (2 (x+y)+0.878) \sqrt{2} \sqrt{\pi } \sqrt{x}}\ =
  \\ \nonumber &\frac{8 (x-0.377966) (x+0.746178)}{(2 (2 x+y)+1) (2
(x+y)+0.878) \sqrt{2} \sqrt{\pi } \sqrt{x}} \ > \ 0 \ .
\end{align}

We explain this chain of inequalities:
\begin{itemize}
\item First inequality: We applied Lemma~\ref{lem:Abramowitz} two times.

\item Equalities factor out $\sqrt{2} \sqrt{x}$ and reformulate.

\item Second inequality part 1: we applied
\begin{align}
& 0<2 y\Longrightarrow (2 x+y)^2+4 x+1<(2 x+y)^2+2 (2 x+y)+1=(2 x+y+1)^2 \ .
\end{align}
\item Second inequality part 2: we show that for $a=\frac{1}{10} \left(\sqrt{\frac{960+169 \pi }{\pi }}-13\right)$ following holds:
$\frac{8 x}{\pi }-\left(a^2+2 a (x+y)\right) \geq 0$. 
We have $\frac{\partial }{\partial x }\frac{8 x}{\pi }-\left(a^2+2 a
  (x+y)\right)=\frac{8}{\pi }-2 a>0$ and
 $\frac{\partial }{\partial y }\frac{8 x}{\pi }-\left(a^2+2 a
  (x+y)\right)=-2 a<0$. 
Therefore the minimum is at border for minimal $x$ and maximal $y$:
\begin{align}
& \frac{8 \cdot 1.2}{\pi }-\left(\frac{2}{10} \left(\sqrt{\frac{960+169 \pi }{\pi }}-13\right) (1.2+0.1)+\left(\frac{1}{10} \left(\sqrt{\frac{960+169 \pi }{\pi }}-13\right)\right)^2\right) \
  = \ 0 \ .
\end{align}
Thus
\begin{align}
& \frac{8 x}{\pi } \ \geq \ a^2+2 a (x+y) \ .
\end{align}
for $a=\frac{1}{10} \left(\sqrt{\frac{960+169 \pi }{\pi }}-13\right) > 0.878$.

\item Equalities only solve square root and factor out the resulting
  terms $(2(2 x+y)+1)$ and $(2(x+y)+0.878)$.

\item We set $\alpha=\alpha_{01}$ and multiplied out. Thereafter we
  also factored out $x$ in the numerator. Finally a quadratic
  equations was solved. 
\end{itemize}

The sub-function has its minimal value for 
minimal
$x=\nu\tau=1.5 \cdot 0.8=1.2$ and  minimal 
$y=\mu \omega=-1 \cdot 0.1=-0.1$. 
We further minimize the function
\begin{align}
& \mu \omega e^{\frac{\mu^2 \omega^2}{2 \nu \tau}}
  \left(2-\erfc\left(\frac{\mu \omega}{\sqrt{2} \sqrt{\nu
  \tau}}\right)\right) \ > \ -0.1 e^{\frac{0.1^2}{2 \cdot 1.2}}
  \left(2-\erfc\left(\frac{0.1}{\sqrt{2}
  \sqrt{1.2}}\right)\right) \ .
\end{align}

We compute the minimum of the term in brackets of $\frac{\partial
}{\partial \mu }
\xinn(\mu,\omega,\nu,\tau,\lambda ,\alpha )$
in Eq.~\eqref{eq:J21New}:  
\begin{align}
&\mu \omega e^{\frac{\mu^2 \omega^2}{2 \nu \tau}} \left(2-\erfc\left(\frac{\mu \omega}{\sqrt{2} \sqrt{\nu \tau}}\right)\right)+ \\ \nonumber 
&\alpha_{\rm 01}^2 \left(-\left(e^{\left(\frac{\mu \omega+\nu \tau}{\sqrt{2} \sqrt{\nu \tau}}\right)^2} 
\erfc\left(\frac{\mu \omega+\nu \tau}{\sqrt{2} \sqrt{\nu \tau}}\right)-e^{\left(\frac{\mu \omega+2 \nu \tau}{\sqrt{2} \sqrt{\nu \tau}}\right)^2} 
\erfc\left(\frac{\mu \omega+2 \nu \tau}{\sqrt{2} \sqrt{\nu \tau}}\right)\right)\right)+ \sqrt{\frac{2}{\pi }} \sqrt{\nu \tau} \ > \\ \nonumber
&\alpha_{\rm 01}^2 \left(-\left(e^{\left(\frac{1.2 -0.1}{\sqrt{2} \sqrt{1.2}}\right)^2} 
\erfc\left(\frac{1.2-0.1}{\sqrt{2} \sqrt{1.2}}\right)-e^{\left(\frac{2 \cdot 1.2-0.1}{\sqrt{2} \sqrt{1.2}}\right)^2} 
\erfc\left(\frac{2 \cdot 1.2-0.1}{\sqrt{2} \sqrt{1.2}}\right)\right)\right)-\\ \nonumber 
&0.1 e^{\frac{0.1^2}{2 \cdot 1.2}} \left(2-\erfc\left(\frac{0.1}{\sqrt{2} \sqrt{1.2}}\right)\right)+\sqrt{1.2} \sqrt{\frac{2}{\pi }}
\ = \ 0.212234 \ .
\end{align}
Therefore the term in brackets  of Eq.~\eqref{eq:J21New}
is larger than zero.
Thus, $\frac{\partial }{\partial \mu }
\xinn(\mu,\omega,\nu,\tau,\lambda ,\alpha )$
has the sign of $\omega$.
Since $\xinn$ is a function in $\mu \omega$
(these variables only appear as this product), we
have for $x=\mu \omega$
\begin{align}
\frac{\partial}{\partial \nu}\xinn \ = \ \frac{\partial}{\partial
x}\xinn \ \frac{\partial x}{\partial
\mu} \ = \ \frac{\partial}{\partial
x}\xinn \ \omega
\end{align}
and 
\begin{align}
\frac{\partial}{\partial
\omega}\xinn \ = \ \frac{\partial}{\partial
x}\xinn \ \frac{\partial x}{\partial
\omega} \ = \ \frac{\partial}{\partial
x}\xinn \ \mu \ .
\end{align}

\begin{align}
\frac{\partial}{\partial
\omega}\xinn(\mu,\omega,\nu,\tau,\lambda_{\rm 01},\alpha{\rm 01}) \ &= \ \frac{\mu}{\omega} \ \frac{\partial}{\partial
\mu}\xinn(\mu,\omega,\nu,\tau,\lambda_{\rm 01},\alpha{\rm 01}) \ .
\end{align}
Since $\frac{\partial}{\partial
\mu}\xinn$ has the sign of $\omega$, 
 $\frac{\partial}{\partial
\mu}\xinn$ has the sign of $\mu$.
Therefore
\begin{align}
\frac{\partial}{\partial
\omega}g(\mu,\omega,\nu,\tau,\lambda_{\rm 01},\alpha{\rm 01} ) \ &= \ 
\frac{\partial}{\partial
\omega}\xinn(\mu,\omega,\nu,\tau,\lambda_{\rm 01},\alpha{\rm 01} )
\end{align}
has the sign of $\mu$.

We now divide the $\mu$-domain into
$-1 \leq \mu \leq 0$ and $0 \leq \mu \leq 1$.
Analogously we divide the $\omega$-domain into
$-0.1 \leq \omega \leq 0$ and $0 \leq \omega \leq 0.1$.
In this domains $g$ is strictly monotonically.

For all domains
$g$ is strictly monotonically decreasing in $\nu$
and strictly monotonically increasing in $\tau$.
Note that we now consider the range  $3 \leq \nu \leq 16$.
For the maximal value of $g$ we set  $\nu=3$ (we set it to 3!)
and $\tau=1.25$.

We consider now all combination of these domains:
\begin{itemize}
\item $-1 \leq \mu \leq 0$ and $-0.1 \leq \omega \leq 0$:

$g$ is decreasing in  $\mu$ and decreasing in $\omega$.
We set  $\mu=-1$ and  $\omega=-0.1$.
\begin{align}
g(-1,-0.1,3,1.25,\lambda_{\rm 01},\alpha_{\rm 01})
\ &= \ -0.0180173 \ .
\end{align}

\item $-1 \leq \mu \leq 0$ and $0 \leq \omega \leq 0.1$:

$g$ is increasing in $\mu$ and decreasing in $\omega$.
We set  $\mu=0$ and  $\omega=0$.
\begin{align}
g(0,0,3,1.25,\lambda_{\rm 01},\alpha_{\rm 01})
\ &= \ -0.148532 \ .
\end{align}

\item $0 \leq \mu \leq 1$ and $-0.1 \leq \omega \leq 0$:

$g$ is decreasing in $\mu$ and increasing in $\omega$.
We set  $\mu=0$ and  $\omega=0$.
\begin{align}
g(0,0,3,1.25,\lambda_{\rm 01},\alpha_{\rm 01})
\ &= \ -0.148532 \ .
\end{align}

\item $0 \leq \mu \leq 1$ and $0 \leq \omega \leq 0.1$:

$g$ is increasing in $\mu$ and increasing in $\omega$.
We set  $\mu=1$ and  $\omega=0.1$.
\begin{align}
g(1,0.1,3,1.25,\lambda_{\rm 01},\alpha_{\rm 01})
\ &= \ -0.0180173 \ .
\end{align}

Therefore the maximal value of $g$ is  $-0.0180173$.
\end{itemize}

\end{proof}

\subsection{Proof of Theorem 3}\index{Theorem 3!proof}

First we recall Theorem~\ref{th:s2Increase}:
\begin{theorem*}[Increasing $\nu$]
We consider $\lambda=\lambda_{\rm 01}$, $\alpha=\alpha_{\rm 01}$
and the two domains 
$\Omega_1 ^-=\{(\mu,\omega,\nu,\tau) \ | \ -0.1 \leq \mu \leq 0.1, -0.1 \leq \omega \leq 0.1, 0.05 \leq \nu \leq 0.16, 0.8 \leq \tau \leq 1.25 \}$ 
 and
$\Omega_2 ^-=\{(\mu,\omega,\nu,\tau) \ | \ -0.1 \leq \mu \leq 0.1, -0.1 \leq \omega \leq 0.1, 0.05 \leq \nu \leq 0.24, 0.9 \leq \tau \leq 1.25 \}$ .

The mapping of the variance
$\nunn(\mu,\omega,\nu,\tau,\lambda,\alpha )$  given in Eq.~\eqref{eq:mappingVar} increases
\begin{align}
\nunn(\mu,\omega,\nu,\tau,\lambda_{\rm 01},\alpha_{\rm 01}) \ &> \ \nu 
\end{align}
in both $\Omega_1^-$ and $\Omega_2^-$.
All fixed
points $(\mu,\nu)$ of mapping Eq.~\eqref{eq:mappingVar} and
Eq.~\eqref{eq:mappingMean} ensure for $0.8 \leq \tau$ that
$\nunn>0.16$ 
and for $0.9 \leq \tau$ that $\nunn>0.24$.
Consequently, the variance mapping Eq.~\eqref{eq:mappingVar} and
Eq.~\eqref{eq:mappingMean} ensures a lower bound on the variance $\nu$. 
\end{theorem*}

\begin{proof}
The mean value theorem states that there exists a $t\in [0,1]$  for which
\begin{align}
&\xinn(\mu,\omega,\nu,\tau,\lambda_{\rm 01},\alpha_{\rm 01})
  \ - \
  \xinn(\mu,\omega,\nu_{\mathrm{min}},\tau,\lambda_{\rm 01},\alpha_{\rm 01})
\ = \\\nonumber &\frac{\partial }{\partial \nu}
  \xinn(\mu,\omega,\nu+t(\nu_{\mathrm{min}}-\nu),\tau,\lambda_{\rm 01},\alpha_{\rm 01})
  \ (\nu-\nu_{\mathrm{min}}) \ .
\end{align}
Therefore
\begin{align}
&\xinn(\mu,\omega,\nu,\tau,\lambda_{\rm 01},\alpha_{\rm 01})
  \ = \ \xinn(\mu,\omega,\nu_{\mathrm{min}},\tau,\lambda_{\rm 01},\alpha_{\rm 01})
\ + \\ \nonumber &\frac{\partial }{\partial \nu}
  \xinn(\mu,\omega,\nu+t(\nu_{\mathrm{min}}-\nu),\tau,\lambda_{\rm 01},\alpha_{\rm 01})
  \ (\nu-\nu_{\mathrm{min}}) \ .
\end{align}

Therefore we are interested to bound the derivative of the $\xi$-mapping
Eq.~\eqref{eq:mappingSecondMom} with respect to  $\nu$:
\begin{align}
\label{eq:fo1}
&\frac{\partial }{\partial \nu}
  \xinn(\mu,\omega,\nu,\tau,\lambda_{\rm 01},\alpha_{\rm 01})\
  = \\\nonumber
&\frac{1}{2} \lambda ^2  \tau e^{-\frac{\mu^2 \omega^2}{2  \nu \tau}} 
\left(\alpha ^2 \left(-\left(e^{\left(\frac{\mu \omega+\nu \tau}{\sqrt{2} \sqrt{\nu \tau}}\right)^2} \erfc \left(\frac{\mu \omega+\nu \tau}{\sqrt{2} \sqrt{\nu \tau}}\right)-
2 e^{\left(\frac{\mu \omega+2  \nu \tau}{\sqrt{2} \sqrt{\nu \tau}}\right)^2} \erfc \left(\frac{\mu \omega+2  \nu \tau}{\sqrt{2} \sqrt{\nu \tau}}\right)\right)\right)-
\right.\\\nonumber &\left.\erfc \left(\frac{\mu \omega}{\sqrt{2} \sqrt{\nu \tau}}\right)+2\right) \ .
\end{align}

The sub-term Eq.~\eqref{eq:subx1} enters the derivative
Eq.~\eqref{eq:fo1} with a negative sign!
According to Lemma~\ref{th:s2monotone},
the minimal value of sub-term Eq.~\eqref{eq:subx1}
is obtained by the largest largest  $\nu$, 
by the smallest $\tau$, and the largest $y=\mu \omega=0.01$.
Also the positive term
$\erfc \left(\frac{\mu \omega}{\sqrt{2} \sqrt{\nu
\tau}}\right)+2$ is multiplied by $\tau$, which is minimized
by using the smallest $\tau$.
Therefore we can use the smallest  $\tau$ in whole formula
Eq.~\eqref{eq:fo1} to lower bound it.

First we consider the domain 
$0.05 \leq \nu \leq 0.16$  and $0.8 \leq \tau \leq 1.25$.
The factor consisting of the exponential in front of the brackets has
its smallest value for $e^{-\frac{0.01 \cdot 0.01}{2 \cdot 0.05 \cdot 0.8}}$.
Since $\erfc $ is monotonically decreasing we inserted the
smallest argument via $\erfc \left(-\frac{0.01}{\sqrt{2}
\sqrt{0.05 \cdot 0.8}}\right)$ in order to obtain the maximal negative contribution.
Thus, applying Lemma~\ref{th:s2monotone}, we obtain the lower bound on the derivative:
\begin{align}
&\frac{1}{2} \lambda ^2  \tau e^{-\frac{\mu^2 \omega^2}{2  \nu \tau}} \left(\alpha ^2 \left(-\left(e^{\left(\frac{\mu \omega+\nu \tau}{\sqrt{2} \sqrt{\nu \tau}}\right)^2} \erfc \left(\frac{\mu \omega+\nu \tau}{\sqrt{2} \sqrt{\nu \tau}}\right)-2 e^{\left(\frac{\mu \omega+2  \nu \tau}{\sqrt{2} \sqrt{\nu \tau}}\right)^2} \erfc \left(\frac{\mu \omega+2  \nu \tau}{\sqrt{2} \sqrt{\nu \tau}}\right)\right)\right)- \right.\\\nonumber &\left.\erfc \left(\frac{\mu \omega}{\sqrt{2} \sqrt{\nu \tau}}\right)+2\right)\ > \\ \nonumber &
\frac{1}{2} 0.8 e^{-\frac{0.01 \cdot 0.01}{2 \cdot 0.05 \cdot 0.8}} \lambda_{\rm 01}^2 \left(\alpha_{\rm 01}^2 \left(-\left(e^{\left(\frac{0.16 \cdot 0.8+0.01}{\sqrt{2} \sqrt{0.16 \cdot 0.8}}\right)^2} \erfc \left(\frac{0.16 \cdot 0.8+0.01}{\sqrt{2} \sqrt{0.16 \cdot 0.8}}\right)-\right.\right.\right.\\\nonumber &\left.\left.\left.2 e^{\left(\frac{2 \cdot  0.16 \cdot 0.8+0.01}{\sqrt{2} \sqrt{0.16 \cdot 0.8}}\right)^2} \erfc \left(\frac{2 \cdot 0.16 \cdot 0.8+0.01}{\sqrt{2} \sqrt{0.16 \cdot 0.8}}\right)\right)\right)-\erfc \left(-\frac{0.01}{\sqrt{2} \sqrt{0.05 \cdot 0.8}}\right)+2\right))\ > \ 0.969231 \ .
\end{align}

For applying the mean value theorem, we require the smallest $\nunn (\nu)$.
We follow the proof of Lemma~\ref{lem:mapDerivatives}, which shows
that at the minimum $y=\mu \omega$ must be maximal 
and $x=\nu \tau$ must be minimal.
Thus, the smallest 
$\xinn(\mu,\omega,\nu,\tau,\lambda_{\rm 01},\alpha_{\rm 01})$
is 
$\xinn(0.01,0.01,0.05,0.8,\lambda_{\rm 01},\alpha_{\rm 01})=0.0662727$
for $0.05 \leq \nu$ and $0.8 \leq \tau$.

Therefore the mean value theorem and the bound on $(\munn)^2$ (Lemma~\ref{lem:musquared}) provide
\begin{align}
&\nunn = 
\xinn(\mu,\omega,\nu,\tau,\lambda_{\rm 01},\alpha_{\rm 01}) - \left(\munn(\mu,\omega,\nu,\tau,\lambda_{\rm 01},\alpha_{\rm 01})\right)^2 > \\ \nonumber
& 0.0662727 \ + \  0.969231 (\nu-0.05) - 0.005\ = \  0.01281115 + 0.969231 \nu \ > \\ \nonumber 
& 0.08006969  \cdot  0.16  + 0.969231
  \nu \geq 1.049301 \nu \ > \ \nu \ .
\end{align}

Next we consider the domain 
$0.05 \leq \nu \leq 0.24$ 
and $0.9 \leq \tau \leq 1.25$.
The factor consisting of the exponential in front of the brackets has
its smallest value for $e^{-\frac{0.01 \cdot 0.01}{2 \cdot 0.05 \cdot 0.9}}$. 
Since $\erfc $ is monotonically decreasing we inserted the
smallest argument via $\erfc \left(-\frac{0.01}{\sqrt{2}
\sqrt{0.05 \cdot 0.9}}\right)$ in order to obtain the maximal negative contribution.

Thus, applying Lemma~\ref{th:s2monotone}, we obtain the lower bound on the derivative:
\begin{align}
&\frac{1}{2} \lambda ^2  \tau e^{-\frac{\mu^2 \omega^2}{2  \nu \tau}} \left(\alpha ^2 \left(-\left(e^{\left(\frac{\mu \omega+\nu \tau}{\sqrt{2} \sqrt{\nu \tau}}\right)^2} \erfc \left(\frac{\mu \omega+\nu \tau}{\sqrt{2} \sqrt{\nu \tau}}\right)-2 e^{\left(\frac{\mu \omega+2  \nu \tau}{\sqrt{2} \sqrt{\nu \tau}}\right)^2} \erfc \left(\frac{\mu \omega+2  \nu \tau}{\sqrt{2} \sqrt{\nu \tau}}\right)\right)\right)- \right.\\\nonumber &\left.\erfc \left(\frac{\mu \omega}{\sqrt{2} \sqrt{\nu \tau}}\right)+2\right)\ > \\ \nonumber &
\frac{1}{2} 0.9 e^{-\frac{0.01 \cdot 0.01}{2 \cdot 0.05 \cdot 0.9}} \lambda_{\rm 01}^2 \left(\alpha_{\rm 01}^2 \left(-\left(e^{\left(\frac{0.24 \cdot 0.9+0.01}{\sqrt{2} \sqrt{0.24 \cdot 0.9}}\right)^2} \erfc \left(\frac{0.24 \cdot 0.9+0.01}{\sqrt{2} \sqrt{0.24 \cdot 0.9}}\right)-\right.\right.\right.\\\nonumber &\left.\left.\left.2 e^{\left(\frac{2 \cdot  0.24 \cdot 0.9+0.01}{\sqrt{2} \sqrt{0.24 \cdot 0.9}}\right)^2} \erfc \left(\frac{2 \cdot 0.24 \cdot 0.9+0.01}{\sqrt{2} \sqrt{0.24 \cdot 0.9}}\right)\right)\right)-\erfc \left(-\frac{0.01}{\sqrt{2} \sqrt{0.05 \cdot 0.9}}\right)+2\right))\ > \ 0.976952 \ .
\end{align}

For applying the mean value theorem, we require the smallest $\nunn(\nu)$.
We follow the proof of Lemma~\ref{lem:mapDerivatives}, which shows
that at the minimum $y=\mu \omega$ must be maximal 
and $x=\nu \tau$ must be minimal.
Thus, the smallest 
$\xinn(\mu,\omega,\nu,\tau,\lambda_{\rm 01},\alpha_{\rm 01})$
is 
$\xinn(0.01,0.01,0.05,0.9,\lambda_{\rm 01},\alpha_{\rm 01})=0.0738404$
for $0.05 \leq \nu$ and $0.9 \leq \tau$.
Therefore the mean value theorem and the bound on $(\munn)^2$ (Lemma~\ref{lem:musquared}) gives
\begin{align}
&\nunn = \xinn(\mu,\omega,\nu,\tau,\lambda_{\rm 01},\alpha_{\rm 01}) - \left(\munn(\mu,\omega,\nu,\tau,\lambda_{\rm 01},\alpha_{\rm 01})\right)^2 > \\ \nonumber
& \ 0.0738404 \ + \  0.976952 (\nu-0.05) - 0.005\ = \ 0.0199928 + 0.976952 \nu \ > \\ \nonumber 
&0.08330333  \cdot  0.24  + 0.976952  \nu \geq 1.060255 \nu \ > \ \nu \ .
\end{align}

\end{proof}

\subsection{Lemmata and Other Tools Required for the Proofs} \index{lemmata}

\subsubsection{Lemmata for proofing Theorem 1 (part 1): Jacobian norm smaller than one} \index{lemmata!Jacobian bound}
\label{sec:S}
In this section, we show that the largest singular value of the Jacobian of the 
mapping $g$ is smaller than one. Therefore, $g$ is a contraction mapping. 
This is even true in a larger domain than the original $\Omega$. We 
do not need to restrict $\tau \in [0.95,1.1]$, but we can extend to
$\tau \in [0.8,1.25]$. The range of the other variables is unchanged such that 
we consider the following domain throughout this section: $\mu \in [-0.1,0.1]$,
$\omega \in [-0.1,0.1]$,
$\nu \in [0.8,1.5]$, and
$\tau \in [0.8,1.25]$. \index{domain!singular value}

\paragraph{Jacobian of the mapping.}\index{Jacobian}
In the following, we denote two Jacobians: \index{Jacobian!definition}
(1) the Jacobian $\mathcal J$ of the mapping  $h:(\mu,\nu) \mapsto (\munn,\xinn)$, and
(2) the Jacobian $\mathcal H$ of the mapping $g:(\mu,\nu) \mapsto (\munn,\nunn)$
because the influence of $\munn$ on $\nunn$ is small, 
and many properties of the system can already be seen on $\mathcal J$. 

\begin{align}
\mathcal J \ &= \ \left(
\begin{array}{cc}
{\mathcal J}_{11} & {\mathcal J}_{12} \\
 {\mathcal J}_{21} & {\mathcal J}_{22} \\
\end{array}
\right) = 
\ \left(
\begin{array}{cc}
\frac{\partial }{\partial \mu} \munn & \frac{\partial }{\partial \nu} \munn \\
 \frac{\partial }{\partial \mu} \xinn  & \frac{\partial }{\partial \nu} \xinn  \\
\end{array}
\right)
\\
\mathcal H \ &= \ \left(
\begin{array}{cc}
{\mathcal H}_{11} & {\mathcal H}_{12} \\
 {\mathcal H}_{21} & {\mathcal H}_{22} \\
\end{array}
\right) = 
\ \left(
\begin{array}{cc}
{\mathcal J}_{11} & {\mathcal J}_{12} \\
 {\mathcal J}_{21} - 2 \munn {\mathcal J}_{11} & {\mathcal J}_{22} - 2 \munn {\mathcal J}_{12} \\
\end{array}
\right)
\end{align}

The definition of the entries of the Jacobian $\mathcal J$ is: \index{Jacobian!entries}
\begin{align}
\label{eq:JacobianEntries}
& {\mathcal J}_{11}(\mu,\omega,\nu,\tau,\lambda ,\alpha ) \
= \frac{\partial }{\partial \mu} \munn(\mu,\omega,\nu,\tau,\lambda ,\alpha ) \ = \\ \nonumber &\frac{1}{2} \lambda  \omega 
\left(\alpha  e^{\mu \omega+\frac{\nu \tau}{2}} \erfc \left(\frac{\mu \omega+\nu \tau}{\sqrt{2} \sqrt{\nu \tau}}\right)
-\erfc \left(\frac{\mu \omega}{\sqrt{2} \sqrt{\nu \tau}}\right)+2\right)\\
&{\mathcal J}_{12}(\mu,\omega,\nu,\tau,\lambda ,\alpha ) \ 
= \frac{\partial }{\partial \nu} \munn(\mu,\omega,\nu,\tau,\lambda ,\alpha ) \  = \\ \nonumber &\frac{1}{4} \lambda  \tau \left(\alpha
  e^{\mu \omega+\frac{\nu \tau}{2}}
  \erfc \left(\frac{\mu \omega+\nu \tau}{\sqrt{2}
      \sqrt{\nu \tau}}\right)-(\alpha -1)
  \sqrt{\frac{2}{\pi  \nu \tau}} e^{-\frac{\mu^2
      \omega^2}{2 \nu \tau}}\right)\\
&{\mathcal J}_{21}(\mu,\omega,\nu,\tau,\lambda ,\alpha ) \ 
= \ \frac{\partial }{\partial \mu } \xinn(\mu,\omega,\nu,\tau,\lambda ,\alpha )\ = \\ \nonumber &\lambda ^2 \omega \left(\alpha ^2 \left(-e^{\mu
      \omega+\frac{\nu \tau}{2}}\right)
  \erfc \left(\frac{\mu \omega+\nu \tau}{\sqrt{2}
      \sqrt{\nu \tau}}\right)+\right. \\ \nonumber
&\left. \alpha ^2 e^{2 \mu \omega+2
    \nu \tau} \erfc \left(\frac{\mu \omega+2
      \nu \tau}{\sqrt{2} \sqrt{\nu
        \tau}}\right)+\mu \omega
  \left(2-\erfc \left(\frac{\mu \omega}{\sqrt{2} \sqrt{\nu
          \tau}}\right)\right)+\sqrt{\frac{2}{\pi }}
  \sqrt{\nu \tau} e^{-\frac{\mu^2 \omega^2}{2 \nu
      \tau}}\right)\\
&{\mathcal J}_{22}(\mu,\omega,\nu,\tau,\lambda ,\alpha ) \ 
= \frac{\partial }{\partial \nu } \xinn(\mu,\omega,\nu,\tau,\lambda ,\alpha )\ 
= \\ \nonumber &\frac{1}{2} \lambda ^2 \tau \left(\alpha ^2
  \left(-e^{\mu \omega+\frac{\nu \tau}{2}}\right)
  \erfc \left(\frac{\mu \omega+\nu \tau}{\sqrt{2}
      \sqrt{\nu \tau}}\right)+\right. \\ \nonumber
&\left. 2 \alpha ^2 e^{2 \mu
    \omega+2 \nu \tau} \erfc \left(\frac{\mu \omega+2
      \nu \tau}{\sqrt{2} \sqrt{\nu
        \tau}}\right)-\erfc \left(\frac{\mu
      \omega}{\sqrt{2} \sqrt{\nu \tau}}\right)+2\right)
\end{align}

\paragraph{Proof sketch: Bounding the largest singular value of the Jacobian.}\index{Jacobian}

If the largest singular value of the Jacobian is smaller than 1, then
the spectral norm of the Jacobian is smaller than 1.
Then the mapping  Eq.~\eqref{eq:mappingMean} 
and Eq.~\eqref{eq:mappingVar} 
of the mean and variance to the mean and variance in the next layer is contracting.


We show that the largest singular value is smaller than 1 by
evaluating the function
$S(\mu, \omega, \nu, \tau, \lambda, \alpha)$ on a grid.
Then we use the Mean Value Theorem to bound the deviation of the
function $S$ between grid points. 
Toward this end we have to bound the gradient of $S$ with respect to
$(\mu, \omega, \nu, \tau)$. If all function values plus
gradient times the deltas (differences between grid points and evaluated
points) is still smaller than 1, then we have proofed that the
function is below 1. 

The singular values of the $2 \times2$ matrix
\begin{align}
\BA \ &= \ \left(
\begin{array}{cc}
 a_{11} & a_{12} \\
 a_{21} & a_{22} \\
\end{array}
\right)
\end{align}
are
\begin{align}
s_1 \ &= \ \frac{1}{2} \
\left(\sqrt{(a_{11}+a_{22})^2+(a_{21}-a_{12})^2} \ + \
\sqrt{(a_{11}-a_{22})^2+(a_{12}+a_{21})^2}\right) \\
s_2 &= \frac{1}{2} \ 
\left(\sqrt{(a_{11}+a_{22})^2+(a_{21}-a_{12})^2} \ - \
\sqrt{(a_{11}-a_{22})^2+(a_{12}+a_{21})^2}\right).
\end{align}
We used an explicit formula for the singular values \citep{Blinn:96}. We now set 
${\mathcal H}_{11}=a_{11},{\mathcal H}_{12}=a_{12},{\mathcal H}_{21}=a_{21},{\mathcal H}_{22}=a_{22}$
to obtain a formula for the largest singular value of the Jacobian
depending on $(\mu,\omega,\nu,\tau,\lambda ,\alpha )$.
The formula for the largest singular value for the Jacobian is: \index{Jacobian!singular value}

\begin{align}
\label{eq:S}
 S(\mu,\omega,\nu,\tau,\lambda ,\alpha ) \ &= \left(\sqrt{({\mathcal H}_{11}+{\mathcal H}_{22})^2+({\mathcal H}_{21}-{\mathcal H}_{12})^2} \ + \ \sqrt{({\mathcal H}_{11}-{\mathcal H}_{22})^2+({\mathcal H}_{12}+{\mathcal H}_{21})^2}\right) = \\ \nonumber
&=\ \frac{1}{2} \ \left(\sqrt{({\mathcal J}_{11}+{\mathcal J}_{22} - 2 \munn {\mathcal J}_{12})^2+({\mathcal J}_{21} - 2 \munn {\mathcal J}_{11}-{\mathcal J}_{12})^2} \ + \right. \\ \nonumber
& \left. \sqrt{({\mathcal J}_{11}-{\mathcal J}_{22} + 2 \munn {\mathcal J}_{12})^2+({\mathcal J}_{12}+{\mathcal J}_{21} - 2 \munn {\mathcal J}_{11})^2} \right), \\ \nonumber 
\end{align}

where $\mathcal J$ are defined in Eq.~\eqref{eq:JacobianEntries} and we left out the dependencies on 
$(\mu,\omega,\nu,\tau,\lambda ,\alpha)$ in order to keep the notation uncluttered, e.g. we 
wrote ${\mathcal J}_{11}$ instead of ${\mathcal J}_{11} (\mu,\omega,\nu,\tau,\lambda ,\alpha)$.

\paragraph{Bounds on the derivatives of the Jacobian entries.}\index{bounds!derivatives of Jacobian entries}\index{Jacobian!derivatives}

In order to bound the gradient of the singular value, we have to bound
the derivatives of the Jacobian entries 
${\mathcal J}_{11}(\mu,\omega,\nu,\tau,\lambda ,\alpha )$,
${\mathcal J}_{12}(\mu,\omega,\nu,\tau,\lambda ,\alpha )$,
${\mathcal J}_{21}(\mu,\omega,\nu,\tau,\lambda ,\alpha )$, and
${\mathcal J}_{22}(\mu,\omega,\nu,\tau,\lambda ,\alpha )$
with respect to 
$\mu$, $\omega$, $\nu$, and $\tau$. The values 
$\lambda$ and $\alpha$ are fixed to $\lambda_{\rm 01}$ and $\alpha_{\rm 01}$.
The 16 derivatives of the 4 Jacobian entries with respect to the 4
variables are:
\begin{align}
\frac{\partial {\mathcal J}_{11}}{\partial \mu} \ &= \
\frac{1}{2} \lambda  \omega^2 e^{-\frac{\mu^2 \omega^2}{2 \nu \tau}} \left(\alpha  e^{\frac{(\mu \omega+\nu \tau)^2}{2 \nu \tau}} \erfc \left(\frac{\mu \omega+\nu \tau}{\sqrt{2} \sqrt{\nu \tau}}\right)-\frac{\sqrt{\frac{2}{\pi }} (\alpha -1)}{\sqrt{\nu \tau}}\right)
\\ \nonumber
\frac{\partial {\mathcal J}_{11}}{\partial \omega} \ &= \
\frac{1}{2} \lambda  \left(-e^{-\frac{\mu^2 \omega^2}{2 \nu
  \tau}} \left(\frac{\sqrt{\frac{2}{\pi }} (\alpha -1) \mu
  \omega}{\sqrt{\nu \tau}}-\alpha  (\mu \omega+1)
  e^{\frac{(\mu \omega+\nu \tau)^2}{2 \nu
  \tau}} \erfc \left(\frac{\mu \omega+\nu
  \tau}{\sqrt{2} \sqrt{\nu \tau}}\right)\right) 
\ - \right. \\ \nonumber &\left. \erfc \left(\frac{\mu \omega}{\sqrt{2} \sqrt{\nu \tau}}\right)+2\right) \\ \nonumber
\frac{\partial {\mathcal J}_{11}}{\partial \nu} \ &= \
\frac{1}{4} \lambda  \tau \omega e^{-\frac{\mu^2 \omega^2}{2 \nu \tau}} \left(\alpha  e^{\frac{(\mu \omega+\nu \tau)^2}{2 \nu \tau}} \erfc \left(\frac{\mu \omega+\nu \tau}{\sqrt{2} \sqrt{\nu \tau}}\right)+\sqrt{\frac{2}{\pi }} \left(\frac{(\alpha -1) \mu \omega}{(\nu \tau)^{3/2}}-\frac{\alpha }{\sqrt{\nu \tau}}\right)\right)
\\ \nonumber
\frac{\partial {\mathcal J}_{11}}{\partial \tau} \ &= \
\frac{1}{4} \lambda  \nu \omega e^{-\frac{\mu^2 \omega^2}{2 \nu \tau}} \left(\alpha  e^{\frac{(\mu \omega+\nu \tau)^2}{2 \nu \tau}} \erfc \left(\frac{\mu \omega+\nu \tau}{\sqrt{2} \sqrt{\nu \tau}}\right)+\sqrt{\frac{2}{\pi }} \left(\frac{(\alpha -1) \mu \omega}{(\nu \tau)^{3/2}}-\frac{\alpha }{\sqrt{\nu \tau}}\right)\right)
\\ \nonumber
\frac{\partial {\mathcal J}_{12}}{\partial \mu} \ &= \ \frac{\partial {\mathcal J}_{11}}{\partial \nu}
\\ \nonumber
\frac{\partial {\mathcal J}_{12}}{\partial \omega} \ &= \
\frac{1}{4} \lambda  \mu \tau e^{-\frac{\mu^2 \omega^2}{2 \nu \tau}} \left(\alpha  e^{\frac{(\mu \omega+\nu \tau)^2}{2 \nu \tau}} \erfc \left(\frac{\mu \omega+\nu \tau}{\sqrt{2} \sqrt{\nu \tau}}\right)+\sqrt{\frac{2}{\pi }} \left(\frac{(\alpha -1) \mu \omega}{(\nu \tau)^{3/2}}-\frac{\alpha }{\sqrt{\nu \tau}}\right)\right)
\\ \nonumber
\frac{\partial {\mathcal J}_{12}}{\partial \nu} \ &= \
\frac{1}{8} \lambda  e^{-\frac{\mu^2 \omega^2}{2 \nu \tau}} \left(\alpha  \tau^2 e^{\frac{(\mu \omega+\nu \tau)^2}{2 \nu \tau}} \erfc \left(\frac{\mu \omega+\nu \tau}{\sqrt{2} \sqrt{\nu \tau}}\right)\ + \right. \\ \nonumber &\left. \sqrt{\frac{2}{\pi }} \left(\frac{(-1) (\alpha -1) \mu^2 \omega^2}{\nu^{5/2} \sqrt{\tau}}+\frac{\sqrt{\tau} (\alpha +\alpha  \mu \omega-1)}{\nu^{3/2}}-\frac{\alpha  \tau^{3/2}}{\sqrt{\nu}}\right)\right)
\\ \nonumber
\frac{\partial {\mathcal J}_{12}}{\partial \tau} \ &= \
\frac{1}{8} \lambda  e^{-\frac{\mu^2 \omega^2}{2 \nu \tau}} \left(2 \alpha  e^{\frac{(\mu \omega+\nu \tau)^2}{2 \nu \tau}} \erfc \left(\frac{\mu \omega+\nu \tau}{\sqrt{2} \sqrt{\nu \tau}}\right)+\alpha  \nu \tau e^{\frac{(\mu \omega+\nu \tau)^2}{2 \nu \tau}} \erfc \left(\frac{\mu \omega+\nu \tau}{\sqrt{2} \sqrt{\nu \tau}}\right)\ + \right. \\ \nonumber &\left. \sqrt{\frac{2}{\pi }} \left(\frac{(-1) (\alpha -1) \mu^2 \omega^2}{(\nu \tau)^{3/2}}+\frac{-\alpha +\alpha  \mu \omega+1}{\sqrt{\nu \tau}}-\alpha  \sqrt{\nu \tau}\right)\right)
\\ \nonumber
\frac{\partial {\mathcal J}_{21}}{\partial \mu} \ &= \
\lambda ^2 \omega^2 \left(\alpha ^2 \left(-e^{-\frac{\mu^2 \omega^2}{2 \nu \tau}}\right) e^{\frac{(\mu \omega+\nu \tau)^2}{2 \nu \tau}} \erfc \left(\frac{\mu \omega+\nu \tau}{\sqrt{2} \sqrt{\nu \tau}}\right)\ + \right. \\ \nonumber &\left. 2 \alpha ^2 e^{\frac{(\mu \omega+2 \nu \tau)^2}{2 \nu \tau}} e^{-\frac{\mu^2 \omega^2}{2 \nu \tau}} \erfc \left(\frac{\mu \omega+2 \nu \tau}{\sqrt{2} \sqrt{\nu \tau}}\right)-\erfc \left(\frac{\mu \omega}{\sqrt{2} \sqrt{\nu \tau}}\right)+2\right)
\\ \nonumber
\frac{\partial {\mathcal J}_{21}}{\partial \omega} \ &= \
\lambda ^2 \left(\alpha ^2 (\mu \omega+1) \left(-e^{-\frac{\mu^2 \omega^2}{2 \nu \tau}}\right) e^{\frac{(\mu \omega+\nu \tau)^2}{2 \nu \tau}} \erfc \left(\frac{\mu \omega+\nu \tau}{\sqrt{2} \sqrt{\nu \tau}}\right)\ + \right. \\ \nonumber &\left. \alpha ^2 (2 \mu \omega+1) e^{\frac{(\mu \omega+2 \nu \tau)^2}{2 \nu \tau}} e^{-\frac{\mu^2 \omega^2}{2 \nu \tau}} \erfc \left(\frac{\mu \omega+2 \nu \tau}{\sqrt{2} \sqrt{\nu \tau}}\right)\ + \right. \\ \nonumber &\left. 2 \mu \omega \left(2-\erfc \left(\frac{\mu \omega}{\sqrt{2} \sqrt{\nu \tau}}\right)\right)+\sqrt{\frac{2}{\pi }} \sqrt{\nu \tau} e^{-\frac{\mu^2 \omega^2}{2 \nu \tau}}\right)
\\ \nonumber
\frac{\partial {\mathcal J}_{21}}{\partial \nu} \ &= \
\frac{1}{2} \lambda ^2 \tau \omega e^{-\frac{\mu^2 \omega^2}{2 \nu \tau}} \left(\alpha ^2 \left(-e^{\frac{(\mu \omega+\nu \tau)^2}{2 \nu \tau}}\right) \erfc \left(\frac{\mu \omega+\nu \tau}{\sqrt{2} \sqrt{\nu \tau}}\right)\ + \right. \\ \nonumber &\left. 4 \alpha ^2 e^{\frac{(\mu \omega+2 \nu \tau)^2}{2 \nu \tau}} \erfc \left(\frac{\mu \omega+2 \nu \tau}{\sqrt{2} \sqrt{\nu \tau}}\right)+\frac{\sqrt{\frac{2}{\pi }} (-1) \left(\alpha ^2-1\right)}{\sqrt{\nu \tau}}\right)
\\ \nonumber
\frac{\partial {\mathcal J}_{21}}{\partial \tau} \ &= \
\frac{1}{2} \lambda ^2 \nu \omega e^{-\frac{\mu^2 \omega^2}{2 \nu \tau}} \left(\alpha ^2 \left(-e^{\frac{(\mu \omega+\nu \tau)^2}{2 \nu \tau}}\right) \erfc \left(\frac{\mu \omega+\nu \tau}{\sqrt{2} \sqrt{\nu \tau}}\right)\ + \right. \\ \nonumber &\left. 4 \alpha ^2 e^{\frac{(\mu \omega+2 \nu \tau)^2}{2 \nu \tau}} \erfc \left(\frac{\mu \omega+2 \nu \tau}{\sqrt{2} \sqrt{\nu \tau}}\right)+\frac{\sqrt{\frac{2}{\pi }} (-1) \left(\alpha ^2-1\right)}{\sqrt{\nu \tau}}\right)
\\ \nonumber
\frac{\partial {\mathcal J}_{22}}{\partial \mu} \ &= \ \frac{\partial {\mathcal J}_{21}}{\partial \nu}
\\ \nonumber
\frac{\partial {\mathcal J}_{22}}{\partial \omega} \ &= \
\frac{1}{2} \lambda ^2 \mu \tau e^{-\frac{\mu^2 \omega^2}{2 \nu \tau}} \left(\alpha ^2 \left(-e^{\frac{(\mu \omega+\nu \tau)^2}{2 \nu \tau}}\right) \erfc \left(\frac{\mu \omega+\nu \tau}{\sqrt{2} \sqrt{\nu \tau}}\right)\ + \right. \\ \nonumber &\left. 4 \alpha ^2 e^{\frac{(\mu \omega+2 \nu \tau)^2}{2 \nu \tau}} \erfc \left(\frac{\mu \omega+2 \nu \tau}{\sqrt{2} \sqrt{\nu \tau}}\right)+\frac{\sqrt{\frac{2}{\pi }} (-1) \left(\alpha ^2-1\right)}{\sqrt{\nu \tau}}\right)
\\ \nonumber
\frac{\partial {\mathcal J}_{22}}{\partial \nu} \ &= \
\frac{1}{4} \lambda ^2 \tau^2 e^{-\frac{\mu^2 \omega^2}{2 \nu \tau}} \left(\alpha ^2 \left(-e^{\frac{(\mu \omega+\nu \tau)^2}{2 \nu \tau}}\right) \erfc \left(\frac{\mu \omega+\nu \tau}{\sqrt{2} \sqrt{\nu \tau}}\right)\ + \right. \\ \nonumber &\left. 8 \alpha ^2 e^{\frac{(\mu \omega+2 \nu \tau)^2}{2 \nu \tau}} \erfc \left(\frac{\mu \omega+2 \nu \tau}{\sqrt{2} \sqrt{\nu \tau}}\right)+\sqrt{\frac{2}{\pi }} \left(\frac{\left(\alpha ^2-1\right) \mu \omega}{(\nu \tau)^{3/2}}-\frac{3 \alpha ^2}{\sqrt{\nu \tau}}\right)\right)
\\ \nonumber
\frac{\partial {\mathcal J}_{22}}{\partial \tau} \ &= \
\frac{1}{4} \lambda ^2 \left(-2 \alpha ^2 e^{-\frac{\mu^2 \omega^2}{2 \nu \tau}} 
e^{\frac{(\mu \omega+\nu \tau)^2}{2 \nu \tau}} \erfc \left(\frac{\mu \omega+\nu \tau}{\sqrt{2}
\sqrt{\nu \tau}}\right)\ - \right. \\ \nonumber 
&\left. \alpha ^2 \nu \tau e^{-\frac{\mu^2 \omega^2}{2 \nu \tau}} e^{\frac{(\mu \omega+\nu \tau)^2}{2 \nu \tau}} 
\erfc \left(\frac{\mu \omega+\nu \tau}{\sqrt{2} \sqrt{\nu \tau}}\right)+4 \alpha ^2 e^{\frac{(\mu \omega+2 \nu \tau)^2}
{2 \nu \tau}} e^{-\frac{\mu^2 \omega^2}{2 \nu \tau}} \erfc \left(\frac{\mu \omega+2 \nu \tau}{\sqrt{2} \sqrt{\nu \tau}}\right)
+\right. \\ \nonumber 
&\left.8 \alpha ^2 \nu \tau e^{\frac{(\mu \omega+2 \nu \tau)^2}{2 \nu \tau}} e^{-\frac{\mu^2 \omega^2}{2 \nu \tau}} 
\erfc \left(\frac{\mu \omega+2 \nu \tau}{\sqrt{2} \sqrt{\nu \tau}}\right)+2 \left(2-\erfc \left(\frac{\mu \omega}
{\sqrt{2} \sqrt{\nu \tau}}\right)\right)+\right. \\ \nonumber 
&\left.\sqrt{\frac{2}{\pi }} e^{-\frac{\mu^2 \omega^2}{2 \nu \tau}} \left(\frac{\left(\alpha ^2-1\right) 
\mu \omega}{\sqrt{\nu \tau}}-3 \alpha ^2 \sqrt{\nu \tau}\right)\right)
\end{align}

\begin{lemma}[Bounds on the Derivatives]
\label{lem:Bounds}
The following bounds on the absolute values of the 
derivatives of the Jacobian entries ${\mathcal J}_{11}(\mu,\omega,\nu,\tau,\lambda ,\alpha )$,
${\mathcal J}_{12}(\mu,\omega,\nu,\tau,\lambda ,\alpha )$,
${\mathcal J}_{21}(\mu,\omega,\nu,\tau,\lambda ,\alpha )$, and
${\mathcal J}_{22}(\mu,\omega,\nu,\tau,\lambda ,\alpha )$
with respect to 
$\mu$, $\omega$, $\nu$, and $\tau$ hold:
\begin{align}
\left| \frac{\partial {\mathcal J}_{11}}{\partial \mu} \right| \ &< \ 0.0031049101995398316 \\ \nonumber
\left| \frac{\partial {\mathcal J}_{11}}{\partial \omega} \right| \ &< \ 1.055872374194189 \\ \nonumber
\left| \frac{\partial {\mathcal J}_{11}}{\partial \nu} \right| \ &< \  0.031242911235461816 \\ \nonumber
\left| \frac{\partial {\mathcal J}_{11}}{\partial \tau} \right| \ &< \ 0.03749149348255419 \\ \nonumber
\end{align}
\begin{align*}
\left| \frac{\partial {\mathcal J}_{12}}{\partial \mu} \right| \ &< \ 0.031242911235461816 \\ \nonumber
\left| \frac{\partial {\mathcal J}_{12}}{\partial \omega} \right| \ &< \ 0.031242911235461816 \\ \nonumber
\left| \frac{\partial {\mathcal J}_{12}}{\partial \nu} \right| \ &< \ 0.21232788238624354 \\ \nonumber
\left| \frac{\partial {\mathcal J}_{12}}{\partial \tau} \right| \ &< \ 0.2124377655377270 \\ \nonumber
\end{align*}
\begin{align*}
\left| \frac{\partial {\mathcal J}_{21}}{\partial \mu} \right| \ &< \  0.02220441024325437 \\ \nonumber
\left| \frac{\partial {\mathcal J}_{21}}{\partial \omega} \right| \ &< \  1.146955401845684 \\ \nonumber
\left| \frac{\partial {\mathcal J}_{21}}{\partial \nu} \right| \ &< \ 0.14983446469110305 \\ \nonumber
\left| \frac{\partial {\mathcal J}_{21}}{\partial \tau} \right| \ &< \ 0.17980135762932363 \\ \nonumber
\end{align*}
\begin{align*}
\left| \frac{\partial {\mathcal J}_{22}}{\partial \mu} \right| \ &< \ 0.14983446469110305 \\ \nonumber
\left| \frac{\partial {\mathcal J}_{22}}{\partial \omega} \right| \ &< \ 0.14983446469110305 \\ \nonumber
\left| \frac{\partial {\mathcal J}_{22}}{\partial \nu} \right| \ &< \ 1.805740052651535 \\ \nonumber
\left| \frac{\partial {\mathcal J}_{22}}{\partial \tau} \right| \ &< \ 2.396685907216327
\end{align*}
\end{lemma}

\begin{proof}
 See proof~\ref{proof:Bounds}.
\end{proof}

\paragraph{Bounds on the entries of the Jacobian.}\index{bounds!Jacobian entries}\index{Jacobian!entries}\index{Jacobian!bounds}

\begin{lemma}[Bound on J11]
\label{lem:J11}
The absolute value of the function \\ 
$\mathcal J_{11} = \frac{1}{2} \lambda  \omega \left( 
	\alpha  e^{\mu \omega+\frac{\nu \tau}{2}} \erfc \left(\frac{\mu \omega + \nu \tau}{\sqrt{2} \sqrt{\nu \tau}}\right) - 
	\erfc\left(\frac{\mu \omega}{\sqrt{2} \sqrt{\nu \tau}}\right)+2\right)$ is bounded by 
$\left|\mathcal J_{11} \right| \leq 0.104497$ in the domain $-0.1 \leq \mu \leq 0.1$, $-0.1 \leq \omega \leq 0.1$, $0.8 \leq \nu \leq 1.5$, 
and $0.8 \leq \tau \leq 1.25$ for $\alpha=\alpha_{\mathrm{01}}$ and $\lambda=\lambda_{\mathrm{01}}$.
\end{lemma}

\begin{proof}
\begin{align}
& \left|\mathcal J_{11} \right| = \left| \frac{1}{2} \lambda  \omega \left( 
	\alpha  e^{\mu \omega+\frac{\nu \tau}{2}} \erfc \left(\frac{\mu \omega + \nu \tau}{\sqrt{2} \sqrt{\nu \tau}}\right) + 
	2 - \erfc\left(\frac{\mu \omega}{\sqrt{2} \sqrt{\nu \tau}}\right)\right) \right| \nonumber \\
& \leq |\frac{1}{2}| |\lambda| |\omega| \left( |\alpha| 0.587622 + 1.00584 \right) \leq 0.104497, \nonumber \\
\end{align}
where we used that (a) $J_{11}$ is strictly monotonically increasing in $\mu \omega$ and $|2-\erfc\left(\frac{0.01}{\sqrt{2} \sqrt{\nu \tau}}\right)| \leq 1.00584$
and (b) Lemma~\ref{lem:mainsubfunctionJ11J12} that 
$| e^{\mu \omega+\frac{\nu \tau}{2}} \erfc \left(\frac{\mu \omega + \nu \tau}{\sqrt{2} \sqrt{\nu \tau}}\right)| \leq e^{0.01 +\frac{0.64}{2}} \erfc \left(\frac{0.01 + 0.64}{\sqrt{2} \sqrt{0.64}}\right)=0.587622$
\end{proof}

\begin{lemma}[Bound on J12]
\label{lem:J12}
The absolute value of the function \\ 
$\mathcal J_{12} = 
\frac{1}{4} \lambda  \tau \left(\alpha e^{\mu \omega+\frac{\nu \tau}{2}} 
\erfc \left(\frac{\mu \omega+\nu \tau}{\sqrt{2} \sqrt{\nu \tau}}\right)-(\alpha -1) \sqrt{\frac{2}{\pi  \nu \tau}} e^{-\frac{\mu^2\omega^2}{2 \nu \tau}}\right)$ is bounded by 
$\left|\mathcal J_{12} \right| \leq 0.194145$ in the domain $-0.1 \leq \mu \leq 0.1$, $-0.1 \leq \omega \leq 0.1$, $0.8 \leq \nu \leq 1.5$, 
and $0.8 \leq \tau \leq 1.25$ for $\alpha=\alpha_{\mathrm{01}}$ and $\lambda=\lambda_{\mathrm{01}}$.
\end{lemma}

\begin{proof}

\begin{align}
& |J_{12}| \leq 
\frac{1}{4} |\lambda|  |\tau| 
\left| \left( \alpha  e^{\mu \omega+\frac{\nu \tau}{2}} \erfc \left(\frac{\mu \omega+\nu \tau}{\sqrt{2} \sqrt{\nu \tau}}\right) 
-  (\alpha-1) \sqrt{\frac{2}{\pi  \nu \tau}} e^{-\frac{\mu^2\omega^2}{2 \nu \tau}} \right) \right| \leq \nonumber \\
& \frac{1}{4} |\lambda|  |\tau| \left|  0.983247-0.392294 \right|  \leq \nonumber \\
& 0.194035
\end{align}

For the first term we have $0.434947 \leq e^{\mu \omega+\frac{\nu \tau}{2}} \erfc \left(\frac{\mu \omega+\nu \tau}{\sqrt{2} \sqrt{\nu \tau}} \right) \leq 0.587622$  after 
Lemma~\ref{lem:mainsubfunctionJ11J12} and for the second term $0.582677 \leq \sqrt{\frac{2}{\pi  \nu \tau}} e^{-\frac{\mu^2\omega^2}{2 \nu \tau}} \leq 0.997356$, which can easily be seen
by maximizing or minimizing the arguments of the exponential or the square root function. The first term scaled by $\alpha$ is
$0.727780 \leq \alpha e^{\mu \omega+\frac{\nu \tau}{2}} \erfc \left(\frac{\mu \omega+\nu \tau}{\sqrt{2} \sqrt{\nu \tau}} \right) \leq 0.983247$
and the second term scaled by $\alpha-1$ is
$0.392294 \leq (\alpha-1)\sqrt{\frac{2}{\pi  \nu \tau}} e^{-\frac{\mu^2\omega^2}{2 \nu \tau}} \leq 0.671484$.
Therefore, the absolute difference between these terms is at most $0.983247-0.392294$
leading to the derived bound.


\end{proof}


\paragraph{Bounds on mean, variance and second moment.}\index{bounds!mean and variance}
For deriving bounds on $\munn$, $\xinn$, and $\nunn$, we need 
the following lemma.

\begin{lemma}[Derivatives of the Mapping]
\label{lem:mapDerivatives}
We assume $\alpha = \alpha_{\rm 01}$ and $\lambda=\lambda_{\rm 01}$.
We restrict the range of the variables to the domain
$\mu \in [-0.1,0.1]$,
$\omega \in [-0.1,0.1]$,
$\nu \in [0.8,1.5]$, and
$\tau \in [0.8,1.25]$.

The derivative $\frac{\partial }{\partial \mu}
\munn(\mu,\omega,\nu,\tau,\lambda ,\alpha )$
has the sign of $\omega$.

The derivative $\frac{\partial }{\partial \nu} \munn(\mu,\omega,\nu,\tau,\lambda ,\alpha )$
is positive.

The derivative $\frac{\partial }{\partial \mu } \xinn(\mu,\omega,\nu,\tau,\lambda ,\alpha )$
has the sign of $\omega$.

The derivative 
$\frac{\partial }{\partial \nu } \xinn(\mu,\omega,\nu,\tau,\lambda ,\alpha )$
is positive.
\end{lemma}

\begin{proof}
 See \ref{proof:mapDerivatives}.
\end{proof}

\begin{lemma}[Bounds on mean, variance and second moment]
\label{lem:boundsmeanvar}
The expressions $\munn$, $\xinn$, and $\nunn$
for
$\alpha = \alpha_{\rm 01}$ and $\lambda=\lambda_{\rm 01}$
are bounded by
$-0.041160 < \munn < 0.087653$,
$0.703257 < \xinn <1.643705$
and
$0.695574 < \nunn < 1.636023$ 
in the domain $\mu \in [-0.1,0.1]$, 
$\nu \in [0.8,15]$, $\omega \in [-0.1,0.1]$,  $\tau \in [0.8,1.25]$.
\end{lemma}

\begin{proof}
We use Lemma~\ref{lem:mapDerivatives} which states that with given
sign the derivatives of the mapping   Eq.~\eqref{eq:mappingMean}
and Eq.~\eqref{eq:mappingVar} with respect to  $\nu$
and $\mu$ are either positive or have the sign of
$\omega$.
Therefore with given sign of $\omega$ the mappings are strict monotonic and
the their maxima and minima are found at the borders.  The minimum of $\munn$ is obtained at
$\mu \omega = -0.01$ and its maximum at $\mu \omega=0.01$ and $\sigma$ and $\tau$ at minimal or maximal values, respectively.
It follows that
\begin{align}
-0.041160 < \munn(-0.1,0.1, 0.8, 0.8, \lambda_{\rm 01}, \alpha_{\rm 01})  \leq  & \munn \leq \munn(0.1,0.1,1.5, 1.25, \lambda_{\rm 01}, \alpha_{\rm 01}) < 0.087653.
\end{align}

Similarly, the maximum and minimum of $\xinn$ is obtained at the values mentioned above:
\begin{align}
0.703257 <\xinn(-0.1,0.1, 0.8, 0.8, \lambda_{\rm 01}, \alpha_{\rm 01})  \leq  & \xinn \leq \xinn(0.1,0.1,1.5, 1.25, \lambda_{\rm 01}, \alpha_{\rm 01}) < 1.643705. 
\end{align}

Hence we obtain the following bounds on $\nunn$:
\begin{align}
0.703257 - \munn^2 <  \xinn - \munn^2  &  < 1.643705 - \munn^2  \\ \nonumber
0.703257 - 0.007683 <  \nunn  &  < 1.643705 - 0.007682   \\ \nonumber
0.695574 <  \nunn  &  < 1.636023 . 
\end{align}

\end{proof}

\paragraph{Upper Bounds on the Largest Singular Value of the Jacobian.}\index{bounds!singular value}\index{Jacobian!singular value bound}

\begin{lemma}[Upper Bounds on Absolute Derivatives of Largest Singular Value]
\label{lem:Ds1Bounds}
We set
$\alpha = \alpha_{\rm 01}$ and $\lambda=\lambda_{\rm 01}$ and
restrict the range of the variables to
$\mu \in [\mu_{\rm min},\mu_{\rm max}]=[-0.1,0.1]$,
$\omega \in [\omega_{\rm min},\omega_{\rm max}]=[-0.1,0.1]$,
$\nu \in [\nu_{\rm min},\nu_{\rm max}]=[0.8,1.5]$, and
$\tau \in [\tau_{\rm min},\tau_{\rm max}]=[0.8,1.25]$.

The absolute values of derivatives of the largest singular value
$S(\mu,\omega,\nu,\tau,\lambda ,\alpha )$
given in Eq.~\eqref{eq:S} with respect to
$(\mu,\omega,\nu,\tau)$ are bounded as follows:


\begin{align}
\left|\frac{\partial S}{\partial \mu}\right| \ &< \ 0.32112 \ , \\
\left|\frac{\partial S}{\partial \omega}\right| \ &< \ 2.63690 \ , \\
\left|\frac{\partial S}{\partial \nu}\right| \ &< \ 2.28242 \ , \\
\left|\frac{\partial S}{\partial \tau}\right| \ &< \ 2.98610 \ .
\end{align}


\end{lemma}
\begin{proof}

The Jacobian of our mapping  Eq.~\eqref{eq:mappingMean} and 
Eq.~\eqref{eq:mappingVar} is defined as
\begin{align}
\bm H \ &= \ \left(
\begin{array}{cc}
{\mathcal H}_{11} & {\mathcal H}_{12} \\
 {\mathcal H}_{21} & {\mathcal H}_{22} \\
\end{array}
\right) =
\ \left(
\begin{array}{cc}
{\mathcal J}_{11} & {\mathcal J}_{12} \\
 {\mathcal J}_{21} - 2 \munn {\mathcal J}_{11} & {\mathcal J}_{22} - 2 \munn {\mathcal J}_{12} \\
\end{array}
\right)
\end{align}

and has the largest singular value
\begin{align}
S(\mu,\omega,\nu,\tau,\lambda ,\alpha ) \ &= \ 
\frac{1}{2}
\left(\sqrt{({\mathcal H}_{11}-{\mathcal H}_{22})^2+({\mathcal H}_{12}+{\mathcal H}_{21})^2}+
\sqrt{({\mathcal H}_{11}+{\mathcal H}_{22})^2+({\mathcal H}_{12}-{\mathcal H}_{21})^2}\right),
\end{align}
according to the formula of \citet{Blinn:96}.

We obtain
\begin{align}
&\left|\frac{\partial S}{\partial {\mathcal H}_{11}}\right| \ = \
\left|\frac{1}{2} \left(\frac{{\mathcal H}_{11}-{\mathcal H}_{22}}{\sqrt{({\mathcal H}_{11}-{\mathcal H}_{22})^2+({\mathcal H}_{12}+{\mathcal H}_{21})^2}}+\frac{{\mathcal H}_{11}+{\mathcal H}_{22}}{\sqrt{({\mathcal H}_{11}+{\mathcal H}_{22})^2+({\mathcal H}_{21}-{\mathcal H}_{12})^2}}\right)\right|
\ < \\\nonumber
&\frac{1}{2} \left(\left|\frac{1}{\sqrt{\frac{({\mathcal H}_{12}+{\mathcal H}_{21})^2}{({\mathcal H}_{11}-{\mathcal H}_{22})^2}+1}}\right|+\left|\frac{1}{\sqrt{\frac{({\mathcal H}_{21}-{\mathcal H}_{12})^2}{({\mathcal H}_{11}+{\mathcal H}_{22})^2}+1}}\right|\right)
\ < \ \frac{1+1}{2} \ = \ 1
\end{align}
and analogously 
\begin{align}
\left|\frac{\partial S}{\partial {\mathcal H}_{12}}\right|\ = \ \left|\frac{1}{2}
\left(\frac{{\mathcal H}_{12}+{\mathcal H}_{21}}{\sqrt{({\mathcal H}_{11}-{\mathcal H}_{22})^2+({\mathcal H}_{12}+{\mathcal H}_{21})^2}}-\frac{{\mathcal H}_{21}-{\mathcal H}_{12}}{\sqrt{({\mathcal H}_{11}+{\mathcal H}_{22})^2+({\mathcal H}_{21}-{\mathcal H}_{12})^2}}\right)\right|\ < \ 1
\end{align}
and
\begin{align}
\left|\frac{\partial S}{\partial {\mathcal H}_{21}}\right| \ = \ \left|\frac{1}{2} \left(\frac{{\mathcal H}_{21}-{\mathcal H}_{12}}{\sqrt{({\mathcal H}_{11}+{\mathcal H}_{22})^2+({\mathcal H}_{21}-{\mathcal H}_{12})^2}}+\frac{{\mathcal H}_{12}+{\mathcal H}_{21}}{\sqrt{({\mathcal H}_{11}-{\mathcal H}_{22})^2+({\mathcal H}_{12}+{\mathcal H}_{21})^2}}\right)\right|\ < \ 1
\end{align}
and
\begin{align}
\left|\frac{\partial S}{\partial {\mathcal H}_{22}}\right| \ = \ \left|\frac{1}{2}
\left(\frac{{\mathcal H}_{11}+{\mathcal H}_{22}}{\sqrt{({\mathcal H}_{11}+{\mathcal H}_{22})^2+({\mathcal H}_{21}-{\mathcal H}_{12})^2}}-\frac{{\mathcal H}_{11}-{\mathcal H}_{22}}{\sqrt{({\mathcal H}_{11}-{\mathcal H}_{22})^2+({\mathcal H}_{12}+{\mathcal H}_{21})^2}}\right)\right|\
  < \ 1 \ .
\end{align}

We have
\begin{align}
\frac{\partial S}{\partial \mu} \ &= \ 
\frac{\partial S}{\partial {\mathcal H}_{11}}
\frac{\partial {\mathcal H}_{11}}{\partial \mu} \ + \
\frac{\partial S}{\partial {\mathcal H}_{12}}
\frac{\partial {\mathcal H}_{12}}{\partial \mu} \ + \
\frac{\partial S}{\partial {\mathcal H}_{21}}
\frac{\partial {\mathcal H}_{21}}{\partial \mu} \ + \
\frac{\partial S}{\partial {\mathcal H}_{22}}
\frac{\partial {\mathcal H}_{22}}{\partial \mu} \\
\frac{\partial S}{\partial \omega} \ &= \ 
\frac{\partial S}{\partial {\mathcal H}_{11}}
\frac{\partial {\mathcal H}_{11}}{\partial \omega} \ + \
\frac{\partial S}{\partial {\mathcal H}_{12}}
\frac{\partial {\mathcal H}_{12}}{\partial \omega} \ + \
\frac{\partial S}{\partial {\mathcal H}_{21}}
\frac{\partial {\mathcal H}_{21}}{\partial \omega} \ + \
\frac{\partial S}{\partial {\mathcal H}_{22}}
\frac{\partial {\mathcal H}_{22}}{\partial \omega} \\
\frac{\partial S}{\partial \nu} \ &= \ 
\frac{\partial S}{\partial {\mathcal H}_{11}}
\frac{\partial {\mathcal H}_{11}}{\partial \nu} \ + \
\frac{\partial S}{\partial {\mathcal H}_{12}}
\frac{\partial {\mathcal H}_{12}}{\partial \nu} \ + \
\frac{\partial S}{\partial {\mathcal H}_{21}}
\frac{\partial {\mathcal H}_{21}}{\partial \nu} \ + \
\frac{\partial S}{\partial {\mathcal H}_{22}}
\frac{\partial {\mathcal H}_{22}}{\partial \nu} \\
\frac{\partial S}{\partial \tau} \ &= \ 
\frac{\partial S}{\partial {\mathcal H}_{11}}
\frac{\partial {\mathcal H}_{11}}{\partial \tau} \ + \
\frac{\partial S}{\partial {\mathcal H}_{12}}
\frac{\partial {\mathcal H}_{12}}{\partial \tau} \ + \
\frac{\partial S}{\partial {\mathcal H}_{21}}
\frac{\partial {\mathcal H}_{21}}{\partial \tau} \ + \
\frac{\partial S}{\partial {\mathcal H}_{22}}
\frac{\partial {\mathcal H}_{22}}{\partial \tau} \\
\end{align}
from which follows using the bounds from Lemma~\ref{lem:Bounds}:

Derivative of the singular value w.r.t. $\mu$:
\begin{align}
&\left|\frac{\partial S}{\partial \mu}\right| \ \leq \\\nonumber
&\left|\frac{\partial S}{\partial {\mathcal H}_{11}}\right| \left|\frac{\partial {\mathcal H}_{11}}{\partial \mu}\right| + 
\left| \frac{\partial S}{\partial {\mathcal H}_{12}}\right| \left|\frac{\partial {\mathcal H}_{12}}{\partial \mu}\right| + 
\left| \frac{\partial S}{\partial {\mathcal H}_{21}}\right| \left|\frac{\partial {\mathcal H}_{21}}{\partial \mu}\right| + 
\left| \frac{\partial S}{\partial {\mathcal H}_{22}}\right| \left|\frac{\partial {\mathcal H}_{22}}{\partial \mu}\right|
\ \leq \\\nonumber
&\left|\frac{\partial {\mathcal H}_{11}}{\partial \mu}\right| + 
\left|\frac{\partial {\mathcal H}_{12}}{\partial \mu} \right| + 
\left|\frac{\partial {\mathcal H}_{21}}{\partial \mu} \right|+ 
\left|\frac{\partial {\mathcal H}_{22}}{\partial \mu}\right|
\ \leq \\\nonumber
&\left|\frac{\partial {\mathcal J}_{11}}{\partial \mu}\right| + 
\left|\frac{\partial {\mathcal J}_{12}}{\partial \mu} \right| + 
\left|\frac{\partial {\mathcal J}_{21} - 2 \munn {\mathcal J}_{11}}{\partial \mu} \right|+ 
\left|\frac{\partial {\mathcal J}_{22} - 2 \munn {\mathcal J}_{12}}{\partial \mu}\right|
\ \leq \\\nonumber
&\left|\frac{\partial {\mathcal J}_{11}}{\partial \mu}\right| + 
\left|\frac{\partial {\mathcal J}_{12}}{\partial \mu} \right| + 
\left|\frac{\partial {\mathcal J}_{21}}{\partial \mu} \right|+ 
\left|\frac{\partial {\mathcal J}_{22}}{\partial \mu}\right| +
2 \left|\frac{\partial {\mathcal J}_{11}}{\partial \mu}\right| \left| \munn\right|  + 2 \left| \mathcal{J}_{11} \right|^2 +
2 \left|\frac{\partial {\mathcal J}_{12}}{\partial \mu}\right| \left| \munn\right|  + 2 \left| \mathcal{J}_{12} \right| \left| \mathcal{J}_{11} \right|
\ \leq \\\nonumber
& 0.0031049101995398316 + 0.031242911235461816  + 0.02220441024325437 + 0.14983446469110305 + \\ \nonumber
& 2 \cdot 0.104497 \cdot 0.087653 + 2\cdot 0.104497^2 + \\ \nonumber
& 2 \cdot 0.194035 \cdot 0.087653 + 2 \cdot  0.104497 \cdot 0.194035
< \ 0.32112,  \\\nonumber
\end{align}
where we used the results from the lemmata \ref{lem:Bounds}, \ref{lem:J11}, \ref{lem:J12}, and \ref{lem:boundsmeanvar}.

Derivative of the singular value w.r.t. $\omega$:
\begin{align}
&\left|\frac{\partial S}{\partial \omega}\right| \ \leq \\\nonumber
&\left|\frac{\partial S}{\partial {\mathcal H}_{11}}\right| \left|\frac{\partial {\mathcal H}_{11}}{\partial \omega}\right| + 
 \left| \frac{\partial S}{\partial {\mathcal H}_{12}}\right| \left|\frac{\partial {\mathcal H}_{12}}{\partial \omega}\right| + 
 \left|\frac{\partial S}{\partial {\mathcal H}_{21}}\right| \left|\frac{\partial {\mathcal H}_{21}}{\partial \omega}\right| + 
 \left| \frac{\partial S}{\partial {\mathcal H}_{22}}\right| \left|\frac{\partial {\mathcal H}_{22}}{\partial \omega}\right|
\ \leq \\\nonumber
&\left|\frac{\partial {\mathcal H}_{11}}{\partial \omega}\right| + 
\left|\frac{\partial {\mathcal H}_{12}}{\partial \omega} \right| + 
\left|\frac{\partial {\mathcal H}_{21}}{\partial \omega} \right|+ 
\left|\frac{\partial {\mathcal H}_{22}}{\partial \omega}\right|
\ \leq \\\nonumber
&\left|\frac{\partial {\mathcal J}_{11}}{\partial \omega}\right| + 
\left|\frac{\partial {\mathcal J}_{12}}{\partial \omega} \right| + 
\left|\frac{\partial {\mathcal J}_{21} - 2 \munn {\mathcal J}_{11}}{\partial \omega} \right|+ 
\left|\frac{\partial {\mathcal J}_{22} - 2 \munn {\mathcal J}_{12}}{\partial \omega}\right|
\ \leq \\\nonumber
&\left|\frac{\partial {\mathcal J}_{11}}{\partial \omega}\right| + 
\left|\frac{\partial {\mathcal J}_{12}}{\partial \omega} \right| + 
\left|\frac{\partial {\mathcal J}_{21}}{\partial \omega} \right|+ 
\left|\frac{\partial {\mathcal J}_{22}}{\partial \omega}\right| +
2 \left|\frac{\partial {\mathcal J}_{11}}{\partial \omega}\right| \left| \munn\right|  + 2 \left| \mathcal{J}_{11} \right| \left| \frac{\partial \munn}{\partial \omega}\right| + \nonumber \\
& \left. 2 \left|\frac{\partial {\mathcal J}_{12}}{\partial \omega}\right| \left| \munn\right|  + 2 \left| \mathcal{J}_{12} \right| \left| \frac{\partial \munn}{\partial \omega}\right| \right.
\ \leq \\\nonumber
& 2.38392 + 2\cdot 1.055872374194189 \cdot 0.087653 + 2 \cdot 0.104497^2 + 2\cdot 0.031242911235461816 \cdot 0.087653 \\ \nonumber
& + 2\cdot 0.194035 \cdot 0.104497\ <\ 2.63690 \ , \\\nonumber
\end{align}
where we used the results from the lemmata \ref{lem:Bounds}, \ref{lem:J11}, \ref{lem:J12}, and \ref{lem:boundsmeanvar} and that
$\munn$ is symmetric for $\mu, \omega$.

Derivative of the singular value w.r.t. $\nu$:
\begin{align}
&\left|\frac{\partial S}{\partial \nu}\right| \ \leq \\\nonumber
&\left|\frac{\partial S}{\partial {\mathcal H}_{11}}\right| \left|\frac{\partial {\mathcal H}_{11}}{\partial \nu}\right| + 
\left| \frac{\partial S}{\partial {\mathcal H}_{12}}\right| \left|\frac{\partial {\mathcal H}_{12}}{\partial \nu}\right| + 
\left|\frac{\partial S}{\partial {\mathcal H}_{21}}\right| \left|\frac{\partial {\mathcal H}_{21}}{\partial \nu}\right| + 
\left| \frac{\partial S}{\partial {\mathcal H}_{22}}\right| \left|\frac{\partial {\mathcal H}_{22}}{\partial \nu}\right|
\ \leq \\\nonumber
&\left|\frac{\partial {\mathcal H}_{11}}{\partial \nu}\right| + 
\left|\frac{\partial {\mathcal H}_{12}}{\partial \nu} \right| + 
\left|\frac{\partial {\mathcal H}_{21}}{\partial \nu} \right|+ 
\left|\frac{\partial {\mathcal H}_{22}}{\partial \nu}\right|
\ \leq \\\nonumber
&\left|\frac{\partial {\mathcal J}_{11}}{\partial \nu}\right| + 
\left|\frac{\partial {\mathcal J}_{12}}{\partial \nu} \right| + 
\left|\frac{\partial {\mathcal J}_{21} - 2 \munn {\mathcal J}_{11}}{\partial \nu} \right|+ 
\left|\frac{\partial {\mathcal J}_{22} - 2 \munn {\mathcal J}_{12}}{\partial \nu}\right|
\ \leq \\\nonumber
&\left|\frac{\partial {\mathcal J}_{11}}{\partial \nu}\right| +
\left|\frac{\partial {\mathcal J}_{12}}{\partial \nu} \right| + 
\left|\frac{\partial {\mathcal J}_{21}}{\partial \nu} \right|+ 
\left|\frac{\partial {\mathcal J}_{22}}{\partial \nu}\right| +
2 \left|\frac{\partial {\mathcal J}_{11}}{\partial \nu} \right| \left| \munn\right|  + 2 \left| \mathcal{J}_{11} \right| \left| \mathcal{J}_{12} \right|  +
2 \left|\frac{\partial {\mathcal J}_{12}}{\partial \nu}\right| \left| \munn\right|  + 2  \left| \mathcal{J}_{12} \right|^2 
\ \leq \\\nonumber
& 2.19916 + 2\cdot 0.031242911235461816 \cdot 0.087653 + 2\cdot 0.104497 \cdot 0.194035 + \\\nonumber
& 2 \cdot 0.21232788238624354 \cdot 0.087653 + 2\cdot 0.194035^2 \ < \ 2.28242 \ , \\\nonumber
\end{align}
where we used the results from the lemmata \ref{lem:Bounds}, \ref{lem:J11}, \ref{lem:J12}, and \ref{lem:boundsmeanvar}.

Derivative of the singular value w.r.t. $\tau$:
\begin{align}
&\left|\frac{\partial S}{\partial \tau}\right| \ \leq \\\nonumber
&\left|\frac{\partial S}{\partial {\mathcal H}_{11}}\right| \left|\frac{\partial {\mathcal H}_{11}}{\partial \tau}\right| + 
\left| \frac{\partial S}{\partial {\mathcal H}_{12}}\right| \left|\frac{\partial {\mathcal H}_{12}}{\partial \tau}\right| + 
\left|\frac{\partial S}{\partial {\mathcal H}_{21}}\right| \left|\frac{\partial {\mathcal H}_{21}}{\partial \tau}\right| + 
\left| \frac{\partial S}{\partial {\mathcal H}_{22}}\right| \left|\frac{\partial {\mathcal H}_{22}}{\partial \tau}\right|
\ \leq \\\nonumber
&\left|\frac{\partial {\mathcal H}_{11}}{\partial \tau}\right| + 
\left|\frac{\partial {\mathcal H}_{12}}{\partial \tau} \right| + 
\left|\frac{\partial {\mathcal H}_{21}}{\partial \tau} \right|+ 
\left|\frac{\partial {\mathcal H}_{22}}{\partial \tau}\right|
\ \leq \\\nonumber
&\left|\frac{\partial {\mathcal J}_{11}}{\partial \tau}\right| + 
\left|\frac{\partial {\mathcal J}_{12}}{\partial \tau} \right| + 
\left|\frac{\partial {\mathcal J}_{21} - 2 \munn {\mathcal J}_{11}}{\partial \tau} \right|+ 
\left|\frac{\partial {\mathcal J}_{22} - 2 \munn {\mathcal J}_{12}}{\partial \tau}\right|
\ \leq \\\nonumber
&\left|\frac{\partial {\mathcal J}_{11}}{\partial \tau}\right| +
\left|\frac{\partial {\mathcal J}_{12}}{\partial \tau} \right| + 
\left|\frac{\partial {\mathcal J}_{21}}{\partial \tau} \right|+ 
\left|\frac{\partial {\mathcal J}_{22}}{\partial \tau}\right| +
2 \left|\frac{\partial {\mathcal J}_{11}}{\partial \tau} \right| \left| \munn\right|  + 2 \left| \mathcal{J}_{11} \right|  \left| \frac{\partial \munn}{\partial \tau}\right|+ \nonumber \\
& \left. 2 \left|\frac{\partial {\mathcal J}_{12}}{\partial \tau}\right| \left| \munn\right|   + 2 \left| \mathcal{J}_{12} \right|  \left| \frac{\partial \munn}{\partial \tau}\right| \right.
\ \leq \\\nonumber
& 2.82643 +2 \cdot 0.03749149348255419 \cdot 0.087653 + 2\cdot 0.104497 \cdot 0.194035 + \\ \nonumber
& 2\cdot 0.2124377655377270 \cdot 0.087653 + 2\cdot 0.194035^2  \ < \ 2.98610 \ ,
\end{align}
where we used the results from the lemmata \ref{lem:Bounds}, \ref{lem:J11}, \ref{lem:J12}, and \ref{lem:boundsmeanvar} and that
$\munn$ is symmetric for  $\nu, \tau$.

\end{proof}

\begin{lemma}[Mean Value Theorem Bound on Deviation from Largest Singular Value]\index{bounds!singular value}
\label{lem:meanValue}
We set
$\alpha = \alpha_{\rm 01}$ and $\lambda=\lambda_{\rm 01}$ and
restrict the range of the variables to
$\mu \in [\mu_{\rm min},\mu_{\rm max}]=[-0.1,0.1]$,
$\omega \in [\omega_{\rm min},\omega_{\rm max}]=[-0.1,0.1]$,
$\nu \in [\nu_{\rm min},\nu_{\rm max}]=[0.8,1.5]$, and
$\tau \in [\tau_{\rm min},\tau_{\rm max}]=[0.8,1.25]$.

The distance of the singular value at
$S(\mu,\omega,\nu,\tau,\lambda_{\rm 01},\alpha_{\rm 01})$
and that at 
$S(\mu + \Delta \mu,\omega + \Delta \omega,\nu + \Delta \nu,\tau + \Delta \tau,\lambda_{\rm 01},\alpha_{\rm 01})$
is bounded as follows:
\begin{align}
&\left|S(\mu + \Delta \mu,\omega + \Delta \omega,\nu + \Delta
  \nu,\tau + \Delta \tau,\lambda_{\rm 01},\alpha_{\rm 01}) 
\ - \ 
S(\mu,\omega,\nu,\tau,\lambda_{\rm 01},\alpha_{\rm 01}) \right|
\ < \\ \nonumber
&0.32112 \left|\Delta \mu\right| + 2.63690\left|\Delta \omega\right| + 2.28242  \left|\Delta \nu\right| + 2.98610 \left|\Delta \tau\right| \ .
\end{align}
\end{lemma}
\begin{proof}

The mean value theorem states that a $t\in[0,1]$ exists for
which
\begin{align}
&S(\mu + \Delta \mu,\omega + \Delta \omega,\nu + \Delta
  \nu,\tau + \Delta \tau,\lambda_{\rm 01},\alpha_{\rm 01}) 
\ - \ 
S(\mu,\omega,\nu,\tau,\lambda_{\rm 01},\alpha_{\rm 01}) \ = \\ \nonumber
&\frac{\partial S}{\partial \mu}(\mu + t\Delta \mu,\omega + t\Delta \omega,\nu + t\Delta
  \nu,\tau + t\Delta \tau,\lambda_{\rm 01},\alpha_{\rm 01}) \ \Delta \mu
\ + \\ \nonumber
&\frac{\partial S}{\partial \omega}(\mu + t\Delta \mu,\omega + t\Delta \omega,\nu + t\Delta
  \nu,\tau + t\Delta \tau,\lambda_{\rm 01},\alpha_{\rm 01}) \ \Delta \omega
\ + \\ \nonumber
&\frac{\partial S}{\partial \nu}(\mu + t\Delta \mu,\omega + t\Delta \omega,\nu + t\Delta
  \nu,\tau + t\Delta \tau,\lambda_{\rm 01},\alpha_{\rm 01})\ \Delta
  \nu
\ + \\ \nonumber
&\frac{\partial S}{\partial \tau}(\mu + t\Delta \mu,\omega + t\Delta \omega,\nu + t\Delta
  \nu,\tau + t\Delta \tau,\lambda_{\rm 01},\alpha_{\rm 01})\
  \Delta \tau 
\end{align}
from which immediately follows that
\begin{align}
&\left|S(\mu + \Delta \mu,\omega + \Delta \omega,\nu + \Delta
  \nu,\tau + \Delta \tau,\lambda_{\rm 01},\alpha_{\rm 01}) 
\ - \ 
S(\mu,\omega,\nu,\tau,\lambda_{\rm 01},\alpha_{\rm 01})\right| \ \leq \\ \nonumber
&\left|\frac{\partial S}{\partial \mu}(\mu + t\Delta \mu,\omega + t\Delta \omega,\nu + t\Delta
  \nu,\tau + t\Delta \tau,\lambda_{\rm 01},\alpha_{\rm 01})\right| \ \left|\Delta \mu\right|
\ + \\ \nonumber
&\left|\frac{\partial S}{\partial \omega}(\mu + t\Delta \mu,\omega + t\Delta \omega,\nu + t\Delta
  \nu,\tau + t\Delta \tau,\lambda_{\rm 01},\alpha_{\rm 01})\right| \ \left|\Delta \omega\right|
\ + \\ \nonumber
&\left|\frac{\partial S}{\partial \nu}(\mu + t\Delta \mu,\omega + t\Delta \omega,\nu + t\Delta
  \nu,\tau + t\Delta
  \tau,\lambda_{\rm 01},\alpha_{\rm 01})\right|\ \left|\Delta \nu\right|
\ + \\ \nonumber
&\left|\frac{\partial S}{\partial \tau}(\mu + t\Delta \mu,\omega + t\Delta \omega,\nu + t\Delta
  \nu,\tau + t\Delta \tau,\lambda_{\rm 01},\alpha_{\rm 01})\right|\
  \left|\Delta \tau \right| \ .
\end{align}
We now apply Lemma~\ref{lem:Ds1Bounds} which gives bounds on the
derivatives, which immediately gives the statement of the lemma.
\end{proof}

\begin{lemma}[Largest Singular Value Smaller Than One]
\label{lem:sBound}
We set
$\alpha = \alpha_{\rm 01}$ and $\lambda=\lambda_{\rm 01}$ and
restrict the range of the variables to
$\mu \in [-0.1,0.1]$,
$\omega \in[-0.1,0.1]$,
$\nu \in [0.8,1.5]$, and
$\tau \in [0.8,1.25]$.

The the largest singular value of the Jacobian is smaller than 1:
\begin{align}
S(\mu,\omega,\nu,\tau,\lambda_{\rm 01},\alpha_{\rm 01}) \ < \ 1 \ .
\end{align}
Therefore the mapping  Eq.~\eqref{eq:mappingMean} 
and Eq.~\eqref{eq:mappingVar} is a contraction mapping.
\end{lemma}
\begin{proof}
We set 
$\Delta \mu=0.0068097371$,
$\Delta \omega=0.0008292885$,
$\Delta \nu=0.0009580840$, and
$\Delta \tau=0.0007323095$.



According to Lemma~\ref{lem:meanValue} we have
\begin{align}
\label{eq:bb1}
&\left|S(\mu + \Delta \mu,\omega + \Delta \omega,\nu + \Delta
  \nu,\tau + \Delta \tau,\lambda_{\rm 01},\alpha_{\rm 01}) 
\ - \ 
S(\mu,\omega,\nu,\tau,\lambda_{\rm 01},\alpha_{\rm 01}) \right|
\ < \\ \nonumber
&0.32112 \cdot 0.0068097371 + 2.63690 \cdot 0.0008292885 + \\ \nonumber
& 2.28242 \cdot 0.0009580840 + 2.98610 \cdot 0.0007323095\ < \ 0.008747 \ .
\end{align}

For a grid with grid length
$\Delta \mu=0.0068097371$,
$\Delta \omega=0.0008292885$,
$\Delta \nu=0.0009580840$, and
$\Delta \tau=0.0007323095$,
we evaluated the function  Eq.~\eqref{eq:S} 
for the largest singular value
in the domain
$\mu \in [-0.1,0.1]$,
$\omega \in [-0.1,0.1]$,
$\nu \in [0.8,1.5]$, and
$\tau \in [0.8,1.25]$.
We did this using a computer.
According to Subsection~\ref{sec:error}
the precision if regarding error propagation
and precision of the implemented functions is larger than
$10^{-13}$.
We performed the evaluation on different operating systems and
different hardware architectures including CPUs and GPUs.
In all cases the function  Eq.~\eqref{eq:S} for the largest singular
value of the Jacobian is bounded by $0.9912524171058772$.

We obtain from Eq.~\eqref{eq:bb1}:
\begin{align}
S(\mu + \Delta \mu,\omega + \Delta \omega,\nu + \Delta
  \nu,\tau + \Delta \tau,\lambda_{\rm 01},\alpha_{\rm 01})
\ \leq \ 0.9912524171058772 \ + \ 0.008747 \ < \ 1 \ .
\end{align}
\end{proof}

\subsubsection{Lemmata for proofing Theorem 1 (part 2): Mapping within domain}\index{mapping in domain}
\label{sec:maptoregion}

We further have to investigate whether the the mapping  Eq.~\eqref{eq:mappingMean}
and Eq.~\eqref{eq:mappingVar} maps into a predefined domains.

\begin{lemma}[Mapping into the domain]
\label{lem:region}
The mapping   Eq.~\eqref{eq:mappingMean}
and Eq.~\eqref{eq:mappingVar} map for 
 $\alpha = \alpha_{\rm 01}$ and $\lambda=\lambda_{\rm 01}$
into the domain
$\mu \in [-0.03106, 0.06773]$ and
$\nu \in [0.80009,1.48617]$ with $\omega \in [-0.1,0.1]$ and $\tau \in [0.95,1.1]$.
\end{lemma}

\begin{proof}
We use Lemma~\ref{lem:mapDerivatives} which states that with given
sign the derivatives of the mapping Eq.~\eqref{eq:mappingMean}
and Eq.~\eqref{eq:mappingVar} with respect to  $\alpha = \alpha_{\rm 01}$
and $\lambda=\lambda_{\rm 01}$ are either positive or have the sign of
$\omega$.
Therefore with given sign of $\omega$ the mappings are strict monotonic and
the their maxima and minima are found at the borders.  The minimum of $\munn$ is obtained at
$\mu \omega = -0.01$ and its maximum at $\mu \omega=0.01$ and $\sigma$ and $\tau$ at their 
minimal and maximal values, respectively. It follows that:
\begin{align}
-0.03106 <\munn(-0.1,0.1, 0.8, 0.95, \lambda_{\rm 01}, \alpha_{\rm 01})  \leq  & \munn \leq \munn(0.1,0.1,1.5, 1.1, \lambda_{\rm 01}, \alpha_{\rm 01}) < 0.06773,
\end{align}
and that $\munn \in [-0.1,0.1]$.

Similarly, the maximum and minimum of $\xinn($ is obtained at the values mentioned above:
\begin{align}
0.80467 <\xinn(-0.1,0.1, 0.8, 0.95, \lambda_{\rm 01}, \alpha_{\rm 01})  \leq  & \xinn \leq \xinn(0.1,0.1,1.5, 1.1, \lambda_{\rm 01}, \alpha_{\rm 01}) < 1.48617. 
\end{align}
Since $|\xinn-\nunn| = |\munn^2| < 0.004597$, we can conclude that 
$0.80009 < \nunn < 1.48617$ and the variance remains in $[0.8,1.5]$.
\end{proof}

\begin{corollary}
 The image $g(\Omega')$ of the mapping $g:(\mu,\nu) \mapsto (\munn, \nunn)$ (Eq.~\eqref{eq:mapping})
 and the domain $\Omega'=\{(\mu,\nu) | -0.1 \leq \mu \leq 0.1, 0.8 \leq \mu \leq 1.5 \}$ is 
 a subset of $\Omega'$: 
 \begin{align} g(\Omega') \subseteq \Omega', \end{align}
 for all $\omega \in [-0.1,0.1]$ and $\tau \in [0.95,1.1]$.
\end{corollary}

\begin{proof}
 Directly follows from Lemma~\ref{lem:region}.
\end{proof}


\subsubsection{Lemmata for proofing Theorem 2: The variance is contracting}\index{contracting variance}

\paragraph{Main Sub-Function.}
We consider the main sub-function of the derivate of second moment, $J22$ (Eq.~\eqref{eq:JacobianEntries}):
\begin{align}
&\frac{\partial}{\partial \nu} \xinn= \frac{1}{2} \lambda ^2 \tau \left(- \alpha ^2 e^{\mu \omega+\frac{\nu \tau}{2}} \erfc \left(\frac{\mu \omega+\nu \tau}{\sqrt{2} \sqrt{\nu \tau}}\right) +  
    2 \alpha ^2 e^{2 \mu
    \omega+2 \nu \tau} \erfc \left(\frac{\mu \omega+2
      \nu \tau}{\sqrt{2} \sqrt{\nu
        \tau}}\right)-\erfc \left(\frac{\mu
      \omega}{\sqrt{2} \sqrt{\nu \tau}}\right)+2\right)
\end{align}

that depends on $\mu \omega$ and $\nu \tau$, therefore we 
set  $x=\nu \tau$ and $y=\mu \omega$. Algebraic reformulations provide the 
formula in the following form:

\begin{align}
& \frac{\partial}{\partial\nu}\xinn\ = 
\frac{1}{2}\lambda^{2}\tau
\left(\alpha^{2}\left(-e^{-\frac{y^{2}}{2x}}\right) \left(e^{\frac{(x+y)^{2}}{2x}}\erfc\left(\frac{y+x}{\sqrt{2}\sqrt{x}}\right)-2e^{\frac{(2x+y)^{2}}{2x}}\erfc\left(\frac{y+2x}{\sqrt{2}\sqrt{x}}\right)\right)
- \erfc\left(\frac{y}{\sqrt{2}\sqrt{x}}\right)+ 2\right)
\end{align}

For $\lambda=\lambda_{\rm 01}$ and 
$\alpha=\alpha_{\rm 01}$,
we consider the domain
$-1 \leq \mu \leq 1$, 
$-0.1 \leq \omega \leq 0.1$,
$1.5 \leq \nu \leq 16$, and,
$0.8 \leq \tau \leq 1.25$.

For $x$ and $y$ we obtain: $0.8 \cdot 1.5=1.2 \leq x \leq 20=1.25 \cdot 16$ and 
$0.1 \cdot (-1)= -0.1 \leq y \leq 0.1=0.1 \cdot 1$.
In the following we assume to remain within this domain.

\begin{lemma}[Main subfunction]
\label{lem:subfunction}

For $1.2 \leq x \leq 20$ and $-0.1 \leq y \leq 0.1$, 

the function
\begin{align}
\label{eq:subfunction}
e^{\frac{(x+y)^2}{2 x}} \erfc \left(\frac{x+y}{\sqrt{2} \sqrt{x}}\right)-2 e^{\frac{(2 x+y)^2}{2 x}} \erfc \left(\frac{2 x+y}{\sqrt{2} \sqrt{x}}\right)
\end{align}
is smaller than zero, is strictly monotonically increasing in $x$,
and strictly monotonically decreasing in $y$ for the minimal $x=12/10=1.2$.
\end{lemma}

\begin{proof}
See proof~\ref{proof:mainsubfunction}.
\end{proof}

The graph of the subfunction in the specified domain is displayed in Figure~\ref{fig:subfunction}.

\begin{figure}
 \centering
 \includegraphics[width=0.48\textwidth]{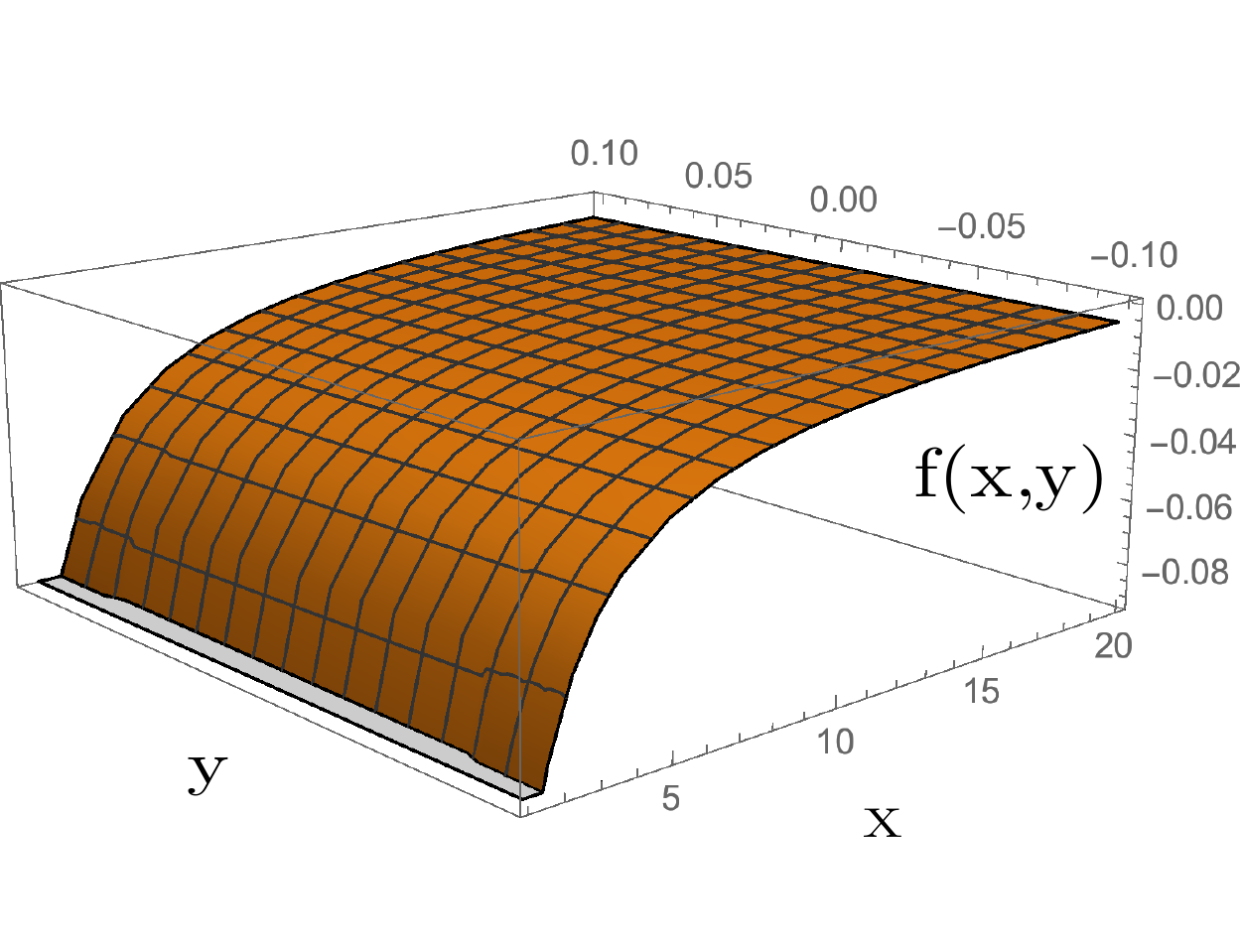}
 \includegraphics[width=0.48\textwidth]{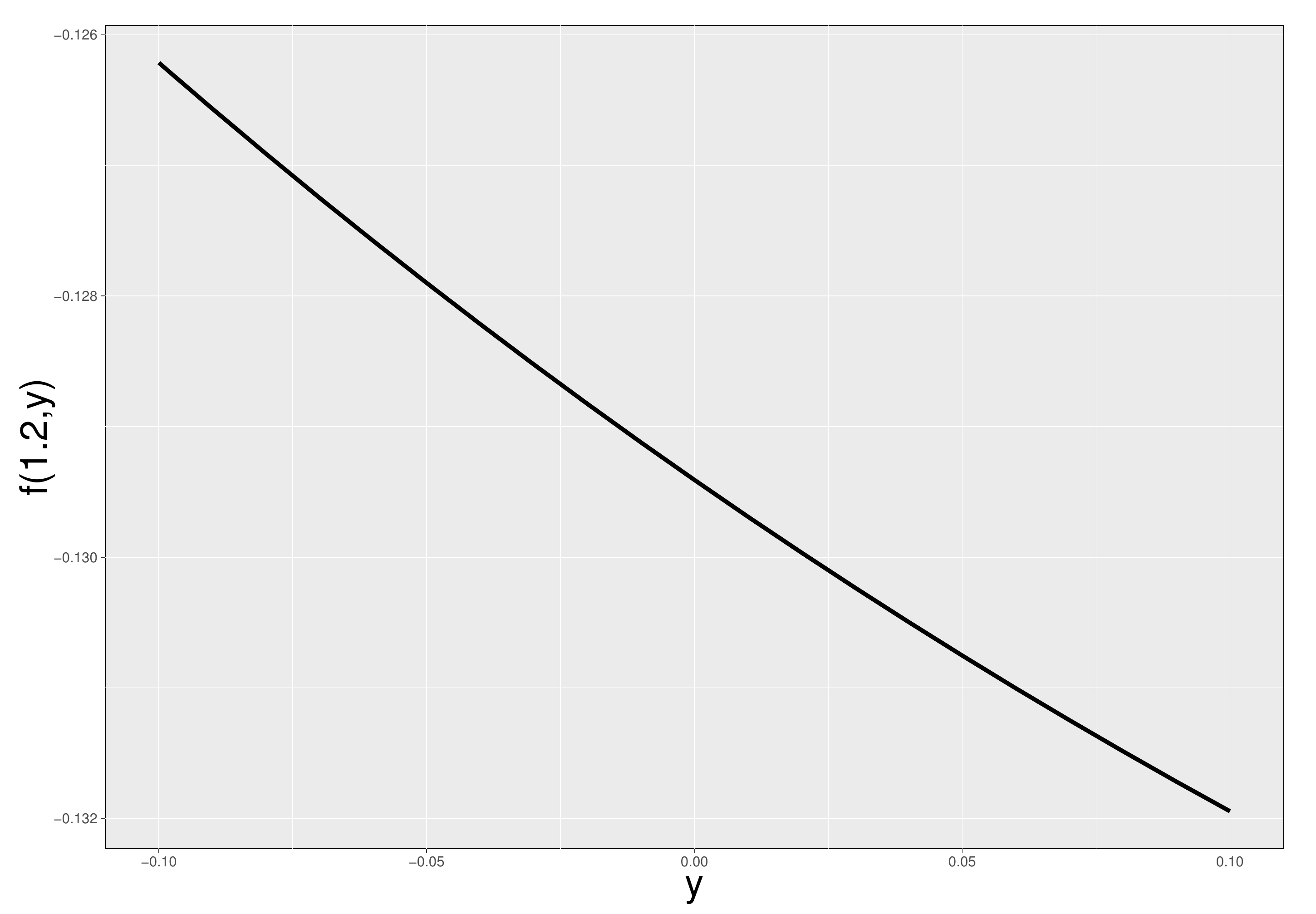}
 \caption[Graph of the main subfunction of the derivative of the second moment]{{\bf Left panel:} Graphs of the main subfunction $f(x,y)=e^{\frac{(x+y)^2}{2 x}} \erfc \left(\frac{x+y}{\sqrt{2} \sqrt{x}}\right)-2 e^{\frac{(2 x+y)^2}{2 x}} \erfc \left(\frac{2 x+y}{\sqrt{2} \sqrt{x}}\right)$
  treated in Lemma~\ref{lem:subfunction}. The function is negative and monotonically increasing with $x$ independent of $y$. 
   {\bf Right panel:}  Graphs of the main subfunction at minimal $x=1.2$. The graph shows that the function $f(1.2,y)$ is strictly monotonically decreasing in $y$.
  \label{fig:subfunction}}
\end{figure}

\begin{theorem}[Contraction $\nu$-mapping]
\label{th:s2Cont}
The mapping of the variance $\nunn(\mu, \omega, \nu, \tau,\lambda,\alpha )$  given in Eq.~\eqref{eq:mappingVar}
is contracting for 
$\lambda=\lambda_{\rm 01}$, $\alpha=\alpha_{\rm 01}$
and the domain $\Omega^+$: 
$-0.1 \leq \mu \leq 0.1$, 
$-0.1 \leq \omega \leq 0.1$,
$1.5 \leq \nu \leq 16$, and 
$0.8 \leq \tau \leq 1.25$, that is,
\begin{align}
&\left| \frac{\partial}{\partial \nu}\nunn(\mu,\omega,\nu,\tau,\lambda_{\rm 01},\alpha_{\rm 01}) \right| \ < \ 1 \ .
\end{align}
\end{theorem}

\begin{proof}
In this domain $\Omega^+$ we have the following three properties (see further below):
$\frac{\partial}{\partial \nu} \xinn < 1$, $\munn > 0$,
and $\frac{\partial}{\partial \nu} \munn > 0$. Therefore, we have

\begin{align}
 & \left| \frac{\partial}{\partial \nu} \nunn \right| = \left| \frac{\partial}{\partial \nu} \xinn- 2 \munn \frac{\partial}{\partial \nu} \munn \right| 
  < \left| \frac{\partial}{\partial \nu} \xinn \right| < 1 
\end{align}

\begin{itemize}
\item We first proof that $\frac{\partial}{\partial \nu} \xinn < 1$ in an even larger domain that fully contains $\Omega^+$.
According to  Eq.~\eqref{eq:JacobianEntries},
the derivative of the mapping  Eq.~\eqref{eq:mappingVar} 
with respect to the variance $\nu$ is
\begin{align}
\label{eq:s2Ds2}
&\frac{\partial}{\partial \nu} \xinn(\mu,\omega,\nu,\tau,\lambda_{\rm 01},\alpha_{\rm 01}) \ = \\ \nonumber 
&\frac{1}{2} \lambda ^2 \tau \left(\alpha ^2 \left(-e^{\mu \omega+\frac{\nu \tau}{2}}\right) \erfc
 \left(\frac{\mu \omega+\nu \tau}{\sqrt{2} \sqrt{\nu \tau}}\right)+\right. \\ \nonumber
&\left. 2 \alpha ^2 e^{2 \mu \omega+2 \nu \tau} \erfc \left(\frac{\mu \omega +2 \nu \tau}{\sqrt{2} \sqrt{\nu \tau}}\right)- 
   \erfc \left(\frac{\mu \omega}{\sqrt{2} \sqrt{\nu \tau}}\right)+2\right) \ .
\end{align}

For 
$\lambda=\lambda_{\rm 01}$, $\alpha=\alpha_{\rm 01}$, 
$-1 \leq \mu \leq 1$, 
$-0.1 \leq \omega \leq 0.1$ 
$1.5 \leq \nu \leq 16$, and 
$0.8 \leq \tau \leq 1.25$, we first show that the derivative is positive
and then upper bound it.

According to  Lemma~\ref{lem:subfunction}, the expression
\begin{align}
e^{\frac{(\mu \omega +\nu \tau)^2}{2 \nu \tau}} \erfc \left(\frac{\mu \omega+\nu \tau}{\sqrt{2} \sqrt{\nu \tau}}\right)- 
 2 e^{\frac{(\mu \omega +2 \nu \tau)^2}{2 \nu \tau}} \erfc\left(\frac{\mu \omega+2 \nu \tau}{\sqrt{2} \sqrt{\nu \tau}}\right) 
\end{align}
is negative. This expression multiplied by positive factors is
subtracted in the derivative Eq.~\eqref{eq:s2Ds2}, therefore, the
whole term is positive.
The remaining term
\begin{align}
2-\erfc\left(\frac{\mu \omega}{\sqrt{2} \sqrt{\nu \tau}}\right)
\end{align}
of the derivative Eq.~\eqref{eq:s2Ds2}
is also positive according to Lemma~\ref{lem:basics}.
All factors outside the brackets in  Eq.~\eqref{eq:s2Ds2} are
positive. Hence, the derivative Eq.~\eqref{eq:s2Ds2} is positive.

The upper bound of the derivative is:
\begin{align}
&\frac{1}{2} \lambda_{\rm 01}^2 \tau \left(\alpha_{\rm 01}^2 \left(-e^{\mu \omega+\frac{\nu \tau}{2}}\right)
  \erfc\left(\frac{\mu \omega+\nu \tau}{\sqrt{2}
  \sqrt{\nu \tau}}\right)+\right. \\ \nonumber 
& \left. 2 \alpha_{\rm 01}^2 e^{2 \mu \omega+2 \nu \tau} \erfc \left(\frac{\mu \omega+2 \nu \tau}{\sqrt{2} \sqrt{\nu \tau}}\right) 
 - \erfc\left(\frac{\mu \omega}{\sqrt{2} \sqrt{\nu \tau}}\right)+2\right)\ = \\ \nonumber 
&\frac{1}{2} \lambda_{\rm 01}^2 \tau \left(\alpha_{\rm 01}^2 \left(-e^{-\frac{\mu^2 \omega^2}{2 \nu \tau}}\right) 
\left(e^{\frac{(\mu \omega+\nu \tau)^2}{2 \nu \tau}} \erfc\left(\frac{\mu \omega+\nu \tau}{\sqrt{2} \sqrt{\nu \tau}}\right)-\right.\right. \\ \nonumber 
&\left. \left. 2 e^{\frac{(\mu \omega+2 \nu \tau)^2}{2 \nu \tau}} \erfc\left(\frac{\mu \omega+2 \nu \tau}{\sqrt{2} \sqrt{\nu \tau}}\right)\right)-\erfc\left(\frac{\mu \omega}{\sqrt{2} \sqrt{\nu \tau}}\right)+2\right)\ \leq \\ \nonumber 
& \frac{1}{2} 1.25 \lambda_{\rm 01}^2 \left(\alpha_{\rm 01}^2 \left(-e^{-\frac{\mu^2 \omega^2}{2 \nu \tau}}\right) 
\left(e^{\frac{(\mu \omega+\nu \tau)^2}{2 \nu \tau}} \erfc\left(\frac{\mu \omega+\nu \tau}{\sqrt{2} \sqrt{\nu \tau}}\right)-\right.\right. \\ \nonumber 
& \left. \left. 2 e^{\frac{(\mu \omega+2 \nu \tau)^2}{2 \nu \tau}} \erfc\left(\frac{\mu \omega+2 \nu \tau}{\sqrt{2} \sqrt{\nu \tau}}\right)\right) - 
\erfc\left(\frac{\mu \omega}{\sqrt{2} \sqrt{\nu \tau}}\right)+2\right) \ \leq \\ \nonumber 
& \frac{1}{2} 1.25 \lambda_{\rm 01}^2 \left(\alpha_{\rm 01}^2 \left(e^{\left(\frac{1.2 +0.1}{\sqrt{2} \sqrt{1.2}}\right)^2} 
\erfc \left(\frac{1.2 +0.1}{\sqrt{2} \sqrt{1.2}}\right)-\right.\right. \\ \nonumber 
& \left. \left. 2 e^{\left(\frac{2 \cdot 1.2+0.1}{\sqrt{2} \sqrt{1.2}}\right)^2} \erfc\left(\frac{2 \cdot 1.2+0.1}{\sqrt{2} \sqrt{1.2}}\right)\right) 
\left(-e^{-\frac{\mu^2 \omega^2}{2 \nu \tau}}\right)-\erfc\left(\frac{\mu \omega}{\sqrt{2} \sqrt{\nu \tau}}\right)+2\right) \ \leq \\ \nonumber 
& \frac{1}{2} 1.25 \lambda_{\rm 01}^2 \left(-e^{0.0} \alpha_{\rm 01}^2 \left(e^{\left(\frac{1.2 +0.1}{\sqrt{2} \sqrt{1.2}}\right)^2} 
\erfc \left(\frac{1.2 +0.1}{\sqrt{2} \sqrt{1.2}}\right)-\right.\right. \\ \nonumber & \left. \left. 2 e^{\left(\frac{2 \cdot 1.2+0.1}{\sqrt{2} \sqrt{1.2}}\right)^2} 
\erfc\left(\frac{2 \cdot 1.2+0.1}{\sqrt{2} \sqrt{1.2}}\right)\right)-\erfc\left(\frac{\mu \omega}{\sqrt{2} \sqrt{\nu \tau}}\right)+2\right) \ \leq \\ \nonumber 
& \frac{1}{2} 1.25 \lambda_{\rm 01}^2 \left(-e^{0.0} \alpha_{\rm 01}^2 \left(e^{\left(\frac{1.2 +0.1}{\sqrt{2} \sqrt{1.2}}\right)^2} 
\erfc \left(\frac{1.2 +0.1}{\sqrt{2} \sqrt{1.2}}\right)-\right.\right. \\ \nonumber 
& \left. \left. 2 e^{\left(\frac{2 \cdot 1.2+0.1}{\sqrt{2} \sqrt{1.2}}\right)^2} \erfc\left(\frac{2 \cdot 1.2+0.1}{\sqrt{2} \sqrt{1.2}}\right)\right)
-\erfc\left(\frac{0.1}{\sqrt{2} \sqrt{1.2}}\right)+2\right) \ \leq \\ \nonumber 
&0.995063 \ < \ 1 \ .
\end{align}
We explain the chain of inequalities:
\begin{itemize}
\item First equality brings the expression
into a shape where we can apply  Lemma~\ref{lem:subfunction} for the
the function Eq.~\eqref{eq:subfunction}.
\item First inequality: The overall factor $\tau$ is bounded by 1.25.
\item Second inequality: We apply Lemma~\ref{lem:subfunction}.
 According to Lemma~\ref{lem:subfunction} the function Eq.~\eqref{eq:subfunction} is negative.
The largest contribution is to subtract the most negative value of 
the function Eq.~\eqref{eq:subfunction}, that is, the minimum of 
function Eq.~\eqref{eq:subfunction}.
According to Lemma~\ref{lem:subfunction} the function
Eq.~\eqref{eq:subfunction} is strictly monotonically increasing in $x$
and strictly monotonically decreasing in $y$ for $x=1.2$.
Therefore the function Eq.~\eqref{eq:subfunction} has its minimum 
at minimal $x=\nu \tau=1.5 \cdot 0.8=1.2$ 
and maximal $y=\mu \omega=1.0 \cdot 0.1=0.1$. We insert these values into
the expression.

\item Third inequality: We use for the whole expression 
the maximal factor 
$e^{-\frac{\mu^2 \omega^2}{2 \nu \tau}}<1$ by setting this
factor to 1.
\item Fourth inequality: $\erfc$ is strictly monotonically
  decreasing. Therefore we maximize its argument to obtain the least
  value which is subtracted. We use the minimal $x=\nu
  \tau=1.5 \cdot 0.8=1.2$ and the maximal $y=\mu \omega=1.0 \cdot 0.1=0.1$.
\item Sixth inequality: evaluation of the terms.
\end{itemize}

\item We now show that $\munn > 0$. The expression
$\munn(\mu,\omega,\nu,\tau)$ (Eq.~\eqref{eq:mappingMean})
is strictly monotonically increasing im $\mu \omega$ and $\nu \tau$. Therefore,
the minimal value in $\Omega^+$ is obtained at 
$\munn(0.01,0.01,1.5,0.8)= 0.008293 > 0$. 

\item Last we show that $\frac{\partial}{\partial \nu} \munn > 0$.
The expression
$\frac{\partial}{\partial \nu}  \munn(\mu,\omega,\nu,\tau)={\mathcal J}_{12}(\mu,\omega,\nu,\tau)$ (Eq.~\eqref{eq:JacobianEntries})
can we reformulated as follows:

\begin{align}
&{\mathcal J}_{12}(\mu,\omega,\nu,\tau,\lambda ,\alpha ) \ =
\frac{\lambda  \tau  e^{-\frac{\mu ^2 \omega ^2}{2 \nu  \tau }} \left(\sqrt{\pi } \alpha  e^{\frac{(\mu  \omega +\nu  \tau )^2}{2 \nu  \tau }} 
\erfc \left(\frac{\mu  \omega +\nu  \tau }{\sqrt{2} \sqrt{\nu  \tau }}\right)-\frac{\sqrt{2} (\alpha -1)}{\sqrt{\nu  \tau }} \right)}{4 \sqrt{\pi } }
\end{align}
is larger than zero when the term $\sqrt{\pi } \alpha   e^{\frac{(\mu  \omega +\nu  \tau )^2}{2 \nu  \tau }} 
\erfc \left(\frac{\mu  \omega +\nu  \tau }{\sqrt{2} \sqrt{\nu  \tau }}\right)- \frac{\sqrt{2} (\alpha -1)}{\sqrt{\nu  \tau }}$
is larger than zero. This term obtains its minimal value 
at $\mu \omega = 0.01$ and $\nu \tau =16 \cdot 1.25$, which can easily be shown using the 
Abramowitz bounds (Lemma~\ref{lem:Abramowitz})
and evaluates to $0.16$, therefore ${\mathcal J}_{12} > 0$ in $\Omega^+$.

\end{itemize}
\end{proof}

%
%

\subsubsection{Lemmata for proofing Theorem 3: The variance is expanding}\index{expanding variance}

\paragraph{Main Sub-Function From Below.}

We consider functions in
$\mu \omega$ and $\nu \tau$, therefore we 
set  $x=\mu \omega$ and $y=\nu \tau$.

For $\lambda=\lambda_{\rm 01}$ and 
$\alpha=\alpha_{\rm 01}$,
we consider the domain
$-0.1 \leq \mu \leq 0.1$, 
$-0.1 \leq \omega \leq 0.1$ 
$0.00875 \leq \nu \leq 0.7$, and 
$0.8 \leq \tau \leq 1.25$.

For $x$ and $y$ we obtain: $0.8 \cdot 0.00875=0.007 \leq x \leq 0.875=1.25 \cdot 0.7$ and 
$0.1 \cdot (-0.1)= -0.01 \leq y \leq 0.01=0.1 \cdot 0.1$.
In the following we assume to be within this domain.

In this domain, we consider the main sub-function of the derivate of second moment in the next layer, $J22$ (Eq.~\eqref{eq:JacobianEntries}):
\begin{align}
&\frac{\partial}{\partial \nu} \xinn= \frac{1}{2} \lambda ^2 \tau \left(- \alpha ^2 e^{\mu \omega+\frac{\nu \tau}{2}} \erfc \left(\frac{\mu \omega+\nu \tau}{\sqrt{2} \sqrt{\nu \tau}}\right) +  
    2 \alpha ^2 e^{2 \mu
    \omega+2 \nu \tau} \erfc \left(\frac{\mu \omega+2
      \nu \tau}{\sqrt{2} \sqrt{\nu
        \tau}}\right)-\erfc \left(\frac{\mu
      \omega}{\sqrt{2} \sqrt{\nu \tau}}\right)+2\right)
\end{align}

that depends on $\mu \omega$ and $\nu \tau$, therefore we 
set  $x=\nu \tau$ and $y=\mu \omega$. Algebraic reformulations provide the 
formula in the following form:

\begin{align}
& \frac{\partial}{\partial\nu}\xinn\ = \\ \nonumber
& \frac{1}{2}\lambda^{2}\tau
\left(\alpha^{2}\left(-e^{-\frac{y^{2}}{2x}}\right) \left(e^{\frac{(x+y)^{2}}{2x}}\erfc\left(\frac{y+x}{\sqrt{2}\sqrt{x}}\right)-2e^{\frac{(2x+y)^{2}}{2x}}\erfc\left(\frac{y+2x}{\sqrt{2}\sqrt{x}}\right)\right)
-\erfc\left(\frac{y}{\sqrt{2}\sqrt{x}}\right)+ 2\right)
\end{align}

\begin{lemma}[Main subfunction Below]
\label{lem:subfunction1}

For $0.007 \leq x \leq 0.875$ and $-0.01 \leq y \leq 0.01$, 
the function
\begin{align}
\label{eq:subfunction1}
e^{\frac{(x+y)^2}{2 x}} \erfc \left(\frac{x+y}{\sqrt{2} \sqrt{x}}\right)-2 e^{\frac{(2 x+y)^2}{2 x}} \erfc \left(\frac{2 x+y}{\sqrt{2} \sqrt{x}}\right)
\end{align}
smaller than zero, is strictly monotonically increasing in $x$
and strictly monotonically increasing in $y$ for the minimal $x=0.007=0.00875 \cdot 0.8$,
$x=0.56=0.7 \cdot 0.8$, $x=0.128=0.16 \cdot 0.8$, and $x=0.216=0.24 \cdot 0.9$ (lower
bound of $0.9$ on $\tau$).
\end{lemma}

\begin{proof}
See proof~\ref{proof:mainsubfunctionbelow}.
\end{proof}

\begin{lemma}[Monotone Derivative]
\label{th:s2monotone}
For $\lambda=\lambda_{\rm 01}$, $\alpha=\alpha_{\rm 01}$
and the domain 
$-0.1 \leq \mu \leq 0.1$, 
$-0.1 \leq \omega \leq 0.1$,
$0.00875 \leq \nu \leq 0.7$, and 
$0.8 \leq \tau \leq 1.25$.
We are interested of the derivative of
\begin{align} 
\label{eq:subx1Th}
\tau \left(e^{\left(\frac{\mu \omega+\nu \tau}{\sqrt{2} \sqrt{\nu \tau}}\right)^2} \erfc \left(\frac{\mu \omega+\nu \tau}{\sqrt{2} \sqrt{\nu \tau}}\right)-2 e^{\left(\frac{\mu \omega+2 \cdot \nu \tau}{\sqrt{2} \sqrt{\nu \tau}}\right)^2} \erfc \left(\frac{\mu \omega+2  \nu \tau}{\sqrt{2} \sqrt{\nu \tau}}\right)\right)\ . 
\end{align}

The derivative of the equation above with
respect to
\begin{itemize}
\item $\nu$ is larger than zero;
\item $\tau$ is smaller than zero for maximal
$\nu=0.7$, $\nu=0.16$, and $\nu=0.24$ (with
$0.9 \leq \tau$);
\item $y=\mu \omega$ is larger than zero for $\nu
\tau=0.00875  0.8=0.007$, $\nu
\tau=0.7  0.8=0.56$, $\nu
\tau=0.16  0.8=0.128$, and $\nu
\tau=0.24 \cdot 0.9=0.216$.
\end{itemize}

\end{lemma}

\begin{proof}
See proof~\ref{proof:monotonederivative}.
\end{proof}

\subsubsection{Computer-assisted proof details for main Lemma 12 in Section A3.4.1.}\index{computer-assisted proof}
\label{sec:error}

\paragraph{Error Analysis.} We investigate the error propagation for the 
singular value (Eq.~\eqref{eq:S}) if the function arguments $\mu, \omega, \nu, \tau$
suffer from numerical imprecisions up to $\epsilon$. To this end, we first 
derive error propagation rules based on the mean value theorem and then 
we apply these rules to the formula for the singular value.

\begin{lemma}[Mean value theorem]
\label{th:mvt}
For a real-valued function $f$ which is differentiable in the closed interval $[a,b]$,
there exists $t \in [0,1]$ with 
\begin{align}
f(\Ba) \ - \ f(\Bb) \ &= \ \nabla f(\Ba+t (\Bb -\Ba)) \ \cdot \ ( \Ba \ - \ \Bb) \ .
\end{align}
\end{lemma}
It follows that 
for computation with error $\Delta x$, there exists a $t \in [0,1]$ with 
\begin{align}
\left| f(\Bx+\Delta \Bx) \ - \ f(\Bx) \right| \ &\leq \ \left\| \nabla
                                                  f(\Bx+t \Delta \Bx)\right\| \ \left\|
                                            \Delta \Bx \right\| \ .
\end{align}
Therefore the increase of the norm of the error after applying
function $f$ is bounded by the norm of the gradient
$\left\| \nabla f(\Bx+t \Delta \Bx)\right\|$.

We now compute for the functions, that we consider their gradient and
its 2-norm:  

\begin{itemize}
\item addition:

$f(\Bx)=x_1+x_2$ and $\nabla f(\Bx)=(1,1)$, which gives
$\left\| \nabla f(\Bx)\right\|=\sqrt{2}$.

We further know that
\begin{align}
\left|f(\Bx +\Delta \Bx) -f(\Bx)\right| \ 
= \ \left| x_1+x_2 +\Delta x_1 + \Delta x_2 - x_1 - x_2 \right| \ \leq
  \ \left| \Delta x_1 \right|  + \left| \Delta x_2 \right| \ .
\end{align}

Adding $n$ terms gives:
\begin{align}
\left| \sum_{i=1}^n x_i + \Delta x_i  \ - \ \sum_{i=1}^n x_i \right| 
\ \leq \ \sum_{i=1}^n \left| \Delta x_i \right| \ \leq \ n \left|
  \Delta x_i\right|_{\mathrm{max}}  \ .
\end{align}

\item subtraction:

$f(\Bx)=x_1-x_2$ and $\nabla f(\Bx)=(1,-1)$, which gives
$\left\| \nabla f(\Bx)\right\|=\sqrt{2}$.

We further know that
\begin{align}
\left|f(\Bx +\Delta \Bx) -f(\Bx)\right| \ 
= \ \left| x_1-x_2 +\Delta x_1 - \Delta x_2 - x_1 + x_2 \right| \ \leq
  \ \left| \Delta x_1 \right|  + \left| \Delta x_2 \right| \ .
\end{align}

Subtracting $n$ terms gives:
\begin{align}
\left| \sum_{i=1}^n -(x_i + \Delta x_i)  \ + \ \sum_{i=1}^n x_i \right| 
\ \leq \ \sum_{i=1}^n \left| \Delta x_i \right| \ \leq \ n \left|
  \Delta x_i\right|_{\mathrm{max}}  \ . 
\end{align}

\item multiplication:

$f(\Bx)=x_1 x_2$ and $\nabla f(\Bx)=(x_2,x_1)$, which gives
$\left\| \nabla f(\Bx)\right\|= \left\| \Bx \right\|$.

We further know that
\begin{align}
&\left|f(\Bx +\Delta \Bx) -f(\Bx)\right| \ 
= \ \left| x_1 \cdot x_2 +\Delta x_1 \cdot x_2 + \Delta x_2 \cdot x_1 + \Delta x_1 \cdot \Delta x_s - x_1  \cdot x_2 \right| \ \leq
  \\ \nonumber 
&\left| \Delta x_1 \right|  \left|  x_2 \right| + \left| \Delta x_2
  \right|  \left|  x_1 \right| +O(\Delta^2) \ .
\end{align}

Multiplying $n$ terms gives:
\begin{align}
&\left| \prod_{i=1}^n (x_i + \Delta x_i)  \ - \ \prod_{i=1}^n x_i \right| 
\ = \ \left| \prod_{i=1}^n x_i \sum_{i=1}^n \frac{\Delta x_i}{x_i} \ +
  \ O(  \Delta^2)  \right|  \ \leq  \\ \nonumber &\prod_{i=1}^n  \left| x_i \right|
  \sum_{i=1}^n \left| \frac{\Delta x_i}{x_i} \right| \ +
  \ O(  \Delta^2) \ \leq \
  n \ \prod_{i=1}^n  \left| x_i \right| \   \left|\frac{\Delta
  x_i}{x_i}  \right|_{\mathrm{max}} \ +
  \ O(  \Delta^2)\ .
\end{align}

\item division:

$f(\Bx)=\frac{x_1}{x_2}$ and $\nabla
f(\Bx)= \left(\frac{1}{x_2},-\frac{x_1}{x_2^2} \right)$, which gives
$\left\| \nabla f(\Bx)\right\|= \frac{\left\| \Bx \right\|}{x_2^2}$.

We further know that
\begin{align}
&\left|f(\Bx +\Delta \Bx) -f(\Bx)\right| \ 
= \ \left| \frac{x_1 +\Delta x_1}{x_2+ \Delta x_2} - \frac{x_1}{x_2} \right|
  \ = \
 \left| \frac{(x_1 +\Delta x_1)x_2-x_1(x_2+ \Delta x_2)}{(x_2+ \Delta
  x_2)x_2} \right| \ = \\\nonumber
&\left| \frac{\Delta x_1 \cdot x_2-\Delta x_2 \cdot x_1}{x_2^2+ \Delta
  x_2 \cdot x_2} \right| \ = \ \left|  \frac{\Delta x_1}{x_2} - \frac{\Delta
  x_2 \cdot x_1}{x_2^2} \right| + O(\Delta^2) \ .
\end{align}

\item square root:

$f(x)= \sqrt{x}$ and $f'(x)= \frac{1}{2 \sqrt{x}}$, which gives
$\left| f'(x)\right|= \frac{1}{2 \sqrt{x}}$.

\item exponential function:

$f(x)= \exp(x)$ and $f'(x)= \exp(x)$, which gives
$\left| f'(x)\right|= \exp(x)$.

\item error function:

$f(x)= \mathrm{erf}(x)$ and $f'(x)= \frac {2}{\sqrt {\pi }} \exp(-x^2)$, which gives
$\left| f'(x)\right|= \frac {2}{\sqrt {\pi }} \exp(-x^2)$.

\item complementary error function:

$f(x)= \mathrm{erfc}(x)$ and $f'(x)= - \frac {2}{\sqrt {\pi }} \exp(-x^2)$, which gives
$\left| f'(x)\right|= \frac {2}{\sqrt {\pi }} \exp(-x^2)$.
\end{itemize}


\begin{lemma}
If the values $\mu, \omega, \nu, \tau$ have a precision of $\epsilon$, 
the singular value (Eq.~\eqref{eq:S}) evaluated with the formulas 
given in Eq.~\eqref{eq:JacobianEntries} and Eq.~\eqref{eq:S} has 
a precision better than $292 \epsilon$.  
\end{lemma}

This means for a machine with a typical precision of $2^{-52}=2.220446 \cdot 10^{-16}$, we have the rounding error $\epsilon \approx 10^{-16}$, the evaluation 
of the singular value (Eq.~\eqref{eq:S}) with the formulas given in Eq.~\eqref{eq:JacobianEntries} and Eq.~\eqref{eq:S} has 
a precision better than $ 10^{-13} > 292 \epsilon $.

\begin{proof}
We have the numerical precision $\epsilon$ of the parameters $\mu, \omega, \nu, \tau$, that we denote by
$\Delta \mu, \Delta \omega, \Delta \nu, \Delta \tau$ together with our domain $\Omega$.

With the error propagation rules that we derived in Subsection~\ref{sec:error}, we 
can obtain bounds for the numerical errors on the following simple expressions:
\begin{align}
  \Delta \left( \mu \omega \right) &\leq \Delta \mu \left| \omega \right| +  \Delta \omega \left| \mu \right| \leq 0.2 \epsilon \\ \nonumber
  \Delta \left( \nu \tau \right) &\leq \Delta \nu \left| \tau \right| +  \Delta \tau \left| \nu \right| \leq 1.5 \epsilon + 1.5 \epsilon = 3 \epsilon \\ \nonumber
  \Delta \left( \frac{\nu \tau}{2} \right) &\leq \left( \Delta (\nu \tau) 2 +  \Delta 2 \left| \nu \tau \right| \right) \frac{1}{2^2} \leq (6\epsilon + 1.25 \cdot 1.5 \epsilon )/4 < 2 \epsilon \\ \nonumber
  \Delta \left(\mu \omega + \nu \tau \right) &\leq \Delta \left( \mu \omega \right) + \Delta \left( \nu \tau \right) = 3.2 \epsilon \\ \nonumber
  \Delta \left(\mu \omega + \frac{\nu \tau}{2} \right) &\leq \Delta \left( \mu \omega \right) + \Delta \left( \frac{\nu \tau}{2} \right) < 2.2 \epsilon \\ \nonumber
  \Delta \left(\sqrt{\nu \tau} \right) &\leq \frac{\Delta \left( \nu \tau \right)}{2 \sqrt{\nu \tau}} \leq \frac{3 \epsilon}{2 \sqrt{0.64}} = 1.875 \epsilon  \\ \nonumber
  \Delta \left(\sqrt{2}\right) &\leq \frac{\Delta 2}{2 \sqrt{2}} \leq \frac{1}{2\sqrt{2}} \epsilon  \\ \nonumber
  \Delta \left(\sqrt{2}  \sqrt{\nu \tau}\right) &\leq \sqrt{2}  \Delta \left(\sqrt{\nu \tau}\right) + \nu \tau  \Delta \left( \sqrt{2} \right) \leq 
        \sqrt{2} \cdot 1.875 \epsilon  + 1.5 \cdot 1.25 \cdot \frac{1}{2\sqrt{2}} \epsilon < 3.5 \epsilon  \\ \nonumber
  \Delta \left(\frac{\mu \omega }{\sqrt{2}\sqrt{\nu \tau}}\right) &\leq  
         \left( \Delta \left(\mu \omega\right)  \sqrt{2}\sqrt{\nu \tau} +  \left| \mu \omega \right|  \Delta \left(\sqrt{2}\sqrt{\nu \tau} \right) \right) \frac{1}{\left(\sqrt{2}\sqrt{\nu \tau}\right)^2} \leq \\ \nonumber
         & \left( 0.2 \epsilon  \sqrt{2}\sqrt{0.64} +  0.01 \cdot 3.5 \epsilon \right) \frac{1}{2\cdot 0.64} < 0.25 \epsilon \\ \nonumber
 \Delta \left(\frac{\mu \omega + \nu \tau}{\sqrt{2}\sqrt{\nu \tau}}\right) &\leq  
        \left( \Delta \left(\mu \omega + \nu \tau \right)  \sqrt{2}\sqrt{\nu \tau} +  \left| \mu \omega + \nu \tau \right|  \Delta \left(\sqrt{2}\sqrt{\nu \tau} \right) \right) \frac{1}{\left(\sqrt{2}\sqrt{\nu \tau}\right)^2} \leq \\ \nonumber
         & \left( 3.2 \epsilon  \sqrt{2}\sqrt{0.64} +  1.885\cdot 3.5 \epsilon \right) \frac{1}{2\cdot 0.64} < 8  \epsilon . \\ \nonumber
\end{align}

Using these bounds on the simple expressions, we can now calculate bounds on the numerical errors of compound expressions:

\begin{align}
 \Delta \left(\erfc \left(\frac{\mu \omega }{\sqrt{2}\sqrt{\nu \tau}}\right) \right) &\leq 
       \frac{2}{\sqrt{\pi}}  e^{ - \left(\frac{\mu \omega }{\sqrt{2}\sqrt{\nu \tau}}\right)^2 } \Delta  \left(\frac{\mu \omega }{\sqrt{2}\sqrt{\nu \tau}}  \right) < \\ \nonumber
       &\frac{2}{\sqrt{\pi}}  0.25 \epsilon < 0.3 \epsilon \\
  \Delta \left(\erfc \left(\frac{\mu \omega + \nu \tau}{\sqrt{2}\sqrt{\nu \tau}}\right) \right) &\leq 
       \frac{2}{\sqrt{\pi}}  e^{ - \left(\frac{\mu \omega + \nu \tau}{\sqrt{2}\sqrt{\nu \tau}}\right)^2 } \Delta  \left(\frac{\mu \omega + \nu \tau}{\sqrt{2}\sqrt{\nu \tau}}  \right) < \\ \nonumber
       &\frac{2}{\sqrt{\pi}}  8 \epsilon < 10 \epsilon \\
  \Delta \left(e^{\mu \omega + \frac{\nu \tau}{2} } \right) &\leq \left(e^{\mu \omega + \frac{\nu \tau}{2} } \right) \Delta \left(e^{\mu \omega + \frac{\nu \tau}{2} } \right) < \\   
       & e^{0.9475}  2.2 \epsilon < 5.7 \epsilon
\end{align}

Subsequently, we can use the above results to get bounds for the numerical errors on the Jacobian entries (Eq.~\eqref{eq:JacobianEntries}), 
applying the rules from Subsection~\ref{sec:error} again:
\begin{align}
& \Delta \left( {\mathcal J}_{11} \right) \  = \ \Delta \left( \frac{1}{2} \lambda  \omega \left(\alpha  e^{\mu \omega+\frac{\nu \tau}{2}} \erfc \left(\frac{\mu \omega+\nu \tau}{\sqrt{2} \sqrt{\nu \tau}}\right)
- \erfc \left(\frac{\mu \omega}{\sqrt{2} \sqrt{\nu \tau}}\right)+2\right) \right) < 6 \epsilon, 
\end{align}

and we obtain $\Delta \left( {\mathcal J}_{12} \right) < 78\epsilon$,  $\Delta \left( {\mathcal J}_{21} \right) < 189\epsilon$, $\Delta \left( {\mathcal J}_{22} \right) < 405\epsilon$
and $\Delta \left(\munn\right) < 52\epsilon$. 
We also have bounds on the absolute values on $\mathcal J_{ij}$ and $\munn$ (see Lemma~\ref{lem:J11},
Lemma~\ref{lem:J12}, and Lemma~\ref{lem:boundsmeanvar}), therefore we can 
propagate the error also through the function that calculates the singular value (Eq.~\eqref{eq:S}). 

\begin{align}
& \Delta \left(S(\mu,\omega,\nu,\tau,\lambda ,\alpha ) \right)\ = \\ \nonumber
& \Delta \left( \frac{1}{2} \ \left(\sqrt{({\mathcal J}_{11}+{\mathcal J}_{22} - 2 \munn {\mathcal J}_{12})^2+({\mathcal J}_{21} - 2 \munn {\mathcal J}_{11}-{\mathcal J}_{12})^2} \ + \right. \right.\\ \nonumber
& \left. \left. \sqrt{({\mathcal J}_{11}-{\mathcal J}_{22} + 2 \munn {\mathcal J}_{12})^2+({\mathcal J}_{12}+{\mathcal J}_{21} - 2 \munn {\mathcal J}_{11})^2} \right)\right) < 292 \epsilon.
\end{align}
\end{proof}

\paragraph{Precision of Implementations.}
We will show that our computations are correct up to 3 ulps. For
our implementation in GNU C library and the hardware architectures
that we used, the precision of all mathematical functions that we used
is at least one ulp.
The term ``ulp'' (acronym for ``unit in the last place'') was coined
by W. Kahan in 1960. It is the highest precision (up to some factor
smaller 1), which can be
achieved for the given hardware and floating point representation.

Kahan defined ulp as \citep{Kahan:04}:
\begin{quote}
``Ulp$(x)$ is the gap between the two {\em finite} floating-point numbers
nearest $x$, even if $x$ is one of them. (But ulp(NaN) is NaN.)''
\end{quote}
Harrison defined ulp as \citep{Harrison:99}:
\begin{quote}
``an ulp in $x$ is the distance
between the two closest {\em straddling} floating point numbers $a$ and $b$, i.e.\ those with
$a \leq x \leq b$ and $a \not= b$ assuming an unbounded exponent range.''
\end{quote}
In the literature we find also slightly different definitions
\citep{Muller:05}.


According to \citep{Muller:05} who refers to \citep{Goldberg:91}:
\begin{quote}
``IEEE-754 mandates four standard rounding modes:''

``Round-to-nearest: $r(x)$ is the floating-point value closest to $x$ with the
usual distance; if two floating-point value are equally close to $x$, then $r(x)$
is the one whose least significant bit is equal to zero.''

``IEEE-754 standardises 5 operations: addition (which we shall note $\oplus$ in order to
distinguish it from the operation over the reals), subtraction ($\ominus$), multiplication
($\otimes$), division ($\oslash$), and also square root.''

``IEEE-754 specifies {em exact rounding} [Goldberg, 1991, \S1.5]: the result of a
floating-point operation is the same as if the operation were performed on the
real numbers with the given inputs, then rounded according to the rules in the
preceding section. Thus, $x \oplus y$ is defined as $r(x + y)$, with $x$ and $y$ taken as
elements of $\dR \cup \{-\infty,+\infty\}$; the same applies for the other operators.''
\end{quote}
Consequently, the IEEE-754 standard guarantees that addition,
subtraction, multiplication, division, and squared root is precise up
to one ulp.

We have to consider transcendental functions. First the is the
exponential function, and then the complementary error
function $\mathrm{erfc}(x)$,
which can be computed via the error function $\mathrm{erf}(x)$.

Intel states \citep{Muller:05}:
\begin{quote}
``With the Intel486 processor and Intel 387 math coprocessor, the worst-
case, transcendental function error is typically $3$ or $3.5$ ulps, but is some-
times as large as $4.5$ ulps.''
\end{quote}

According to \url{https://www.mirbsd.org/htman/i386/man3/exp.htm} and 
\url{http://man.openbsd.org/OpenBSD-current/man3/exp.3}:
\begin{quote}
``exp$(x)$, log$(x)$, expm1$(x)$ and log1p$(x)$ are accurate to within an ulp''
\end{quote}
which is the same for freebsd \url{https://www.freebsd.org/cgi/man.cgi?query=exp&sektion=3&apropos=0&manpath=freebsd}:
\begin{quote}
``The values of exp(0), expm1(0), exp2(integer), and pow(integer, integer)
are exact provided that they are representable.  Otherwise the error in
these functions is generally below one ulp.''
\end{quote}
The same holds for ``FDLIBM'' \url{http://www.netlib.org/fdlibm/readme}:
\begin{quote}
``FDLIBM is intended to provide a reasonably portable (see
assumptions below), reference quality (below one ulp for
major functions like sin,cos,exp,log) math library
(libm.a).''
\end{quote}

In
\url{http://www.gnu.org/software/libc/manual/html_node/Errors-in-Math-Functions.html}
we find that both $\mathrm{exp}$ and
$\mathrm{erf}$ have an error of 1 ulp while $\mathrm{erfc}$ has an
error up to 3 ulps depending on the architecture.
For the most common architectures as used by us, however, the error of
$\mathrm{erfc}$ is 1 ulp.


We implemented the function in the programming language C.
We rely on the GNU C Library \citep{Loosemore:16}.
According to the GNU C Library manual which can be obtained from
\url{http://www.gnu.org/software/libc/manual/pdf/libc.pdf},
the errors of the math functions $\exp$, $\mathrm{erf}$, and
$\mathrm{erfc}$
are not larger than 3 ulps for all architectures
\citep[pp. 528]{Loosemore:16}.
For the architectures ix86, i386/i686/fpu, and m68k/fpmu68k/m680x0/fpu
that we used the error are at least one ulp
\citep[pp. 528]{Loosemore:16}.

\subsubsection{Intermediate Lemmata and Proofs}
\label{sec:smallLemmata}

Since we focus on the fixed point 
$(\mu,\nu)=(0,1)$,
we assume for our whole analysis
that $\alpha = \alpha_{\rm 01}$ and $\lambda=\lambda_{\rm 01}$.
Furthermore, we restrict the range of the variables
$\mu \in [\mu_{\rm min},\mu_{\rm max}]=[-0.1,0.1]$,
$\omega \in [\omega_{\rm min},\omega_{\rm max}]=[-0.1,0.1]$,
$\nu \in [\nu_{\rm min},\nu_{\rm max}]=[0.8,1.5]$, and
$\tau \in [\tau_{\rm min},\tau_{\rm max}]=[0.8,1.25]$.

For bounding different partial derivatives we need properties of
different functions. 
We will bound a the absolute value of a function by computing an upper
bound on its maximum and a lower bound on its minimum. These bounds
are computed by upper or lower bounding terms. The bounds get tighter
if we can combine terms to a more complex function and bound this
function. The following lemmata give some properties of functions that
we will use in bounding complex functions. 

Throughout this work, we use the error function $\erf(x):=\frac{1}{\sqrt{\pi}} \int_{-x} ^x e^{-t^2}$ and the complementary
error function $\erfc (x) = 1 - \erf(x)$. \index{error function!definition} \index{complementary error function!definition}
\index{erf} \index{erfc}

\begin{lemma}[Basic functions]
\label{lem:basics}

$\exp(x)$ is strictly monotonically increasing from $0$ at $-\infty$ to
$\infty$ at $\infty$ and has positive curvature.

According to its definition $\erfc (x)$ is strictly monotonically decreasing from 2 at $-\infty$ to 0 at $\infty$.
\end{lemma}

Next we introduce a bound on $\erfc $:
\begin{lemma}[Erfc bound from Abramowitz]
\label{lem:Abramowitz}
\index{Abramowitz bounds}

\begin{align}
\frac{2 e^{-x^2}}{\sqrt{\pi } \left(\sqrt{x^2+2}+x\right)} \ &< \
\erfc (x) \ \leq \ 
\frac{2 e^{-x^2}}{\sqrt{\pi } \left(\sqrt{x^2+\frac{4}{\pi }}+x\right)},
\end{align}
for $x>0$. 
\end{lemma}
\begin{proof}
The statement follows immediately from 
\citep{Abramowitz:64} (page 298, formula 7.1.13).
\end{proof}

These bounds are displayed in figure~\ref{fig:abramowitz}. \index{Abramowitz bounds} \index{error function!bounds}\index{complementary error function!bounds} \index{erf} \index{erfc}
\begin{figure}
 \centering
 \includegraphics[width=0.49\columnwidth]{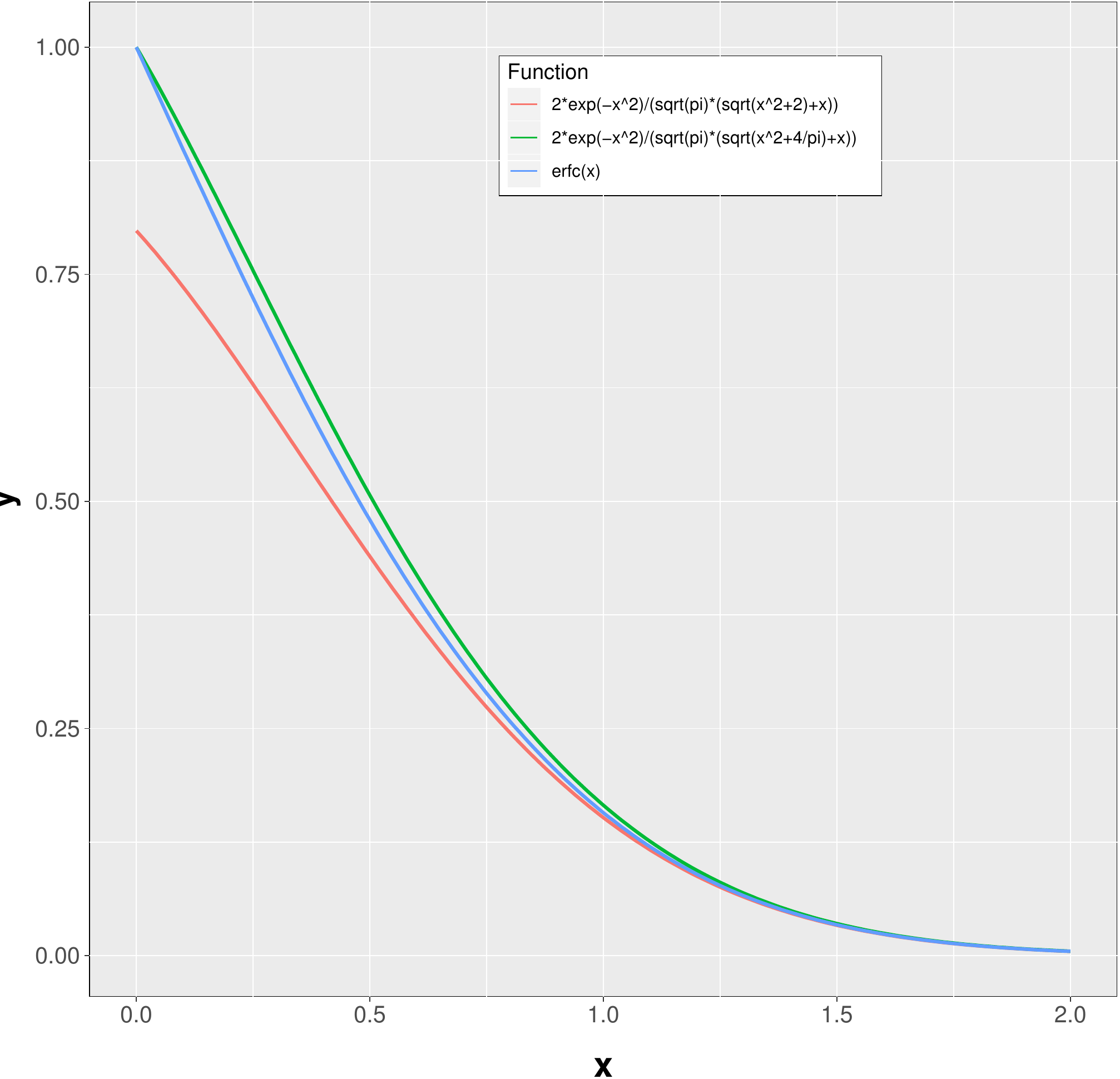}
 \caption[Graph of the Abramowitz bound for the complementary error function.]{Graphs of the upper and lower bounds on $\erfc$.  The lower bound $\frac{2 e^{-x^2}}{\sqrt{\pi } \left(\sqrt{x^2+2}+x\right)}$ (red),  
  the upper bound $\frac{2 e^{-x^2}}{\sqrt{\pi } \left(\sqrt{x^2+\frac{4}{\pi }}+x\right)}$ (green) and the function $\erfc (x)$ (blue) as 
  treated in Lemma~\ref{lem:Abramowitz}. \label{fig:abramowitz}}
\end{figure}

\begin{lemma}[Function $e^{x^2} \erfc (x)$]
\label{lem:exerfc}

$e^{x^2} \erfc (x)$ is strictly monotonically decreasing for $x > 0$ 
and has positive curvature 
(positive 2nd order derivative), that is, the decreasing slowes down.
\end{lemma}

\begin{figure}
 \includegraphics[width=0.49\columnwidth]{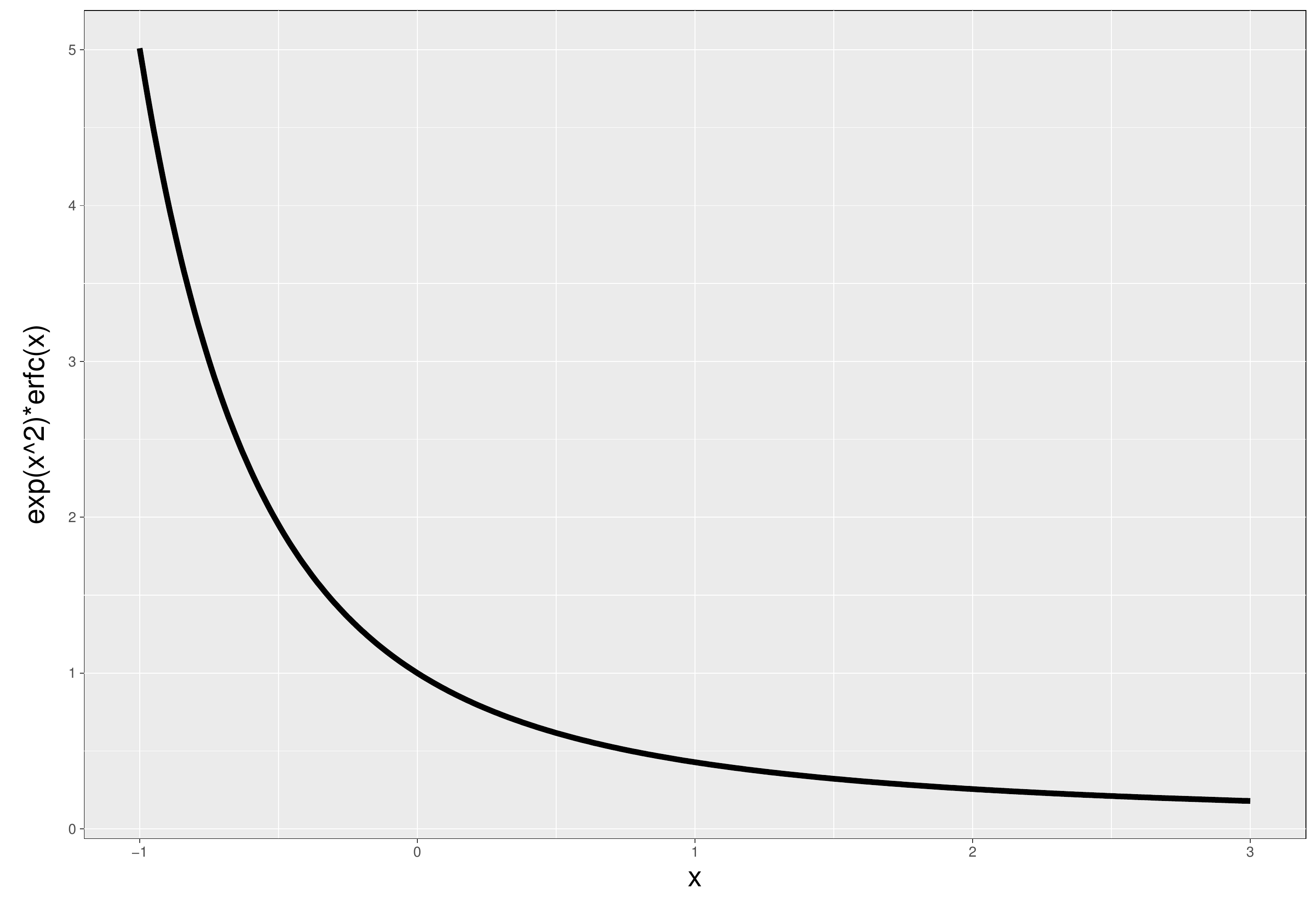}
 \includegraphics[width=0.49\columnwidth]{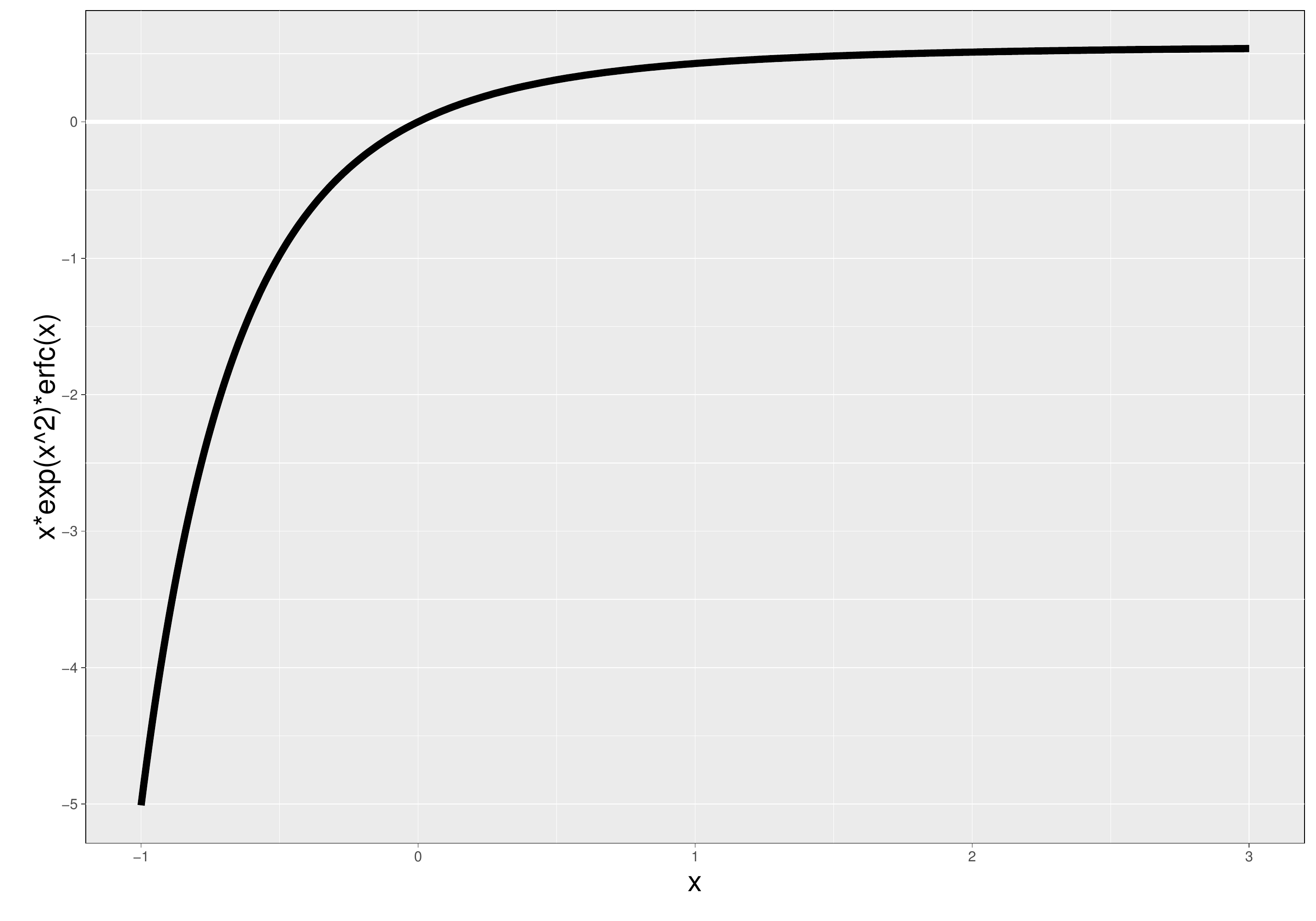}
 \caption[Graphs of the functions  $e^{x^2} \erfc (x)$ and $x e^{x^2} \erfc (x)$.]{Graphs of the functions  $e^{x^2} \erfc (x)$ (left) and $x e^{x^2} \erfc (x)$ (right) treated in Lemma~\ref{lem:exerfc} and Lemma~\ref{lem:xeErfc}, 
 respectively. \label{fig:lemma4}}
\end{figure} 

A graph of the function is displayed in Figure~\ref{fig:lemma4}.

\begin{proof}
The derivative of $e^{x^2} \erfc (x)$ is
\begin{align}
\frac{\partial e^{x^2} \erfc (x)}{\partial x} \ = \  
2 e^{x^2} x \erfc (x)-\frac{2}{\sqrt{\pi }} \ .
\end{align}
Using Lemma~\ref{lem:Abramowitz}, we get
\begin{align}
&\frac{\partial e^{x^2} \erfc (x)}{\partial x} \ = \ 
2 e^{x^2} x \erfc (x)-\frac{2}{\sqrt{\pi }}\ <  &\frac{4 x}{\sqrt{\pi } \left(\sqrt{x^2+\frac{4}{\pi
  }}+x\right)}-\frac{2}{\sqrt{\pi }}=\frac{2
  \left(\frac{2}{\sqrt{\frac{4}{\pi  x^2}+1}+1}-1\right)}{\sqrt{\pi }}
  \ < \ 0
\end{align}
Thus $e^{x^2} \erfc (x)$
is strictly monotonically decreasing for $x > 0$.

The second order derivative of $e^{x^2} \erfc (x)$ is
\begin{align}
\frac{\partial^2 e^{x^2} \erfc (x)}{\partial x^2} \ = \  
4 e^{x^2} x^2 \erfc (x)+2 e^{x^2} \erfc (x)-\frac{4 x}{\sqrt{\pi }} \ .
\end{align}

Again using Lemma~\ref{lem:Abramowitz} (first inequality), we get
\begin{align}
&2 \left(\left(2 x^2+1\right) e^{x^2} \erfc (x)-\frac{2
  x}{\sqrt{\pi }}\right)
\ > \\\nonumber
&\frac{4 \left(2 x^2+1\right)}{\sqrt{\pi } \left(\sqrt{x^2+2}+x\right)}-\frac{4 x}{\sqrt{\pi }}
\ = \\\nonumber
&\frac{4 \left(x^2-\sqrt{x^2+2} x+1\right)}{\sqrt{\pi } \left(\sqrt{x^2+2}+x\right)}
\ = \\\nonumber
&\frac{4 \left(x^2-\sqrt{x^4+2 x^2}+1\right)}{\sqrt{\pi } \left(\sqrt{x^2+2}+x\right)}
\ > \\\nonumber
&\frac{4 \left(x^2-\sqrt{x^4+2 x^2+1}+1\right)}{\sqrt{\pi }
  \left(\sqrt{x^2+2}+x\right)} \ = \ 0
\end{align}
For the last inequality we added 1 in the numerator in the square root
which is subtracted, that is, making a larger negative term in the
numerator. 
\end{proof}

\begin{lemma}[Properties of $x e^{x^2} \erfc (x)$]
\label{lem:xeErfc}
\index{error function!properties}

The function $x e^{x^2} \erfc (x)$ has the sign of $x$ and is
monotonically increasing to $\frac{1}{\sqrt{\pi }}$.
\end{lemma}
\begin{proof}
The derivative of  $x e^{x^2} \erfc (x)$ is
\begin{align}
2 e^{x^2} x^2 \erfc (x)+e^{x^2} \erfc (x)-\frac{2
  x}{\sqrt{\pi }} \ .
\end{align}
This derivative is positive since
\begin{align}
&2 e^{x^2} x^2 \erfc (x)+e^{x^2} \erfc (x)-\frac{2
  x}{\sqrt{\pi }}\ = \\\nonumber
&e^{x^2} \left(2 x^2+1\right) \erfc (x)-\frac{2
  x}{\sqrt{\pi }}>\frac{2 \left(2 x^2+1\right)}{\sqrt{\pi }
  \left(\sqrt{x^2+2}+x\right)}-\frac{2 x}{\sqrt{\pi }}=\frac{2
  \left(\left(2 x^2+1\right)-x
  \left(\sqrt{x^2+2}+x\right)\right)}{\sqrt{\pi }
  \left(\sqrt{x^2+2}+x\right)}\ = \\\nonumber
&\frac{2 \left(x^2-x \sqrt{x^2+2}
  +1\right)}{\sqrt{\pi } \left(\sqrt{x^2+2}+x\right)}=\frac{2
  \left(x^2-x \sqrt{x^2+2} +1\right)}{\sqrt{\pi }
  \left(\sqrt{x^2+2}+x\right)}>\frac{2
  \left(x^2-x \sqrt{x^2+\frac{1}{x^2}+2} +1\right)}{\sqrt{\pi }
  \left(\sqrt{x^2+2}+x\right)}\ = \\\nonumber
&\frac{2 \left(x^2-\sqrt{x^4+2
  x^2+1}+1\right)}{\sqrt{\pi } \left(\sqrt{x^2+2}+x\right)}=\frac{2
  \left(x^2-\sqrt{\left(x^2+1\right)^2}+1\right)}{\sqrt{\pi }
  \left(\sqrt{x^2+2}+x\right)}=0 \ .
\end{align}

We apply Lemma~\ref{lem:Abramowitz}
to $x \erfc (x) e^{x^2}$ and divide the terms of the lemma by $x$,
which gives
\begin{align}
\frac{2}{\sqrt{\pi }
  \left(\sqrt{\frac{2}{x^2}+1}+1\right)}<x\erfc (x) e^{x^2}\leq
  \frac{2}{\sqrt{\pi } \left(\sqrt{\frac{4}{\pi  x^2}+1}+1\right)} \ .
\end{align}
For $\lim_{x \to \infty}$ both the upper and the lower bound go to 
$\frac{1}{\sqrt{\pi }}$.
\end{proof}

\begin{lemma}[Function $\mu \omega$]
\label{lem:x11}

$h_{11}(\mu, \omega) = \mu \omega$ is monotonically increasing in $\mu \omega$.
It has minimal value $t_{11}=-0.01$ and maximal value
$T_{11}=0.01$.
\end{lemma}
\begin{proof}
Obvious.
\end{proof}

\begin{lemma}[Function $\nu \tau$]
\label{lem:x22}

$h_{22}(\nu, \tau)=\nu \tau$ is
monotonically increasing in $\nu \tau$ and is positive.
It has minimal value $t_{22}=0.64$ and maximal value
$T_{22}=1.875$.
\end{lemma}
\begin{proof}
Obvious. 
\end{proof}

\begin{lemma}[Function $\frac{\mu \omega+\nu \tau}{\sqrt{2} \sqrt{\nu \tau}}$]
\label{lem:xx1}

$h_1(\mu, \omega,\nu, \tau)=\frac{\mu \omega+\nu \tau}{\sqrt{2} \sqrt{\nu \tau}}$
is larger than zero and increasing in both $\nu
\tau$ and $\mu \omega$.
It has minimal value $t_1=0.5568$ and maximal value
$T_1=0.9734$.
\end{lemma}
\begin{proof}
The derivative of the function
$\frac{\mu \omega+x}{\sqrt{2} \sqrt{x}}$
with respect to $x$ is
\begin{align}
&\frac{1}{\sqrt{2} \sqrt{x}}-\frac{\mu \omega+x}{2 \sqrt{2} x^{3/2}}
\ = \frac{2 x-(\mu \omega+x)}{2 \sqrt{2} x^{3/2}}
\ = \ \frac{x-\mu \omega}{2 \sqrt{2} x^{3/2}} \ > \ 0 \ ,
\end{align}
since $x>0.8 \cdot 0.8$ and $\mu \omega<0.1 \cdot 0.1$.
\end{proof}

\begin{lemma}[Function $\frac{\mu \omega+2 \nu \tau}{\sqrt{2} \sqrt{\nu \tau}}$]
\label{lem:xx2}

$h_2(\mu, \omega,\nu, \tau)=\frac{\mu \omega+2 \nu \tau}{\sqrt{2} \sqrt{\nu \tau}}$
is larger than zero and increasing in both $\nu
\tau$ and $\mu \omega$.
It has minimal value $t_2=1.1225$ and maximal value
$T_2=1.9417$.
\end{lemma}
\begin{proof}
The derivative of the function
$\frac{\mu \omega+2 x}{\sqrt{2} \sqrt{x}}$
with respect to $x$ is
\begin{align}
\frac{\sqrt{2}}{\sqrt{x}}-\frac{\mu \omega+2 x}{2 \sqrt{2} x^{3/2}}=\frac{4 x-(\mu \omega+2 x)}{2 \sqrt{2} x^{3/2}}=
\frac{2 x-\mu \omega}{2 \sqrt{2} x^{3/2}} \ > \ 0 \ .
\end{align}
\end{proof}

\begin{lemma}[Function $\frac{\mu \omega}{\sqrt{2} \sqrt{\nu \tau}}$]
\label{lem:xx3}

$h_3(\mu, \omega,\nu, \tau)=\frac{\mu \omega}{\sqrt{2} \sqrt{\nu \tau}}$
monotonically decreasing in $\nu
\tau$ and monotonically increasing in $\mu \omega$.
It has minimal value $t_3=-0.0088388$ and maximal value
$T_3=0.0088388$.
\end{lemma}
\begin{proof}
Obvious.
\end{proof}

\begin{lemma}[Function $\left(\frac{\mu \omega}{\sqrt{2} \sqrt{\nu \tau}}\right)^2$]
\label{lem:xx4}

$h_4(\mu, \omega,\nu, \tau)=\left(\frac{\mu \omega}{\sqrt{2} \sqrt{\nu \tau}}\right)^2$
has a minimum at 0 for $\mu=0$ or $\omega=0$ and has a maximum for
the smallest  $\nu \tau$ and largest $\left| \mu \omega
\right|$ and is larger or equal to zero.
It has minimal value $t_4 = 0$ and maximal value
$T_4 = 0.000078126$.
\end{lemma}
\begin{proof}
Obvious.
\end{proof}

\begin{lemma}[Function $\frac{\sqrt{\frac{2}{\pi }} (\alpha -1)}{\sqrt{\nu \tau}}$]
\label{lem:F1}

$\frac{\sqrt{\frac{2}{\pi }} (\alpha -1)}{\sqrt{\nu
    \tau}}>0$
and decreasing in $\nu \tau$.
\end{lemma}
\begin{proof}
Statements follow directly from elementary functions square root and
division.
\end{proof}

\begin{lemma}[Function $2-\erfc \left(\frac{\mu \omega}{\sqrt{2} \sqrt{\nu \tau}}\right)$]
\label{lem:F2}

$2-\erfc \left(\frac{\mu \omega}{\sqrt{2} \sqrt{\nu
      \tau}}\right)>0$
and decreasing in $\nu \tau$ and increasing in 
$\mu \omega$.
\end{lemma}
\begin{proof}
Statements follow directly from Lemma~\ref{lem:basics} and $\erfc $.
\end{proof}

\begin{lemma}[Function $\sqrt{\frac{2}{\pi }} \left(\frac{(\alpha -1) \mu \omega}{(\nu \tau)^{3/2}}-\frac{\alpha }{\sqrt{\nu \tau}}\right)$]
\label{lem:F3}

For $\lambda=\lambda_{\rm 01}$ and $\alpha=\alpha_{\rm 01}$,
 $\sqrt{\frac{2}{\pi }} \left(\frac{(\alpha -1) \mu\omega}{(\nu \tau)^{3/2}} - 
 \frac{\alpha }{\sqrt{\nu \tau}}\right)<0$
and increasing in both $\nu \tau$ and $\mu \omega$.
\end{lemma}
\begin{proof}
We consider the function
    $\sqrt{\frac{2}{\pi }} \left(\frac{(\alpha -1) \mu
    \omega}{x^{3/2}}-\frac{\alpha }{\sqrt{x}}\right)$, 
which has the derivative with respect to $x$:
\begin{align}
\sqrt{\frac{2}{\pi }} \left(\frac{\alpha }{2 x^{3/2}}-\frac{3 (\alpha
  -1) \mu \omega}{2 x^{5/2}}\right) \ .
\end{align}
This derivative is larger than zero, since 
\begin{align}
&\sqrt{\frac{2}{\pi }} \left(\frac{\alpha }{2 (\nu \tau)^{3/2}}-\frac{3 (\alpha -1) \mu \omega}{2 (\nu \tau)^{5/2}}\right)
\ > \frac{\sqrt{\frac{2}{\pi }} \left(\alpha -\frac{3 (\alpha -1)
  \mu \omega}{\nu \tau}\right)}{2 (\nu
  \tau)^{3/2}} \ > \ 0 \ .
\end{align}
The last inequality follows from 
$\alpha -\frac{3 \cdot 0.1 \cdot 0.1(\alpha -1) }{0.8 \cdot 0.8}>0$ for $\alpha=\alpha_{\rm 01}$.

We next consider the function
$\sqrt{\frac{2}{\pi }} \left(\frac{(\alpha -1) x}{(\nu \tau)^{3/2}}-\frac{\alpha }{\sqrt{\nu \tau}}\right)$,
which  has the derivative with respect to $x$:
\begin{align}
\frac{\sqrt{\frac{2}{\pi }} (\alpha -1)}{(\nu \tau)^{3/2}}
  \ > \ 0 \ .
\end{align}
\end{proof}

\begin{lemma}[Function $\sqrt{\frac{2}{\pi }} \left(\frac{(-1) (\alpha -1) \mu^2 \omega^2}{(\nu \tau)^{3/2}}+\frac{-\alpha +\alpha  \mu \omega+1}{\sqrt{\nu \tau}}-\alpha  \sqrt{\nu \tau}\right)$]
\label{lem:F4}

The function \\
$\sqrt{\frac{2}{\pi }} 
    \left(\frac{(-1) (\alpha -1) \mu^2 \omega^2}{(\nu \tau)^{3/2}} + 
    \frac{-\alpha + \alpha  \mu \omega+1}{\sqrt{\nu \tau}} - 
    \alpha  \sqrt{\nu \tau}\right)<0$ 
is decreasing in $\nu \tau$ and increasing in $\mu \omega$.
\end{lemma}
\begin{proof}
We define the function
\begin{align}
\sqrt{\frac{2}{\pi }} \left(\frac{(-1) (\alpha -1) \mu^2 \omega^2}{x^{3/2}}+\frac{-\alpha +\alpha  \mu \omega+1}{\sqrt{x}}-\alpha  \sqrt{x}\right)
\end{align}
which has as derivative with respect to $x$:
\begin{align}
&\sqrt{\frac{2}{\pi }} \left(\frac{3 (\alpha -1) \mu^2 \omega^2}{2
  x^{5/2}}-\frac{-\alpha +\alpha  \mu \omega+1}{2
  x^{3/2}}-\frac{\alpha }{2 \sqrt{x}}\right)
\ = \\\nonumber
&\frac{1}{\sqrt{2 \pi } x^{5/2}}
\left(3 (\alpha -1) \mu^2 \omega^2-x (-\alpha +\alpha  \mu
  \omega+1)-\alpha  x^2\right) \ .
\end{align}
The derivative of the term
$3 (\alpha -1) \mu^2 \omega^2-x (-\alpha +\alpha  \mu
\omega+1)-\alpha  x^2$
with respect to $x$ is 
$-1 +\alpha -\mu \omega \alpha -2 \alpha  x<0$, since
$2 \alpha  x > 1.6 \alpha$.
Therefore the term is maximized with the smallest value for $x$, which
is $x=\nu \tau=0.8 \cdot 0.8$.
For $\mu \omega$ we use for each term the value which gives maximal
contribution. We obtain an upper bound for the term:
\begin{align}
3 (-0.1 \cdot 0.1)^2 (\alpha_{\rm 01}-1)-(0.8 \cdot 0.8)^2 \alpha_{\rm 01}-0.8 \cdot 0.8 ((-0.1 \cdot 0.1) \alpha_{\rm 01}-\alpha_{\rm 01}+1) 
\ = \ -0.243569 \ .
\end{align}
Therefore the derivative with respect to $x=\nu \tau$ 
is smaller than zero and the original function is decreasing in $\nu \tau$

We now consider the derivative with respect to $x=\mu \omega$.
The derivative with respect to $x$ of the function
\begin{align}
\sqrt{\frac{2}{\pi }} \left(-\alpha  \sqrt{\nu \tau}-\frac{(\alpha -1) x^2}{(\nu \tau)^{3/2}}+\frac{-\alpha +\alpha  x+1}{\sqrt{\nu \tau}}\right)
\end{align}
is
\begin{align}
\frac{\sqrt{\frac{2}{\pi }} (\alpha  \nu \tau-2 (\alpha -1)
  x)}{(\nu \tau)^{3/2}} \ .
\end{align}
Since 
$-2 x (-1 + \alpha) + \nu \tau \alpha > -2 \cdot 0.01 \cdot  (-1 + \alpha_{\rm 01}) +
0.8 \cdot 0.8 \alpha_{\rm 01}>1.0574>0$, the derivative is larger than zero.
Consequently, the original function is increasing in  $\mu \omega$.

The maximal value is obtained with the minimal $\nu
\tau=0.8 \cdot 0.8$ and the maximal  $\mu \omega=0.1 \cdot 0.1$.
The maximal value is
\begin{align}
\sqrt{\frac{2}{\pi }} \left(\frac{0.1 \cdot 0.1 \alpha_{\rm 01}-\alpha_{\rm 01}+1}{\sqrt{0.8 \cdot 0.8}}+\frac{0.1^2 0.1^2 (-1)
  (\alpha_{\rm 01}-1)}{(0.8 \cdot 0.8)^{3/2}}-\sqrt{0.8 \cdot 0.8}
  \alpha_{\rm 01}\right) \ = \ -1.72296 \ .
\end{align}
Therefore the original function is smaller than zero.
\end{proof}

\begin{lemma}[Function $\sqrt{\frac{2}{\pi }} \left(\frac{\left(\alpha ^2-1\right) \mu \omega}{(\nu \tau)^{3/2}}-\frac{3 \alpha ^2}{\sqrt{\nu \tau}}\right)$]
\label{lem:F5}

For $\lambda=\lambda_{\rm 01}$ and $\alpha=\alpha_{\rm 01}$, \\
$\sqrt{\frac{2}{\pi }} \left(\frac{\left(\alpha ^2-1\right) \mu \omega}{(\nu \tau)^{3/2}}-\frac{3 \alpha ^2}{\sqrt{\nu \tau}}\right)<0$
and increasing in both $\nu \tau$ and $\mu \omega$.

\end{lemma}
\begin{proof}
The derivative of the function
\begin{align}
\sqrt{\frac{2}{\pi }} \left(\frac{\left(\alpha ^2-1\right) \mu \omega}{x^{3/2}}-\frac{3 \alpha ^2}{\sqrt{x}}\right)
\end{align}
with respect to $x$ is
\begin{align}
\sqrt{\frac{2}{\pi }} \left(\frac{3 \alpha^2}{2 x^{3/2}}-\frac{3
  \left(\alpha^2-1\right) \mu \omega}{2 x^{5/2}}\right)
\ = \ 
\frac{3 \left(\alpha^2 x-\left(\alpha^2-1\right) \mu
  \omega\right)}{\sqrt{2 \pi } x^{5/2}} \ > \ 0 \ ,
\end{align}
since
$\alpha^2 x- \mu \omega (-1+\alpha^2)> \alpha_{\rm 01}^2 0.8 \cdot 0.8- 0.1 \cdot 0.1 \cdot (-1+\alpha_{\rm 01}^2)>1.77387$

The derivative of the function
\begin{align}
\sqrt{\frac{2}{\pi }} \left(\frac{\left(\alpha ^2-1\right) x}{(\nu \tau)^{3/2}}-\frac{3 \alpha ^2}{\sqrt{\nu \tau}}\right)
\end{align}
with respect to $x$ is
\begin{align}
\frac{\sqrt{\frac{2}{\pi }} \left(\alpha ^2-1\right)}{(\nu \tau)^{3/2}} \ > \ 0 \ .
\end{align}

The maximal function value is obtained by maximal $\nu
\tau=1.5 \cdot 1.25$ and the maximal  $\mu \omega=0.1 \cdot 0.1$. 
The maximal value is
$\sqrt{\frac{2}{\pi }} \left(\frac{0.1 \cdot 0.1 \left(\alpha_{\rm 01}^2-1\right)}{(1.5 \cdot 1.25)^{3/2}} - 
\frac{3 \alpha_{\rm 01}^2}{\sqrt{1.5 \cdot 1.25}}\right) \ = \ -4.88869$.
Therefore the function is negative.
\end{proof}

\begin{lemma}[Function $\sqrt{\frac{2}{\pi }} \left(\frac{\left(\alpha ^2-1\right) \mu \omega}{\sqrt{\nu \tau}}-3 \alpha ^2 \sqrt{\nu \tau}\right)$]
\label{lem:F6}

The function
$\sqrt{\frac{2}{\pi }} \left(\frac{\left(\alpha ^2-1\right) \mu \omega}{\sqrt{\nu \tau}}-3 \alpha ^2 \sqrt{\nu \tau}\right)<0$ 
is decreasing in $\nu \tau$ and increasing in $\mu \omega$.
\end{lemma}
\begin{proof}
The derivative of the function
\begin{align}
\sqrt{\frac{2}{\pi }} \left(\frac{\left(\alpha ^2-1\right) \mu \omega}{\sqrt{x}}-3 \alpha ^2 \sqrt{x}\right)
\end{align}
with respect to $x$ is
\begin{align}
\sqrt{\frac{2}{\pi }} \left(-\frac{\left(\alpha^2-1\right) \mu
  \omega}{2 x^{3/2}}-\frac{3 \alpha ^2}{2 \sqrt{x}}\right)
\ = \ 
\frac{-\left(\alpha^2-1\right) \mu \omega-3 \alpha^2 x}{\sqrt{2 \pi } x^{3/2}} \ < \ 0 \ ,
\end{align}
since 
$-3 \alpha^2 x - \mu \omega (-1+\alpha^2)< -3 \alpha_{\rm 01}^2 0.8 \cdot 0.8+0.1 \cdot 0.1 (-1+\alpha_{\rm 01}^2)<-5.35764$.

The derivative of the function
\begin{align}
\sqrt{\frac{2}{\pi }} \left(\frac{\left(\alpha ^2-1\right) x}{\sqrt{\nu \tau}}-3 \alpha ^2 \sqrt{\nu \tau}\right)
\end{align}
with respect to $x$ is
\begin{align}
\frac{\sqrt{\frac{2}{\pi }} \left(\alpha ^2-1\right)}{\sqrt{\nu \tau}} \ > \ 0 \ .
\end{align}

The maximal function value is obtained for 
minimal $\nu
\tau=0.8 \cdot 0.8$ and the maximal  $\mu \omega=0.1 \cdot 0.1$. 
The value is
$\sqrt{\frac{2}{\pi }} \left(\frac{0.1 \cdot 0.1 \left(\alpha_{\rm 01}^2-1\right)}
{\sqrt{0.8 \cdot 0.8}}-3 \sqrt{0.8 \cdot 0.8} \alpha_{\rm 01}^2\right)
\ = \ -5.34347$.
Thus, the function is negative.
\end{proof}

\begin{lemma}[Function $\nu \tau e^{\frac{(\mu \omega+\nu \tau)^2}{2 \nu \tau}} 
\erfc \left(\frac{\mu \omega+\nu \tau}{\sqrt{2} \sqrt{\nu \tau}}\right)$]
\label{lem:F7}

The function 
 $\nu \tau e^{\frac{(\mu \omega+\nu \tau)^2}{2
     \nu \tau}} \erfc \left(\frac{\mu
     \omega+\nu \tau}{\sqrt{2} \sqrt{\nu
       \tau}}\right)>0$ 
is increasing in $\nu \tau$ and decreasing in $\mu \omega$.
\end{lemma}

\begin{proof}

The derivative of the function
\begin{align}
x e^{\frac{(\mu \omega+x)^2}{2 x}} \erfc \left(\frac{\mu \omega+x}{\sqrt{2} \sqrt{x}}\right)
\end{align}
with respect to $x$ is
\begin{align}
\frac{e^{\frac{(\mu \omega+x)^2}{2 x}} \left(x (x+2)-\mu^2 \omega^2\right) \erfc \left(\frac{\mu \omega+x}{\sqrt{2} \sqrt{x}}\right)}{2 x}+\frac{\mu \omega-x}{\sqrt{2 \pi } \sqrt{x}} \ .
\end{align}

This derivative is larger than zero, since
\begin{align}
&\frac{e^{\frac{(\mu \omega+\nu \tau)^2}{2 \nu \tau}} \left(\nu \tau (\nu \tau+2)-\mu^2 \omega^2\right) \erfc \left(\frac{\mu \omega+\nu \tau}{\sqrt{2} \sqrt{\nu \tau}}\right)}{2 \nu \tau}+\frac{\mu \omega-\nu \tau}{\sqrt{2 \pi } \sqrt{\nu \tau}}
\ > \\\nonumber
&\frac{0.4349 \left(\nu \tau (\nu \tau+2)-\mu^2 \omega^2\right)}{2 \nu \tau}+\frac{\mu \omega-\nu \tau}{\sqrt{2 \pi } \sqrt{\nu \tau}}
\ > \\\nonumber
&\frac{0.5 \left(\nu \tau (\nu \tau+2)-\mu^2 \omega^2\right)}{\sqrt{2 \pi } \nu \tau}+\frac{\mu \omega-\nu \tau}{\sqrt{2 \pi } \sqrt{\nu \tau}}
\ = \\\nonumber
&\frac{0.5 \left(\nu \tau (\nu \tau+2)-\mu^2 \omega^2\right)+\sqrt{\nu \tau} (\mu \omega-\nu \tau)}{\sqrt{2 \pi } \nu \tau}
\ = \\\nonumber
&\frac{-0.5 \mu^2 \omega^2+\mu \omega \sqrt{\nu \tau}+0.5 (\nu \tau)^2-\nu \tau \sqrt{\nu \tau}+\nu \tau}{\sqrt{2 \pi } \nu \tau}
\ = \\\nonumber
&\frac{-0.5 \mu^2 \omega^2+\mu \omega \sqrt{\nu
  \tau}+\left(0.5 \nu \tau-\sqrt{\nu
  \tau}\right)^2+0.25 (\nu \tau)^2}{\sqrt{2 \pi }
  \nu \tau} \ > \ 0 \ .
\end{align}

We explain this chain of inequalities:
\begin{itemize}
\item The first inequality follows by applying Lemma~\ref{lem:exerfc}
  which says that $e^{\frac{(\mu \omega+\nu \tau)^2}{2
      \nu \tau}} \erfc \left(\frac{\mu
      \omega+\nu \tau}{\sqrt{2} \sqrt{\nu
        \tau}}\right)$ 
is strictly monotonically decreasing. The minimal value that is larger
than 0.4349 is taken on at
the maximal values $\nu
\tau=1.5 \cdot 1.25$ and $\mu \omega=0.1 \cdot 0.1$. 
\item The second inequality uses 
$\frac{1}{2} 0.4349 \sqrt{2 \pi } = 0.545066 > 0.5$.
\item The equalities are just algebraic reformulations.
\item The last inequality follows from
$-0.5 \mu^2 \omega^2+\mu \omega \sqrt{\nu \tau}+0.25 (\nu \tau)^2>0.25 (0.8 \cdot 0.8)^2-0.5 \cdot (0.1)^2(0.1)^2-0.1 \cdot 0.1 \cdot \sqrt{0.8 \cdot 0.8}=0.09435>0$.
\end{itemize}
Therefore the function is increasing in  $\nu \tau$.

Decreasing in $\mu \omega$ follows from decreasing of $e^{x^2} \erfc (x)$
according to Lemma~\ref{lem:exerfc}.
Positivity follows form the fact that $\erfc $ and the
exponential function are positive and that  $\nu
\tau>0$. 
\end{proof}

\begin{lemma}[Function $\nu \tau e^{\frac{(\mu \omega+2 \nu \tau)^2}{2 \nu \tau}} \erfc \left(\frac{\mu \omega+2 \nu \tau}{\sqrt{2} \sqrt{\nu \tau}}\right)$]
\label{lem:F8}

The function 
$\nu \tau e^{\frac{(\mu \omega+2 \nu \tau)^2}{2
    \nu \tau}} \erfc \left(\frac{\mu \omega+2
    \nu \tau}{\sqrt{2} \sqrt{\nu \tau}}\right)>0$
is increasing in $\nu \tau$ and decreasing in $\mu \omega$.
\end{lemma}
\begin{proof}
The derivative of the function
\begin{align}
x e^{\frac{(\mu \omega+2 x)^2}{2 x}} \erfc \left(\frac{\mu \omega+2 x}{\sqrt{2} \sqrt{2 x}}\right)
\end{align}
is
\begin{align}
\frac{e^{\frac{(\mu \omega+2 x)^2}{4 x}} \left(\sqrt{\pi }
  e^{\frac{(\mu \omega+2 x)^2}{4 x}} \left(2 x (2 x+1)-\mu^2
  \omega^2\right) \erfc \left(\frac{\mu \omega+2 x}{2
  \sqrt{x}}\right)+\sqrt{x} (\mu \omega-2 x)\right)}{2 \sqrt{\pi } x}
  \ .
\end{align}
We only have to determine the sign of
$\sqrt{\pi } e^{\frac{(\mu \omega+2 x)^2}{4 x}} \left(2 x (2
  x+1)-\mu^2 \omega^2\right) \erfc \left(\frac{\mu \omega+2
    x}{2 \sqrt{x}}\right)+\sqrt{x} (\mu \omega-2 x)$ 
since all other factors are obviously larger than zero.

This derivative is larger than zero, since
\begin{align}
&\sqrt{\pi } e^{\frac{(\mu \omega+2 \nu \tau)^2}{4 \nu \tau}} \left(2 \nu \tau (2 \nu \tau+1)-\mu^2 \omega^2\right) \erfc \left(\frac{\mu \omega+2 \nu \tau}{2 \sqrt{\nu \tau}}\right)+\sqrt{\nu \tau} (\mu \omega-2 \nu \tau)
\ > \\\nonumber
&0.463979 \left(2 \nu \tau (2 \nu \tau+1)-\mu^2 \omega^2\right)+\sqrt{\nu \tau} (\mu \omega-2 \nu \tau)
\ = \\\nonumber
&-0.463979 \mu^2 \omega^2+\mu \omega \sqrt{\nu \tau}+1.85592 (\nu \tau)^2+0.927958 \nu \tau-2 \nu \tau \sqrt{\nu \tau}
\ = \\\nonumber
&\mu \omega \left(\sqrt{\nu \tau}-0.463979 \mu
  \omega\right)+0.85592 (\nu \tau)^2+\left(\nu
  \tau-\sqrt{\nu \tau}\right)^2-0.0720421 \nu
  \tau \ > \ 0 \ .
\end{align}
We explain this chain of inequalities:
\begin{itemize}
\item The first inequality follows by applying Lemma~\ref{lem:exerfc}
  which says that $e^{\frac{(\mu \omega+2 \nu \tau)^2}{2 \nu \tau}} \erfc \left(\frac{\mu \omega+2 \nu \tau}{\sqrt{2} \sqrt{\nu \tau}}\right)$ 
is strictly monotonically decreasing. The minimal value that is larger
than 0.261772 is taken on at
the maximal values $\nu
\tau=1.5 \cdot 1.25$ and $\mu \omega=0.1 \cdot 0.1$. 
$0.261772 \sqrt{\pi } > 0.463979$.
\item The equalities are just algebraic reformulations.
\item The last inequality follows from
$\mu \omega \left(\sqrt{\nu \tau}-0.463979 \mu \omega\right)+0.85592 (\nu \tau)^2-0.0720421 \nu 
\tau>0.85592 \cdot (0.8 \cdot 0.8)^2-0.1 \cdot 0.1 \left(\sqrt{1.5 \cdot 1.25}+0.1 \cdot 0.1 \cdot 0.463979\right)-0.0720421 \cdot 1.5 \cdot 1.25>0.201766$.
\end{itemize}
Therefore the function is increasing in  $\nu \tau$.

Decreasing in $\mu \omega$ follows from decreasing of $e^{x^2} \erfc (x)$
according to Lemma~\ref{lem:exerfc}.
Positivity follows from the fact that $\erfc $ and the
exponential function are positive and that  $\nu
\tau>0$. 
\end{proof}

\begin{lemma}[Bounds on the Derivatives]
\label{proof:Bounds}
The following bounds on the absolute values of the 
derivatives of the Jacobian entries ${\mathcal J}_{11}(\mu,\omega,\nu,\tau,\lambda ,\alpha )$,
${\mathcal J}_{12}(\mu,\omega,\nu,\tau,\lambda ,\alpha )$,
${\mathcal J}_{21}(\mu,\omega,\nu,\tau,\lambda ,\alpha )$, and
${\mathcal J}_{22}(\mu,\omega,\nu,\tau,\lambda ,\alpha )$
with respect to 
$\mu$, $\omega$, $\nu$, and $\tau$ hold:
\begin{align}
\left| \frac{\partial {\mathcal J}_{11}}{\partial \mu} \right| \ &< \ 0.0031049101995398316 \\ \nonumber
\left| \frac{\partial {\mathcal J}_{11}}{\partial \omega} \right| \ &< \ 1.055872374194189 \\ \nonumber
\left| \frac{\partial {\mathcal J}_{11}}{\partial \nu} \right| \ &< \  0.031242911235461816 \\ \nonumber
\left| \frac{\partial {\mathcal J}_{11}}{\partial \tau} \right| \ &< \ 0.03749149348255419 \\ \nonumber
\end{align}
\begin{align*}
\left| \frac{\partial {\mathcal J}_{12}}{\partial \mu} \right| \ &< \ 0.031242911235461816 \\ \nonumber
\left| \frac{\partial {\mathcal J}_{12}}{\partial \omega} \right| \ &< \ 0.031242911235461816 \\ \nonumber
\left| \frac{\partial {\mathcal J}_{12}}{\partial \nu} \right| \ &< \ 0.21232788238624354 \\ \nonumber
\left| \frac{\partial {\mathcal J}_{12}}{\partial \tau} \right| \ &< \ 0.2124377655377270 \\ \nonumber
\end{align*}
\begin{align*}
\left| \frac{\partial {\mathcal J}_{21}}{\partial \mu} \right| \ &< \  0.02220441024325437 \\ \nonumber
\left| \frac{\partial {\mathcal J}_{21}}{\partial \omega} \right| \ &< \  1.146955401845684 \\ \nonumber
\left| \frac{\partial {\mathcal J}_{21}}{\partial \nu} \right| \ &< \ 0.14983446469110305 \\ \nonumber
\left| \frac{\partial {\mathcal J}_{21}}{\partial \tau} \right| \ &< \ 0.17980135762932363 \\ \nonumber
\end{align*}
\begin{align*}
\left| \frac{\partial {\mathcal J}_{22}}{\partial \mu} \right| \ &< \ 0.14983446469110305 \\ \nonumber
\left| \frac{\partial {\mathcal J}_{22}}{\partial \omega} \right| \ &< \ 0.14983446469110305 \\ \nonumber
\left| \frac{\partial {\mathcal J}_{22}}{\partial \nu} \right| \ &< \ 1.805740052651535 \\ \nonumber
\left| \frac{\partial {\mathcal J}_{22}}{\partial \tau} \right| \ &< \ 2.396685907216327
\end{align*}
\end{lemma}

\begin{proof}
For each derivative we compute a lower and an upper bound and take the
maximum of the absolute value. 
A lower bound is determined by minimizing the single terms of the
functions that represents the derivative. An upper bound is determined
by maximizing the single terms of the functions that represent the
derivative. Terms can be combined to larger terms for which
the maximum and the minimum must be known. We apply many previous lemmata
which state properties of functions representing single or combined
terms. The more terms are combined, the tighter the bounds can be
made. 

Next we go through all the derivatives, where we use 
Lemma~\ref{lem:x11}, 
Lemma~\ref{lem:x22}, 
Lemma~\ref{lem:xx1}, 
Lemma~\ref{lem:xx2}, 
Lemma~\ref{lem:xx3},
Lemma~\ref{lem:xx4},
Lemma~\ref{lem:basics}, and
Lemma~\ref{lem:exerfc} without citing. Furthermore, we use the bounds on the simple
expressions $t_{11}$,$t_{22}$, ..., and $T_4$ as defined the aforementioned lemmata:
\begin{itemize}
\item $\frac{\partial {\mathcal J}_{11}}{\partial \mu}$

We use Lemma~\ref{lem:F1} and
consider the expression 
$\alpha  e^{\frac{(\mu \omega+\nu \tau)^2}{2 \nu \tau}} \erfc \left(\frac{\mu \omega+\nu \tau}{\sqrt{2} \sqrt{\nu \tau}}\right)-\frac{\sqrt{\frac{2}{\pi }} (\alpha -1)}{\sqrt{\nu \tau}}$
in brackets.
An upper bound on the maximum of is
\begin{align}
\alpha_{\rm 01} e^{t_1^2}
  \erfc (t_1)-\frac{\sqrt{\frac{2}{\pi }}
  (\alpha_{\rm 01}-1)}{\sqrt{T_{22}}} \ = \ 0.591017 \ .
\end{align}
A lower bound on the minimum is
\begin{align}
\alpha_{\rm 01} e^{T_1^2}
  \erfc (T_1)-\frac{\sqrt{\frac{2}{\pi }}
  (\alpha_{\rm 01}-1)}{\sqrt{t_{22}}} \ = \ 0.056318 \ .
\end{align}
Thus, an upper bound on the maximal absolute value is 
\begin{align}
\frac{1}{2} \lambda_{\rm 01} \omega_{\rm max}^2 e^{t_4}
  \left(\alpha_{\rm 01} e^{t_1^2}
  \erfc (t_1)-\frac{\sqrt{\frac{2}{\pi }}
  (\alpha_{\rm 01}-1)}{\sqrt{T_{22}}}\right) \ = \
0.0031049101995398316\ .
\end{align}

\item $\frac{\partial {\mathcal J}_{11}}{\partial \omega}$

We use Lemma~\ref{lem:F1} and
consider the expression 
$\frac{\sqrt{\frac{2}{\pi }} (\alpha -1) \mu \omega}{\sqrt{\nu \tau}}-\alpha  (\mu \omega+1) e^{\frac{(\mu \omega+\nu \tau)^2}{2 \nu \tau}} \erfc \left(\frac{\mu \omega+\nu \tau}{\sqrt{2} \sqrt{\nu \tau}}\right)$
in brackets.

An upper bound on the maximum is
\begin{align}
\frac{\sqrt{\frac{2}{\pi }} (\alpha_{\rm 01}-1)
  T_{11}}{\sqrt{t_{22}}}-\alpha_{\rm 01}
  (t_{11}+1) e^{T_1^2} \erfc (T_1) \ = \ 
-0.713808\ .
\end{align}
A lower bound on the minimum is
\begin{align}
\frac{\sqrt{\frac{2}{\pi }} (\alpha_{\rm 01}-1)
  t_{11}}{\sqrt{t_{22}}}-\alpha_{\rm 01}
  (T_{11}+1) e^{t_1^2} \erfc (t_1) 
\ = \ 
-0.99987\ .
\end{align}
This term is subtracted, and $2-\erfc (x)>0$, therefore we have
to use the minimum and the maximum for the argument of $\erfc $.

Thus, an upper bound on the maximal absolute value is 
\begin{align}
&\frac{1}{2} \lambda_{\rm 01} \left(-e^{t_4}
  \left(\frac{\sqrt{\frac{2}{\pi }} (\alpha_{\rm 01}-1)
  t_{11}}{\sqrt{t_{22}}}-\alpha_{\rm 01}
  (T_{11}+1) e^{t_1^2}
  \erfc (t_1)\right)\ -  \erfc (T_3)+2\right)
  \ = \ \\ \nonumber
&1.055872374194189\ .
\end{align}

\item $\frac{\partial {\mathcal J}_{11}}{\partial \nu}$

We consider the term in brackets
\begin{align}
\alpha  e^{\frac{(\mu \omega+\nu \tau)^2}{2 \nu
  \tau}} \erfc \left(\frac{\mu \omega+\nu
  \tau}{\sqrt{2} \sqrt{\nu
  \tau}}\right)+\sqrt{\frac{2}{\pi }} \left(\frac{(\alpha -1)
  \mu \omega}{(\nu \tau)^{3/2}}-\frac{\alpha
  }{\sqrt{\nu \tau}}\right) \ .
\end{align}

We apply Lemma~\ref{lem:F3} for the first sub-term.
An upper bound on the maximum is
\begin{align}
\alpha_{\rm 01} e^{t_1^2}
  \erfc (t_1)+\sqrt{\frac{2}{\pi }}
  \left(\frac{(\alpha_{\rm 01}-1)
  T_{11} }{T_{22}^{3/2}}-\frac{\alpha_{\rm 01} }{\sqrt{T_{22}}}\right) \ = \ 
0.0104167\ .
\end{align}
A lower bound on the minimum is
\begin{align}
  \alpha_{\rm 01} e^{T_1^2}
  \erfc (T_1)+\sqrt{\frac{2}{\pi }}
  \left(\frac{(\alpha_{\rm 01}-1)
  t_{11}}{t_{22}^{3/2}}-\frac{\alpha_{\rm 01}}{\sqrt{t_{22}}}\right) \ = \
-0.95153\ .
\end{align}

Thus, an upper bound on the maximal absolute value is 
\begin{align}
&-\frac{1}{4} \lambda_{\rm 01} \tau_{\rm max} \omega_{\rm max}
  e^{t_4} \left(\alpha_{\rm 01} e^{T_1^2}
  \erfc (T_1)+\sqrt{\frac{2}{\pi }}
  \left(\frac{(\alpha_{\rm 01}-1)
  t_{11}}{t_{22}^{3/2}}-\frac{\alpha_{\rm 01}}{\sqrt{t_{22}}}\right)\right) \ = \\ \nonumber
&0.031242911235461816\ .
\end{align}

\item $\frac{\partial {\mathcal J}_{11}}{\partial \tau}$

We use the results of  item $\frac{\partial {\mathcal J}_{11}}{\partial \nu}$
were the brackets are only differently scaled.
Thus, an upper bound on the maximal absolute value is 
\begin{align}
&-\frac{1}{4} \lambda_{\rm 01} \nu_{\rm max} \omega_{\rm max}
  e^{t_4} \left(\alpha_{\rm 01} e^{T_1^2}
  \erfc (T_1)+\sqrt{\frac{2}{\pi }}
  \left(\frac{(\alpha_{\rm 01}-1)
  t_{11}}{t_{22}^{3/2}}-\frac{\alpha_{\rm 01}}{\sqrt{t_{22}}}\right)\right) \ = \\ \nonumber
& 0.03749149348255419\ .
\end{align}

\item $\frac{\partial {\mathcal J}_{12}}{\partial \mu}$

Since $\frac{\partial {\mathcal J}_{12}}{\partial \mu}=\frac{\partial {\mathcal J}_{11}}{\partial \nu}$,
an upper bound on the maximal absolute value is 
\begin{align}
&-\frac{1}{4} \lambda_{\rm 01} \tau_{\rm max} \omega_{\rm max}
  e^{t_4} \left(\alpha_{\rm 01} e^{T_1^2}
  \erfc (T_1)+\sqrt{\frac{2}{\pi }}
  \left(\frac{(\alpha_{\rm 01}-1)
  t_{11}}{t_{22}^{3/2}}-\frac{\alpha_{\rm 01}}{\sqrt{t_{22}}}\right)\right) \ = \\ \nonumber
& 0.031242911235461816\ .
\end{align}

\item $\frac{\partial {\mathcal J}_{12}}{\partial \omega}$

We use the results of  item $\frac{\partial {\mathcal J}_{11}}{\partial \nu}$
were the brackets are only differently scaled.
Thus, an upper bound on the maximal absolute value is 
\begin{align}
&-\frac{1}{4} \lambda_{\rm 01} \mu_{\rm max} \tau_{\rm max}
  e^{t_4} \left(\alpha_{\rm 01} e^{T_1^2}
  \erfc (T_1)+\sqrt{\frac{2}{\pi }}
  \left(\frac{(\alpha_{\rm 01}-1)
  t_{11}}{t_{22}^{3/2}}-\frac{\alpha_{\rm 01}}{\sqrt{t_{22}}}\right)\right) \ = \\ \nonumber
& 0.031242911235461816\ .
\end{align}

\item $\frac{\partial {\mathcal J}_{12}}{\partial \nu}$

For the second term in brackets, we see that
$\alpha_{\rm 01} \tau_{\rm min}^2 e^{T_1^2} \erfc (T_1)=0.465793$ and $\alpha_{\rm 01} \tau_{\rm max}^2 e^{t_1^2} \erfc (t_1)=1.53644$.

We now check different values for 
\begin{align}
\sqrt{\frac{2}{\pi }} \left(\frac{(-1) (\alpha -1) \mu^2
  \omega^2}{\nu^{5/2} \sqrt{\tau}}+\frac{\sqrt{\tau}
  (\alpha +\alpha  \mu \omega-1)}{\nu^{3/2}}-\frac{\alpha
  \tau^{3/2}}{\sqrt{\nu}}\right) \ ,
\end{align}
where we maximize or minimize all single terms.

A lower bound on the minimum of this expression is
\begin{align}
  &\sqrt{\frac{2}{\pi }} \left(\frac{(-1) (\alpha_{\rm 01}-1)
  \mu_{\rm max}^2 \omega_{\rm max}^2}{\nu_{\rm min}^{5/2}
  \sqrt{\tau_{\rm min}}}
  +\frac{\sqrt{\tau_{\rm min}} (\alpha_{\rm 01}+\alpha_{\rm 01}
  t_{11}-1)}{\nu_{\rm max}^{3/2}}
  -\frac{\alpha_{\rm 01} \tau_{\rm max}^{3/2}}{\sqrt{\nu_{\rm min}}}\right) \ = \\ \nonumber
  & -1.83112\ .
\end{align}
An upper bound on the maximum of this expression is
\begin{align}
  &\sqrt{\frac{2}{\pi }} \left(\frac{(-1) (\alpha_{\rm 01}-1)
  \mu_{\rm min}^2 \omega_{\rm min}^2}{\nu_{\rm max}^{5/2}
  \sqrt{\tau_{\rm max}}} +
  \frac{\sqrt{\tau_{\rm max}}(\alpha_{\rm 01}+\alpha_{\rm 01} T_{11}-1)}{\nu_{\rm min}^{3/2}} -
  \frac{\alpha_{\rm 01}\tau_{\rm min}^{3/2}}{\sqrt{\nu_{\rm max}}} \right)
  \ = \\ \nonumber
  & 0.0802158\ .
\end{align}

An upper bound on the maximum is
\begin{align}
&\frac{1}{8} \lambda_{\rm 01} e^{t_4}
  \left(\sqrt{\frac{2}{\pi }} \left(\frac{(-1) (\alpha_{\rm 01}-1)
  \mu_{\rm min}^2 \omega_{\rm min}^2}{\nu_{\rm max}^{5/2}
  \sqrt{\tau_{\rm max}}}-\frac{\alpha_{\rm 01}
  \tau_{\rm min}^{3/2}}{\sqrt{\nu_{\rm max}}}\ + \right. \right. \\
  \nonumber &\left. \left.  \frac{\sqrt{\tau_{\rm max}}
  (\alpha_{\rm 01}+\alpha_{\rm 01}
  T_{11}-1)}{\nu_{\rm min}^{3/2}}\right)+\alpha_{\rm 01}
  \tau_{\rm max}^2 e^{t_1^2} \erfc (t_1)\right)
  \ = \ 
0.212328\ .
\end{align}
A lower bound on the minimum is
\begin{align}
&\frac{1}{8} \lambda_{\rm 01} e^{t_4} \left(\alpha_{\rm 01} \tau_{\rm min}^2 e^{T_1^2}
  \erfc (T_1)\ + \right.  \\
  \nonumber &\left. \sqrt{\frac{2}{\pi }} \left(\frac{(-1)
  (\alpha_{\rm 01}-1) \mu_{\rm max}^2
  \omega_{\rm max}^2}{\nu_{\rm min}^{5/2}
  \sqrt{\tau_{\rm min}}}+\frac{\sqrt{\tau_{\rm min}} (\alpha_{\rm 01}+\alpha_{\rm 01}
  t_{11}-1)}{\nu_{\rm max}^{3/2}}-\frac{\alpha_{\rm 01}
  \tau_{\rm max}^{3/2}}{\sqrt{\nu_{\rm min}}}\right)\right) \ = \\ \nonumber 
& -0.179318\ .
\end{align}
Thus, an upper bound on the maximal absolute value is 
\begin{align}
&\frac{1}{8} \lambda_{\rm 01} e^{t_4}
  \left(\sqrt{\frac{2}{\pi }} \left(\frac{(-1) (\alpha_{\rm 01}-1)
  \mu_{\rm min}^2 \omega_{\rm min}^2}{\nu_{\rm max}^{5/2}
  \sqrt{\tau_{\rm max}}}-\frac{\alpha_{\rm 01}
  \tau_{\rm min}^{3/2}}{\sqrt{\nu_{\rm max}}}\ + \right. \right. \\
  \nonumber &\left. \left.  \frac{\sqrt{\tau_{\rm max}}
  (\alpha_{\rm 01}+\alpha_{\rm 01}
  T_{11}-1)}{\nu_{\rm min}^{3/2}}\right)+\alpha_{\rm 01}
  \tau_{\rm max}^2 e^{t_1^2} \erfc (t_1)\right)
  \ = \ 0.21232788238624354 \ .
\end{align}

\item $\frac{\partial {\mathcal J}_{12}}{\partial \tau}$

We use Lemma~\ref{lem:F4} to obtain
an upper bound on the maximum of the expression of the lemma:
\begin{align}
\sqrt{\frac{2}{\pi }} \left(\frac{0.1^2 \cdot 0.1^2 (-1) (\alpha_{\rm 01}-1)}{(0.8 \cdot 0.8)^{3/2}}-\sqrt{0.8 \cdot 0.8} \alpha_{\rm 01}+
 \frac{(0.1 \cdot 0.1) \alpha_{\rm 01}-\alpha_{\rm 01}+1}{\sqrt{0.8 \cdot 0.8}}\right) \ = \
-1.72296 \ .
\end{align}
We use Lemma~\ref{lem:F4} to obtain
an lower bound on the minimum of the expression of the lemma:
\begin{align}
\sqrt{\frac{2}{\pi }} \left(\frac{0.1^2 \cdot 0.1^2 (-1) (\alpha_{\rm 01}-1)}{(1.5 \cdot 1.25)^{3/2}}-
 \sqrt{1.5 \cdot 1.25} \alpha_{\rm 01}+\frac{(-0.1 \cdot 0.1) \alpha_{\rm 01}-
 \alpha_{\rm 01}+1}{\sqrt{1.5 \cdot 1.25}}\right) \ = \ 
-2.2302 \ .
\end{align}

Next we apply Lemma~\ref{lem:F7} for the expression $\nu \tau e^{\frac{(\mu \omega+\nu \tau)^2}{2 \nu \tau}} \erfc \left(\frac{\mu \omega+\nu \tau}{\sqrt{2} \sqrt{\nu \tau}}\right)$.
We use Lemma~\ref{lem:F7} to obtain
an upper bound on the maximum of this expression:
\begin{align}
1.5 \cdot 1.25 e^{\frac{(1.5 \cdot 1.25-0.1 \cdot 0.1)^2}{2 \cdot 1.5 \cdot 1.25}} \alpha_{\rm 01}
  \erfc \left(\frac{1.5 \cdot 1.25-0.1 \cdot 0.1}{\sqrt{2} \sqrt{1.5 \cdot 
  1.25}}\right) \ = \
1.37381 \ .
\end{align}
We use Lemma~\ref{lem:F7} to obtain
an lower bound on the minimum of this expression:
\begin{align}
0.8 \cdot 0.8 e^{\frac{(0.8 \cdot 0.8 + 0.1 \cdot 0.1)^2}{2 \cdot 0.8 \cdot 0.8}} \alpha_{\rm 01} \erfc \left(\frac{0.8 \cdot 0.8 + 0.1 \cdot 0.1}{\sqrt{2} \sqrt{0.8 \cdot 0.8}}\right) \ = \
0.620462\ .
\end{align}

Next we apply Lemma~\ref{lem:exerfc} for $2 \alpha  e^{\frac{(\mu \omega+\nu \tau)^2}{2 \nu \tau}} \erfc \left(\frac{\mu \omega+\nu \tau}{\sqrt{2} \sqrt{\nu \tau}}\right)$.
An upper bound on this expression is
\begin{align}
2 e^{\frac{(0.8 \cdot 0.8-0.1 \cdot 0.1)^2}{2 0.8 \cdot 0.8}} \alpha_{\rm 01}
  \erfc \left(\frac{0.8 \cdot 0.8-0.1 \cdot 0.1}{\sqrt{2} 
\sqrt{0.8 \cdot 0.8}}\right) \ = \ 1.96664 \ .
\end{align}
A lower bound on this expression is
\begin{align}
2 e^{\frac{(1.5 \cdot 1.25+0.1 \cdot 0.1)^2}{2 \cdot 1.5 \cdot 1.25}} \alpha_{\rm 01}
  \erfc \left(\frac{1.5 \cdot 1.25+0.1 \cdot 0.1}{\sqrt{2} \sqrt{1.5 \cdot 
  1.25}}\right) \ = \ 1.4556 \ .
\end{align}

The sum of the minimal values of the terms is
$-2.23019 + 0.62046 + 1.45560 = -0.154133$.

The sum of the maximal values of the terms is
$-1.72295 + 1.37380 + 1.96664 = 1.61749$.

Thus, an upper bound on the maximal absolute value is 
\begin{align}
&\frac{1}{8} \lambda_{\rm 01} e^{t_4} \left(\alpha_{\rm 01} T_{22} e^{\frac{(t_{11}+T_{22})^2}{2 T_{22}}}
  \erfc \left(\frac{t_{11}+T_{22}}{\sqrt{2}
  \sqrt{T_{22}}}\right)\ + \right. \\
  \nonumber &\left.   2 \alpha_{\rm 01}
  e^{t_1^2} \erfc (t_1)+\sqrt{\frac{2}{\pi }}
  \left(-\frac{(\alpha_{\rm 01}-1)
  T_{11}^2}{t_{22}^{3/2}}+\frac{-\alpha_{\rm 01} + \alpha_{\rm 01}
  T_{11}+1}{\sqrt{t_{22}}}\ - \right. \right. \\
  \nonumber &\left.\left.  \alpha_{\rm 01}
  \sqrt{t_{22}}\right)\right) \ = \ 0.2124377655377270\ .
\end{align}

\item $\frac{\partial {\mathcal J}_{21}}{\partial \mu}$

An upper bound on the maximum is
\begin{align}
&\lambda_{\rm 01}^2 \omega_{\rm max}^2 \left(\alpha_{\rm 01}^2
  e^{T_1^2} \left(-e^{-T_4}\right)
  \erfc (T_1)+2 \alpha_{\rm 01}^2
  e^{t_2^2} e^{t_4}
  \erfc (t_2)\ - \erfc (T_3)+2\right) \ = \ \\ \nonumber
& 0.0222044\ .
\end{align}
A upper bound on the absolute minimum is
\begin{align}
&\lambda_{\rm 01}^2 \omega_{\rm max}^2 \left(\alpha_{\rm 01}^2
  e^{t_1^2} \left(-e^{-t_4}\right)
  \erfc (t_1)+2 \alpha_{\rm 01}^2
  e^{T_2^2} e^{T_4}
  \erfc (T_2)\ -  \erfc (t_3)+2\right) \ = \  \\ \nonumber
& 0.00894889\ .
\end{align}
Thus, an upper bound on the maximal absolute value is 
\begin{align}
&\lambda_{\rm 01}^2 \omega_{\rm max}^2 \left(\alpha_{\rm 01}^2
  e^{T_1^2} \left(-e^{-T_4}\right)
  \erfc (T_1)+2 \alpha_{\rm 01}^2
  e^{t_2^2} e^{t_4}
  \erfc (t_2)\ - \erfc (T_3)+2\right) \ = \ \\ \nonumber
&0.02220441024325437 \ .
\end{align}

\item $\frac{\partial {\mathcal J}_{21}}{\partial \omega}$

An upper bound on the maximum is
\begin{align}
&\lambda_{\rm 01}^2 \left(\alpha_{\rm 01}^2 (2 T_{11}+1)
  e^{t_2^2} e^{-t_4} \erfc (t_2)+2
  T_{11} (2-\erfc (T_3))\ + \right. \\
  \nonumber &\left. \alpha_{\rm 01}^2
  (t_{11}+1) e^{T_1^2}
  \left(-e^{-T_4}\right)
  \erfc (T_1)+\sqrt{\frac{2}{\pi }}
  \sqrt{T_{22}} e^{-t_4}\right) \ = \
1.14696 \ .
\end{align}
A lower bound on the minimum is
\begin{align}
&\lambda_{\rm 01}^2 \left(\alpha_{\rm 01}^2 (T_{11}+1)
  e^{t_1^2} \left(-e^{-t_4}\right)
  \erfc (t_1)\ + \right. \\
  \nonumber &\left. \alpha_{\rm 01}^2 (2 t_{11}+1)
  e^{T_2^2} e^{-T_4} \erfc (T_2)+2
  t_{11} (2-\erfc (T_3))+ \right.\\  \nonumber 
&\left. \sqrt{\frac{2}{\pi }}
  \sqrt{t_{22}} e^{-T_4}\right) \ = \
-0.359403\ .
\end{align}
Thus, an upper bound on the maximal absolute value is 
\begin{align}
&\lambda_{\rm 01}^2 \left(\alpha_{\rm 01}^2 (2 T_{11}+1)
  e^{t_2^2} e^{-t_4} \erfc (t_2)+2
  T_{11} (2-\erfc (T_3))\ + \right. \\
  \nonumber &\left. \alpha_{\rm 01}^2
  (t_{11}+1) e^{T_1^2}
  \left(-e^{-T_4}\right)
  \erfc (T_1)+\sqrt{\frac{2}{\pi }}
  \sqrt{T_{22}} e^{-t_4}\right) \ = \
1.146955401845684 \ .
\end{align}

\item $\frac{\partial {\mathcal J}_{21}}{\partial \nu}$

An upper bound on the maximum is
\begin{align}
&\frac{1}{2} \lambda_{\rm 01}^2 \tau_{\rm max} \omega_{\rm max}
  e^{-t_4} \left(\alpha_{\rm 01}^2
  \left(-e^{T_1^2}\right) \erfc (T_1)+4
  \alpha_{\rm 01}^2 e^{t_2^2}
  \erfc (t_2)\ +  \frac{\sqrt{\frac{2}{\pi }} (-1)
  \left(\alpha_{\rm 01}^2-1\right)}{\sqrt{T_{22}}}\right) = \ \\ \nonumber
   & 0.149834\ .
\end{align}
A lower bound on the minimum is
\begin{align}
&\frac{1}{2} \lambda_{\rm 01}^2 \tau_{\rm max} \omega_{\rm max}
  e^{-t_4} \left(\alpha_{\rm 01}^2
  \left(-e^{t_1^2}\right) \erfc (t_1)+4
  \alpha_{\rm 01}^2 e^{T_2^2}
  \erfc (T_2)\ +  \frac{\sqrt{\frac{2}{\pi }} (-1)
  \left(\alpha_{\rm 01}^2-1\right)}{\sqrt{t_{22}}}\right) = \\ \nonumber
  &-0.0351035\ .
\end{align}
Thus, an upper bound on the maximal absolute value is 
\begin{align}
&\frac{1}{2} \lambda_{\rm 01}^2 \tau_{\rm max} \omega_{\rm max}
  e^{-t_4} \left(\alpha_{\rm 01}^2
  \left(-e^{T_1^2}\right) \erfc (T_1)+4
  \alpha_{\rm 01}^2 e^{t_2^2}
  \erfc (t_2)\ + \frac{\sqrt{\frac{2}{\pi }} (-1)
  \left(\alpha_{\rm 01}^2-1\right)}{\sqrt{T_{22}}}\right) \ 
  = \ \\ \nonumber
&0.14983446469110305 \ .
\end{align}

\item $\frac{\partial {\mathcal J}_{21}}{\partial \tau}$

An upper bound on the maximum is
\begin{align}
&\frac{1}{2} \lambda_{\rm 01}^2 \nu_{\rm max} \omega_{\rm max}
  e^{-t_4} \left(\alpha_{\rm 01}^2
  \left(-e^{T_1^2}\right) \erfc (T_1)+4
  \alpha_{\rm 01}^2 e^{t_2^2}
  \erfc (t_2)\ + \frac{\sqrt{\frac{2}{\pi }} (-1)
  \left(\alpha_{\rm 01}^2-1\right)}{\sqrt{T_{22}}}\right) \
  = \ \\ \nonumber & 0.179801 \ .
\end{align}
A lower bound on the minimum is
\begin{align}
&\frac{1}{2} \lambda_{\rm 01}^2 \nu_{\rm max} \omega_{\rm max}
  e^{-t_4} \left(\alpha_{\rm 01}^2
  \left(-e^{t_1^2}\right) \erfc (t_1)+4
  \alpha_{\rm 01}^2 e^{T_2^2}
  \erfc (T_2)\ + \frac{\sqrt{\frac{2}{\pi }} (-1)
  \left(\alpha_{\rm 01}^2-1\right)}{\sqrt{t_{22}}}\right) \
  = \ \\ \nonumber
&-0.0421242 \ .
\end{align}
Thus, an upper bound on the maximal absolute value is 
\begin{align}
&\frac{1}{2} \lambda_{\rm 01}^2 \nu_{\rm max} \omega_{\rm max}
  e^{-t_4} \left(\alpha_{\rm 01}^2
  \left(-e^{T_1^2}\right) \erfc (T_1)+4
  \alpha_{\rm 01}^2 e^{t_2^2}
  \erfc (t_2)\ +  \frac{\sqrt{\frac{2}{\pi }} (-1)
  \left(\alpha_{\rm 01}^2-1\right)}{\sqrt{T_{22}}}\right) \
  = \ \\ \nonumber 
&0.17980135762932363 \ .
\end{align}

\item $\frac{\partial {\mathcal J}_{22}}{\partial \mu}$

We use the fact that $\frac{\partial {\mathcal J}_{22}}{\partial \mu}=\frac{\partial {\mathcal J}_{21}}{\partial \nu}$.
Thus, an upper bound on the maximal absolute value is 
\begin{align}
&\frac{1}{2} \lambda_{\rm 01}^2 \tau_{\rm max} \omega_{\rm max}
  e^{-t_4} \left(\alpha_{\rm 01}^2
  \left(-e^{T_1^2}\right) \erfc (T_1)+4
  \alpha_{\rm 01}^2 e^{t_2^2}
  \erfc (t_2)\ + \frac{\sqrt{\frac{2}{\pi }} (-1)
  \left(\alpha_{\rm 01}^2-1\right)}{\sqrt{T_{22}}}\right) \
  = \ \\ \nonumber
&0.14983446469110305 \ .
\end{align}

\item $\frac{\partial {\mathcal J}_{22}}{\partial \omega}$

An upper bound on the maximum is
\begin{align}
&\frac{1}{2} \lambda_{\rm 01}^2 \mu_{\rm max} \tau_{\rm max}
  e^{-t_4} \left(\alpha_{\rm 01}^2
  \left(-e^{T_1^2}\right) \erfc (T_1)+4
  \alpha_{\rm 01}^2 e^{t_2^2}
  \erfc (t_2)\ +  \frac{\sqrt{\frac{2}{\pi }} (-1)
  \left(\alpha_{\rm 01}^2-1\right)}{\sqrt{T_{22}}}\right) \
  = \ \\ \nonumber
&0.149834 \ .
\end{align}
A lower bound on the minimum is
\begin{align}
&\frac{1}{2} \lambda_{\rm 01}^2 \mu_{\rm max} \tau_{\rm max}
  e^{-t_4} \left(\alpha_{\rm 01}^2
  \left(-e^{t_1^2}\right) \erfc (t_1)+4
  \alpha_{\rm 01}^2 e^{T_2^2}
  \erfc (T_2)\ + \frac{\sqrt{\frac{2}{\pi }} (-1)
  \left(\alpha_{\rm 01}^2-1\right)}{\sqrt{t_{22}}}\right) \
  = \ \\ \nonumber
&-0.0351035 \ .
\end{align}
Thus, an upper bound on the maximal absolute value is 
\begin{align}
&\frac{1}{2} \lambda_{\rm 01}^2 \mu_{\rm max} \tau_{\rm max}
  e^{-t_4} \left(\alpha_{\rm 01}^2
  \left(-e^{T_1^2}\right) \erfc (T_1)+4
  \alpha_{\rm 01}^2 e^{t_2^2}
  \erfc (t_2)\ +  \frac{\sqrt{\frac{2}{\pi }} (-1)
  \left(\alpha_{\rm 01}^2-1\right)}{\sqrt{T_{22}}}\right) \
  = \  \\ \nonumber
&0.14983446469110305 \ .
\end{align}

\item $\frac{\partial {\mathcal J}_{22}}{\partial \nu}$

We apply Lemma~\ref{lem:F5} to the expression
$\sqrt{\frac{2}{\pi }} \left(\frac{\left(\alpha ^2-1\right) \mu \omega}{(\nu \tau)^{3/2}}-\frac{3 \alpha ^2}{\sqrt{\nu \tau}}\right)$.
Using  Lemma~\ref{lem:F5}, an upper bound on the maximum is
\begin{align}
&\frac{1}{4} \lambda_{\rm 01}^2 \tau_{\rm max}^2 e^{-t_4}
  \left(\alpha_{\rm 01}^2 \left(-e^{T_1^2}\right)
  \erfc (T_1)+8 \alpha_{\rm 01}^2
  e^{t_2^2} \erfc (t_2)\ + \right. \\
  \nonumber &\left. \sqrt{\frac{2}{\pi }}
  \left(\frac{\left(\alpha_{\rm 01}^2-1\right)
  T_{11}}{T_{22}^{3/2}}-\frac{3 \alpha_{\rm 01}^2}{\sqrt{T_{22}}}\right)\right) \ = \
1.19441 \ .
\end{align}
Using  Lemma~\ref{lem:F5}, a lower bound on the minimum is
\begin{align}
&\frac{1}{4} \lambda_{\rm 01}^2 \tau_{\rm max}^2 e^{-t_4}
  \left(\alpha_{\rm 01}^2 \left(-e^{t_1^2}\right)
  \erfc (t_1)+8 \alpha_{\rm 01}^2
  e^{T_2^2} \erfc (T_2)\ + \right. \\
  \nonumber &\left. \sqrt{\frac{2}{\pi }}
  \left(\frac{\left(\alpha_{\rm 01}^2-1\right)
  t_{11}}{t_{22}^{3/2}}-\frac{3 \alpha_{\rm 01}^2}{\sqrt{t_{22}}}\right)\right) \ = \
-1.80574 \ .
\end{align}
Thus, an upper bound on the maximal absolute value is 
\begin{align}
&-\frac{1}{4} \lambda_{\rm 01}^2 \tau_{\rm max}^2 e^{-t_4}
  \left(\alpha_{\rm 01}^2 \left(-e^{t_1^2}\right)
  \erfc (t_1)+8 \alpha_{\rm 01}^2
  e^{T_2^2} \erfc (T_2)\ + \right. \\
  \nonumber &\left. \sqrt{\frac{2}{\pi }}
  \left(\frac{\left(\alpha_{\rm 01}^2-1\right)
  t_{11}}{t_{22}^{3/2}}-\frac{3 \alpha_{\rm 01}^2}{\sqrt{t_{22}}}\right)\right) \ = \
1.805740052651535 \ .
\end{align}

\item $\frac{\partial {\mathcal J}_{22}}{\partial \tau}$

We apply Lemma~\ref{lem:F6} to the expression
$\sqrt{\frac{2}{\pi }} \left(\frac{\left(\alpha ^2-1\right) \mu
    \omega}{\sqrt{\nu \tau}}-3 \alpha ^2 \sqrt{\nu
    \tau}\right)$. \\
We apply Lemma~\ref{lem:F7} to the expression 
$\nu \tau e^{\frac{(\mu \omega+\nu \tau)^2}{2 \nu \tau}} \erfc \left(\frac{\mu \omega+\nu \tau}{\sqrt{2} \sqrt{\nu \tau}}\right)$. 
We apply Lemma~\ref{lem:F8} to the expression
$\nu \tau e^{\frac{(\mu \omega+2 \nu \tau)^2}{2 \nu \tau}} \erfc \left(\frac{\mu \omega+2 \nu \tau}{\sqrt{2} \sqrt{\nu \tau}}\right)$. 

We combine the results of these lemmata to obtain 
an upper bound on the maximum:
\begin{align}
&\frac{1}{4} \lambda_{\rm 01}^2 \left(-\alpha_{\rm 01}^2
  t_{22} e^{-T_4}
  e^{\frac{(T_{11}+t_{22})^2}{2 t_{22}}}
  \erfc \left(\frac{T_{11}+t_{22}}{\sqrt{2}
  \sqrt{t_{22}}}\right)\ + \right. \\
  \nonumber &\left. 8 \alpha_{\rm 01}^2 T_{22}
  e^{-t_4} e^{\frac{(t_{11}+2 T_{22})^2}{2
  T_{22}}} \erfc \left(\frac{t_{11}+2
  T_{22}}{\sqrt{2} \sqrt{T_{22}}}\right)\ - \right. \\
  \nonumber &\left. 2 \alpha_{\rm 01}^2 e^{T_1^2} e^{-T_4}
  \erfc (T_1)+4 \alpha_{\rm 01}^2
  e^{t_2^2} e^{-t_4} \erfc (t_2)+2
  (2-\erfc (T_3))\ + \right. \\
  \nonumber &\left. \sqrt{\frac{2}{\pi }}
  e^{-T_4} \left(\frac{\left(\alpha_{\rm 01}^2-1\right)
  T_{11}}{\sqrt{t_{22}}}-3 \alpha_{\rm 01}^2
  \sqrt{t_{22}}\right)\right) \ = \ 
2.39669 \ .
\end{align}
We combine the results of these lemmata to obtain 
an lower bound on the minimum:
\begin{align}
&\frac{1}{4} \lambda_{\rm 01}^2 \left(8 \alpha_{\rm 01}^2
  t_{22} e^{-T_4} e^{\frac{(T_{11}+2
  t_{22})^2}{2 t_{22}}}
  \erfc \left(\frac{T_{11}+2 t_{22}}{\sqrt{2}
  \sqrt{t_{22}}}\right)\ + \right. \\
  \nonumber &\left. \alpha_{\rm 01}^2 T_{22}
  e^{-t_4} e^{\frac{(t_{11}+T_{22})^2}{2
  T_{22}}}
  \erfc \left(\frac{t_{11}+T_{22}}{\sqrt{2}
  \sqrt{T_{22}}}\right)\ - \right. \\
  \nonumber &\left. 2 \alpha_{\rm 01}^2
  e^{t_1^2} e^{-t_4} \erfc (t_1)+4
  \alpha_{\rm 01}^2 e^{T_2^2} e^{-T_4}
  \erfc (T_2)\ + \right. \\
  \nonumber &\left. 2
  (2-\erfc (t_3))+\sqrt{\frac{2}{\pi }}
  e^{-t_4} \left(\frac{\left(\alpha_{\rm 01}^2-1\right)
  t_{11}}{\sqrt{T_{22}}}-3 \alpha_{\rm 01}^2
  \sqrt{T_{22}}\right)\right) \ = \
-1.17154 \ .
\end{align}
Thus, an upper bound on the maximal absolute value is 
\begin{align}
&\frac{1}{4} \lambda_{\rm 01}^2 \left(-\alpha_{\rm 01}^2
  t_{22} e^{-T_4}
  e^{\frac{(T_{11}+t_{22})^2}{2 t_{22}}}
  \erfc \left(\frac{T_{11}+t_{22}}{\sqrt{2}
  \sqrt{t_{22}}}\right)\ + \right. \\
  \nonumber &\left. 8 \alpha_{\rm 01}^2 T_{22}
  e^{-t_4} e^{\frac{(t_{11}+2 T_{22})^2}{2
  T_{22}}} \erfc \left(\frac{t_{11}+2
  T_{22}}{\sqrt{2} \sqrt{T_{22}}}\right)\ - \right. \\
  \nonumber &\left. 2 \alpha_{\rm 01}^2 e^{T_1^2} e^{-T_4}
  \erfc (T_1)+4 \alpha_{\rm 01}^2
  e^{t_2^2} e^{-t_4} \erfc (t_2)+2
  (2-\erfc (T_3))\ + \right. \\
  \nonumber &\left. \sqrt{\frac{2}{\pi }}
  e^{-T_4} \left(\frac{\left(\alpha_{\rm 01}^2-1\right)
  T_{11}}{\sqrt{t_{22}}}-3 \alpha_{\rm 01}^2
  \sqrt{t_{22}}\right)\right) \ = \ 
2.396685907216327\ .
\end{align}
\end{itemize}
\end{proof}

\begin{lemma}[Derivatives of the Mapping]
\label{proof:mapDerivatives}
We assume $\alpha = \alpha_{\rm 01}$ and $\lambda=\lambda_{\rm 01}$.
We restrict the range of the variables to the domain
$\mu \in [-0.1,0.1]$,
$\omega \in [-0.1,0.1]$,
$\nu \in [0.8,1.5]$, and
$\tau \in [0.8,1.25]$.

The derivative $\frac{\partial }{\partial \mu}
\munn(\mu,\omega,\nu,\tau,\lambda ,\alpha )$
has the sign of $\omega$.

The derivative $\frac{\partial }{\partial \nu}
 \munn(\mu,\omega,\nu,\tau,\lambda ,\alpha )$
is positive.

The derivative $\frac{\partial }{\partial \mu } \xinn(\mu,\omega,\nu,\tau,\lambda ,\alpha )$
has the sign of $\omega$.

The derivative 
$\frac{\partial }{\partial \nu } \xinn(\mu,\omega,\nu,\tau,\lambda ,\alpha )$
is positive.
\end{lemma}

\begin{proof}
\begin{itemize}
\item $\frac{\partial }{\partial \mu}
\munn(\mu,\omega,\nu,\tau,\lambda ,\alpha )$

$(2 - \erfc (x) > 0$ according to 
Lemma~\ref{lem:basics} and $e^{x^2} \erfc (x)$
is also larger than zero according to Lemma~\ref{lem:exerfc}.
Consequently, has $\frac{\partial }{\partial \mu}
\munn(\mu,\omega,\nu,\tau,\lambda ,\alpha )$
the sign of $\omega$.

\item $\frac{\partial }{\partial \nu}
 \munn(\mu,\omega,\nu,\tau,\lambda ,\alpha )$

Lemma~\ref{lem:exerfc} says 
$e^{x^2} \erfc (x)$ is decreasing in $\frac{\mu \omega+\nu \tau}{\sqrt{2} \sqrt{\nu \tau}}$.
The first term (negative) is increasing in $\nu\tau$ since it is
proportional to minus
one over the squared root of  $\nu\tau$.

We obtain a lower bound by
setting $\frac{\mu \omega+\nu \tau}{\sqrt{2}
  \sqrt{\nu \tau}}=\frac{1.5 \cdot 1.25+0.1 \cdot 0.1}{\sqrt{2}
  \sqrt{1.5 \cdot 1.25}}$ for the $e^{x^2} \erfc (x)$ term.
The term in brackets is larger than
$e^{\left(\frac{1.5 \cdot 1.25+0.1 \cdot 0.1}{\sqrt{2} \sqrt{1.5 \cdot 1.25}}\right)^2}
\alpha_{\rm 01} \ \erfc \left(\frac{1.5 \cdot 1.25+0.1 \cdot 0.1}{\sqrt{2} \sqrt{1.5 \cdot  1.25}}\right)-\sqrt{\frac{2}{\pi  0.8 \cdot 0.8}} (\alpha_{\rm 01}-1)=0.056$
Consequently, the function is larger than zero.

\item $\frac{\partial }{\partial \mu } \xinn(\mu,\omega,\nu,\tau,\lambda ,\alpha )$

We consider the sub-function
\begin{align}
\sqrt{\frac{2}{\pi }} \sqrt{\nu \tau}-\alpha ^2
  \left(e^{\left(\frac{\mu \omega+\nu \tau}{\sqrt{2}
  \sqrt{\nu \tau}}\right)^2}
  \erfc \left(\frac{\mu \omega+\nu \tau}{\sqrt{2}
  \sqrt{\nu \tau}}\right)-e^{\left(\frac{\mu \omega+2
  \nu \tau}{\sqrt{2} \sqrt{\nu \tau}}\right)^2}
  \erfc \left(\frac{\mu \omega+2 \nu \tau}{\sqrt{2}
  \sqrt{\nu \tau}}\right)\right) \ .
\end{align}
We set $x=\nu\tau$ and $y=\mu \omega$ and obtain
\begin{align}
\sqrt{\frac{2}{\pi }} \sqrt{x}-\alpha ^2
  \left(e^{\left(\frac{x+y}{\sqrt{2} \sqrt{x}}\right)^2}
  \erfc \left(\frac{x+y}{\sqrt{2}
  \sqrt{x}}\right)-e^{\left(\frac{2 x+y}{\sqrt{2} \sqrt{x}}\right)^2}
  \erfc \left(\frac{2 x+y}{\sqrt{2} \sqrt{x}}\right)\right) \ .
\end{align}

The derivative of this sub-function with respect to $y$ is
\begin{align}
&\frac{\alpha ^2 \left(e^{\frac{(2 x+y)^2}{2 x}} (2 x+y)
   \erfc \left(\frac{2 x+y}{\sqrt{2}
   \sqrt{x}}\right)-e^{\frac{(x+y)^2}{2 x}} (x+y)
   \erfc \left(\frac{x+y}{\sqrt{2} \sqrt{x}}\right)\right)}{x} \ =
  \\ \nonumber &\frac{\sqrt{2} \alpha ^2 \sqrt{x} \left(\frac{e^{\frac{(2 x+y)^2}{2 x}} (x+y) \erfc \left(\frac{x+y}{\sqrt{2} \sqrt{x}}\right)}{\sqrt{2} \sqrt{x}}-\frac{e^{\frac{(x+y)^2}{2 x}} (x+y) \erfc \left(\frac{x+y}{\sqrt{2} \sqrt{x}}\right)}{\sqrt{2} \sqrt{x}}\right)}{x}\ >\ 0 \ .
\end{align}

The inequality follows from Lemma~\ref{lem:xeErfc}, which states that 
$z e^{z^2} \erfc (z)$ is monotonically increasing in $z$.
Therefore the sub-function is increasing in $y$. 

The derivative of this sub-function with respect to $x$ is
\begin{align}
& \frac{\sqrt{\pi } \alpha ^2 \left(e^{\frac{(2 x+y)^2}{2 x}} \left(4 x^2-y^2\right) \erfc \left(\frac{2 x+y}{\sqrt{2} \sqrt{x}}\right) 
- e^{\frac{(x+y)^2}{2 x}} (x-y) (x+y) \erfc \left(\frac{x+y}{\sqrt{2} \sqrt{x}}\right)\right)-
\sqrt{2} \left(\alpha ^2-1\right) x^{3/2}}{2 \sqrt{\pi } x^2}\ . 
\end{align}

The sub-function is increasing in $x$, since the
derivative is larger than zero:
\begin{align}
&\frac{\sqrt{\pi } \alpha ^2 \left(e^{\frac{(2 x+y)^2}{2 x}} \left(4 x^2-y^2\right) \erfc \left(\frac{2 x+y}{\sqrt{2} \sqrt{x}}\right)-e^{\frac{(x+y)^2}{2 x}} (x-y) (x+y) \erfc \left(\frac{x+y}{\sqrt{2} \sqrt{x}}\right)\right)-\sqrt{2} x^{3/2} \left(\alpha ^2-1\right)}{2 \sqrt{\pi } x^2}\ \geq \\ \nonumber 
&\frac{\sqrt{\pi } \alpha ^2 \left(\frac{(2 x-y) (2 x+y) 2}{\sqrt{\pi } \left(\frac{2 x+y}{\sqrt{2} \sqrt{x}}+\sqrt{\left(\frac{2 x+y}{\sqrt{2} \sqrt{x}}\right)^2+2}\right)}-\frac{(x-y) (x+y) 2}{\sqrt{\pi } \left(\frac{x+y}{\sqrt{2} \sqrt{x}}+\sqrt{\left(\frac{x+y}{\sqrt{2} \sqrt{x}}\right)^2+\frac{4}{\pi }}\right)}\right)-\sqrt{2} x^{3/2} \left(\alpha ^2-1\right)}{2 \sqrt{\pi } x^2}\ =\\ \nonumber 
&\frac{\sqrt{\pi } \alpha ^2 \left(\frac{(2 x-y) (2 x+y) 2 \left(\sqrt{2} \sqrt{x}\right)}{\sqrt{\pi } \left(2 x+y+\sqrt{(2 x+y)^2+4 x}\right)}-\frac{(x-y) (x+y) 2 \left(\sqrt{2} \sqrt{x}\right)}{\sqrt{\pi } \left(x+y+\sqrt{(x+y)^2+\frac{8 x}{\pi }}\right)}\right)-\sqrt{2} x^{3/2} \left(\alpha ^2-1\right)}{2 \sqrt{\pi } x^2}\ = \\ \nonumber 
&\frac{\sqrt{\pi } \alpha ^2 \left(\frac{(2 x-y) (2 x+y) 2}{\sqrt{\pi } \left(2 x+y+\sqrt{(2 x+y)^2+4 x}\right)}-\frac{(x-y) (x+y) 2}{\sqrt{\pi } \left(x+y+\sqrt{(x+y)^2+\frac{8 x}{\pi }}\right)}\right)-x \left(\alpha ^2-1\right)}{\sqrt{2} \sqrt{\pi } x^{3/2}}\ > \\ \nonumber 
  & \frac{\sqrt{\pi } \alpha ^2 \left(\frac{(2 x-y) (2 x+y) 2}{\sqrt{\pi } \left(2 x+y+\sqrt{(2 x+y)^2+2 (2 x+y)+1}\right)}-\frac{(x-y) (x+y) 2}{\sqrt{\pi } \left(x+y+\sqrt{(x+y)^2+0.782 \cdot 2 (x+y)+0.782^2}\right)}\right)-x \left(\alpha ^2-1\right)}{\sqrt{2} \sqrt{\pi } x^{3/2}}\ =
  \\ \nonumber &\frac{\sqrt{\pi } \alpha ^2 \left(\frac{(2 x-y) (2 x+y) 2}{\sqrt{\pi } \left(2 x+y+\sqrt{(2 x+y+1)^2}\right)}-\frac{(x-y) (x+y) 2}{\sqrt{\pi } \left(x+y+\sqrt{(x+y+0.782)^2}\right)}\right)-x \left(\alpha ^2-1\right)}{\sqrt{2} \sqrt{\pi } x^{3/2}}\ =
  \\ \nonumber 
  &\frac{\sqrt{\pi } \alpha ^2 \left(\frac{(2 x-y) (2 x+y) 2}{\sqrt{\pi } (2 (2 x+y)+1)}-\frac{(x-y) (x+y) 2}{\sqrt{\pi } (2 (x+y)+0.782)}\right)-x \left(\alpha ^2-1\right)}{\sqrt{2} \sqrt{\pi } x^{3/2}}\ = \\ \nonumber 
  &\frac{\sqrt{\pi } \alpha ^2 \left(\frac{(2(x+y)+0.782) (2 x-y) (2 x+y) 2}{\sqrt{\pi
                 }}-\frac{(x-y) (x+y) (2 (2 x+y)+1) 2}{\sqrt{\pi
                 }}\right)}{(2 (2 x+y)+1) (2 (x+y)+0.782) \sqrt{2}
                 \sqrt{\pi } x^{3/2}} \ + \\ \nonumber 
  &\frac{\sqrt{\pi } \alpha ^2 \left(-x \left(\alpha ^2-1\right) (2 (2 x+y)+1) (2 (x+y)+0.782)\right)}{(2 (2 x+y)+1) (2 (x+y)+0.782) \sqrt{2} \sqrt{\pi } x^{3/2}}\ =\\ \nonumber 
 &\frac{8 x^3+(12 y+2.68657) x^2+(y (4 y-6.41452)-1.40745) x+1.22072 y^2}{(2 (2 x+y)+1) (2 (x+y)+0.782) \sqrt{2} \sqrt{\pi } x^{3/2}}\ >
  \\ \nonumber 
  &\frac{8 x^3+(2.68657 -12 0.01) x^2+(0.01 (-6.41452-4 0.01)-1.40745) x+1.22072 (0.0)^2}{(2 (2 x+y)+1) (2 (x+y)+0.782) \sqrt{2} \sqrt{\pi } x^{3/2}}\ =
  \\ \nonumber &\frac{8 x^2+2.56657 x-1.472}{(2 (2 x+y)+1) (2 (x+y)+0.782) \sqrt{2} \sqrt{\pi } \sqrt{x}}\ =
  \\ \nonumber 
  &\frac{8 x^2+2.56657 x-1.472}{(2 (2 x+y)+1) (2 (x+y)+0.782) \sqrt{2} \sqrt{\pi } \sqrt{x}}\ =
  \\ \nonumber &\frac{8 (x+0.618374) (x-0.297553)}{(2 (2 x+y)+1) (2
(x+y)+0.782) \sqrt{2} \sqrt{\pi } \sqrt{x}} \ > \ 0 \ .
\end{align}
We explain this chain of inequalities:
\begin{itemize}
\item First inequality: We applied Lemma~\ref{lem:Abramowitz} two times.

\item Equalities factor out $\sqrt{2} \sqrt{x}$ and reformulate.

\item Second inequality part 1: we applied
\begin{align}
0<2 y\Longrightarrow (2 x+y)^2+4 x+1<(2 x+y)^2+2 (2 x+y)+1=(2 x+y+1)^2
  \ .
\end{align}
\item Second inequality part 2: we show that for $a=\frac{1}{20}
  \left(\sqrt{\frac{2048+169 \pi }{\pi }}-13\right)$ following holds:
$\frac{8 x}{\pi }-\left(a^2+2 a (x+y)\right) \geq 0$. 
We have $\frac{\partial }{\partial x }\frac{8 x}{\pi }-\left(a^2+2 a
  (x+y)\right)=\frac{8}{\pi }-2 a>0$ and
 $\frac{\partial }{\partial y }\frac{8 x}{\pi }-\left(a^2+2 a
  (x+y)\right)=-2 a>0$. 
Therefore the minimum is at border for minimal $x$ and maximal $y$:
\begin{align}
\frac{8 \cdot 0.64}{\pi }-\left(\frac{2}{20} \left(\sqrt{\frac{2048+169 \pi
  }{\pi }}-13\right) (0.64+0.01)+\left(\frac{1}{20}
  \left(\sqrt{\frac{2048+169 \pi }{\pi }}-13\right)\right)^2\right) \
  = \ 0 \ .
\end{align}
Thus
\begin{align}
\frac{8 x}{\pi } \ \geq \ a^2+2 a (x+y) \ .
\end{align}
for $a=\frac{1}{20}
  \left(\sqrt{\frac{2048+169 \pi }{\pi }}-13\right) > 0.782$.

\item Equalities only solve square root and factor out the resulting
  terms $(2(2 x+y)+1)$ and $(2(x+y)+0.782)$.

\item We set $\alpha=\alpha_{\rm 01}$ and multiplied out. Thereafter we
  also factored out $x$ in the numerator. Finally a quadratic
  equations was solved. 
\end{itemize}

The sub-function has its minimal value for 
minimal $x$ and minimal $y$
$x=\nu\tau=0.8 \cdot 0.8=0.64$ and $y=\mu \omega=-0.1 \cdot 0.1=-0.01$. 
We further minimize the function
\begin{align}
\mu \omega e^{\frac{\mu^2 \omega^2}{2 \nu \tau}}
  \left(2-\erfc \left(\frac{\mu \omega}{\sqrt{2} \sqrt{\nu
  \tau}}\right)\right)>-0.01 e^{\frac{0.01^2}{2 0.64}}
  \left(2-\erfc \left(\frac{0.01}{\sqrt{2}
  \sqrt{0.64}}\right)\right) \ .
\end{align}

We compute the minimum of the term in brackets of $\frac{\partial }{\partial \mu } \xinn(\mu,\omega,\nu,\tau,\lambda ,\alpha )$:  
\begin{align}
&\mu \omega e^{\frac{\mu^2 \omega^2}{2 \nu \tau}}
  \left(2-\erfc \left(\frac{\mu \omega}{\sqrt{2} \sqrt{\nu \tau}}\right)\right)+ \\ \nonumber
&\alpha_{\rm 01}^2 \left(-\left(e^{\left(\frac{\mu \omega+\nu \tau}{\sqrt{2} \sqrt{\nu \tau}}\right)^2} 
\erfc \left(\frac{\mu \omega+\nu \tau}{\sqrt{2} \sqrt{\nu \tau}}\right)-e^{\left(\frac{\mu \omega+2 \nu \tau}{\sqrt{2} 
\sqrt{\nu \tau}}\right)^2} \erfc \left(\frac{\mu \omega+2 \nu \tau}{\sqrt{2} \sqrt{\nu \tau}}\right)\right)\right)+ \sqrt{\frac{2}{\pi }} \sqrt{\nu \tau} \ > \\ \nonumber
&\alpha_{\rm 01}^2 \left(-\left(e^{\left(\frac{0.64 -0.01}{\sqrt{2} \sqrt{0.64}}\right)^2} 
\erfc \left(\frac{0.64 -0.01}{\sqrt{2} \sqrt{0.64}}\right)-e^{\left(\frac{2 0.64-0.01}{\sqrt{2} \sqrt{0.64}}\right)^2}
\erfc \left(\frac{2 \cdot 0.64-0.01}{\sqrt{2} \sqrt{0.64}}\right)\right)\right)-\\ \nonumber 
&0.01 e^{\frac{0.01^2}{2 0.64}} \left(2-\erfc \left(\frac{0.01}{\sqrt{2} \sqrt{0.64}}\right)\right)+\sqrt{0.64} \sqrt{\frac{2}{\pi }} \ = \ 0.0923765 \ .
\end{align}
Therefore the term in brackets is larger than zero.

Thus, $\frac{\partial }{\partial \mu }
\xinn(\mu,\omega,\nu,\tau,\lambda ,\alpha )$
has the sign of $\omega$.

\item $\frac{\partial }{\partial \nu } \xinn(\mu,\omega,\nu,\tau,\lambda ,\alpha )$

We look at the sub-term
\begin{align}
2 e^{\left(\frac{2 x+y}{\sqrt{2} \sqrt{x}}\right)^2}
  \erfc \left(\frac{2 x+y}{\sqrt{2}
  \sqrt{x}}\right)-e^{\left(\frac{x+y}{\sqrt{2} \sqrt{x}}\right)^2}
  \erfc \left(\frac{x+y}{\sqrt{2} \sqrt{x}}\right) \ .
\end{align}
We obtain a chain of inequalities:
\begin{align}
&2 e^{\left(\frac{2 x+y}{\sqrt{2} \sqrt{x}}\right)^2} \erfc \left(\frac{2 x+y}{\sqrt{2} \sqrt{x}}\right)-e^{\left(\frac{x+y}{\sqrt{2} \sqrt{x}}\right)^2} \erfc \left(\frac{x+y}{\sqrt{2} \sqrt{x}}\right)\ > \\ \nonumber
&\frac{2 \cdot 2}{\sqrt{\pi } \left(\frac{2 x+y}{\sqrt{2} \sqrt{x}}+\sqrt{\left(\frac{2 x+y}{\sqrt{2} \sqrt{x}}\right)^2+2}\right)}-\frac{2}{\sqrt{\pi } \left(\frac{x+y}{\sqrt{2} \sqrt{x}}+\sqrt{\left(\frac{x+y}{\sqrt{2} \sqrt{x}}\right)^2+\frac{4}{\pi }}\right)}\ = \\ \nonumber
&\frac{2 \sqrt{2}  \sqrt{x} \left(\frac{2}{\sqrt{(2 x+y)^2+4 x}+2 x+y}-\frac{1}{\sqrt{(x+y)^2+\frac{8 x}{\pi }}+x+y}\right)}{\sqrt{\pi }}\ > \\ \nonumber
&\frac{2 \sqrt{2}  \sqrt{x} \left(\frac{2}{\sqrt{(2 x+y)^2+2 (2
  x+y)+1}+2 x+y}-\frac{1}{\sqrt{(x+y)^2+0.782 \cdot 2
  (x+y)+0.782^2}+x+y}\right)}{\sqrt{\pi }}\ = \\ \nonumber
&\frac{2 \sqrt{2}  \sqrt{x}
  \left(\frac{2}{2 (2 x+y)+1}-\frac{1}{2
  (x+y)+0.782}\right)}{\sqrt{\pi }}\ = \\ \nonumber
&\frac{\left(2 \sqrt{2} 
  \sqrt{x}\right) (2 (2 (x+y)+0.782)-(2 (2 x+y)+1))}{\sqrt{\pi } ((2
  (x+y)+0.782) (2 (2 x+y)+1))}\ = \\ \nonumber
&\frac{\left(2 \sqrt{2}  \sqrt{x}\right)
  (2 y+0.782 \cdot 2-1)}{\sqrt{\pi } ((2 (x+y)+0.782) (2 (2 x+y)+1))} \ > \
  0 \ .
\end{align}
We explain this chain of inequalities:
\begin{itemize}
\item First inequality: We applied Lemma~\ref{lem:Abramowitz} two times.

\item Equalities factor out $\sqrt{2} \sqrt{x}$ and reformulate.

\item Second inequality part 1: we applied
\begin{align}
0<2 y\Longrightarrow (2 x+y)^2+4 x+1<(2 x+y)^2+2 (2 x+y)+1=(2 x+y+1)^2
  \ .
\end{align}
\item Second inequality part 2: we show that for $a=\frac{1}{20}
  \left(\sqrt{\frac{2048+169 \pi }{\pi }}-13\right)$ following holds:
$\frac{8 x}{\pi }-\left(a^2+2 a (x+y)\right) \geq 0$. 
We have $\frac{\partial }{\partial x }\frac{8 x}{\pi }-\left(a^2+2 a
  (x+y)\right)=\frac{8}{\pi }-2 a>0$ and
 $\frac{\partial }{\partial y }\frac{8 x}{\pi }-\left(a^2+2 a
  (x+y)\right)=-2 a<0$. 
Therefore the minimum is at border for minimal $x$ and maximal $y$:
\begin{align}
\frac{8 \cdot 0.64}{\pi }-\left(\frac{2}{20} \left(\sqrt{\frac{2048+169 \pi
  }{\pi }}-13\right) (0.64+0.01)+\left(\frac{1}{20}
  \left(\sqrt{\frac{2048+169 \pi }{\pi }}-13\right)\right)^2\right) \
  = \ 0 \ .
\end{align}
Thus
\begin{align}
\frac{8 x}{\pi } \ \geq \ a^2+2 a (x+y) \ .
\end{align}
for $a=\frac{1}{20}
  \left(\sqrt{\frac{2048+169 \pi }{\pi }}-13\right) > 0.782$.

\item Equalities only solve square root and factor out the resulting
  terms $(2(2 x+y)+1)$ and $(2(x+y)+0.782)$.
\end{itemize}

We know that $(2 - \erfc (x) > 0$ according to 
Lemma~\ref{lem:basics}.
For the sub-term we derived  
\begin{align}
&2 e^{\left(\frac{2 x+y}{\sqrt{2} \sqrt{x}}\right)^2} \erfc \left(\frac{2 x+y}{\sqrt{2} \sqrt{x}}\right)-e^{\left(\frac{x+y}{\sqrt{2} \sqrt{x}}\right)^2} \erfc \left(\frac{x+y}{\sqrt{2} \sqrt{x}}\right)\ > \ 0 \ .
\end{align}

Consequently, both terms in the brackets of  $\frac{\partial
}{\partial \nu }
\xinn(\mu,\omega,\nu,\tau,\lambda ,\alpha )$
are larger than zero.
Therefore  $\frac{\partial
}{\partial \nu }
\xinn(\mu,\omega,\nu,\tau,\lambda ,\alpha )$
is larger than zero.
\end{itemize}

\end{proof}

\begin{lemma}[Mean at low variance]
\label{lem:meanLowVar}
The mapping of the mean $\munn$ (Eq.~\eqref{eq:mappingMean}) 
\begin{align}
&\munn(\mu,\omega,\nu,\tau, \lambda ,\alpha ) \
  = \frac{1}{2} \lambda  \left(-(\alpha +\mu  \omega ) \erfc \left(\frac{\mu  \omega }{\sqrt{2} \sqrt{\nu  \tau }}\right)+ \right. \\ \nonumber & \left.
    \alpha  e^{\mu  \omega +\frac{\nu  \tau }{2}} \erfc 
    \left(\frac{\mu  \omega +\nu  \tau }{\sqrt{2} \sqrt{\nu  \tau }}\right)+ 
    \sqrt{\frac{2}{\pi }} \sqrt{\nu  \tau } e^{-\frac{\mu ^2 \omega ^2}{2 \nu  \tau }}+2 \mu  \omega \right) 
\end{align}
in the domain $-0.1 \leq \mu \leq -0.1$,  $-0.1 \leq \omega \leq -0.1$,
and $0.02 \leq \nu \tau \leq 0.5$ is bounded by  
\begin{align}
&| \munn(\mu,\omega,\nu,\tau, \lambda_{\mathrm{01}} ,\alpha_{\mathrm{01}} ) | < 0.289324
\end{align}
and
\begin{align}
 \lim_{\nu \rightarrow 0} | \munn(\mu,\omega,\nu,\tau, \lambda_{\mathrm{01}} ,\alpha_{\mathrm{01}} ) | = \lambda \mu \omega.
\end{align}
\end{lemma}

We can consider $\munn$ with given $\mu \omega$ as a function in $x = \nu \tau$. We show the graph of this function at the
maximal $\mu \omega=0.01$ in the interval $x \in [0, 1]$ in Figure~\ref{fig:meanAtLowVar}.

\begin{figure}
  \centering
 \includegraphics[width=0.5\textwidth]{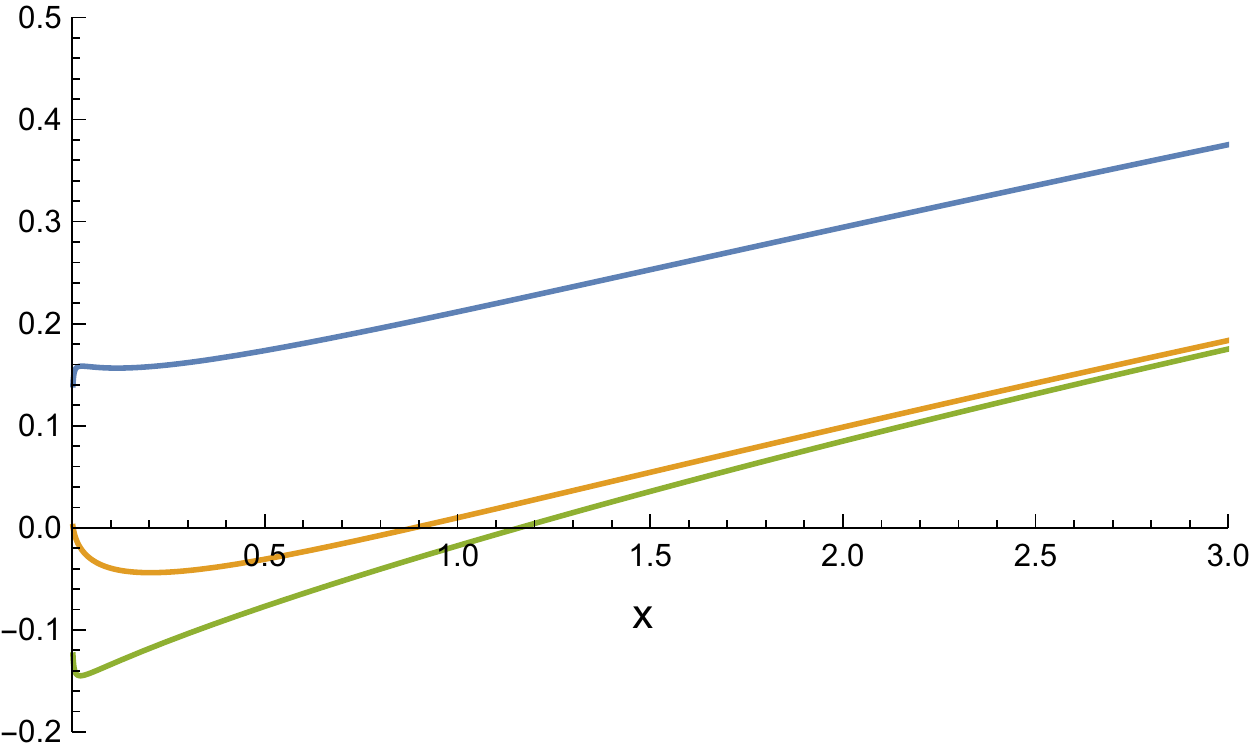}
 \caption[The graph of function $\munn$ for low variances]{The graph of function $\munn$ for low variances $x=\nu \tau$ for $\mu \omega=0.01$, where $x \in [0, 3]$, is
 displayed in yellow. 
 Lower and upper bounds based on the Abramowitz bounds (Lemma~\ref{lem:Abramowitz}) are displayed in green and blue, respectively. 
 \label{fig:meanAtLowVar}}
\end{figure}

\begin{proof}
Since $\munn$ is strictly monotonically increasing with $\mu \omega$ 
\begin{align}
&   \munn(\mu,\omega,\nu,\tau, \lambda ,\alpha )   \leq \\ \nonumber
&   \munn(0.1, 0.1,\nu,\tau, \lambda ,\alpha )   \leq \\ \nonumber
&   \frac{1}{2} \lambda  \left(-(\alpha +0.01) \erfc \left(\frac{0.01}{\sqrt{2} \sqrt{\nu  \tau }}\right)+
    \alpha  e^{0.01+\frac{\nu  \tau }{2}} \erfc \left(\frac{0.01 +\nu  \tau }{\sqrt{2} \sqrt{\nu  \tau }}\right)
    +\sqrt{\frac{2}{\pi }} \sqrt{\nu  \tau } e^{-\frac{0.01 ^2}{2 \nu  \tau }}+2 \cdot 0.01 \right)  \leq \\ \nonumber 
&  \frac{1}{2} \lambda_{\mathrm{01}} \left(e^{\frac{0.05}{2}+0.01} \alpha_{\mathrm{01}} \erfc \left(\frac{0.02 + 0.01}{\sqrt{2} \sqrt{0.02}}\right)-
  (\alpha_{\mathrm{01}} +0.01) \erfc \left(\frac{0.01}{\sqrt{2} \sqrt{0.02}}\right)+e^{-\frac{0.01 ^2}{2\cdot 0.5}} \sqrt{0.5} \sqrt{\frac{2}{\pi }}+0.01 \cdot 2\right) \\ \nonumber
  & < 0.21857,
\end{align}
where we have used the monotonicity of the terms in $\nu \tau$. 

Similarly, we can use the monotonicity of the terms in $\nu \tau$ to show that 
\begin{align}
&   \munn(\mu,\omega,\nu,\tau, \lambda ,\alpha )   \geq   \munn(0.1, -0.1,\nu,\tau, \lambda ,\alpha )   >  -0.289324,
\end{align}

such that $\left| \munn \right| < 0.289324$ at low variances.

Furthermore, when $(\nu \tau) \rightarrow 0$, the terms with the arguments of the complementary error functions $\erfc$ and the exponential function 
go to infinity, therefore these three terms converge to zero. Hence, the remaining terms are only $2 \mu  \omega  \frac{1}{2} \lambda$.
\end{proof}

\begin{lemma}[Bounds on derivatives of $\munn$ in $\Omega^-$]
\label{lem:muBounds}
The derivatives of the function $\munn (\mu, \omega, \nu,\tau,\lambda_{\mathrm 01}, \alpha_{\mathrm 01}$ 
(Eq.~\eqref{eq:mappingMean})
with respect to $\mu, \omega, \nu,\tau$ in the domain 
$\Omega^- = \{ \mu, \omega, \nu, \tau \ | \ -0.1 \leq \mu \leq 0.1, -0.1 \leq \omega \leq 0.1, 0.05 \leq \nu \leq 0.24, 0.8 \leq \tau \leq 1.25 \}$
can be bounded as follows:

\begin{align}
 & \left|\frac{\partial}{\partial \mu} \munn \right| < 0.14  \\ \nonumber
 & \left|\frac{\partial}{\partial \omega} \munn \right| < 0.14  \\ \nonumber
 & \left|\frac{\partial}{\partial \nu} \munn \right| < 0.52  \\ \nonumber
 & \left|\frac{\partial}{\partial \tau} \munn \right| < 0.11.
\end{align}
\end{lemma}

\begin{proof}
The expression
\begin{align}
 & \frac{\partial}{\partial \mu} \munn = J_{11} = \frac{1}{2}\lambda\omega e^{\frac{-(\mu\omega)^{2}}{2\nu\tau}}\left(2e^{\frac{(\mu\omega)^{2}}{2\nu\tau}}-
 e^{\frac{(\mu\omega)^{2}}{2\nu\tau}}\erfc \left(\frac{\mu\omega}{\sqrt{2}\sqrt{\nu\tau}}\right)+
 \alpha e^{\frac{(\mu\omega+\nu\tau)^{2}}{2\nu\tau}}\erfc\left(\frac{\mu\omega+\nu\tau}{\sqrt{2}\sqrt{\nu\tau}}\right)\right) \\ \nonumber
\end{align}
contains the terms $e^{\frac{(\mu\omega)^{2}}{2\nu\tau}}\erfc \left(\frac{\mu\omega}{\sqrt{2}\sqrt{\nu\tau}}\right)$
and $e^{\frac{(\mu\omega+\nu\tau)^{2}}{2\nu\tau}}\erfc\left(\frac{\mu\omega+\nu\tau}{\sqrt{2}\sqrt{\nu\tau}}\right)$
which are monotonically decreasing in their arguments (Lemma~\ref{lem:exerfc}). We can therefore obtain their
minima and maximal at the minimal and maximal arguments. Since the first term has a negative sign in the expression, both terms
reach their maximal value at $\mu \omega=-0.01$, $\nu = 0.05$, and $\tau=0.8$. 
\begin{align}
 & \left| \frac{\partial}{\partial \mu} \munn \right| \leq    \frac{1}{2} \left| \lambda\omega \right| \left| \left(2 -
  e^{0.0353553^2}\erfc \left(0.0353553 \right)  +
 \alpha e^{0.106066^2}\erfc\left(0.106066\right)\right) \right| <  0.133
\end{align}

Since, $\munn$ is symmetric in $\mu$ and $\omega$, these bounds also hold for the derivate to $\omega$.

We use the argumentation that the term with the error function is monotonically decreasing (Lemma~\ref{lem:exerfc})
again for the expression
\begin{align}
 & \frac{\partial}{\partial \nu} \munn = J_{12} = \\ \nonumber
 & =\frac{1}{4}\lambda\tau e^{-\frac{\mu^{2}\omega^{2}}{2\nu\tau}} \left(
 \alpha e^{\frac{(\mu\omega+\nu\tau)^{2}}{2\nu\tau}} \erfc \left(\frac{\mu\omega+\nu\tau}{\sqrt{2}\sqrt{\nu\tau}}\right)
 -(\alpha-1)\sqrt{\frac{2}{\pi\nu\tau}}\ \right) \leq \\ \nonumber
 & \left| \frac{1}{4}\lambda \tau \right| 
 \left( \left| 1.1072 -  2.68593   \right| \right) <  0.52.  
\end{align}

We have used that the term $1.1072 \leq \alpha_{\mathrm{01}} e^{\frac{(\mu\omega+\nu\tau)^{2}}{2\nu\tau}} \erfc \left(\frac{\mu\omega+\nu\tau}{\sqrt{2}\sqrt{\nu\tau}}\right) \leq 1.49042$
and the term $0.942286 \leq (\alpha-1)\sqrt{\frac{2}{\pi\nu\tau}}\  \leq 2.68593$.
Since $\munn$ is symmetric in $\nu$ and $\tau$, we only have to chance outermost
term $\left| \frac{1}{4}\lambda \tau \right|$ to  $\left| \frac{1}{4}\lambda \nu \right|$ to 
obtain the estimate $\left| \frac{\partial}{\partial \tau} \munn \right| < 0.11$.

\end{proof}

\begin{lemma}[Tight bound on ${\munn}^2$ in $\Omega^-$]
\label{lem:musquared}
The function ${\munn}^2 (\mu, \omega, \nu,\tau,\lambda_{\mathrm 01}, \alpha_{\mathrm 01})$ 
(Eq.~\eqref{eq:mappingMean})
is bounded by
\begin{align}
 & \left| {\munn}^2 \right| < 0.005  \\ 
\end{align}

in the domain 
$\Omega^- = \{ \mu, \omega, \nu, \tau \ | \ -0.1 \leq \mu \leq 0.1, -0.1 \leq \omega \leq 0.1, 0.05 \leq \nu \leq 0.24, 0.8 \leq \tau \leq 1.25 \}$.
\end{lemma}

We visualize the function ${\munn}^2$ at its maximal $\mu \nu=-0.01$ and for $x=\nu \tau$ in the form
$h(x)={\munn}^2 (0.1, -0.1, x,1,\lambda_{\mathrm 01}, \alpha_{\mathrm 01})$  in Figure~\ref{fig:meanSqu}.

\begin{figure}
 \centering
 \includegraphics[width=0.5\textwidth]{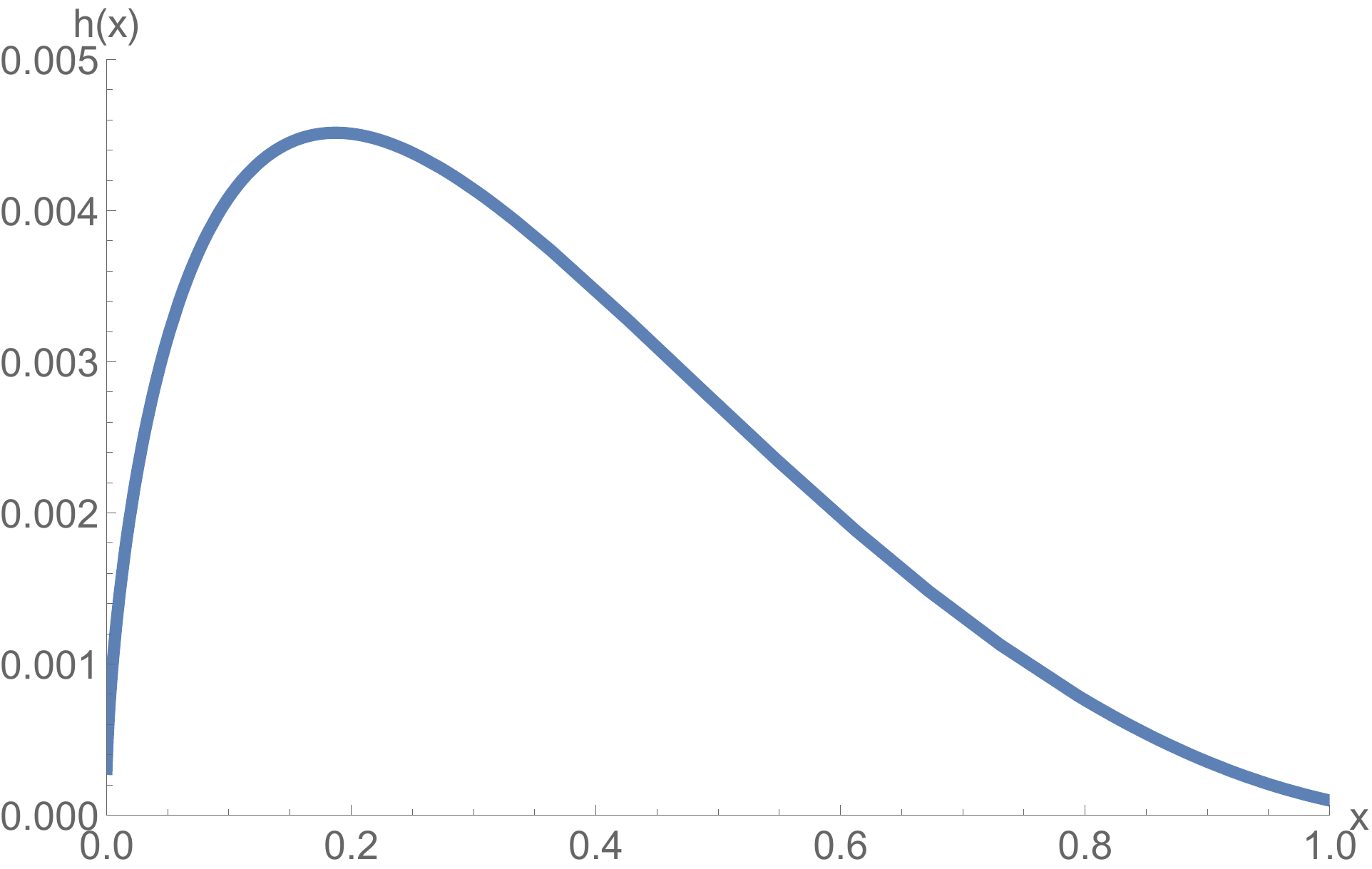}
 \caption[Graph of the function $h(x)={\munn}^2 (0.1, -0.1, x,1,\lambda_{\mathrm 01}, \alpha_{\mathrm 01})$]{The graph of the function $h(x)={\munn}^2 (0.1, -0.1, x,1,\lambda_{\mathrm 01}, \alpha_{\mathrm 01})$ is displayed. It has a local 
  maximum at $x=\nu \tau \approx 0.187342$ and $h(x)\approx 0.00451457$ in the domain $x \in [0,1]$. \label{fig:meanSqu}}
\end{figure}

\begin{proof}
We use a similar strategy to the one we have used to show the bound on the singular value (Lemmata~\ref{lem:Ds1Bounds}, \ref{lem:meanValue}, and \ref{lem:sBound}), where
we evaluted the function on a grid and used bounds on the derivatives together with the mean value theorem.
Here we have
\begin{align}
&\left| {\munn}^2 (\mu, \omega, \nu,\tau,\lambda_{\mathrm 01}, \alpha_{\mathrm 01}) - 
{\munn}^2 (\mu +\Delta \mu, \omega+\Delta \omega, \nu + \Delta \nu,\tau + \Delta \tau,\lambda_{\mathrm 01}, \alpha_{\mathrm 01}) \right|
\leq \\ \nonumber
&\left| \frac{\partial}{\partial \mu}{\munn}^2 \right| |\Delta \mu| + 
\left| \frac{\partial}{\partial \omega}{\munn}^2 \right| |\Delta \omega|+
\left| \frac{\partial}{\partial \nu}{\munn}^2 \right| |\Delta \nu|+
\left| \frac{\partial}{\partial \tau}{\munn}^2 \right| |\Delta \tau|.
\end{align}

We use Lemma~\ref{lem:muBounds} and Lemma~\ref{lem:meanLowVar}, to obtain 
\begin{align}
& \left| \frac{\partial}{\partial \mu}{\munn}^2 \right| = 2  \left| \munn \right|  \left| \frac{\partial}{\partial \mu} \munn \right| <
2 \cdot 0.289324 \cdot 0.14 = 0.08101072 \\ \nonumber
& \left| \frac{\partial}{\partial \omega}{\munn}^2 \right| = 2  \left| \munn \right|  \left| \frac{\partial}{\partial \omega} \munn \right| <
2 \cdot 0.289324 \cdot 0.14 = 0.08101072 \\ \nonumber
& \left| \frac{\partial}{\partial \nu}{\munn}^2 \right| = 2  \left| \munn \right|  \left| \frac{\partial}{\partial \nu} \munn \right| <
2 \cdot 0.289324 \cdot 0.52 = 0.30089696 \\ \nonumber
& \left| \frac{\partial}{\partial \tau}{\munn}^2 \right| = 2  \left| \munn \right|  \left| \frac{\partial}{\partial \tau} \munn \right| <
2 \cdot 0.289324 \cdot 0.11 =  0.06365128 
\end{align}

We evaluated the function ${\munn}^2$ in a grid $G$ of $\Omega^-$ with $\Delta \mu =0.001498041 $,
$\Delta \omega = 0.001498041$,
$\Delta \nu = 0.0004033190$, and
$\Delta \tau = 0.0019065994$ using a computer and obtained the maximal value $\max _{G} (\munn)^2=0.00451457$, therefore 
the maximal value of   ${\munn}^2$ is bounded by
\begin{align}
&\max _ {(\mu, \omega, \nu, \tau) \in \Omega^-} (\munn)^2 \leq \\ \nonumber
&0.00451457 +  0.001498041 \cdot  0.08101072 + 0.001498041 \cdot  0.08101072 + \\
& 0.0004033190 \cdot 0.30089696 + 0.0019065994 \cdot 0.06365128 < 0.005.
\end{align}

Furthermore we used error propagation to estimate the numerical error on the function evaluation. Using the error propagation rules 
derived in Subsection~\ref{sec:error}, we found that the numerical error is smaller than $10^{-13}$ in the worst case.
\end{proof}


\begin{lemma}[Main subfunction]
\label{proof:mainsubfunction}
For $1.2 \leq x \leq 20$ and $-0.1 \leq y \leq 0.1$, 

the function
\begin{align}
e^{\frac{(x+y)^2}{2 x}} \erfc \left(\frac{x+y}{\sqrt{2} \sqrt{x}}\right)-2 e^{\frac{(2 x+y)^2}{2 x}} \erfc \left(\frac{2 x+y}{\sqrt{2} \sqrt{x}}\right)
\end{align}
is smaller than zero, is strictly monotonically increasing in $x$,
and strictly monotonically decreasing in $y$ for the minimal $x=12/10=1.2$.
\end{lemma}

\begin{proof}
We first consider the derivative of sub-function
Eq.~\eqref{eq:subfunction} with respect to $x$.
The derivative of the function 
\begin{align}
e^{\frac{(x+y)^2}{2 x}} \erfc \left(\frac{x+y}{\sqrt{2} \sqrt{x}}\right)-2 e^{\frac{(2 x+y)^2}{2 x}} \erfc \left(\frac{2 x+y}{\sqrt{2} \sqrt{x}}\right)
\end{align}
with respect to $x$ is
\begin{align}
&\frac{\sqrt{\pi } \left(e^{\frac{(x+y)^2}{2 x}} (x-y) (x+y)
  \erfc \left(\frac{x+y}{\sqrt{2} \sqrt{x}}\right)-2 e^{\frac{(2
  x+y)^2}{2 x}} \left(4 x^2-y^2\right) \erfc \left(\frac{2
  x+y}{\sqrt{2} \sqrt{x}}\right)\right)+\sqrt{2} \sqrt{x} (3 x-y)}{2
  \sqrt{\pi } x^2}\ = \\ \nonumber &\frac{\sqrt{\pi } \left(e^{\frac{(x+y)^2}{2 x}}
  (x-y) (x+y) \erfc \left(\frac{x+y}{\sqrt{2} \sqrt{x}}\right)-2
  e^{\frac{(2 x+y)^2}{2 x}} (2 x+y) (2 x-y) \erfc \left(\frac{2
  x+y}{\sqrt{2} \sqrt{x}}\right)\right)+\sqrt{2} \sqrt{x} (3 x-y)}{2
  \sqrt{\pi } x^2} \ = \\ \nonumber &\frac{\sqrt{\pi } \left(\frac{e^{\frac{(x+y)^2}{2
  x}} (x-y) (x+y) \erfc \left(\frac{x+y}{\sqrt{2}
  \sqrt{x}}\right)}{\sqrt{2} \sqrt{x}}-\frac{2 e^{\frac{(2 x+y)^2}{2
  x}} (2 x+y) (2 x-y) \erfc \left(\frac{2 x+y}{\sqrt{2}
  \sqrt{x}}\right)}{\sqrt{2} \sqrt{x}}\right)+(3 x-y)}{2 \sqrt{2} 
  \sqrt{\pi }  x^2 \sqrt{x}} \ .
\end{align}

We consider the numerator
\begin{align}
\sqrt{\pi } \left(\frac{e^{\frac{(x+y)^2}{2 x}} (x-y) (x+y)
  \erfc \left(\frac{x+y}{\sqrt{2} \sqrt{x}}\right)}{\sqrt{2}
  \sqrt{x}}-\frac{2 e^{\frac{(2 x+y)^2}{2 x}} (2 x+y) (2 x-y)
  \erfc \left(\frac{2 x+y}{\sqrt{2} \sqrt{x}}\right)}{\sqrt{2}
  \sqrt{x}}\right)+(3 x-y) \ .
\end{align}

For bounding this value, we use the approximation 
\begin{align}
\label{eq:ren}
e^{z^2} \erfc (z) \ \approx \ \frac{2.911}{\sqrt{\pi } (2.911 -1)
  z+\sqrt{\pi  z^2+2.911^2}} \ .
\end{align}
from \citet{Ren:07}.
We start with an error analysis of this approximation.
According to \citet{Ren:07} (Figure~1), the approximation 
error is positive in the range
$[0.7,3.2]$.  This range contains all possible
arguments of $\erfc$ that we consider.
Numerically we maximized and minimized the approximation error of the
whole expression
\begin{align}
E(x,y) \ &= \ \left(\frac{e^{\frac{(x+y)^2}{2 x}} (x-y) (x+y)
  \erfc \left(\frac{x+y}{\sqrt{2} \sqrt{x}}\right)}{\sqrt{2}
  \sqrt{x}}-\frac{2 e^{\frac{(2 x+y)^2}{2 x}} (2 x-y) (2 x+y)
  \erfc \left(\frac{2 x+y}{\sqrt{2} \sqrt{x}}\right)}{\sqrt{2}
  \sqrt{x}}\right) \ - \\ \nonumber 
&\left(\frac{2.911 (x-y) (x+y)}{\left(\sqrt{2} \sqrt{x}\right) \left(\frac{\sqrt{\pi } (2.911 -1) 
  (x+y)}{\sqrt{2} \sqrt{x}}+\sqrt{\pi  \left(\frac{x+y}{\sqrt{2} \sqrt{x}}\right)^2+2.911^2}\right)}\ - \right. \\ \nonumber 
&\left. \frac{2 \cdot 2.911 (2 x-y) (2 x+y)}{\left(\sqrt{2} \sqrt{x}\right) \left(\frac{\sqrt{\pi } (2.911 -1) (2 x+y)}{\sqrt{2} \sqrt{x}}+
  \sqrt{\pi  \left(\frac{2 x+y}{\sqrt{2} \sqrt{x}}\right)^2+2.911^2}\right)}\right) \ .
\end{align}
We numerically determined $0.0113556 \leq E(x,y) \leq 0.0169551$ for 
$1.2 \leq x \leq 20$ and $-0.1 \leq y \leq 0.1$. 
We used different numerical optimization techniques like  
gradient based constraint BFGS algorithms and 
non-gradient-based Nelder-Mead methods with different start points.
Therefore our approximation is smaller than the function that we
approximate. 
We subtract an additional safety gap of 0.0131259 from our
approximation to ensure that the inequality via the approximation
holds true. With this safety gap the inequality would hold true even 
for negative $x$, where the approximation error becomes negative and
the safety gap would compensate.
Of course, the safety gap of 0.0131259 is not necessary for our
analysis but may help or future investigations.

We have the sequences of inequalities using the approximation of \citet{Ren:07}:
\begin{align}
\label{eq:ineqX}
&(3 x-y)+\left(\frac{e^{\frac{(x+y)^2}{2 x}} (x-y) (x+y)
  \erfc \left(\frac{x+y}{\sqrt{2} \sqrt{x}}\right)}{\sqrt{2}
  \sqrt{x}}-\frac{2 e^{\frac{(2 x+y)^2}{2 x}} (2 x-y) (2 x+y)
  \erfc \left(\frac{2 x+y}{\sqrt{2} \sqrt{x}}\right)}{\sqrt{2}
  \sqrt{x}}\right) \sqrt{\pi }\ \geq \\ \nonumber & (3 x-y)+\left(\frac{2.911 (x-y)
  (x+y)}{\left(\sqrt{\pi  \left(\frac{x+y}{\sqrt{2}
  \sqrt{x}}\right)^2+2.911^2}+\frac{(2.911 -1) \sqrt{\pi }
  (x+y)}{\sqrt{2} \sqrt{x}}\right) \left(\sqrt{2}
  \sqrt{x}\right)}\ - \right. \\ \nonumber &\left. \frac{2 (2 x-y) (2 x+y) 2.911}{\left(\sqrt{2}
  \sqrt{x}\right) \left(\sqrt{\pi  \left(\frac{2 x+y}{\sqrt{2}
  \sqrt{x}}\right)^2+2.911^2}+\frac{(2.911 -1) \sqrt{\pi } (2
  x+y)}{\sqrt{2} \sqrt{x}}\right)}\right) \sqrt{\pi }-0.0131259 \ = \\ \nonumber 
  &(3 x-y)+\left(\frac{\left(\sqrt{2} \sqrt{x} 2.911\right) (x-y)
  (x+y)}{\left(\sqrt{\pi  (x+y)^2+2 \cdot 2.911^2 x}+(2.911 -1) (x+y)
  \sqrt{\pi }\right) \left(\sqrt{2} \sqrt{x}\right)}\ - \right. \\ \nonumber &\left. \frac{2 (2 x-y)
  (2 x+y) \left(\sqrt{2} \sqrt{x} 2.911\right)}{\left(\sqrt{2}
  \sqrt{x}\right) \left(\sqrt{\pi  (2 x+y)^2+2 \cdot 2.911^2 x}+(2.911 -1)
  (2 x+y) \sqrt{\pi }\right)}\right) \sqrt{\pi }-0.0131259 \ = \\ \nonumber 
  &(3 x-y)+2.911 \left(\frac{(x-y) (x+y)}{(2.911 -1)
  (x+y)+\sqrt{(x+y)^2+\frac{2 \cdot 2.911^2 x}{\pi }}}\ - \right. \\ \nonumber &\left. \frac{2 (2 x-y) (2
  x+y)}{(2.911 -1) (2 x+y)+\sqrt{(2 x+y)^2+\frac{2 \cdot 2.911^2 x}{\pi
  }}}\right)-0.0131259\ \geq \\ \nonumber & (3 x-y)+2.911 \left(\frac{(x-y)
  (x+y)}{(2.911 -1) (x+y)+\sqrt{\left(\frac{2.911^2}{\pi
  }\right)^2+(x+y)^2+\frac{2 \cdot 2.911^2 x}{\pi }+\frac{2 \cdot 2.911^2 y}{\pi
  }}}\ - \right. \\ \nonumber &\left. \frac{2 (2 x-y) (2 x+y)}{(2.911 -1) (2 x+y)+\sqrt{(2
  x+y)^2+\frac{2 \cdot 2.911^2 x}{\pi }}}\right)-0.0131259 \ = \\ \nonumber 
  &(3 x-y)+2.911
  \left(\frac{(x-y) (x+y)}{(2.911 -1)
  (x+y)+\sqrt{\left(x+y+\frac{2.911^2}{\pi }\right)^2}}\ - \right. \\ \nonumber &\left. \frac{2 (2
  x-y) (2 x+y)}{(2.911 -1) (2 x+y)+\sqrt{(2 x+y)^2+\frac{2 \cdot 2.911^2
  x}{\pi }}}\right)-0.0131259 \ = \\ \nonumber &(3 x-y)+2.911 \left(\frac{(x-y)
  (x+y)}{2.911 (x+y)+\frac{2.911^2}{\pi }}-\frac{2 (2 x-y) (2
  x+y)}{(2.911 -1) (2 x+y)+\sqrt{(2 x+y)^2+\frac{2 \cdot 2.911^2 x}{\pi
  }}}\right)-0.0131259 \ = \\ \nonumber &(3 x-y)+\frac{(x-y)
  (x+y)}{(x+y)+\frac{2.911}{\pi }}-\frac{2 (2 x-y) (2 x+y)
  2.911}{(2.911 -1) (2 x+y)+\sqrt{(2 x+y)^2+\frac{2 \cdot 2.911^2 x}{\pi
  }}}-0.0131259 \ = 
&(3 x-y)+\frac{(x-y) (x+y)}{(x+y)+\frac{2.911}{\pi
  }}-\frac{2 (2 x-y) (2 x+y) 2.911}{(2.911 -1) (2 x+y)+\sqrt{(2
  x+y)^2+\frac{2 \cdot 2.911^2 x}{\pi }}}-0.0131259 \ = \\ \nonumber 
  &\left(-2 (2 x-y) 2.911
  \left((x+y)+\frac{2.911}{\pi }\right) (2
  x+y) \right. \ + \\ \nonumber &\left. \left((x+y)+\frac{2.911}{\pi }\right) (3 x-y-0.0131259)
  \left((2.911 -1) (2 x+y)+\sqrt{(2 x+y)^2+\frac{2 \cdot 2.911^2 x}{\pi
  }}\right)\right. \ + \\ \nonumber &\left. (x-y) (x+y) \left((2.911 -1) (2 x+y)+\sqrt{(2
  x+y)^2+\frac{2 \cdot 2.911^2 x}{\pi
  }}\right)\right) \\ \nonumber &\left(\left((x+y)+\frac{2.911}{\pi }\right) \left((2.911 -1)
  (2 x+y)+\sqrt{(2 x+y)^2+\frac{2 \cdot 2.911^2 x}{\pi
  }}\right)\right)^{-1} \ = \\ &\left(((x-y) (x+y)+(3 x-y-0.0131259) (x+y+0.9266)) \left(\sqrt{(2 x+y)^2+5.39467 x}+3.822 x+1.911 y\right)\right. \ - \\ \nonumber &\left. 5.822 (2 x-y) (x+y+0.9266) (2 x+y)\right)\\ \nonumber &    \left(\left((x+y)+\frac{2.911}{\pi }\right) \left((2.911 -1) (2 x+y)+\sqrt{(2 x+y)^2+\frac{2 2.911^2 x}{\pi }}\right)\right)^{-1} \ > \ 0 \ .
\end{align}

We explain this sequence of inequalities:
\begin{itemize} 
\item First inequality: The approximation of \citet{Ren:07}
and then subtracting a safety gap (which would not be necessary for the
current analysis).

\item Equalities: The factor $\sqrt{2} \sqrt{x}$ is factored out and
  canceled. 

\item Second inequality: adds a positive term in the first root to
  obtain a binomial form. The term containing the root 
is positive and the root is in the denominator, 
therefore the whole term becomes smaller.
\end{itemize}
\begin{itemize} 

\item Equalities: solve for the term and factor out.

\item Bringing all terms to the denominator
$\left((x+y)+\frac{2.911}{\pi }\right) \left((2.911 -1) (2 x+y)+\sqrt{(2 x+y)^2+\frac{2 \cdot 2.911^2 x}{\pi }}\right)$.

\item Equalities: Multiplying out and expanding terms.

\item Last inequality $>0$ is proofed in the following sequence of
  inequalities.
\end{itemize}

We look at the numerator of the last expression of 
Eq.~\eqref{eq:ineqX}, which we show to be
positive in order to show $>0$ in 
Eq.~\eqref{eq:ineqX}. The numerator is 
\begin{align}
&((x-y) (x+y)+(3 x-y-0.0131259) (x+y+0.9266)) \left(\sqrt{(2 x+y)^2+5.39467 x}+3.822 x+1.911 y\right)-\\ \nonumber &5.822 (2 x-y) (x+y+0.9266) (2 x+y)\ = \\ \nonumber &-5.822 (2 x-y) (x+y+0.9266) (2 x+y)+(3.822 x+1.911 y) ((x-y) (x+y)+\\ \nonumber &(3 x-y-0.0131259) (x+y+0.9266))+((x-y) (x+y)+\\ \nonumber &(3 x-y-0.0131259) (x+y+0.9266)) \sqrt{(2 x+y)^2+5.39467 x}\ = \\ \nonumber &-8.0 x^3+\left(4 x^2+2 x y+2.76667 x-2 y^2-0.939726 y-0.0121625\right) \sqrt{(2 x+y)^2+5.39467 x}-\\ \nonumber &8.0 x^2 y-11.0044 x^2+2.0 x y^2+1.69548 x y-0.0464849 x+2.0 y^3+3.59885 y^2-0.0232425 y\ = \\ \nonumber &-8.0 x^3+\left(4 x^2+2 x y+2.76667 x-2 y^2-0.939726 y-0.0121625\right) \sqrt{(2 x+y)^2+5.39467 x}-\\ \nonumber &8.0 x^2 y-11.0044 x^2+2.0 x y^2+1.69548 x y-0.0464849 x+2.0 y^3+3.59885 y^2-0.0232425 y \ .
\end{align}
The factor in front of the root is positive.
If the term, that does not contain the root, was positive, then the whole expression would be positive and 
we would have proofed that the numerator is positive. 
Therefore we consider the case that the term, that does not contain the root, is negative.
The term that contains the root must be larger than the other term in absolute values. 
\begin{align}
&-\left(-8.0 x^3-8.0 x^2 y-11.0044 x^2+2. x y^2+1.69548 x y-0.0464849 x+2. y^3+3.59885 y^2-0.0232425 y\right)\ < \\ \nonumber 
&\left(4 x^2+2 x y+2.76667 x-2 y^2-0.939726 y-0.0121625\right) \sqrt{(2 x+y)^2+5.39467 x} \ . 
\end{align}
Therefore the squares of the root term have to be larger 
than the square of the other term to show $>0$ in 
Eq.~\eqref{eq:ineqX}.
Thus, we have the inequality:
\begin{align}
&\left(-8.0 x^3-8.0 x^2 y-11.0044 x^2+2. x y^2+1.69548 x y-0.0464849 x+2. y^3+3.59885 y^2-0.0232425 y\right)^2\ < \\ \nonumber 
&\left(4 x^2+2 x y+2.76667 x-2 y^2-0.939726 y-0.0121625\right)^2 \left((2 x+y)^2+5.39467 x\right) \ . 
\end{align}

This is equivalent to
\begin{align}
& 0\ < \ \left(4 x^2+2 x y+2.76667 x-2 y^2-0.939726 y-0.0121625\right)^2 \left((2 x+y)^2+5.39467 x\right)-\\ \nonumber 
&\left(-8.0 x^3-8.0 x^2 y-11.0044 x^2+2.0 x y^2+1.69548 x y-0.0464849 x+2.0 y^3+3.59885 y^2-0.0232425 y\right)^2\ = \\ \nonumber &-1.2227 x^5+40.1006 x^4 y+27.7897 x^4+41.0176 x^3 y^2+64.5799 x^3 y+39.4762 x^3+10.9422 x^2 y^3-\\ \nonumber &13.543 x^2 y^2-28.8455 x^2 y-0.364625 x^2+0.611352 x y^4+6.83183 x y^3+5.46393 x y^2+\\ \nonumber &0.121746 x y+0.000798008 x-10.6365 y^5-11.927 y^4+0.190151 y^3-0.000392287 y^2 \ .
\end{align}
We obtain the inequalities: 
\begin{align}
&-1.2227 x^5+40.1006 x^4 y+27.7897 x^4+41.0176 x^3 y^2+64.5799 x^3 y+39.4762 x^3+10.9422 x^2 y^3-\\ \nonumber &13.543 x^2 y^2-28.8455 x^2 y-0.364625 x^2+0.611352 x y^4+6.83183 x y^3+5.46393 x y^2+\\ \nonumber &0.121746 x y+0.000798008 x-10.6365 y^5-11.927 y^4+0.190151 y^3-0.000392287 y^2\ = \\ \nonumber &-1.2227 x^5+27.7897 x^4+41.0176 x^3 y^2+39.4762 x^3-13.543 x^2 y^2-0.364625 x^2+\\ \nonumber &y \left(40.1006 x^4+64.5799 x^3+10.9422 x^2 y^2-28.8455 x^2+6.83183 x y^2+0.121746 x\ - \right. \\ \nonumber &\left. 10.6365 y^4+0.190151 y^2\right)+0.611352 x y^4+5.46393 x y^2+0.000798008 x-11.927 y^4-0.000392287 y^2\ > \\ \nonumber &
-1.2227 x^5+27.7897 x^4+41.0176 \cdot (0.0)^2 x^3+39.4762 x^3-13.543 \cdot (0.1)^2 x^2-0.364625 x^2- \\ \nonumber &0.1 \cdot  \left(40.1006 x^4+64.5799 x^3+10.9422 \cdot (0.1)^2 x^2-28.8455 x^2+6.83183 \cdot (0.1)^2 x+0.121746 x\ + \right. \\ \nonumber &\left. 10.6365 \cdot (0.1)^4+0.190151 \cdot (0.1)^2\right)+\\ \nonumber &0.611352 \cdot (0.0)^4 x+5.46393 \cdot (0.0)^2 x+0.000798008 x-11.927 \cdot (0.1)^4-0.000392287 \cdot (0.1)^2
\ = \\ \nonumber &-1.2227 x^5+23.7796 x^4+(20+13.0182) x^3+2.37355 x^2-0.0182084 x-0.000194074\ \geq \\ \nonumber & -1.2227 x^5+24.7796 x^4+13.0182 x^3+2.37355 x^2-0.0182084 x-0.000194074\ > \\ \nonumber &13.0182 x^3+2.37355 x^2-0.0182084 x-0.000194074\ >\ 0\ .
\end{align}
We used $24.7796 \cdot (20)^4-1.2227 \cdot (20)^5= 52090.9>0 $ and $x \leq 20$.
We have proofed the last inequality $>0$ of Eq.~\eqref{eq:ineqX}.

Consequently the derivative is always positive independent of $y$,
thus 
\begin{align}
e^{\frac{(x+y)^2}{2 x}} \erfc \left(\frac{x+y}{\sqrt{2} \sqrt{x}}\right)-2 e^{\frac{(2 x+y)^2}{2 x}} \erfc \left(\frac{2 x+y}{\sqrt{2} \sqrt{x}}\right)
\end{align}
is strictly monotonically increasing in $x$.

\paragraph{The main subfunction is smaller than zero.} 
Next we show that the 
sub-function Eq.~\eqref{eq:subfunction} is smaller
than zero.
We consider the limit:
\begin{align}
&\lim_{x \to \infty}e^{\frac{(x+y)^2}{2 x}}
  \erfc \left(\frac{x+y}{\sqrt{2} \sqrt{x}}\right) \ - \ 2 e^{\frac{(2 x+y)^2}{2 x}} \erfc \left(\frac{2 x+y}{\sqrt{2} \sqrt{x}}\right)
\ = \ 0
\end{align}
The limit follows from Lemma~\ref{lem:Abramowitz}.
Since the function is monotonic increasing in $x$, it has to approach
$0$ from below. Thus,
\begin{align}
e^{\frac{(x+y)^2}{2 x}} \erfc \left(\frac{x+y}{\sqrt{2} \sqrt{x}}\right)-2 e^{\frac{(2 x+y)^2}{2 x}} \erfc \left(\frac{2 x+y}{\sqrt{2} \sqrt{x}}\right)
\end{align}
is smaller than zero.

\paragraph{Behavior of the main subfunction with respect to $y$ at minimal $x$.} 
We now consider the derivative of sub-function
Eq.~\eqref{eq:subfunction} with respect to $y$.
We proofed that sub-function
Eq.~\eqref{eq:subfunction} is  strictly monotonically increasing 
independent of $y$. 
In the proof of Theorem~\ref{th:s2Cont}, we need the minimum
of  sub-function
Eq.~\eqref{eq:subfunction}. Therefore we are only interested in the
derivative of sub-function
Eq.~\eqref{eq:subfunction} with respect to $y$
for the minimum $x=12/10=1.2$ 

Consequently, we insert the minimum $x=12/10=1.2$ into the  sub-function
Eq.~\eqref{eq:subfunction}. The main terms become
\begin{align}
\frac{x+y}{\sqrt{2} \sqrt{x}} \ = \ \frac{y+1.2}{\sqrt{2} \sqrt{1.2}} \ = \
\frac{y}{\sqrt{2} \sqrt{1.2}}+\frac{\sqrt{1.2}}{\sqrt{2}} \ = \
\frac{5 y+6}{2 \sqrt{15}} 
\end{align}
and
\begin{align}
\frac{2 x+y}{\sqrt{2} \sqrt{x}} \ = \ \frac{y+1.2 \cdot 2}{\sqrt{2} \sqrt{1.2}} \ = \
\frac{y}{\sqrt{2} \sqrt{1.2}}+\sqrt{1.2} \sqrt{2} \ = \
\frac{5 y+12}{2 \sqrt{15}} \ .
\end{align}
Sub-function
Eq.~\eqref{eq:subfunction} becomes:
\begin{align}
e^{\left(\frac{y}{\sqrt{2} \sqrt{\frac{12}{10}}}+\frac{\sqrt{\frac{12}{10}}}{\sqrt{2}}\right)^2} \erfc \left(\frac{y}{\sqrt{2} \sqrt{\frac{12}{10}}}+\frac{\sqrt{\frac{12}{10}}}{\sqrt{2}}\right)-2 e^{\left(\frac{y}{\sqrt{2} \sqrt{\frac{12}{10}}}+\sqrt{2} \sqrt{\frac{12}{10}}\right)^2} \erfc \left(\frac{y}{\sqrt{2} \sqrt{\frac{12}{10}}}+\sqrt{2} \sqrt{\frac{12}{10}}\right) \ .
\end{align}
The derivative of this function with respect to $y$ is
\begin{align}
\frac{\sqrt{15 \pi } \left(e^{\frac{1}{60} (5 y+6)^2} (5 y+6) \erfc \left(\frac{5 y+6}{2 \sqrt{15}}\right)-2 e^{\frac{1}{60} (5 y+12)^2} (5 y+12) \erfc \left(\frac{5 y+12}{2 \sqrt{15}}\right)\right)+30}{6 \sqrt{15 \pi }} \ .
\end{align}

We again will use the approximation of \citet{Ren:07}
\begin{align}
e^{z^2} \erfc (z) \ = \ \frac{2.911}{\sqrt{\pi } (2.911 -1)
  z+\sqrt{\pi  z^2+2.911^2}} \ .
\end{align}
Therefore we first perform an error analysis.
We estimated the maximum and minimum of 
\begin{align}
&\sqrt{15 \pi } \left(\frac{2 \cdot 2.911 (5 y+12)}{\frac{\sqrt{\pi } (2.911 -1) (5 y+12)}{2 \sqrt{15}}+\sqrt{\pi  \left(\frac{5 y+12}{2 \sqrt{15}}\right)^2+2.911^2}} - \frac{2.911 (5 y+6)}{\frac{\sqrt{\pi } (2.911 -1) (5 y+6)}{2 \sqrt{15}}+\sqrt{\pi  \left(\frac{5 y+6}{2 \sqrt{15}}\right)^2+2.911^2}} \right)+30\ + \\ \nonumber 
&\sqrt{15 \pi } \left(e^{\frac{1}{60} (5 y+6)^2} (5 y+6) \erfc \left(\frac{5 y+6}{2 \sqrt{15}}\right)-2 e^{\frac{1}{60} (5 y+12)^2} (5 y+12) \erfc \left(\frac{5 y+12}{2 \sqrt{15}}\right)\right)+30 \ .
\end{align}
We obtained for the maximal absolute error the value $0.163052$.
We added an approximation 
error of $0.2$ to the approximation of the derivative.
Since we want to show that the approximation upper bounds the true
expression, the addition of the approximation error is required here.
We get a sequence of inequalities:

\begin{align}
\label{eq:ineqY}
&\sqrt{15 \pi } \left(e^{\frac{1}{60} (5 y+6)^2} (5 y+6) \erfc \left(\frac{5 y+6}{2 \sqrt{15}}\right)-2 e^{\frac{1}{60} (5 y+12)^2} (5 y+12) \erfc \left(\frac{5 y+12}{2 \sqrt{15}}\right)\right)+30\ \leq \\ \nonumber & \sqrt{15 \pi } \left(\frac{2.911 (5 y+6)}{\frac{\sqrt{\pi } (2.911 -1) (5 y+6)}{2 \sqrt{15}}+\sqrt{\pi  \left(\frac{5 y+6}{2 \sqrt{15}}\right)^2+2.911^2}}-\frac{2 \cdot 2.911 (5 y+12)}{\frac{\sqrt{\pi } (2.911 -1) (5 y+12)}{2 \sqrt{15}}+\sqrt{\pi  \left(\frac{5 y+12}{2 \sqrt{15}}\right)^2+2.911^2}}\right)+\\ \nonumber &30+0.2\ = \\ \nonumber &\frac{(30 \cdot 2.911) (5 y+6)}{(2.911 -1) (5 y+6)+\sqrt{(5 y+6)^2+\left(\frac{2 \sqrt{15} \cdot 2.911}{\sqrt{\pi }}\right)^2}}-\frac{2 (30 \cdot 2.911) (5 y+12)}{(2.911 -1) (5 y+12)+\sqrt{(5 y+12)^2+\left(\frac{2 \sqrt{15} \cdot 2.911}{\sqrt{\pi }}\right)^2}}+\\ \nonumber &30+0.2\ = \\ \nonumber &\left((0.2 +30) \left((2.911 -1) (5 y+12)+\sqrt{(5 y+12)^2+\left(\frac{2 \sqrt{15} \cdot 2.911}{\sqrt{\pi }}\right)^2}\right) \right.  \\ \nonumber &\left.  \left((2.911 -1) (5 y+6)+\sqrt{(5 y+6)^2+\left(\frac{2 \sqrt{15} \cdot 2.911}{\sqrt{\pi }}\right)^2}\right)-\right.  \\ \nonumber &\left.2 \cdot  30 \cdot 2.911 (5 y+12) \left((2.911 -1) (5 y+6)+\sqrt{(5 y+6)^2+\left(\frac{2 \sqrt{15} \cdot 2.911}{\sqrt{\pi }}\right)^2}\right)+\right.  \\ \nonumber &\left. 2.911 \cdot 30 (5 y+6) \left((2.911 -1) (5 y+12)+\sqrt{(5 y+12)^2+\left(\frac{2 \sqrt{15} \cdot 2.911}{\sqrt{\pi }}\right)^2}\right)\right)\\ \nonumber &\left(\left((2.911 -1) (5 y+6)+\sqrt{(5 y+6)^2+\left(\frac{2 \sqrt{15} \cdot 2.911}{\sqrt{\pi }}\right)^2}\right) \right.  \\ \nonumber &\left.\left((2.911 -1) (5 y+12)+\sqrt{(5 y+12)^2+\left(\frac{2 \sqrt{15} \cdot 2.911}{\sqrt{\pi }}\right)^2}\right)\right)^{-1} \ < \ 0 \ .
\end{align}
We explain this sequence of inequalities.
\begin{itemize}
\item First inequality: The approximation of \citet{Ren:07}
and then adding the error bound to ensure that the approximation
is larger than the true value.

\item First equality: The factor $2 \sqrt{15}$ and $2 \sqrt{\pi}$
  are factored out and canceled.

\item Second equality: Bringing all terms to the denominator
\begin{align}
&\left((2.911 -1) (5 y+6)+\sqrt{(5 y+6)^2+\left(\frac{2 \sqrt{15}
  2.911}{\sqrt{\pi }}\right)^2}\right) \\ \nonumber & \left((2.911-1) (5 y+12)+\sqrt{(5 y+12)^2+\left(\frac{2 \sqrt{15} \cdot 2.911}{\sqrt{\pi }}\right)^2}\right) \ .
\end{align}

\item Last inequality $<0$ is proofed in the following sequence of
  inequalities.
\end{itemize}
We look at the numerator of the last term in Eq.~\eqref{eq:ineqY}. We
have to proof that this numerator is smaller than zero in order to
proof the last inequality of  Eq.~\eqref{eq:ineqY}.
The numerator is
\begin{align}
&(0.2 +30) \left((2.911 -1) (5 y+12)+\sqrt{(5 y+12)^2+\left(\frac{2 \sqrt{15} \cdot 2.911}{\sqrt{\pi }}\right)^2}\right) \\ \nonumber & \left((2.911 -1) (5 y+6)+\sqrt{(5 y+6)^2+\left(\frac{2 \sqrt{15} \cdot 2.911}{\sqrt{\pi }}\right)^2}\right)\ - \\ \nonumber &2 \cdot 30 \cdot 2.911 (5 y+12) \left((2.911 -1) (5 y+6)+\sqrt{(5 y+6)^2+\left(\frac{2 \sqrt{15} \cdot 2.911}{\sqrt{\pi }}\right)^2}\right)+ \\ \nonumber &2.911 \cdot 30 (5 y+6) \left((2.911 -1) (5 y+12)+\sqrt{(5 y+12)^2+\left(\frac{2 \sqrt{15} \ . 2.911}{\sqrt{\pi }}\right)^2}\right) \ .
\end{align}
We now compute upper bounds for this numerator:

\begin{align}
&(0.2 +30) \left((2.911 -1) (5 y+12)+\sqrt{(5 y+12)^2+\left(\frac{2 \sqrt{15} \cdot 2.911}{\sqrt{\pi }}\right)^2}\right)\\ \nonumber & \left((2.911 -1) (5 y+6)+\sqrt{(5 y+6)^2+\left(\frac{2 \sqrt{15} \cdot 2.911}{\sqrt{\pi }}\right)^2}\right)-\\ \nonumber &2 \cdot 30 \cdot 2.911 (5 y+12) \left((2.911 -1) (5 y+6)+\sqrt{(5 y+6)^2+\left(\frac{2 \sqrt{15} \cdot 2.911}{\sqrt{\pi }}\right)^2}\right)+\\ \nonumber &2.911 \cdot 30 (5 y+6) \left((2.911 -1) (5 y+12)+\sqrt{(5 y+12)^2+\left(\frac{2 \sqrt{15} \cdot 2.911}{\sqrt{\pi }}\right)^2}\right)\ = \\ \nonumber &-1414.99 y^2-584.739 \sqrt{(5 y+6)^2+161.84} y+725.211 \sqrt{(5 y+12)^2+161.84} y-\\ \nonumber &5093.97 y-1403.37 \sqrt{(5 y+6)^2+161.84}+30.2 \sqrt{(5 y+6)^2+161.84} \sqrt{(5 y+12)^2+161.84}+\\ \nonumber &870.253 \sqrt{(5 y+12)^2+161.84}-4075.17\ < \\ \nonumber &-1414.99 y^2-584.739 \sqrt{(5 y+6)^2+161.84} y+725.211 \sqrt{(5 y+12)^2+161.84} y-\\ \nonumber &5093.97 y-1403.37 \sqrt{(6+5 \cdot (-0.1))^2+161.84}+30.2 \sqrt{(6+5 \cdot 0.1)^2+161.84} \sqrt{(12+5 \cdot 0.1)^2+161.84}+\\ \nonumber &870.253 \sqrt{(12+5 \cdot 0.1)^2+161.84}-4075.17\ = \\ \nonumber &-1414.99 y^2-584.739 \sqrt{(5 y+6)^2+161.84} y+725.211 \sqrt{(5 y+12)^2+161.84} y-5093.97 y-309.691\ < \\ \nonumber &y \left(-584.739 \sqrt{(5 y+6)^2+161.84}+725.211 \sqrt{(5 y+12)^2+161.84}-5093.97\right)-309.691\ < \\ \nonumber &-0.1 \left(725.211 \sqrt{(12+5 \cdot (-0.1))^2+161.84}-584.739 \sqrt{(6+5 \cdot 0.1)^2+161.84}-5093.97\right)-309.691 
 \ = \\ \nonumber &-208.604 \ .
\end{align}
For the first inequality we choose $y$ in the roots, so that 
positive terms maximally increase and negative terms maximally decrease.
The second inequality just removed the $y^2$ term which is always
negative, therefore increased the expression.
For the last inequality, the term in brackets 
is negative for all settings of $y$. 
Therefore we make the brackets as negative as possible 
and make the whole term positive by multiplying with $y=-0.1$.

Consequently 
\begin{align}
e^{\frac{(x+y)^2}{2 x}} \erfc \left(\frac{x+y}{\sqrt{2} \sqrt{x}}\right)-2 e^{\frac{(2 x+y)^2}{2 x}} \erfc \left(\frac{2 x+y}{\sqrt{2} \sqrt{x}}\right)
\end{align}
is strictly monotonically decreasing in $y$ for the minimal
$x=1.2$. 
\end{proof}

\begin{lemma}[Main subfunction below]
\label{proof:mainsubfunctionbelow} 
For $0.007 \leq x \leq 0.875$ and $-0.01 \leq y \leq 0.01$, 
the function
\begin{align}
e^{\frac{(x+y)^2}{2 x}} \erfc \left(\frac{x+y}{\sqrt{2} \sqrt{x}}\right)-2 e^{\frac{(2 x+y)^2}{2 x}} \erfc \left(\frac{2 x+y}{\sqrt{2} \sqrt{x}}\right)
\end{align}
smaller than zero, is strictly monotonically increasing in $x$
and strictly monotonically increasing in $y$ for the minimal $x=0.007=0.00875 \cdot 0.8$,
$x=0.56=0.7 \cdot 0.8$, $x=0.128=0.16 \cdot 0.8$, and $x=0.216=0.24 \cdot 0.9$ (lower
bound of $0.9$ on $\tau$).
\end{lemma}

\begin{proof}
We first consider the derivative of sub-function
Eq.~\eqref{eq:subfunction1} with respect to $x$.
The derivative of the function 
\begin{align}
e^{\frac{(x+y)^2}{2 x}} \erfc \left(\frac{x+y}{\sqrt{2} \sqrt{x}}\right)-2 e^{\frac{(2 x+y)^2}{2 x}} \erfc \left(\frac{2 x+y}{\sqrt{2} \sqrt{x}}\right)
\end{align}
with respect to $x$ is
\begin{align}
&\frac{\sqrt{\pi } \left(e^{\frac{(x+y)^2}{2 x}} (x-y) (x+y)
  \erfc \left(\frac{x+y}{\sqrt{2} \sqrt{x}}\right)-2 e^{\frac{(2
  x+y)^2}{2 x}} \left(4 x^2-y^2\right) \erfc \left(\frac{2
  x+y}{\sqrt{2} \sqrt{x}}\right)\right)+\sqrt{2} \sqrt{x} (3 x-y)}{2
  \sqrt{\pi } x^2}\ = \\ \nonumber &\frac{\sqrt{\pi } \left(e^{\frac{(x+y)^2}{2 x}}
  (x-y) (x+y) \erfc \left(\frac{x+y}{\sqrt{2} \sqrt{x}}\right)-2
  e^{\frac{(2 x+y)^2}{2 x}} (2 x+y) (2 x-y) \erfc \left(\frac{2
  x+y}{\sqrt{2} \sqrt{x}}\right)\right)+\sqrt{2} \sqrt{x} (3 x-y)}{2
  \sqrt{\pi } x^2} \ = \\ \nonumber &\frac{\sqrt{\pi } \left(\frac{e^{\frac{(x+y)^2}{2
  x}} (x-y) (x+y) \erfc \left(\frac{x+y}{\sqrt{2}
  \sqrt{x}}\right)}{\sqrt{2} \sqrt{x}}-\frac{2 e^{\frac{(2 x+y)^2}{2
  x}} (2 x+y) (2 x-y) \erfc \left(\frac{2 x+y}{\sqrt{2}
  \sqrt{x}}\right)}{\sqrt{2} \sqrt{x}}\right)+(3 x-y)}{\sqrt{2} 2
  \sqrt{\pi } \sqrt{x} x^2} \ .
\end{align}

We consider the numerator
\begin{align}
\sqrt{\pi } \left(\frac{e^{\frac{(x+y)^2}{2 x}} (x-y) (x+y)
  \erfc \left(\frac{x+y}{\sqrt{2} \sqrt{x}}\right)}{\sqrt{2}
  \sqrt{x}}-\frac{2 e^{\frac{(2 x+y)^2}{2 x}} (2 x+y) (2 x-y)
  \erfc \left(\frac{2 x+y}{\sqrt{2} \sqrt{x}}\right)}{\sqrt{2}
  \sqrt{x}}\right)+(3 x-y) \ .
\end{align}

For bounding this value, we use the approximation 
\begin{align}
e^{z^2} \erfc (z) \ \approx \ \frac{2.911}{\sqrt{\pi } (2.911 -1)
  z+\sqrt{\pi  z^2+2.911^2}} \ .
\end{align}
from \citet{Ren:07}.
We start with an error analysis of this approximation.
According to \citet{Ren:07} (Figure~1), the approximation 
error is both positive and negative in the range
$[0.175,1.33]$.  This range contains all possible
arguments of $\erfc $ that we consider in this subsection.
Numerically we maximized and minimized the approximation error of the
whole expression
\begin{align}
E(x,y) \ &= \ \left(\frac{e^{\frac{(x+y)^2}{2 x}} (x-y) (x+y)
  \erfc \left(\frac{x+y}{\sqrt{2} \sqrt{x}}\right)}{\sqrt{2}
  \sqrt{x}}-\frac{2 e^{\frac{(2 x+y)^2}{2 x}} (2 x-y) (2 x+y)
  \erfc \left(\frac{2 x+y}{\sqrt{2} \sqrt{x}}\right)}{\sqrt{2}
  \sqrt{x}}\right) \ - \\ \nonumber 
&\left(\frac{2.911 (x-y) (x+y)}{\left(\sqrt{2} \sqrt{x}\right) \left(\frac{\sqrt{\pi } (2.911 -1) (x+y)}{\sqrt{2} \sqrt{x}}+\sqrt{\pi  \left(\frac{x+y}{\sqrt{2} \sqrt{x}}\right)^2+2.911^2}\right)}\ - \right. \\ \nonumber 
&\left. \frac{2 \cdot 2.911 (2 x-y) (2 x+y)}{\left(\sqrt{2} \sqrt{x}\right) \left(\frac{\sqrt{\pi } (2.911 -1) (2 x+y)}{\sqrt{2} \sqrt{x}}+\sqrt{\pi  \left(\frac{2 x+y}{\sqrt{2} \sqrt{x}}\right)^2+2.911^2}\right)}\right) \ .
\end{align}
We numerically determined $-0.000228141 \leq E(x,y) \leq 0.00495688$ for 
$0.08 \leq x \leq 0.875$ and $-0.01 \leq y \leq 0.01$. 
We used different numerical optimization techniques like  
gradient based constraint BFGS algorithms and 
non-gradient-based Nelder-Mead methods with different start points.
Therefore our approximation is smaller than the function that we
approximate. 

We use an error gap of $-0.0003$ to countermand the error due to the
approximation. We have the sequences of inequalities using the approximation of 
\citet{Ren:07}:
\begin{align}
\label{eq:ineqX1}
&(3 x-y)+\left(\frac{e^{\frac{(x+y)^2}{2 x}} (x-y) (x+y)
  \erfc \left(\frac{x+y}{\sqrt{2} \sqrt{x}}\right)}{\sqrt{2}
  \sqrt{x}}-\frac{2 e^{\frac{(2 x+y)^2}{2 x}} (2 x-y) (2 x+y)
  \erfc \left(\frac{2 x+y}{\sqrt{2} \sqrt{x}}\right)}{\sqrt{2}
  \sqrt{x}}\right) \sqrt{\pi }\ \geq \\ \nonumber & (3 x-y)+\left(\frac{2.911 (x-y)
  (x+y)}{\left(\sqrt{\pi  \left(\frac{x+y}{\sqrt{2}
  \sqrt{x}}\right)^2+2.911^2}+\frac{(2.911 -1) \sqrt{\pi }
  (x+y)}{\sqrt{2} \sqrt{x}}\right) \left(\sqrt{2}
  \sqrt{x}\right)}\ - \right. \\ \nonumber &\left. \frac{2 (2 x-y) (2 x+y) 2.911}{\left(\sqrt{2}
  \sqrt{x}\right) \left(\sqrt{\pi  \left(\frac{2 x+y}{\sqrt{2}
  \sqrt{x}}\right)^2+2.911^2}+\frac{(2.911 -1) \sqrt{\pi } (2
  x+y)}{\sqrt{2} \sqrt{x}}\right)}\right) \sqrt{\pi }-0.0003 \ = \\ \nonumber 
  &(3 x-y)+\left(\frac{\left(\sqrt{2} \sqrt{x} 2.911\right) (x-y)
  (x+y)}{\left(\sqrt{\pi  (x+y)^2+2 \cdot 2.911^2 x}+(2.911 -1) (x+y)
  \sqrt{\pi }\right) \left(\sqrt{2} \sqrt{x}\right)}\ - \right. \\ \nonumber &\left. \frac{2 (2 x-y)
  (2 x+y) \left(\sqrt{2} \sqrt{x} 2.911\right)}{\left(\sqrt{2}
  \sqrt{x}\right) \left(\sqrt{\pi  (2 x+y)^2+2 \cdot 2.911^2 x}+(2.911 -1)
  (2 x+y) \sqrt{\pi }\right)}\right) \sqrt{\pi }-0.0003 \ = \\ \nonumber 
  &(3 x-y)+2.911 \left(\frac{(x-y) (x+y)}{(2.911 -1)
  (x+y)+\sqrt{(x+y)^2+\frac{2 \cdot 2.911^2 x}{\pi }}}\ - \right. \\ \nonumber &\left. \frac{2 (2 x-y) (2
  x+y)}{(2.911 -1) (2 x+y)+\sqrt{(2 x+y)^2+\frac{2 \cdot 2.911^2 x}{\pi
  }}}\right)-0.0003\ \geq \\ \nonumber & (3 x-y)+2.911 \left(\frac{(x-y)
  (x+y)}{(2.911 -1) (x+y)+\sqrt{\left(\frac{2.911^2}{\pi
  }\right)^2+(x+y)^2+\frac{2 \cdot 2.911^2 x}{\pi }+\frac{2 \cdot 2.911^2 y}{\pi
  }}}\ - \right. \\ \nonumber &\left. \frac{2 (2 x-y) (2 x+y)}{(2.911 -1) (2 x+y)+\sqrt{(2
  x+y)^2+\frac{2 \cdot 2.911^2 x}{\pi }}}\right)-0.0003 \ = \\ \nonumber 
  &(3 x-y)+2.911 \left(\frac{(x-y) (x+y)}{(2.911 -1)
  (x+y)+\sqrt{\left(x+y+\frac{2.911^2}{\pi }\right)^2}}\ - \right. \\ \nonumber &\left. \frac{2 (2
  x-y) (2 x+y)}{(2.911 -1) (2 x+y)+\sqrt{(2 x+y)^2+\frac{2 \cdot 2.911^2
  x}{\pi }}}\right)-0.0003 \ = \\ \nonumber &(3 x-y)+2.911 \left(\frac{(x-y)
  (x+y)}{2.911 (x+y)+\frac{2.911^2}{\pi }}-\frac{2 (2 x-y) (2
  x+y)}{(2.911 -1) (2 x+y)+\sqrt{(2 x+y)^2+\frac{2 \cdot 2.911^2 x}{\pi
  }}}\right)-0.0003 \ = \\ \nonumber &(3 x-y)+\frac{(x-y)
  (x+y)}{(x+y)+\frac{2.911}{\pi }}-\frac{2 (2 x-y) (2 x+y)
  2.911}{(2.911 -1) (2 x+y)+\sqrt{(2 x+y)^2+\frac{2 \cdot 2.911^2 x}{\pi
  }}}-0.0003 \ = \\ \nonumber
&(3 x-y)+\frac{(x-y) (x+y)}{(x+y)+\frac{2.911}{\pi
  }}-\frac{2 (2 x-y) (2 x+y) 2.911}{(2.911 -1) (2 x+y)+\sqrt{(2
  x+y)^2+\frac{2 \cdot 2.911^2 x}{\pi }}}-0.0003 \ = \\ \nonumber &
\left(-2 (2 x-y) 2.911
  \left((x+y)+\frac{2.911}{\pi }\right) (2
  x+y) \right. \ + \\ \nonumber &\left. \left((x+y)+\frac{2.911}{\pi }\right) (3 x-y-0.0003)
  \left((2.911 -1) (2 x+y)+\sqrt{(2 x+y)^2+\frac{2 \cdot 2.911^2 x}{\pi
  }}\right)\right. \ + \\ \nonumber &\left. (x-y) (x+y) \left((2.911 -1) (2 x+y)+\sqrt{(2
  x+y)^2+\frac{2 \cdot 2.911^2 x}{\pi
  }}\right)\right) \\ \nonumber &\left(\left((x+y)+\frac{2.911}{\pi }\right) \left((2.911 -1)
  (2 x+y)+\sqrt{(2 x+y)^2+\frac{2 \cdot 2.911^2 x}{\pi
  }}\right)\right)^{-1} \ = \\  \nonumber &
\left(-8 x^3-8 x^2 y+4 x^2 \sqrt{(2 x+y)^2+5.39467 x}-10.9554 x^2+2 x y^2-2 y^2 \sqrt{(2 x+y)^2+5.39467 x} \ +\right.\\\nonumber &\left.1.76901 x y+2 x y \sqrt{(2 x+y)^2+5.39467 x}+2.7795 x \sqrt{(2 x+y)^2+5.39467 x} \ -\right.\\\nonumber &\left.0.9269 y \sqrt{(2 x+y)^2+5.39467 x}-0.00027798 \sqrt{(2 x+y)^2+5.39467 x}-0.00106244 x \ +\right.\\\nonumber &\left.2 y^3+3.62336 y^2-0.00053122 y\right) \\\nonumber &\left(\left((x+y)+\frac{2.911}{\pi }\right) \left((2.911 -1) (2 x+y)+\sqrt{(2 x+y)^2+\frac{2 \cdot 2.911^2 x}{\pi }}\right)\right)^{-1}\ = \\\nonumber &\left(-8 x^3+\left(4 x^2+2 x y+2.7795 x-2 y^2-0.9269 y-0.00027798\right) \sqrt{(2 x+y)^2+5.39467 x} \ -\right.\\\nonumber &\left.8 x^2 y-10.9554 x^2+2 x y^2+1.76901 x y-0.00106244 x+2 y^3+3.62336 y^2-0.00053122 y\right) \\\nonumber &\left(\left((x+y)+\frac{2.911}{\pi }\right) \left((2.911 -1) (2 x+y)+\sqrt{(2 x+y)^2+\frac{2 \cdot 2.911^2 x}{\pi }}\right)\right)^{-1}  \ > \ 0 \ .
\end{align}
We explain this sequence of inequalities:
\begin{itemize} 
\item First inequality: The approximation of \citet{Ren:07}
and then subtracting an error gap of $0.0003$.

\item Equalities: The factor $\sqrt{2} \sqrt{x}$ is factored out and
  canceled. 

\item Second inequality: adds a positive term in the first root to
  obtain a binomial form. The term containing the root 
is positive and the root is in the denominator, 
therefore the whole term becomes smaller.

\item Equalities: solve for the term and factor out.

\item Bringing all terms to the denominator
$\left((x+y)+\frac{2.911}{\pi }\right) \left((2.911 -1) (2 x+y)+\sqrt{(2 x+y)^2+\frac{2 \cdot 2.911^2 x}{\pi }}\right)$.

\item Equalities: Multiplying out and expanding terms.

\item Last inequality $>0$ is proofed in the following sequence of
  inequalities.
\end{itemize}

We look at the numerator of the last expression of 
Eq.~\eqref{eq:ineqX1}, which we show to be
positive in order to show $>0$ in 
Eq.~\eqref{eq:ineqX1}. The numerator is 
\begin{align}
&-8 x^3+\left(4 x^2+2 x y+2.7795 x-2 y^2-0.9269 y-0.00027798\right) \sqrt{(2 x+y)^2+5.39467 x}\ -\\\nonumber &8 x^2 y-10.9554 x^2+2 x y^2+1.76901 x y-0.00106244 x+2 y^3+3.62336 y^2-0.00053122 y \ .
\end{align}
The factor $4 x^2+2 x y+2.7795 x-2 y^2-0.9269 y-0.00027798$
in front of the root is positive:
\begin{align}
4 x^2+2 x y+2.7795 x-2 y^2-0.9269 y-0.00027798 \ > \\ \nonumber
-2 y^2+0.007 \cdot 2 y-0.9269 y+4 \cdot 0.007^2+2.7795 \cdot 0.007-0.00027798  \ = \\ \nonumber
-2 y^2-0.9129 y+2.77942 \ = -2 (y + 1.42897)(y - 0.972523) \ > 0 \ .
\end{align}
If the term that does not contain the root would be positive, 
then everything is positive and we have proofed the the numerator is
positive. Therefore we consider the case that the term that does
not contain the root is negative.
The term that contains the root must be larger than 
the other term in absolute values. 
\begin{align}
&-\left(-8 x^3-8 x^2 y-10.9554 x^2+2 x y^2+1.76901 x y-0.00106244 x+2 y^3+3.62336 y^2-0.00053122 y\right)\ < \\ \nonumber &\left(4 x^2+2 x y+2.7795 x-2 y^2-0.9269 y-0.00027798\right) \sqrt{(2 x+y)^2+5.39467 x} \ . 
\end{align}
Therefore the squares of the root term have to be larger 
than the square of the other term to show $>0$ in 
Eq.~\eqref{eq:ineqX1}.
Thus, we have the inequality:
\begin{align}
&\left(-8 x^3-8 x^2 y-10.9554 x^2+2 x y^2+1.76901 x y-0.00106244 x+2 y^3+3.62336 y^2-0.00053122 y\right)^2\ < \\ \nonumber&\left(4 x^2+2 x y+2.7795 x-2 y^2-0.9269 y-0.00027798\right)^2 \left((2 x+y)^2+5.39467 x\right)\ . 
\end{align}

This is equivalent to
\begin{align}
& 0\ < \ \left(4 x^2+2 x y+2.7795 x-2 y^2-0.9269 y-0.00027798\right)^2 \left((2 x+y)^2+5.39467 x\right)-\\ \nonumber &\left(-8 x^3-8 x^2 y-10.9554 x^2+2 x y^2+1.76901 x y-0.00106244 x+2 y^3+3.62336 y^2-0.00053122 y\right)^2\ = \\ 
\nonumber &x \cdot 4.168614250 \cdot 10^{-7}-y^2 2.049216091 \cdot 10^{-7}-0.0279456 x^5+\\ \nonumber &43.0875 x^4 y+30.8113 x^4+43.1084 x^3 y^2+68.989 x^3 y+41.6357 x^3+10.7928 x^2 y^3-13.1726 x^2 y^2-\\ \nonumber &27.8148 x^2 y-0.00833715 x^2+0.0139728 x y^4+5.47537 x y^3+\\ \nonumber &4.65089 x y^2+0.00277916 x y-10.7858 y^5-12.2664 y^4+0.00436492 y^3\ .
\end{align}
We obtain the inequalities: 
\begin{align}
&x \cdot 4.168614250 \cdot 10^{-7}-y^2 2.049216091 \cdot 10^{-7}-0.0279456 x^5+\\ \nonumber &43.0875 x^4 y+30.8113 x^4+43.1084 x^3 y^2+68.989 x^3 y+41.6357 x^3+10.7928 x^2 y^3-\\ \nonumber &13.1726 x^2 y^2-27.8148 x^2 y-0.00833715 x^2+\\ \nonumber &0.0139728 x y^4+5.47537 x y^3+4.65089 x y^2+0.00277916 x y-10.7858 y^5-12.2664 y^4+0.00436492 y^3\ > \\ \nonumber & 
x \cdot 4.168614250 \cdot 10^{-7}-(0.01)^2 2.049216091 \cdot 10^{-7}-0.0279456 x^5+\\ \nonumber &0.0 \cdot  43.0875 x^4+30.8113 x^4+43.1084 (0.0)^2 x^3+0.0 \cdot 68.989 x^3+41.6357 x^3+\\ \nonumber &10.7928 (0.0)^3 x^2-13.1726 (0.01)^2 x^2-27.8148 (0.01) x^2-0.00833715 x^2+\\ \nonumber &0.0139728 (0.0)^4 x+5.47537 (0.0)^3 x+4.65089 (0.0)^2 x+\\ \nonumber &0.0 \cdot 0.00277916 x-10.7858 (0.01)^5-12.2664 (0.01)^4+0.00436492 (0.0)^3\ = \\ \nonumber &x \cdot 4.168614250 \cdot 10^{-7}-1.237626189 \cdot 10^{-7}-0.0279456 x^5+30.8113 x^4+41.6357 x^3-0.287802 x^2\ > \\ \nonumber &-\left(\frac{x}{0.007}\right)^3 1.237626189 \cdot 10^{-7}+30.8113 x^4-(0.875) \cdot 0.0279456 x^4+41.6357 x^3-\frac{(0.287802 x) x^2}{0.007}\ = \\ \nonumber &30.7869 x^4+0.160295 x^3\ >\ 0\ .
\end{align}
We used $x \geq 0.007$ and $x \leq 0.875$ (reducing the negative $x^4$-term to a
$x^3$-term).
We have proofed the last inequality $>0$ of Eq.~\eqref{eq:ineqX1}.

Consequently the derivative is always positive independent of $y$,
thus 
\begin{align}
e^{\frac{(x+y)^2}{2 x}} \erfc \left(\frac{x+y}{\sqrt{2} \sqrt{x}}\right)-2 e^{\frac{(2 x+y)^2}{2 x}} \erfc \left(\frac{2 x+y}{\sqrt{2} \sqrt{x}}\right)
\end{align}
is strictly monotonically increasing in $x$.

Next we show that the 
sub-function Eq.~\eqref{eq:subfunction1} is smaller
than zero.
We consider the limit:
\begin{align}
&\lim_{x \to \infty}e^{\frac{(x+y)^2}{2 x}}
  \erfc \left(\frac{x+y}{\sqrt{2} \sqrt{x}}\right) \ - \ 2 e^{\frac{(2 x+y)^2}{2 x}} \erfc \left(\frac{2 x+y}{\sqrt{2} \sqrt{x}}\right)
\ = \ 0
\end{align}
The limit follows from Lemma~\ref{lem:Abramowitz}.
Since the function is monotonic increasing in $x$, it has to approach
$0$ from below. Thus,
\begin{align}
e^{\frac{(x+y)^2}{2 x}} \erfc \left(\frac{x+y}{\sqrt{2} \sqrt{x}}\right)-2 e^{\frac{(2 x+y)^2}{2 x}} \erfc \left(\frac{2 x+y}{\sqrt{2} \sqrt{x}}\right)
\end{align}
is smaller than zero.

We now consider the derivative of sub-function
Eq.~\eqref{eq:subfunction1} with respect to $y$.
We proofed that sub-function
Eq.~\eqref{eq:subfunction1} is  strictly monotonically increasing 
independent of $y$. 
In the proof of Theorem~\ref{th:s2Increase}, we need the minimum
of  sub-function
Eq.~\eqref{eq:subfunction1}. First, we are interested in the
derivative of sub-function
Eq.~\eqref{eq:subfunction1} with respect to $y$
for the minimum $x=0.007=7/1000$.

Consequently, we insert the minimum $x=0.007=7/1000$ into the sub-function
Eq.~\eqref{eq:subfunction1}:
\begin{align}
&e^{\left(\frac{y}{\sqrt{2} \sqrt{\frac{7}{1000}}}+\frac{\sqrt{\frac{7}{1000}}}{\sqrt{2}}\right)^2} \erfc \left(\frac{y}{\sqrt{2} \sqrt{\frac{7}{1000}}}+\frac{\sqrt{\frac{7}{1000}}}{\sqrt{2}}\right)- \\ \nonumber &2 e^{\left(\frac{y}{\sqrt{2} \sqrt{\frac{7}{1000}}}+\sqrt{2} \sqrt{\frac{7}{1000}}\right)^2} \erfc \left(\frac{y}{\sqrt{2} \sqrt{\frac{7}{1000}}}+\sqrt{2} \sqrt{\frac{7}{1000}}\right)\ = \\ \nonumber &e^{\frac{500 y^2}{7}+y+\frac{7}{2000}} \erfc \left(\frac{1000 y+7}{20 \sqrt{35}}\right)-2 e^{\frac{(500 y+7)^2}{3500}} \erfc \left(\frac{500 y+7}{10 \sqrt{35}}\right) \ .
\end{align}
The derivative of this function with respect to $y$ is
\begin{align}
&\left(\frac{1000 y}{7}+1\right) e^{\frac{500 y^2}{7}+y+\frac{7}{2000}} \erfc \left(\frac{1000 y+7}{20 \sqrt{35}}\right)- \\ \nonumber &\frac{1}{7} 4 e^{\frac{(500 y+7)^2}{3500}} (500 y+7) \erfc \left(\frac{500 y+7}{10 \sqrt{35}}\right)+20 \sqrt{\frac{5}{7 \pi }}\ > \\ \nonumber&\left(1+\frac{1000 \cdot (-0.01)}{7}\right) e^{-0.01 +\frac{7}{2000}+\frac{500 \cdot (-0.01)^2}{7}} \erfc \left(\frac{7+1000+(-0.01)}{20 \sqrt{35}}\right)- \\ \nonumber &\frac{1}{7} 4 e^{\frac{(7+500 \cdot 0.01)^2}{3500}} (7+500 \cdot 0.01) \erfc \left(\frac{7+500 \cdot 0.01}{10 \sqrt{35}}\right)+20 \sqrt{\frac{5}{7 \pi }}\ >\ 3.56\ .
\end{align}
For the first inequality, we use Lemma~\ref{lem:xeErfc}.
Lemma~\ref{lem:xeErfc} says that 
the function $x e^{x^2}\erfc (x) $ has the sign of $x$ and is
monotonically increasing to $\frac {1} {\sqrt {\pi}} $.
Consequently, we inserted the maximal $y=0.01$ to
make the negative term more negative and the minimal $y=-0.01$
to make the positive term less positive.

Consequently 
\begin{align}
e^{\frac{(x+y)^2}{2 x}} \erfc \left(\frac{x+y}{\sqrt{2} \sqrt{x}}\right)-2 e^{\frac{(2 x+y)^2}{2 x}} \erfc \left(\frac{2 x+y}{\sqrt{2} \sqrt{x}}\right)
\end{align}
is strictly monotonically increasing in $y$ for the minimal
$x=0.007$.

Next, we consider $x=0.7 \cdot 0.8=0.56$, which is the maximal $\nu=0.7$
and minimal $\tau=0.8$.
We insert the minimum $x=0.56=56/100$ into the  sub-function
Eq.~\eqref{eq:subfunction1}:
\begin{align}
&e^{\left(\frac{y}{\sqrt{2} \sqrt{\frac{56}{100}}}+\frac{\sqrt{\frac{56}{100}}}{\sqrt{2}}\right)^2} \erfc \left(\frac{y}{\sqrt{2} \sqrt{\frac{56}{100}}}+\frac{\sqrt{\frac{56}{100}}}{\sqrt{2}}\right)-\\ \nonumber &2 e^{\left(\frac{y}{\sqrt{2} \sqrt{\frac{56}{100}}}+\sqrt{2} \sqrt{\frac{56}{100}}\right)^2} \erfc \left(\frac{y}{\sqrt{2} \sqrt{\frac{56}{100}}}+\sqrt{2} \sqrt{\frac{56}{100}}\right)\ .
\end{align}
The derivative with respect to $y$ is:
\begin{align}
&\frac{5 e^{\left(\frac{5 y}{2 \sqrt{7}}+\frac{\sqrt{7}}{5}\right)^2} \left(\frac{5 y}{2 \sqrt{7}}+\frac{\sqrt{7}}{5}\right) \erfc \left(\frac{5 y}{2 \sqrt{7}}+\frac{\sqrt{7}}{5}\right)}{\sqrt{7}}-\\ \nonumber &\frac{10 e^{\left(\frac{5 y}{2 \sqrt{7}}+\frac{2 \sqrt{7}}{5}\right)^2} \left(\frac{5 y}{2 \sqrt{7}}+\frac{2 \sqrt{7}}{5}\right) \erfc \left(\frac{5 y}{2 \sqrt{7}}+\frac{2 \sqrt{7}}{5}\right)}{\sqrt{7}}+\frac{5}{\sqrt{7 \pi }}\ > \ \\ \nonumber &
\frac{5 e^{\left(\frac{\sqrt{7}}{5}-\frac{0.01 \cdot 5}{2 \sqrt{7}}\right)^2} \left(\frac{\sqrt{7}}{5}-\frac{0.01 \cdot 5}{2 \sqrt{7}}\right) \erfc \left(\frac{\sqrt{7}}{5}-\frac{0.01 \cdot 5}{2 \sqrt{7}}\right)}{\sqrt{7}}-\\ \nonumber &\frac{10 e^{\left(\frac{2 \sqrt{7}}{5}+\frac{0.01 \cdot 5}{2 \sqrt{7}}\right)^2} \left(\frac{2 \sqrt{7}}{5}+\frac{0.01 \cdot 5}{2 \sqrt{7}}\right) \erfc \left(\frac{2 \sqrt{7}}{5}+\frac{0.01 \cdot 5}{2 \sqrt{7}}\right)}{\sqrt{7}}+\frac{5}{\sqrt{7 \pi }}\ > \ 0.00746 \ .
\end{align}
For the first inequality we applied Lemma~\ref{lem:xeErfc}
which states that the function $x e^{x^2} \erfc (x)$ is
monotonically increasing.
Consequently, we inserted the maximal $y=0.01$ to
make the negative term more negative and the minimal $y=-0.01$
to make the positive term less positive.

Consequently 
\begin{align}
e^{\frac{(x+y)^2}{2 x}} \erfc \left(\frac{x+y}{\sqrt{2} \sqrt{x}}\right)-2 e^{\frac{(2 x+y)^2}{2 x}} \erfc \left(\frac{2 x+y}{\sqrt{2} \sqrt{x}}\right)
\end{align}
is strictly monotonically increasing in $y$ for $x=0.56$.

Next, we consider $x=0.16 \cdot 0.8=0.128$, which is the minimal $\tau=0.8$.
We insert the minimum $x=0.128=128/1000$ into the  sub-function
Eq.~\eqref{eq:subfunction1}:
\begin{align}
&e^{\left(\frac{y}{\sqrt{2} \sqrt{\frac{128}{1000}}}+\frac{\sqrt{\frac{128}{1000}}}{\sqrt{2}}\right)^2} \erfc \left(\frac{y}{\sqrt{2} \sqrt{\frac{128}{1000}}}+\frac{\sqrt{\frac{128}{1000}}}{\sqrt{2}}\right)-\\ \nonumber &2 e^{\left(\frac{y}{\sqrt{2} \sqrt{\frac{128}{1000}}}+\sqrt{2} \sqrt{\frac{128}{1000}}\right)^2} \erfc \left(\frac{y}{\sqrt{2} \sqrt{\frac{128}{1000}}}+\sqrt{2} \sqrt{\frac{128}{1000}}\right)=\\ \nonumber &e^{\frac{125 y^2}{32}+y+\frac{8}{125}} \erfc \left(\frac{125 y+16}{20 \sqrt{10}}\right)-2 e^{\frac{(125 y+32)^2}{4000}} \erfc \left(\frac{125 y+32}{20 \sqrt{10}}\right)\ .
\end{align}
The derivative with respect to $y$ is:
\begin{align}
&\frac{1}{16} \left(e^{\frac{125 y^2}{32}+y+\frac{8}{125}} (125 y+16) \erfc \left(\frac{125 y+16}{20 \sqrt{10}}\right)-\right. \\ \nonumber &\left.2 e^{\frac{(125 y+32)^2}{4000}} (125 y+32) \erfc \left(\frac{125 y+32}{20 \sqrt{10}}\right)+20 \sqrt{\frac{10}{\pi }}\right)\ > \ \\ \nonumber &\frac{1}{16} \left((16+125 (-0.01)) e^{-0.01 +\frac{8}{125}+\frac{125 (-0.01)^2}{32}} \erfc \left(\frac{16+125 (-0.01)}{20 \sqrt{10}}\right)-\right. \\ \nonumber &\left.2 e^{\frac{(32+125 0.01)^2}{4000}} (32+125 0.01) \erfc \left(\frac{32+125 0.01}{20 \sqrt{10}}\right)+20 \sqrt{\frac{10}{\pi }}\right)\ > \ 0.4468 \ .
\end{align}
For the first inequality we applied Lemma~\ref{lem:xeErfc}
which states that the function $x e^{x^2} \erfc (x)$ is
monotonically increasing.
Consequently, we inserted the maximal $y=0.01$ to
make the negative term more negative and the minimal $y=-0.01$
to make the positive term less positive.

Consequently 
\begin{align}
e^{\frac{(x+y)^2}{2 x}} \erfc \left(\frac{x+y}{\sqrt{2} \sqrt{x}}\right)-2 e^{\frac{(2 x+y)^2}{2 x}} \erfc \left(\frac{2 x+y}{\sqrt{2} \sqrt{x}}\right)
\end{align}
is strictly monotonically increasing in $y$ for $x=0.128$.

Next, we consider $x=0.24 \cdot 0.9=0.216$, which is the minimal
$\tau=0.9$ (here we consider $0.9$ as lower bound for $\tau$).
We insert the minimum $x=0.216=216/1000$ into the  sub-function
Eq.~\eqref{eq:subfunction1}:
\begin{align}
&e^{\left(\frac{y}{\sqrt{2} \sqrt{\frac{216}{1000}}}+\frac{\sqrt{\frac{216}{1000}}}{\sqrt{2}}\right)^2} \erfc \left(\frac{y}{\sqrt{2} \sqrt{\frac{216}{1000}}}+\frac{\sqrt{\frac{216}{1000}}}{\sqrt{2}}\right)-\\ \nonumber &2 e^{\left(\frac{y}{\sqrt{2} \sqrt{\frac{216}{1000}}}+\sqrt{2} \sqrt{\frac{216}{1000}}\right)^2} \erfc \left(\frac{y}{\sqrt{2} \sqrt{\frac{216}{1000}}}+\sqrt{2} \sqrt{\frac{216}{1000}}\right)=\\ \nonumber &e^{\frac{(125 y+27)^2}{6750}} \erfc \left(\frac{125 y+27}{15 \sqrt{30}}\right)-2 e^{\frac{(125 y+54)^2}{6750}} \erfc \left(\frac{125 y+54}{15 \sqrt{30}}\right)
\end{align}
The derivative with respect to $y$ is: 
\begin{align}
&\frac{1}{27} \left(e^{\frac{(125 y+27)^2}{6750}} (125 y+27) \erfc \left(\frac{125 y+27}{15 \sqrt{30}}\right)-\right. \\ \nonumber &\left.2 e^{\frac{(125 y+54)^2}{6750}} (125 y+54) \erfc \left(\frac{125 y+54}{15 \sqrt{30}}\right)+15 \sqrt{\frac{30}{\pi }}\right)\ > \ \\ \nonumber&\frac{1}{27} \left((27+125 (-0.01)) e^{\frac{(27+125 (-0.01))^2}{6750}} \erfc \left(\frac{27+125 (-0.01)}{15 \sqrt{30}}\right)-\right. \\ \nonumber &\left.2 e^{\frac{(54+125 0.01)^2}{6750}} (54+125 0.01) \erfc \left(\frac{54+125 0.01}{15 \sqrt{30}}\right)+15 \sqrt{\frac{30}{\pi }}\right))\ > \ 0.211288\ .
\end{align}
For the first inequality we applied Lemma~\ref{lem:xeErfc}
which states that the function $x e^{x^2} \erfc (x)$ is
monotonically increasing.
Consequently, we inserted the maximal $y=0.01$ to
make the negative term more negative and the minimal $y=-0.01$
to make the positive term less positive.

Consequently 
\begin{align}
e^{\frac{(x+y)^2}{2 x}} \erfc \left(\frac{x+y}{\sqrt{2} \sqrt{x}}\right)-2 e^{\frac{(2 x+y)^2}{2 x}} \erfc \left(\frac{2 x+y}{\sqrt{2} \sqrt{x}}\right)
\end{align}
is strictly monotonically increasing in $y$ for $x=0.216$. 
\end{proof}

\begin{lemma}[Monotone Derivative]
\label{proof:monotonederivative}
For $\lambda=\lambda_{\rm 01}$, $\alpha=\alpha_{\rm 01}$
and the domain 
$-0.1 \leq \mu \leq 0.1$, 
$-0.1 \leq \omega \leq 0.1$,
$0.00875 \leq \nu \leq 0.7$, and 
$0.8 \leq \tau \leq 1.25$.
We are interested of the derivative of
\begin{align} 
\tau \left(e^{\left(\frac{\mu \omega+\nu \tau}{\sqrt{2} \sqrt{\nu \tau}}\right)^2} \erfc \left(\frac{\mu \omega+\nu \tau}{\sqrt{2} \sqrt{\nu \tau}}\right)-2 e^{\left(\frac{\mu \omega+2 \cdot \nu \tau}{\sqrt{2} \sqrt{\nu \tau}}\right)^2} \erfc \left(\frac{\mu \omega+2 \cdot \nu \tau}{\sqrt{2} \sqrt{\nu \tau}}\right)\right)\ . 
\end{align}

The derivative of the equation above with
respect to
\begin{itemize}
\item $\nu$ is larger than zero;
\item $\tau$ is smaller than zero for maximal
$\nu=0.7$, $\nu=0.16$, and $\nu=0.24$ (with
$0.9 \leq \tau$);
\item $y=\mu \omega$ is larger than zero for $\nu
\tau=0.00875 \cdot 0.8=0.007$, $\nu
\tau=0.7 \cdot 0.8=0.56$, $\nu
\tau=0.16 \cdot 0.8=0.128$, and $\nu
\tau=0.24 \cdot 0.9=0.216$.
\end{itemize}

\end{lemma}

\begin{proof}
We consider the domain:
$-0.1 \leq \mu \leq 0.1$, 
$-0.1 \leq \omega \leq 0.1$,
$0.00875 \leq \nu \leq 0.7$, and 
$0.8 \leq \tau \leq 1.25$.

We use Lemma~\ref{lem:subfunction1} to determine the derivatives.
Consequently, the derivative of 
\begin{align}
\tau \left(e^{\left(\frac{\mu \omega+\nu \tau}{\sqrt{2} \sqrt{\nu \tau}}\right)^2} \erfc \left(\frac{\mu \omega+\nu \tau}{\sqrt{2} \sqrt{\nu \tau}}\right)-2 e^{\left(\frac{\mu \omega+2  \nu \tau}{\sqrt{2} \sqrt{\nu \tau}}\right)^2} \erfc \left(\frac{\mu \omega+2  \nu \tau}{\sqrt{2} \sqrt{\nu \tau}}\right)\right)
\end{align}
with respect to  $\nu$ is larger than zero, which follows
directly from  Lemma~\ref{lem:subfunction1} using the chain rule.

{Consequently,} the derivative of 
\begin{align}
\tau \left(e^{\left(\frac{\mu \omega+\nu \tau}{\sqrt{2} \sqrt{\nu \tau}}\right)^2} \erfc \left(\frac{\mu \omega+\nu \tau}{\sqrt{2} \sqrt{\nu \tau}}\right)-2 e^{\left(\frac{\mu \omega+2  \nu \tau}{\sqrt{2} \sqrt{\nu \tau}}\right)^2} \erfc \left(\frac{\mu \omega+2  \nu \tau}{\sqrt{2} \sqrt{\nu \tau}}\right)\right)
\end{align}
with respect to  $y=\mu \omega$ is larger than zero for $\nu
\tau=0.00875 \cdot 0.8=0.007$, $\nu
\tau=0.7 \cdot 0.8=0.56$, $\nu
\tau=0.16 \cdot 0.8=0.128$, and $\nu
\tau=0.24 \cdot 0.9=0.216$,
which also follows
directly from  Lemma~\ref{lem:subfunction1}.

We now consider the derivative with respect to $\tau$,
which is not trivial since $\tau$ is a factor of the whole expression.
The sub-expression should be maximized as it appears with
negative sign in the mapping for $\nu$.

First,
we consider the function for  
the largest $\nu = 0.7$ and the largest $y=\mu \omega = 0.01$ 
for determining the derivative with respect to $\tau$.

The expression becomes
\begin{align}
&\tau \left(e^{\left(\frac{\frac{7  \tau}{10}+\frac{1}{100}}{\sqrt{2} \sqrt{\frac{7  \tau}{10}}}\right)^2} \erfc \left(\frac{\frac{7  \tau}{10}+\frac{1}{100}}{\sqrt{2} \sqrt{\frac{7  \tau}{10}}}\right)-
2 e^{\left(\frac{\frac{2 \cdot 7  \tau}{10}+\frac{1}{100}}{\sqrt{2} \sqrt{\frac{7  \tau}{10}}}\right)^2} \erfc \left(\frac{\frac{2 \cdot 7  \tau}{10}+\frac{1}{100}}{\sqrt{2} \sqrt{\frac{7  \tau}{10}}}\right)\right) \ .
\end{align}

The derivative with respect to $\tau$ is 
\begin{align}
&\left(\sqrt{\pi } \left(e^{\frac{(70 \tau+1)^2}{14000 \tau}} (700 \tau (7  \tau+20)-1) \erfc \left(\frac{70 \tau+1}{20 \sqrt{35} \sqrt{\tau}}\right)\ -\right.\right.\\\nonumber &\left.\left.2 e^{\frac{(140 \tau+1)^2}{14000 \tau}} (2800 \tau (7  \tau+5)-1) \erfc \left(\frac{140 \tau+1}{20 \sqrt{35} \sqrt{\tau}}\right)\right)+20 \sqrt{35} (210 \tau-1) \sqrt{\tau}\right)\\ \nonumber &\left(14000 \sqrt{\pi } \tau\right)^{-1} \ .
\end{align}

We are considering only the numerator and use again the approximation
of \citet{Ren:07}.
The error analysis on the whole numerator gives an approximation error $97<E<186$. Therefore
we add 200 to the numerator when we use the approximation \citet{Ren:07}.
We obtain the inequalities:
\begin{align}
&\sqrt{\pi } \left(e^{\frac{(70 \tau+1)^2}{14000 \tau}} (700 \tau (7  \tau+20)-1) \erfc \left(\frac{70 \tau+1}{20 \sqrt{35} \sqrt{\tau}}\right) \ -\right.\\\nonumber &\left.2 e^{\frac{(140 \tau+1)^2}{14000 \tau}} (2800 \tau (7  \tau+5)-1) \erfc \left(\frac{140 \tau+1}{20 \sqrt{35} \sqrt{\tau}}\right)\right)+20 \sqrt{35} (210 \tau-1) \sqrt{\tau}\ \leq \\ \nonumber & \sqrt{\pi } \left(\frac{2.911 (700 \tau (7  \tau+20)-1)}{\frac{\sqrt{\pi } (2.911 -1) (70 \tau+1)}{20 \sqrt{35} \sqrt{\tau}}+\sqrt{\pi  \left(\frac{70 \tau+1}{20 \sqrt{35} \sqrt{\tau}}\right)^2+2.911^2}} \ -\right.\\\nonumber &\left.\frac{2 \cdot 2.911 (2800 \tau (7  \tau+5)-1)}{\frac{\sqrt{\pi } (2.911 -1) (140 \tau+1)}{20 \sqrt{35} \sqrt{\tau}}+\sqrt{\pi  \left(\frac{140 \tau+1}{20 \sqrt{35} \sqrt{\tau}}\right)^2+2.911^2}}\right) \\ \nonumber &\ +20 \sqrt{35} (210 \tau-1) \sqrt{\tau}+200\ = \\ \nonumber &
\sqrt{\pi } \left(\frac{(700 \tau (7  \tau+20)-1) \left(20 \cdot  \sqrt{35} \cdot  2.911 \sqrt{\tau}\right)}{\sqrt{\pi } (2.911 -1) (70 \tau+1)+\sqrt{\left(20 \cdot   2.911 \sqrt{35} \sqrt{\tau}\right)^2+\pi  (70 \tau+1)^2}} \ -\right.\\\nonumber &\left.\frac{2 (2800 \tau (7  \tau+5)-1) \left(20 \cdot \sqrt{35} \cdot  2.911 \sqrt{\tau}\right)}{\sqrt{\pi } (2.911 -1) (140 \tau+1)+\sqrt{\left(20 \cdot \sqrt{35} \cdot 2.911 \sqrt{\tau}\right)^2+\pi  (140 \tau+1)^2}}\right)+\\ \nonumber &\left(20 \sqrt{35} (210 \tau-1) \sqrt{\tau}+200\right)\ = \\ \nonumber &
\left(\left(20 \sqrt{35} (210 \tau-1) \sqrt{\tau}+200\right) \left(\sqrt{\pi } (2.911 -1) (70 \tau+1)+\sqrt{\left(20 \cdot \sqrt{35} \cdot 2.911 \sqrt{\tau}\right)^2+\pi  (70 \tau+1)^2}\right)\right.\\\nonumber &\left. \left(\sqrt{\pi } (2.911 -1) (140 \tau+1)+\sqrt{\left(20 \cdot \sqrt{35} \cdot 2.911 \sqrt{\tau}\right)^2+\pi  (140 \tau+1)^2}\right)+\right.\\\nonumber &\left.2.911 \cdot 20 \sqrt{35} \sqrt{\pi } (700 \tau (7  \tau+20)-1) \sqrt{\tau}  \right.\\\nonumber &\left.\left(\sqrt{\pi } (2.911 -1) (140 \tau+1)+\sqrt{\left(20 \cdot \sqrt{35} \cdot 2.911 \sqrt{\tau}\right)^2+\pi  (140 \tau+1)^2}\right)- \right.\\\nonumber &\left.\sqrt{\pi } 2 \cdot 20 \cdot \sqrt{35}\cdot  2.911 (2800 \tau (7  \tau+5)-1) \right.\\\nonumber &\left. \sqrt{\tau} \left(\sqrt{\pi } (2.911 -1) (70 \tau+1)+\sqrt{\left(20 \cdot \sqrt{35} \cdot 2.911 \sqrt{\tau}\right)^2+\pi  (70 \tau+1)^2}\right)\right)\\ \nonumber &\left(\left(\sqrt{\pi } (2.911 -1) (70 \tau+1)+\sqrt{\left(20 \sqrt{35}\cdot  2.911 \cdot \sqrt{\tau}\right)^2+\pi  (70 \tau+1)^2}\right) \right.\\\nonumber &\left. \left(\sqrt{\pi } (2.911 -1) (140 \tau+1)+\sqrt{\left(20 \sqrt{35} \cdot 2.911 \cdot \sqrt{\tau}\right)^2+\pi  (140 \tau+1)^2}\right)\right)^{-1} \ .
\end{align}
After applying the approximation
of \citet{Ren:07} and adding 200,
we first factored out $20 \sqrt{35} \sqrt{\tau}$.
Then we brought all terms to the same denominator.

We now consider the numerator:
\begin{align}
\label{eq:start}
&\left(20 \sqrt{35} (210 \tau-1) \sqrt{\tau}+200\right) \left(\sqrt{\pi } (2.911 -1) (70 \tau+1)+\sqrt{\left(20 \cdot  \sqrt{35} \cdot  2.911 \sqrt{\tau}\right)^2+\pi  (70 \tau+1)^2}\right) \\\nonumber 
&\left(\sqrt{\pi } (2.911 -1) (140 \tau+1)+\sqrt{\left(20 \cdot  \sqrt{35} \cdot  2.911 \sqrt{\tau}\right)^2+\pi  (140 \tau+1)^2}\right) \ +\\ \nonumber
&2.911 \cdot 20 \sqrt{35} \sqrt{\pi } (700 \tau (7  \tau+20)-1) \sqrt{\tau} \\\nonumber & \left(\sqrt{\pi } (2.911 -1) (140 \tau+1)+\sqrt{\left(20 \cdot \sqrt{35}\cdot  2.911 \sqrt{\tau}\right)^2+\pi  (140 \tau+1)^2}\right)-\\ \nonumber 
&\sqrt{\pi } 2 \cdot 20 \cdot  \sqrt{35} \cdot  2.911 (2800 \tau (7  \tau+5)-1) \sqrt{\tau}  \\ \nonumber 
&\left(\sqrt{\pi } (2.911 -1) (70 \tau+1)+\sqrt{\left(20 \cdot  \sqrt{35} \cdot  2.911 \sqrt{\tau}\right)^2+\pi  (70 \tau+1)^2}\right)\ = \\ \nonumber 
&-1.70658\times 10^7 \sqrt{\pi  (70 \tau+1)^2+118635 \tau} \tau^{3/2}+\\ \nonumber 
&4200 \sqrt{35} \sqrt{\pi  (70 \tau+1)^2+118635 \tau} \sqrt{\pi  (140 \tau+1)^2+118635 \tau} \tau^{3/2}\ + \\ \nonumber 
&8.60302\times 10^6 \sqrt{\pi  (140 \tau+1)^2+118635 \tau} \tau^{3/2}-2.89498\times 10^7 \tau^{3/2}\ - \\ \nonumber 
&1.21486\times 10^7 \sqrt{\pi  (70 \tau+1)^2+118635 \tau} \tau^{5/2}+8.8828\times 10^6 \sqrt{\pi  (140 \tau+1)^2+118635 \tau} \tau^{5/2}\ - \\ \nonumber 
&2.43651\times 10^7 \tau^{5/2}-1.46191\times 10^9 \tau^{7/2}+2.24868\times 10^7 \tau^2+94840.5 \sqrt{\pi  (70 \tau+1)^2+118635 \tau} \tau\ + \\ \nonumber 
&47420.2 \sqrt{\pi  (140 \tau+1)^2+118635 \tau} \tau+481860 \tau+710.354 \sqrt{\tau}\ + \\ \nonumber 
&820.213 \sqrt{\tau} \sqrt{\pi  (70 \tau+1)^2+118635 \tau}+677.432 \sqrt{\pi  (70 \tau+1)^2+118635 \tau}\ - \\ \nonumber 
&1011.27 \sqrt{\tau} \sqrt{\pi  (140 \tau+1)^2+118635 \tau}\ - \\ \nonumber 
&20 \sqrt{35} \sqrt{\tau} \sqrt{\pi  (70 \tau+1)^2+118635 \tau} \sqrt{\pi  (140 \tau+1)^2+118635 \tau}\ + \\ \nonumber 
&200 \sqrt{\pi  (70 \tau+1)^2+118635 \tau} \sqrt{\pi  (140 \tau+1)^2+118635 \tau}\ + \\ \nonumber 
&677.432 \sqrt{\pi  (140 \tau+1)^2+118635 \tau}+2294.57\ = \\ \nonumber
&-2.89498\times 10^7 \tau^{3/2}-2.43651\times 10^7 \tau^{5/2}-1.46191\times 10^9 \tau^{7/2}\ + \\ \nonumber 
&\left(-1.70658\times 10^7 \tau^{3/2}-1.21486\times 10^7 \tau^{5/2}+94840.5 \tau+820.213 \sqrt{\tau}+677.432\right) \\ \nonumber 
& \sqrt{\pi  (70 \tau+1)^2+118635 \tau}\ + \\ \nonumber 
&\left(8.60302\times 10^6 \tau^{3/2}+8.8828\times 10^6 \tau^{5/2}+47420.2 \tau-1011.27 \sqrt{\tau}+677.432\right) \\ \nonumber 
& \sqrt{\pi  (140 \tau+1)^2+118635 \tau}\ + \\ \nonumber 
&\left(4200 \sqrt{35} \tau^{3/2}-20 \sqrt{35} \sqrt{\tau}+200\right) \sqrt{\pi  (70 \tau+1)^2+118635 \tau} \sqrt{\pi  (140 \tau+1)^2+118635 \tau}\ + \\ \nonumber 
&2.24868\times 10^7 \tau^2+481860. \tau+710.354 \sqrt{\tau}+2294.57\ \leq \\ \nonumber 
&-2.89498\times 10^7 \tau^{3/2}-2.43651\times 10^7 \tau^{5/2}-1.46191\times 10^9 \tau^{7/2}+\\ \nonumber 
&\left(-1.70658\times 10^7 \tau^{3/2}-1.21486\times 10^7 \tau^{5/2}+820.213 \sqrt{1.25}+1.25 \cdot 94840.5+677.432\right)\\ \nonumber 
& \sqrt{\pi  (70 \tau+1)^2+118635 \tau}+\\ \nonumber 
&\left(8.60302\times 10^6 \tau^{3/2}+8.8828\times 10^6 \tau^{5/2}-1011.27 \sqrt{0.8}+1.25 \cdot 47420.2+677.432\right) \\ \nonumber 
&\sqrt{\pi  (140 \tau+1)^2+118635 \tau}+\\ \nonumber 
&\left(4200 \sqrt{35} \tau^{3/2}-20 \sqrt{35} \sqrt{\tau}+200\right) \\ \nonumber 
&\sqrt{\pi  (70 \tau+1)^2+118635 \tau} \sqrt{\pi  (140 \tau+1)^2+118635 \tau}+\\ \nonumber 
&2.24868\times 10^7 \tau^2+710.354 \sqrt{1.25}+1.25 \cdot 481860+2294.57\ = \\ \nonumber 
&-2.89498\times 10^7 \tau^{3/2}-2.43651\times 10^7 \tau^{5/2}-1.46191\times 10^9 \tau^{7/2}+\\ \nonumber 
&\left(-1.70658\times 10^7 \tau^{3/2}-1.21486\times 10^7 \tau^{5/2}+120145.\right) \sqrt{\pi  (70 \tau+1)^2+118635 \tau}+\\ \nonumber 
&\left(8.60302\times 10^6 \tau^{3/2}+8.8828\times 10^6 \tau^{5/2}+59048.2\right) \sqrt{\pi  (140 \tau+1)^2+118635 \tau}+\\ \nonumber 
&\left(4200 \sqrt{35} \tau^{3/2}-20 \sqrt{35} \sqrt{\tau}+200\right) \sqrt{\pi  (70 \tau+1)^2+118635 \tau} \sqrt{\pi  (140 \tau+1)^2+118635 \tau}+\\ \nonumber 
&2.24868\times 10^7 \tau^2+605413\ = \\ \nonumber
&-2.89498\times 10^7 \tau^{3/2}-2.43651\times 10^7 \tau^{5/2}-1.46191\times 10^9 \tau^{7/2}+\\ \nonumber 
&\left(8.60302\times 10^6 \tau^{3/2}+8.8828\times 10^6 \tau^{5/2}+59048.2\right) \sqrt{19600 \pi  (\tau+1.94093) (\tau+0.0000262866)}+\\ \nonumber 
&\left(-1.70658\times 10^7 \tau^{3/2}-1.21486\times 10^7 \tau^{5/2}+120145.\right) \sqrt{4900 \pi  (\tau+7.73521) (\tau+0.0000263835)}+\\ \nonumber 
&\left(4200 \sqrt{35} \tau^{3/2}-20 \sqrt{35} \sqrt{\tau}+200\right) \\ \nonumber 
&\sqrt{19600 \pi  (\tau+1.94093) (\tau+0.0000262866)} \sqrt{4900 \pi  (\tau+7.73521) (\tau+0.0000263835)}+\\ \nonumber 
&2.24868\times 10^7 \tau^2+605413\ \leq \\ \nonumber 
&-2.89498\times 10^7 \tau^{3/2}-2.43651\times 10^7 \tau^{5/2}-1.46191\times 10^9 \tau^{7/2}+\\ \nonumber
&\left(8.60302\times 10^6 \tau^{3/2}+8.8828\times 10^6 \tau^{5/2}+59048.2\right) \sqrt{19600 \pi (\tau+1.94093) \tau}+\\ \nonumber 
&\left(-1.70658\times 10^7 \tau^{3/2}-1.21486\times 10^7 \tau^{5/2}+120145.\right) \sqrt{4900 \pi  1.00003 (\tau+7.73521) \tau}+\\ \nonumber 
&\left(4200 \sqrt{35} \tau^{3/2}-20 \sqrt{35} \sqrt{\tau}+200\right) \sqrt{19600 \pi  1.00003 (\tau+1.94093) \tau} \\ \nonumber 
&\sqrt{4900 \pi  1.00003 (\tau+7.73521) \tau}+\\ \nonumber 
&2.24868\times 10^7 \tau^2+605413\ = \\ \nonumber 
&-2.89498\times 10^7 \tau^{3/2}-2.43651\times 10^7 \tau^{5/2}-1.46191\times 10^9 \tau^{7/2}+\\ \nonumber 
&\left(-3.64296\times 10^6 \tau^{3/2}+7.65021\times 10^8 \tau^{5/2}+6.15772\times 10^6 \tau\right) \\ \nonumber 
&\sqrt{\tau+1.94093} \sqrt{\tau+7.73521}+2.24868\times 10^7 \tau^2+\\ \nonumber 
&\left(2.20425\times 10^9 \tau^3+2.13482\times 10^9 \tau^2+1.46527\times 10^7 \sqrt{\tau}\right) \sqrt{\tau+1.94093}+\\ \nonumber 
&\left(-1.5073\times 10^9 \tau^3-2.11738\times 10^9 \tau^2+1.49066\times 10^7 \sqrt{\tau}\right) \sqrt{\tau+7.73521}+605413 \ \leq \\ \nonumber
&\sqrt{1.25 +1.94093} \sqrt{1.25 +7.73521} \left(-3.64296\times 10^6 \tau^{3/2}+7.65021\times 10^8 \tau^{5/2}+6.15772\times 10^6 \tau\right)+\\ \nonumber 
&\sqrt{1.25 +1.94093} \left(2.20425\times 10^9 \tau^3+2.13482\times 10^9 \tau^2+1.46527\times 10^7 \sqrt{\tau}\right)+\\ \nonumber 
&\sqrt{0.8 +7.73521} \left(-1.5073\times 10^9 \tau^3-2.11738\times 10^9 \tau^2+1.49066\times 10^7 \sqrt{\tau}\right)-\\ \nonumber 
&2.89498\times 10^7 \tau^{3/2}-2.43651\times 10^7 \tau^{5/2}-1.46191\times 10^9 \tau^{7/2}+2.24868\times 10^7 \tau^2+605413 \ = \\ \nonumber 
&-4.84561\times 10^7 \tau^{3/2}+4.07198\times 10^9 \tau^{5/2}-1.46191\times 10^9 \tau^{7/2}-\\ \nonumber 
&4.66103\times 10^8 \tau^3-2.34999\times 10^9 \tau^2+\\ \nonumber 
&3.29718\times 10^7 \tau+6.97241\times 10^7 \sqrt{\tau}+605413\ \leq \\ \nonumber 
& \frac{605413 \tau^{3/2}}{0.8^{3/2}}-4.84561\times 10^7 \tau^{3/2}+\\ \nonumber 
&4.07198\times 10^9 \tau^{5/2}-1.46191\times 10^9 \tau^{7/2}-\\ \nonumber 
&4.66103\times 10^8 \tau^3-2.34999\times 10^9 \tau^2+\frac{3.29718\times 10^7 \sqrt{\tau} \tau}{\sqrt{0.8}}+\frac{6.97241\times 10^7 \tau \sqrt{\tau}}{0.8}\ = \\ \nonumber 
&\tau^{3/2} \left(-4.66103\times 10^8 \tau^{3/2}-1.46191\times 10^9 \tau^2-2.34999\times 10^9 \sqrt{\tau}+\right.\\ \nonumber 
&\left.4.07198\times 10^9 \tau+7.64087\times 10^7\right)\ \leq \\ \nonumber 
& \tau^{3/2} \left(-4.66103\times 10^8 \tau^{3/2}-1.46191\times 10^9 \tau^2+\frac{7.64087\times 10^7 \sqrt{\tau}}{\sqrt{0.8}}-\right.\\ \nonumber 
&\left.2.34999\times 10^9 \sqrt{\tau}+4.07198\times 10^9 \tau\right)\ = \\ \nonumber 
&\tau^2 \left(-1.46191\times 10^9 \tau^{3/2}+4.07198\times 10^9 \sqrt{\tau}-4.66103\times 10^8 \tau-2.26457\times 10^9\right)\ \leq \\ \nonumber 
&  \left(-2.26457\times 10^9+4.07198\times 10^9 \sqrt{0.8}-4.66103\times 10^8 0.8-1.46191\times 10^9 0.8^{3/2}\right) \tau^2\ = \\ \nonumber 
&-4.14199\times 10^7 \tau^2\ < \ 0 \ . 
\end{align}

First we expanded the term (multiplied it out).
The we put the terms multiplied by the same square root into brackets.
The next inequality sign stems from inserting the maximal value of $1.25$ for $\tau$ for
some positive terms and value of $0.8$ for negative terms. 
These terms are then expanded at the $=$-sign.
The next equality factors the terms under the squared root.
We decreased the negative term by setting
$\tau=\tau+0.0000263835$ under the root.
We increased positive terms by setting  
$\tau+0.000026286=1.00003\tau$ and
$\tau+0.000026383 = 1.00003\tau$
under the root for positive terms.
The positive terms are increase, since
$\frac{0.8 + 0.000026383}{0.8}=1.00003$, thus
$\tau+0.000026286<\tau+0.000026383 \leq
1.00003\tau$.
For the next inequality we decreased negative terms by inserting
$\tau=0.8$ and increased positive terms by inserting
$\tau=1.25$. The next equality expands the terms.
We use upper bound of $1.25$ and lower bound of $0.8$ to obtain terms with
corresponding exponents of $\tau$.

For the last $\leq$-sign we used the function
\begin{align}
\label{eq:funcA}
-1.46191\times 10^9 \tau^{3/2}+4.07198\times 10^9
  \sqrt{\tau}-4.66103\times 10^8 \tau-2.26457\times 10^9
\end{align}
The derivative of this function is
\begin{align}
-2.19286\times 10^9 \sqrt{\tau}+\frac{2.03599\times
  10^9}{\sqrt{\tau}}-4.66103\times 10^8
\end{align}
and the second order derivative is
\begin{align}
-\frac{1.01799\times 10^9}{\tau^{3/2}}-\frac{1.09643\times
  10^9}{\sqrt{\tau}} \ < \ 0 \ .
\end{align}
The derivative at 0.8 is smaller than zero:
\begin{align}
&-2.19286\times 10^9 \sqrt{0.8}-4.66103\times 10^8+\frac{2.03599\times
  10^9}{\sqrt{0.8}}
= \\ \nonumber 
& -1.51154\times 10^8 \ < \ 0 \ .
\end{align}
Since the second order derivative is negative, the derivative
decreases with increasing $\tau$. Therefore the derivative is
negative for all values of  $\tau$ that we consider, that is, the
function Eq.~\eqref{eq:funcA} is strictly monotonically decreasing.
The maximum of the function Eq.~\eqref{eq:funcA} is therefore at $0.8$.
We inserted $0.8$ to obtain the maximum.

Consequently, the derivative of
\begin{align}
\label{eq:subx1}
\tau \left(e^{\left(\frac{\mu \omega+\nu \tau}{\sqrt{2} \sqrt{\nu \tau}}\right)^2} \erfc \left(\frac{\mu \omega+\nu \tau}{\sqrt{2} \sqrt{\nu \tau}}\right)-2 e^{\left(\frac{\mu \omega+2  \nu \tau}{\sqrt{2} \sqrt{\nu \tau}}\right)^2} \erfc \left(\frac{\mu \omega+2  \nu \tau}{\sqrt{2} \sqrt{\nu \tau}}\right)\right)
\end{align}
with respect to $\tau$ is smaller than zero for maximal $\nu = 0.7$.

Next,
we consider the function for  
the largest $\nu = 0.16$ and the largest $y=\mu \omega = 0.01$ 
for determining the derivative with respect to $\tau$.

The expression becomes
\begin{align}
&\tau \left(e^{\left(\frac{\frac{16 \tau}{100}+\frac{1}{100}}{\sqrt{2} \sqrt{\frac{16 \tau}{100}}}\right)^2} \erfc \left(\frac{\frac{16 \tau}{100}+\frac{1}{100}}{\sqrt{2} \sqrt{\frac{16 \tau}{100}}}\right)- e^{\left(\frac{\frac{2\ 16 \tau}{100}+\frac{1}{100}}{\sqrt{2} \sqrt{\frac{16 \tau}{100}}}\right)^2} \erfc \left(\frac{\frac{2\ 16 \tau}{100}+\frac{1}{100}}{\sqrt{2} \sqrt{\frac{16 \tau}{100}}}\right)\right) \ .
\end{align}

The derivative with respect to $\tau$ is 
\begin{align}
&\left(\sqrt{\pi } \left(e^{\frac{(16 \tau+1)^2}{3200 \tau}} (128 \tau (2  \tau+25)-1) \erfc \left(\frac{16 \tau+1}{40 \sqrt{2} \sqrt{\tau}}\right)-\right.\right.\\\nonumber 
&\left.\left.2 e^{\frac{(32  \tau+1)^2}{3200 \tau}} (128 \tau (8 \tau+25)-1) \erfc \left(\frac{32  \tau+1}{40 \sqrt{2} \sqrt{\tau}}\right)\right)+40 \sqrt{2} (48 \tau-1) \sqrt{\tau}\right)\\ \nonumber 
&\left(3200 \sqrt{\pi } \tau\right)^{-1} \ .
\end{align}

We are considering only the numerator and use again the approximation
of \citet{Ren:07}.
The error analysis on the whole numerator gives an approximation error $1.1<E<12$. Therefore
we add 20 to the numerator when we use the approximation of \citet{Ren:07}.
We obtain the inequalities:
\begin{align}
&\sqrt{\pi } \left(e^{\frac{(16 \tau+1)^2}{3200 \tau}} (128 \tau (2  \tau+25)-1) \erfc \left(\frac{16 \tau+1}{40 \sqrt{2} \sqrt{\tau}}\right)-\right.\\\nonumber 
&\left.2 e^{\frac{(32  \tau+1)^2}{3200 \tau}} (128 \tau (8 \tau+25)-1) \erfc \left(\frac{32  \tau+1}{40 \sqrt{2} \sqrt{\tau}}\right)\right)+40 \sqrt{2} (48 \tau-1) \sqrt{\tau}\ \leq \\ \nonumber 
& \sqrt{\pi } \left(\frac{2.911 (128 \tau (2  \tau+25)-1)}{\frac{\sqrt{\pi } (2.911 -1) (16 \tau+1)}{40 \sqrt{2} \sqrt{\tau}}+\sqrt{\pi  \left(\frac{16 \tau+1}{40 \sqrt{2} \sqrt{\tau}}\right)^2+2.911^2}} \ -\right.\\\nonumber 
&\left.\frac{2 \cdot 2.911 (128 \tau (8 \tau+25)-1)}{\frac{\sqrt{\pi } (2.911 -1) (32  \tau+1)}{40 \sqrt{2} \sqrt{\tau}}+\sqrt{\pi  \left(\frac{32  \tau+1}{40 \sqrt{2} \sqrt{\tau}}\right)^2+2.911^2}}\right)\\ \nonumber 
&\ +40 \sqrt{2} (48 \tau-1) \sqrt{\tau}+20\ = \\ \nonumber 
&\sqrt{\pi } \left(\frac{(128 \tau (2  \tau+25)-1) \left(40 \sqrt{2} 2.911 \sqrt{\tau}\right)}{\sqrt{\pi } (2.911 -1) (16 \tau+1)+\sqrt{\left(40 \sqrt{2} 2.911 \sqrt{\tau}\right)^2+\pi  (16 \tau+1)^2}} \ -\right.\\\nonumber
&\left.\frac{2 (128 \tau (8 \tau+25)-1) \left(40 \sqrt{2} 2.911 \sqrt{\tau}\right)}{\sqrt{\pi } (2.911 -1) (32  \tau+1)+\sqrt{\left(40 \sqrt{2} 2.911 \sqrt{\tau}\right)^2+\pi  (32  \tau+1)^2}}\right)+\\ \nonumber 
&40 \sqrt{2} (48 \tau-1) \sqrt{\tau}+20\ = \\ \nonumber 
&\left(\left(40 \sqrt{2} (48 \tau-1) \sqrt{\tau}+20\right) \left(\sqrt{\pi } (2.911 -1) (16 \tau+1)+\sqrt{\left(40 \sqrt{2} 2.911 \sqrt{\tau}\right)^2+\pi  (16 \tau+1)^2}\right) \right.\\\nonumber 
&\left. \left(\sqrt{\pi } (2.911 -1) (32  \tau+1)+\sqrt{\left(40 \sqrt{2} 2.911 \sqrt{\tau}\right)^2+\pi  (32  \tau+1)^2}\right)++\right.\\\nonumber 
&\left.2.911 \cdot 40 \sqrt{2} \sqrt{\pi } (128 \tau (2  \tau+25)-1) \sqrt{\tau}  \right.\\\nonumber 
&\left.\left(\sqrt{\pi } (2.911 -1) (32  \tau+1)+\sqrt{\left(40 \sqrt{2} 2.911 \sqrt{\tau}\right)^2+\pi  (32  \tau+1)^2}\right)- \right.\\\nonumber 
&\left.2 \sqrt{\pi } 40 \sqrt{2} 2.911 (128 \tau (8 \tau+25)-1)  \right.\\\nonumber 
&\left. \sqrt{\tau} \left(\sqrt{\pi } (2.911 -1) (16 \tau+1)+\sqrt{\left(40 \sqrt{2} 2.911 \sqrt{\tau}\right)^2+\pi  (16 \tau+1)^2}\right)\right)\\ \nonumber 
&\left(\left(\sqrt{\pi } (2.911 -1) (32  \tau+1)+\sqrt{\left(40 \sqrt{2} 2.911 \sqrt{\tau}\right)^2+\pi  (32  \tau+1)^2}\right)  \right.\\\nonumber 
&\left. \left(\sqrt{\pi } (2.911 -1) (32  \tau+1)+\sqrt{\left(40 \sqrt{2} 2.911 \sqrt{\tau}\right)^2+\pi  (32  \tau+1)^2}\right)\right)^{-1} \ .
\end{align}
After applying the approximation
of \citet{Ren:07} and adding 20,
we first factored out $40 \sqrt{2} \sqrt{\tau}$.
Then we brought all terms to the same denominator.

We now consider the numerator:
\begin{align}
\label{eq:start1}
&\left(40 \sqrt{2} (48 \tau-1) \sqrt{\tau}+20\right) \left(\sqrt{\pi } (2.911 -1) (16 \tau+1)+\sqrt{\left(40 \sqrt{2} 2.911 \sqrt{\tau}\right)^2+\pi  (16 \tau+1)^2}\right) \\\nonumber
& \left(\sqrt{\pi } (2.911 -1) (32  \tau+1)+\sqrt{\left(40 \sqrt{2} 2.911 \sqrt{\tau}\right)^2+\pi  (32  \tau+1)^2}\right) \ +\\ \nonumber
&2.911 \cdot 40 \sqrt{2} \sqrt{\pi } (128 \tau (2  \tau+25)-1) \sqrt{\tau}  \\\nonumber 
&\left(\sqrt{\pi } (2.911 -1) (32  \tau+1)+\sqrt{\left(40 \sqrt{2} 2.911 \sqrt{\tau}\right)^2+\pi  (32  \tau+1)^2}\right)-\\ \nonumber 
&2 \sqrt{\pi } 40 \sqrt{2} 2.911 (128 \tau (8 \tau+25)-1) \sqrt{\tau} \\ \nonumber 
&\left(\sqrt{\pi } (2.911 -1) (16 \tau+1)+\sqrt{\left(40 \sqrt{2} 2.911 \sqrt{\tau}\right)^2+\pi  (16 \tau+1)^2}\right)\ = \\ \nonumber 
& -1.86491\times 10^6 \sqrt{\pi  (16 \tau+1)^2+27116.5 \tau} \tau^{3/2}+\\ \nonumber 
&1920 \sqrt{2} \sqrt{\pi  (16 \tau+1)^2+27116.5 \tau} \sqrt{\pi  (32  \tau+1)^2+27116.5 \tau} \tau^{3/2}+ \\ \nonumber 
&940121 \sqrt{\pi  (32  \tau+1)^2+27116.5 \tau} \tau^{3/2}-3.16357\times 10^6 \tau^{3/2}- \\ \nonumber 
&303446 \sqrt{\pi  (16 \tau+1)^2+27116.5 \tau} \tau^{5/2}+221873 \sqrt{\pi  (32  \tau+1)^2+27116.5 \tau} \tau^{5/2}-608588 \tau^{5/2}- \\ \nonumber 
&8.34635\times 10^6 \tau^{7/2}+117482. \tau^2+2167.78 \sqrt{\pi  (16 \tau+1)^2+27116.5 \tau} \tau+ \\ \nonumber 
&1083.89 \sqrt{\pi  (32  \tau+1)^2+27116.5 \tau} \tau+ \\ \nonumber 
&11013.9 \tau+339.614 \sqrt{\tau}+392.137 \sqrt{\tau} \sqrt{\pi  (16 \tau+1)^2+27116.5 \tau}+ \\ \nonumber 
&67.7432 \sqrt{\pi  (16 \tau+1)^2+27116.5 \tau}-483.478 \sqrt{\tau} \sqrt{\pi  (32  \tau+1)^2+27116.5 \tau}- \\ \nonumber 
&40 \sqrt{2} \sqrt{\tau} \sqrt{\pi  (16 \tau+1)^2+27116.5 \tau} \sqrt{\pi  (32  \tau+1)^2+27116.5 \tau}+ \\ \nonumber 
&20 \sqrt{\pi  (16 \tau+1)^2+27116.5 \tau} \sqrt{\pi  (32  \tau+1)^2+27116.5 \tau}+ \\ \nonumber 
&67.7432 \sqrt{\pi  (32  \tau+1)^2+27116.5 \tau}+229.457\ =\\ \nonumber 
&-3.16357\times 10^6 \tau^{3/2}-608588 \tau^{5/2}-8.34635\times 10^6 \tau^{7/2}+\\ \nonumber 
&\left(-1.86491\times 10^6 \tau^{3/2}-303446 \tau^{5/2}+2167.78 \tau+392.137 \sqrt{\tau}+67.7432\right) \\ \nonumber
& \sqrt{\pi  (16 \tau+1)^2+27116.5 \tau}+ \\ \nonumber 
&\left(940121 \tau^{3/2}+221873 \tau^{5/2}+1083.89 \tau-483.478 \sqrt{\tau}+67.7432\right)  \\ \nonumber 
&\sqrt{\pi  (32  \tau+1)^2+27116.5 \tau}+ \\ \nonumber 
&\left(1920 \sqrt{2} \tau^{3/2}-40 \sqrt{2} \sqrt{\tau}+20\right) \sqrt{\pi  (16 \tau+1)^2+27116.5 \tau} \sqrt{\pi  (32  \tau+1)^2+27116.5 \tau}+ \\ \nonumber
&117482. \tau^2+11013.9 \tau+339.614 \sqrt{\tau}+229.457\ \leq \\ \nonumber 
&-3.16357\times 10^6 \tau^{3/2}-608588 \tau^{5/2}-8.34635\times 10^6 \tau^{7/2}+\\ \nonumber
&\left(-1.86491\times 10^6 \tau^{3/2}-303446 \tau^{5/2}+392.137 \sqrt{1.25}+1.25 2167.78+67.7432\right) \\ \nonumber 
&\sqrt{\pi  (16 \tau+1)^2+27116.5 \tau}+\\ \nonumber 
&\left(940121 \tau^{3/2}+221873 \tau^{5/2}-483.478 \sqrt{0.8}+1.25 1083.89+67.7432\right) \\ \nonumber 
&\sqrt{\pi  (32  \tau+1)^2+27116.5 \tau}+\\ \nonumber
&\left(1920 \sqrt{2} \tau^{3/2}-40 \sqrt{2} \sqrt{\tau}+20\right) \sqrt{\pi  (16 \tau+1)^2+27116.5 \tau} \sqrt{\pi  (32  \tau+1)^2+27116.5 \tau}+\\ \nonumber 
&117482. \tau^2+339.614 \sqrt{1.25}+1.25 11013.9+229.457\ = \\ \nonumber 
&-3.16357\times 10^6 \tau^{3/2}-608588 \tau^{5/2}-8.34635\times 10^6 \tau^{7/2}+\\ \nonumber 
&\left(-1.86491\times 10^6 \tau^{3/2}-303446 \tau^{5/2}+3215.89\right) \sqrt{\pi  (16 \tau+1)^2+27116.5 \tau}+\\ \nonumber 
&\left(940121 \tau^{3/2}+221873 \tau^{5/2}+990.171\right) \sqrt{\pi  (32  \tau+1)^2+27116.5 \tau}+\\ \nonumber 
&\left(1920 \sqrt{2} \tau^{3/2}-40 \sqrt{2} \sqrt{\tau}+20\right) \sqrt{\pi  (16 \tau+1)^2+27116.5 \tau} \sqrt{\pi  (32  \tau+1)^2+27116.5 \tau}+\\ \nonumber 
&117482 \tau^2+14376.6\ = \\ \nonumber
&-3.16357\times 10^6 \tau^{3/2}-608588 \tau^{5/2}-8.34635\times 10^6 \tau^{7/2}+\\ \nonumber 
&\left(940121 \tau^{3/2}+221873 \tau^{5/2}+990.171\right) \sqrt{1024 \pi  (\tau+8.49155) (\tau+0.000115004)}+\\ \nonumber
&\left(-1.86491\times 10^6 \tau^{3/2}-303446 \tau^{5/2}+3215.89\right) \sqrt{256 \pi  (\tau+33.8415) (\tau+0.000115428)}+\\ \nonumber 
&\left(1920 \sqrt{2} \tau^{3/2}-40 \sqrt{2} \sqrt{\tau}+20\right) \sqrt{1024 \pi  (\tau+8.49155) (\tau+0.000115004)} \\ \nonumber
&\sqrt{256 \pi  (\tau+33.8415) (\tau+0.000115428)}+\\ \nonumber 
&117482. \tau^2+14376.6\ \leq \\ \nonumber 
&-3.16357\times 10^6 \tau^{3/2}-608588 \tau^{5/2}-8.34635\times 10^6 \tau^{7/2}+\\ \nonumber
&\left(940121 \tau^{3/2}+221873 \tau^{5/2}+990.171\right) \sqrt{1024 \pi  1.00014 (\tau+8.49155) \tau}+\\ \nonumber 
&\left(1920 \sqrt{2} \tau^{3/2}-40 \sqrt{2} \sqrt{\tau}+20\right) \sqrt{256 \pi  1.00014 (\tau+33.8415) \tau} \sqrt{1024 \pi  1.00014 (\tau+8.49155) \tau}+\\ \nonumber 
&\left(-1.86491\times 10^6 \tau^{3/2}-303446 \tau^{5/2}+3215.89\right) \sqrt{256 \pi  (\tau+33.8415) \tau}+\\ \nonumber
&117482. \tau^2+14376.6\ = \\ \nonumber 
&-3.16357\times 10^6 \tau^{3/2}-608588 \tau^{5/2}-8.34635\times 10^6 \tau^{7/2}+\\ \nonumber 
&\left(-91003 \tau^{3/2}+4.36814\times 10^6 \tau^{5/2}+32174.4 \tau\right) \sqrt{\tau+8.49155} \sqrt{\tau+33.8415}+117482. \tau^2+\\ \nonumber 
&\left(1.25852\times 10^7 \tau^3+5.33261\times 10^7 \tau^2+56165.1 \sqrt{\tau}\right) \sqrt{\tau+8.49155}+\\ \nonumber 
&\left(-8.60549\times 10^6 \tau^3-5.28876\times 10^7 \tau^2+91200.4 \sqrt{\tau}\right) \sqrt{\tau+33.8415}+14376.6\ \leq \\ \nonumber
&\sqrt{1.25 +8.49155} \sqrt{1.25 +33.8415} \left(-91003 \tau^{3/2}+4.36814\times 10^6 \tau^{5/2}+32174.4 \tau\right)+\\ \nonumber 
&\sqrt{1.25 +8.49155} \left(1.25852\times 10^7 \tau^3+5.33261\times 10^7 \tau^2+56165.1 \sqrt{\tau}\right)+\\ \nonumber 
&\sqrt{0.8 +33.8415} \left(-8.60549\times 10^6 \tau^3-5.28876\times 10^7 \tau^2+91200.4 \sqrt{\tau}\right)-\\ \nonumber 
&3.16357\times 10^6 \tau^{3/2}-608588 \tau^{5/2}-8.34635\times 10^6 \tau^{7/2}+117482. \tau^2+14376.6\ = \\ \nonumber 
&-4.84613\times 10^6 \tau^{3/2}+8.01543\times 10^7 \tau^{5/2}-8.34635\times 10^6 \tau^{7/2}-\\ \nonumber 
&1.13691\times 10^7 \tau^3-1.44725\times 10^8 \tau^2+\\ \nonumber 
&594875. \tau+712078. \sqrt{\tau}+14376.6\ \leq \\ \nonumber
&\frac{14376.6 \tau^{3/2}}{0.8^{3/2}}-4.84613\times 10^6 \tau^{3/2}+\\ \nonumber 
&8.01543\times 10^7 \tau^{5/2}-8.34635\times 10^6 \tau^{7/2}-\\ \nonumber
&1.13691\times 10^7 \tau^3-1.44725\times 10^8 \tau^2+\frac{594875. \sqrt{\tau} \tau}{\sqrt{0.8}}+\frac{712078. \tau \sqrt{\tau}}{0.8}\ = \\ \nonumber
&-3.1311 \cdot 10^6 \tau^{3/2}-1.44725 \cdot 10^8 \tau^2+8.01543 \cdot 10^7 \tau^{5/2}-1.13691 \cdot 10^7 \tau^3-\\ \nonumber 
&8.34635 \cdot 10^6 \tau^{7/2}\ \leq \\ \nonumber 
& -3.1311\times 10^6 \tau^{3/2}+\frac{8.01543\times 10^7 \sqrt{1.25} \tau^{5/2}}{\sqrt{\tau}}-\\ \nonumber
&8.34635\times 10^6 \tau^{7/2}-1.13691\times 10^7 \tau^3-1.44725\times 10^8 \tau^2\ = \\ \nonumber 
&-3.1311\times 10^6 \tau^{3/2}-8.34635\times 10^6 \tau^{7/2}-1.13691\times 10^7 \tau^3-5.51094\times 10^7 \tau^22\ < \ 0 \ . 
\end{align}

First we expanded the term (multiplied it out).
The we put the terms multiplied by the same square root into brackets.
The next inequality sign stems from inserting the maximal value of $1.25$ for $\tau$ for
some positive terms and value of $0.8$ for negative terms. 
These terms are then expanded at the $=$-sign.
The next equality factors the terms under the squared root.
We decreased the negative term by setting
$\tau=\tau+0.00011542$ under the root.
We increased positive terms by setting  
$\tau+0.00011542=1.00014\tau$ and
$\tau+0.000115004 = 1.00014\tau$
under the root for positive terms.
The positive terms are increase, since
$\frac{0.8 + 0.00011542}{0.8}<1.000142$, thus
$\tau+0.000115004<\tau+0.00011542 \leq
1.00014\tau$.
For the next inequality we decreased negative terms by inserting
$\tau=0.8$ and increased positive terms by inserting
$\tau=1.25$. The next equality expands the terms.
We use upper bound of $1.25$ and lower bound of $0.8$ to obtain terms with
corresponding exponents of $\tau$.

Consequently, the derivative of
\begin{align}
\label{eq:subx1a}
\tau \left(e^{\left(\frac{\mu \omega+\nu \tau}{\sqrt{2} \sqrt{\nu \tau}}\right)^2} \erfc \left(\frac{\mu \omega+\nu \tau}{\sqrt{2} \sqrt{\nu \tau}}\right)-2 e^{\left(\frac{\mu \omega+2  \nu \tau}{\sqrt{2} \sqrt{\nu \tau}}\right)^2} \erfc \left(\frac{\mu \omega+2  \nu \tau}{\sqrt{2} \sqrt{\nu \tau}}\right)\right)
\end{align}
with respect to $\tau$ is smaller than zero for maximal $\nu = 0.16$.

Next,
we consider the function for  
the largest $\nu = 0.24$ and the largest $y=\mu \omega = 0.01$ 
for determining the derivative with respect to $\tau$. 
However we assume $0.9 \leq \tau$, in order to restrict the
domain of $\tau$.

The expression becomes
\begin{align}
&\tau \left(e^{\left(\frac{\frac{24 \tau}{100}+\frac{1}{100}}{\sqrt{2} \sqrt{\frac{24 \tau}{100}}}\right)^2} \erfc \left(\frac{\frac{24 \tau}{100}+\frac{1}{100}}{\sqrt{2} \sqrt{\frac{24 \tau}{100}}}\right)-
e^{\left(\frac{\frac{2\ 24 \tau}{100}+\frac{1}{100}}{\sqrt{2} \sqrt{\frac{24 \tau}{100}}}\right)^2} \erfc \left(\frac{\frac{2\ 24 \tau}{100}+\frac{1}{100}}{\sqrt{2} \sqrt{\frac{24 \tau}{100}}}\right)\right) \ .
\end{align}

The derivative with respect to $\tau$ is 
\begin{align}
&\left(\sqrt{\pi } \left(e^{\frac{(24 \tau+1)^2}{4800 \tau}} (192  \tau (3 \tau+25)-1) \erfc \left(\frac{24 \tau+1}{40 \sqrt{3} \sqrt{\tau}}\right)-\right.\right.\\\nonumber 
&\left.\left.2 e^{\frac{(48 \tau+1)^2}{4800 \tau}} (192  \tau (12  \tau+25)-1) \erfc \left(\frac{48 \tau+1}{40 \sqrt{3} \sqrt{\tau}}\right)\right)+40 \sqrt{3} (72  \tau-1) \sqrt{\tau}\right)\\ \nonumber 
&\left(4800 \sqrt{\pi } \tau\right)^{-1} \ .
\end{align}

We are considering only the numerator and use again the approximation
of \citet{Ren:07}.
The error analysis on the whole numerator gives an approximation error $14<E<32$. Therefore
we add 32 to the numerator when we use the approximation of \citet{Ren:07}.
We obtain the inequalities:
\begin{align}
&\sqrt{\pi } \left(e^{\frac{(24 \tau+1)^2}{4800 \tau}} (192  \tau (3 \tau+25)-1) \erfc \left(\frac{24 \tau+1}{40 \sqrt{3} \sqrt{\tau}}\right)\ -\right.\\\nonumber 
&\left.2 e^{\frac{(48 \tau+1)^2}{4800 \tau}} (192  \tau (12  \tau+25)-1) \erfc \left(\frac{48 \tau+1}{40 \sqrt{3} \sqrt{\tau}}\right)\right)+40 \sqrt{3} (72  \tau-1) \sqrt{\tau}\ \leq \\ \nonumber 
&  \sqrt{\pi } \left(\frac{2.911 (192  \tau (3 \tau+25)-1)}{\frac{\sqrt{\pi } (2.911 -1) (24 \tau+1)}{40 \sqrt{3} \sqrt{\tau}}+\sqrt{\pi  \left(\frac{24 \tau+1}{40 \sqrt{3} \sqrt{\tau}}\right)^2+2.911^2}} \ -\right.\\\nonumber 
&\left.\frac{2 \cdot 2.911 (192  \tau (12  \tau+25)-1)}{\frac{\sqrt{\pi } (2.911 -1) (48 \tau+1)}{40 \sqrt{3} \sqrt{\tau}}+\sqrt{\pi  \left(\frac{48 \tau+1}{40 \sqrt{3} \sqrt{\tau}}\right)^2+2.911^2}}\right)+\\ \nonumber 
&40 \sqrt{3} (72  \tau-1) \sqrt{\tau}+32\ = \\ \nonumber 
&\sqrt{\pi } \left(\frac{(192  \tau (3 \tau+25)-1) \left(40 \sqrt{3} 2.911 \sqrt{\tau}\right)}{\sqrt{\pi } (2.911 -1) (24 \tau+1)+\sqrt{\left(40 \sqrt{3} 2.911 \sqrt{\tau}\right)^2+\pi  (24 \tau+1)^2}} \ -\right.\\\nonumber 
&\left.\frac{2 (192  \tau (12  \tau+25)-1) \left(40 \sqrt{3} 2.911 \sqrt{\tau}\right)}{\sqrt{\pi } (2.911 -1) (48 \tau+1)+\sqrt{\left(40 \sqrt{3} 2.911 \sqrt{\tau}\right)^2+\pi  (48 \tau+1)^2}}\right)+\\ \nonumber 
&40 \sqrt{3} (72  \tau-1) \sqrt{\tau}+32\ = \\ \nonumber 
&\left(\left(40 \sqrt{3} (72  \tau-1) \sqrt{\tau}+32\right) \left(\sqrt{\pi } (2.911 -1) (24 \tau+1)+\sqrt{\left(40 \sqrt{3} 2.911 \sqrt{\tau}\right)^2+\pi  (24 \tau+1)^2}\right) \right.\\\nonumber 
&\left. \left(\sqrt{\pi } (2.911 -1) (48 \tau+1)+\sqrt{\left(40 \sqrt{3} 2.911 \sqrt{\tau}\right)^2+\pi  (48 \tau+1)^2}\right)+\right.\\\nonumber 
&\left.2.911 \cdot 40 \sqrt{3} \sqrt{\pi } (192  \tau (3 \tau+25)-1) \sqrt{\tau}   \right.\\\nonumber 
&\left.\left(\sqrt{\pi } (2.911 -1) (48 \tau+1)+\sqrt{\left(40 \sqrt{3} 2.911 \sqrt{\tau}\right)^2+\pi  (48 \tau+1)^2}\right)- \right.\\\nonumber 
&\left.2 \sqrt{\pi } 40 \sqrt{3} 2.911 (192  \tau (12  \tau+25)-1) \right.\\\nonumber 
&\left. \sqrt{\tau} \left(\sqrt{\pi } (2.911 -1) (24 \tau+1)+\sqrt{\left(40 \sqrt{3} 2.911 \sqrt{\tau}\right)^2+\pi  (24 \tau+1)^2}\right)\right)\\ \nonumber 
&\left(\left(\sqrt{\pi } (2.911 -1) (24 \tau+1)+\sqrt{\left(40 \sqrt{3} 2.911 \sqrt{\tau}\right)^2+\pi  (24 \tau+1)^2}\right) \right.\\\nonumber 
&\left.  \left(\sqrt{\pi } (2.911 -1) (48 \tau+1)+\sqrt{\left(40 \sqrt{3} 2.911 \sqrt{\tau}\right)^2+\pi  (48 \tau+1)^2}\right)\right)^{-1} \ .
\end{align}
After applying the approximation
of \citet{Ren:07} and adding 200,
we first factored out $40 \sqrt{3} \sqrt{\tau}$.
Then we brought all terms to the same denominator.

We now consider the numerator:
\begin{align}
\label{eq:start2}
&\left(40 \sqrt{3} (72  \tau-1) \sqrt{\tau}+32\right) \left(\sqrt{\pi } (2.911 -1) (24 \tau+1)+\sqrt{\left(40 \sqrt{3} 2.911 \sqrt{\tau}\right)^2+\pi  (24 \tau+1)^2}\right) \\\nonumber
& \left(\sqrt{\pi } (2.911 -1) (48 \tau+1)+\sqrt{\left(40 \sqrt{3} 2.911 \sqrt{\tau}\right)^2+\pi  (48 \tau+1)^2}\right) \ +\\ \nonumber 
&2.911 \cdot 40 \sqrt{3} \sqrt{\pi } (192  \tau (3 \tau+25)-1) \sqrt{\tau} \\\nonumber 
&  \left(\sqrt{\pi } (2.911 -1) (48 \tau+1)+\sqrt{\left(40 \sqrt{3} 2.911 \sqrt{\tau}\right)^2+\pi  (48 \tau+1)^2}\right)- \\\nonumber 
& 2 \sqrt{\pi } 40 \sqrt{3} 2.911 (192  \tau (12  \tau+25)-1) \sqrt{\tau} \\\nonumber 
&  \left(\sqrt{\pi } (2.911 -1) (24 \tau+1)+\sqrt{\left(40 \sqrt{3} 2.911 \sqrt{\tau}\right)^2+\pi  (24 \tau+1)^2}\right)\ = \\ \nonumber &
-3.42607\times 10^6 \sqrt{\pi  (24 \tau+1)^2+40674.8 \tau} \tau^{3/2}+\\ \nonumber 
&2880 \sqrt{3} \sqrt{\pi  (24 \tau+1)^2+40674.8 \tau} \sqrt{\pi  (48 \tau+1)^2+40674.8 \tau} \tau^{3/2}+\\ \nonumber
&1.72711\times 10^6 \sqrt{\pi  (48 \tau+1)^2+40674.8 \tau} \tau^{3/2}-5.81185\times 10^6 \tau^{3/2}\ - \\ \nonumber
&836198 \sqrt{\pi  (24 \tau+1)^2+40674.8 \tau} \tau^{5/2}+611410 \sqrt{\pi  (48 \tau+1)^2+40674.8 \tau} \tau^{5/2}-\\ \nonumber 
&1.67707\times 10^6 \tau^{5/2}\ - \\ \nonumber &3.44998\times 10^7 \tau^{7/2}+422935. \tau^2+5202.68 \sqrt{\pi  (24 \tau+1)^2+40674.8 \tau} \tau+\\ \nonumber 
&2601.34 \sqrt{\pi  (48 \tau+1)^2+40674.8 \tau} \tau\ + \\ \nonumber 
&26433.4 \tau+415.94 \sqrt{\tau}+480.268 \sqrt{\tau} \sqrt{\pi  (24 \tau+1)^2+40674.8 \tau}\ + \\ \nonumber 
&108.389 \sqrt{\pi  (24 \tau+1)^2+40674.8 \tau}-592.138 \sqrt{\tau} \sqrt{\pi  (48 \tau+1)^2+40674.8 \tau}-\\ \nonumber
&40 \sqrt{3} \sqrt{\tau} \sqrt{\pi  (24 \tau+1)^2+40674.8 \tau} \sqrt{\pi  (48 \tau+1)^2+40674.8 \tau}\ + \\ \nonumber 
&32 \sqrt{\pi  (24 \tau+1)^2+40674.8 \tau} \sqrt{\pi  (48 \tau+1)^2+40674.8 \tau}\ + \\ \nonumber 
&108.389 \sqrt{\pi  (48 \tau+1)^2+40674.8 \tau}+367.131\ = \\ \nonumber
&-5.81185\times 10^6 \tau^{3/2}-1.67707\times 10^6 \tau^{5/2}-3.44998\times 10^7 \tau^{7/2}+\\ \nonumber
&\left(-3.42607\times 10^6 \tau^{3/2}-836198 \tau^{5/2}+5202.68 \tau+480.268 \sqrt{\tau}+108.389\right) \\ \nonumber 
&\sqrt{\pi  (24 \tau+1)^2+40674.8 \tau}+\\ \nonumber 
&\left(1.72711\times 10^6 \tau^{3/2}+611410 \tau^{5/2}+2601.34 \tau-592.138 \sqrt{\tau}+108.389\right) \\ \nonumber 
&\sqrt{\pi  (48 \tau+1)^2+40674.8 \tau}+\\ \nonumber 
&\left(2880 \sqrt{3} \tau^{3/2}-40 \sqrt{3} \sqrt{\tau}+32\right) \sqrt{\pi  (24 \tau+1)^2+40674.8 \tau} \sqrt{\pi  (48 \tau+1)^2+40674.8 \tau}+\\ \nonumber 
&422935. \tau^2+26433.4 \tau+415.94 \sqrt{\tau}+367.131\ \leq \\ \nonumber &
-5.81185\times 10^6 \tau^{3/2}-1.67707\times 10^6 \tau^{5/2}-3.44998\times 10^7 \tau^{7/2}+\\ \nonumber 
&\left(-3.42607\times 10^6 \tau^{3/2}-836198 \tau^{5/2}+480.268 \sqrt{1.25}+1.25 5202.68+108.389\right) \\ \nonumber 
&\sqrt{\pi  (24 \tau+1)^2+40674.8 \tau}+\\ \nonumber 
&\left(1.72711\times 10^6 \tau^{3/2}+611410 \tau^{5/2}-592.138 \sqrt{0.9}+1.25 2601.34+108.389\right) \\ \nonumber 
&\sqrt{\pi  (48 \tau+1)^2+40674.8 \tau}+\\ \nonumber 
&\left(2880 \sqrt{3} \tau^{3/2}-40 \sqrt{3} \sqrt{\tau}+32\right) \sqrt{\pi  (24 \tau+1)^2+40674.8 \tau} \sqrt{\pi  (48 \tau+1)^2+40674.8 \tau}+\\ \nonumber 
&422935 \tau^2+415.94 \sqrt{1.25}+1.25 26433.4+367.131\ = \\ \nonumber &
-5.81185\times 10^6 \tau^{3/2}-1.67707\times 10^6 \tau^{5/2}-3.44998\times 10^7 \tau^{7/2}+\\ \nonumber 
&\left(-3.42607\times 10^6 \tau^{3/2}-836198 \tau^{5/2}+7148.69\right) \sqrt{\pi  (24 \tau+1)^2+40674.8 \tau}+\\ \nonumber 
&\left(1.72711\times 10^6 \tau^{3/2}+611410 \tau^{5/2}+2798.31\right) \sqrt{\pi  (48 \tau+1)^2+40674.8 \tau}+\\ \nonumber 
&\left(2880 \sqrt{3} \tau^{3/2}-40 \sqrt{3} \sqrt{\tau}+32\right) \sqrt{\pi  (24 \tau+1)^2+40674.8 \tau} \sqrt{\pi  (48 \tau+1)^2+40674.8 \tau}+\\ \nonumber 
&422935 \tau^2+33874\ = \\ \nonumber
&-5.81185\times 10^6 \tau^{3/2}-1.67707\times 10^6 \tau^{5/2}-3.44998\times 10^7 \tau^{7/2}+\\ \nonumber
&\left(1.72711\times 10^6 \tau^{3/2}+611410 \tau^{5/2}+2798.31\right) \sqrt{2304 \pi  (\tau+5.66103) (\tau+0.0000766694)}+\\ \nonumber 
&\left(-3.42607\times 10^6 \tau^{3/2}-836198 \tau^{5/2}+7148.69\right) \sqrt{576 \pi  (\tau+22.561) (\tau+0.0000769518)}+\\ \nonumber 
&\left(2880 \sqrt{3} \tau^{3/2}-40 \sqrt{3} \sqrt{\tau}+32\right) \sqrt{2304 \pi  (\tau+5.66103) (\tau+0.0000766694)}\\ \nonumber 
& \sqrt{576 \pi  (\tau+22.561) (\tau+0.0000769518)}+\\ \nonumber 
&422935 \tau^2+33874\ \leq \\ \nonumber 
&-5.81185 10^6 \tau^{3/2}-1.67707\times 10^6 \tau^{5/2}-3.44998\times 10^7 \tau^{7/2}+\\ \nonumber 
&\left(1.72711\times 10^6 \tau^{3/2}+611410 \tau^{5/2}+2798.31\right) \sqrt{2304 \pi  1.0001 (\tau+5.66103) \tau}+\\ \nonumber 
&\left(2880 \sqrt{3} \tau^{3/2}-40 \sqrt{3} \sqrt{\tau}+32\right) \sqrt{2304 \pi  1.0001 (\tau+5.66103) \tau} \sqrt{576 \pi  1.0001 (\tau+22.561) \tau}+\\ \nonumber
&\left(-3.42607\times 10^6 \tau^{3/2}-836198 \tau^{5/2}+7148.69\right) \\ \nonumber 
&\sqrt{576 \pi  (\tau+22.561) \tau}+\\ \nonumber 
&422935 \tau^2+33874.\ = \\ \nonumber &
-5.81185 10^6 \tau^{3/2}-1.67707\times 10^6 \tau^{5/2}-3.44998\times 10^7 \tau^{7/2}+\\ \nonumber 
&\left(-250764. \tau^{3/2}+1.8055\times 10^7 \tau^{5/2}+115823. \tau\right) \\ \nonumber 
&\sqrt{\tau+5.66103} \sqrt{\tau+22.561}+422935. \tau^2+\\ \nonumber 
&\left(5.20199\times 10^7 \tau^3+1.46946\times 10^8 \tau^2+238086. \sqrt{\tau}\right) \sqrt{\tau+5.66103}+\\ \nonumber 
&\left(-3.55709\times 10^7 \tau^3-1.45741\times 10^8 \tau^2+304097. \sqrt{\tau}\right) \sqrt{\tau+22.561}+33874. \ \leq \\ \nonumber
&\sqrt{1.25 +5.66103} \sqrt{1.25 +22.561} \left(-250764. \tau^{3/2}+1.8055\times 10^7 \tau^{5/2}+115823. \tau\right)+\\ \nonumber 
&\sqrt{1.25 +5.66103} \left(5.20199\times 10^7 \tau^3+1.46946\times 10^8 \tau^2+238086. \sqrt{\tau}\right)+\\ \nonumber 
&\sqrt{0.9 +22.561} \left(-3.55709\times 10^7 \tau^3-1.45741\times 10^8 \tau^2+304097. \sqrt{\tau}\right)-\\ \nonumber 
&5.81185 10^6 \tau^{3/2}-1.67707\times 10^6 \tau^{5/2}-3.44998\times 10^7 \tau^{7/2}+422935. \tau^2+33874. \ \leq \\ \nonumber 
&
\frac{33874. \tau^{3/2}}{0.9^{3/2}}-9.02866\times 10^6 \tau^{3/2}+2.29933\times 10^8 \tau^{5/2}-3.44998\times 10^7 \tau^{7/2}-\\ \nonumber 
&3.5539\times 10^7 \tau^3-3.19193\times 10^8 \tau^2+\frac{1.48578\times 10^6 \sqrt{\tau} \tau}{\sqrt{0.9}}+\frac{2.09884\times 10^6 \tau \sqrt{\tau}}{0.9}\ = \\ \nonumber 
&
-5.09079\times 10^6 \tau^{3/2}+2.29933\times 10^8 \tau^{5/2}-\\ \nonumber 
&3.44998\times 10^7 \tau^{7/2}-3.5539\times 10^7 \tau^3-3.19193\times 10^8 \tau^2 \ \leq \\ \nonumber 
&
-5.09079\times 10^6 \tau^{3/2}+\frac{2.29933\times 10^8 \sqrt{1.25} \tau^{5/2}}{\sqrt{\tau}}-3.44998\times 10^7 \tau^{7/2}-\\ \nonumber 
&3.5539\times 10^7 \tau^3-3.19193\times 10^8 \tau^2\ = \\ \nonumber &
-5.09079\times 10^6 \tau^{3/2}-3.44998\times 10^7 \tau^{7/2}-3.5539\times 10^7 \tau^3-6.21197\times 10^7 \tau^2\ < \ 0 \ . 
\end{align}

First we expanded the term (multiplied it out).
The we put the terms multiplied by the same square root into brackets.
The next inequality sign stems from inserting the maximal value of $1.25$ for $\tau$ for
some positive terms and value of $0.9$ for negative terms. 
These terms are then expanded at the $=$-sign.
The next equality factors the terms under the squared root.
We decreased the negative term by setting
$\tau=\tau+0.0000769518$ under the root.
We increased positive terms by setting  
$\tau+0.0000769518=1.0000962\tau$ and
$\tau+0.0000766694 =1.0000962\tau$
under the root for positive terms.
The positive terms are increase, since
$\frac{0.8 + 0.0000769518}{0.8}<1.0000962$, thus
$\tau+0.0000766694<\tau+0.0000769518 \leq
1.0000962\tau$.
For the next inequality we decreased negative terms by inserting
$\tau=0.9$ and increased positive terms by inserting
$\tau=1.25$. The next equality expands the terms.
We use upper bound of $1.25$ and lower bound of $0.9$ to obtain terms with
corresponding exponents of $\tau$.

{Consequently}, the derivative of
\begin{align}
\label{eq:subx1b}
\tau \left(e^{\left(\frac{\mu \omega+\nu \tau}{\sqrt{2} \sqrt{\nu \tau}}\right)^2} \erfc \left(\frac{\mu \omega+\nu \tau}{\sqrt{2} \sqrt{\nu \tau}}\right)-2 e^{\left(\frac{\mu \omega+2  \nu \tau}{\sqrt{2} \sqrt{\nu \tau}}\right)^2} \erfc \left(\frac{\mu \omega+2  \nu \tau}{\sqrt{2} \sqrt{\nu \tau}}\right)\right)
\end{align}
with respect to $\tau$ is smaller than zero for maximal
$\nu = 0.24$ and the domain $0.9 \leq \tau \leq 1.25$.
\end{proof}

\begin{lemma}
\label{lem:mainsubfunctionJ11J12}
In the domain $-0.01 \leq y \leq 0.01$ and $0.64 \leq x \leq 1.875$, 
the function $f(x,y)=e^{\frac{1}{2} (2 y + x)} \erfc\left(\frac{x+y}{\sqrt{2 x}}\right)$ has a global 
maximum at $y=0.64$ and $x=-0.01$ and a global minimum at $y=1.875$ and $x=0.01$. 
\end{lemma}

\begin{proof}
$f(x,y)=e^{\frac{1}{2} (2 y + x)} \erfc\left(\frac{x+y}{\sqrt{2 x}}\right)$ is strictly monotonically decreasing
in $x$, since its derivative with respect to $x$ is negative:
\begin{align}
 &\frac{e^{-\frac{y^2}{2 x}} \left(\sqrt{\pi } x^{3/2} e^{\frac{(x+y)^2}{2 x}} \erfc \left(\frac{x+y}{\sqrt{2} \sqrt{x}}\right)+\sqrt{2} (y-x)\right)}{2 \sqrt{\pi } x^{3/2}} < 0 \nonumber \\
 &\iff \sqrt{\pi } x^{3/2} e^{\frac{(x+y)^2}{2 x}} \erfc \left(\frac{x+y}{\sqrt{2} \sqrt{x}}\right)+\sqrt{2} (y-x) < 0  \nonumber \\
 &\sqrt{\pi } x^{3/2} e^{\frac{(x+y)^2}{2 x}} \erfc \left(\frac{x+y}{\sqrt{2} \sqrt{x}}\right)+\sqrt{2} (y-x) \leq \nonumber \\
 &\frac{2 x^{3/2}}{\frac{x+y}{\sqrt{2} \sqrt{x}}+\sqrt{\frac{(x+y)^2}{2 x}+\frac{4}{\pi }}}+ y \sqrt{2} - x \sqrt{2} \leq \nonumber \\ 
 &\frac{2\cdot 0.64^{3/2}}{\frac{0.01+0.64}{\sqrt{2} \sqrt{0.64}}+\sqrt{\frac{(0.01+0.64)^2}{2 \cdot 0.64}+\frac{4}{\pi }}}+0.01\sqrt{2}  - 0.64\sqrt{2} =  -0.334658 < 0.  
\end{align}
The two last inqualities come from applying Abramowitz bounds~\ref{lem:Abramowitz} and from the fact that the expression
$\frac{2 x^{3/2}}{\frac{x+y}{\sqrt{2} \sqrt{x}}+\sqrt{\frac{(x+y)^2}{2 x}+\frac{4}{\pi }}}+ y \sqrt{2} - x \sqrt{2}$
does not change monotonicity in the 
domain and hence the maximum must be found at the border. For $x=0.64$ that maximizes the function $f(x,y)$ is monotonically   in $y$, because 
its derivative w.r.t. $y$ at $x=0.64$ is 
\begin{align}
&e^y \left(1.37713 \erfc (0.883883 y+0.565685)-1.37349 e^{-0.78125 (y+0.64)^2}\right) < 0 \nonumber \\
&\iff  \left(1.37713 \erfc (0.883883 y+0.565685)-1.37349 e^{-0.78125 (y+0.64)^2}\right) < 0 \nonumber \\
&\left(1.37713 \erfc (0.883883 y+0.565685) - 1.37349 e^{-0.78125 (y+0.64)^2}\right) \leq \nonumber \\
&\left(1.37713 \erfc (0.883883 \cdot -0.01 + 0.565685) - 1.37349 e^{-0.78125 (0.01+0.64)^2}\right) =  \nonumber \\
&0.5935272325870631 - 0.987354705867739 < 0.
\end{align}
Therefore, the values $y=0.64$ and $x=-0.01$ give
a global maximum of the function $f(x,y)$ in the domain $-0.01\leq y \leq 0.01$ and $0.64 \leq x \leq 1.875$ and
the values $y=1.875$ and $x=0.01$ give the global minimum.
\end{proof}

\section{Additional information on experiments}
\label{sec:experiments} \index{experiments}
In this section, we report the hyperparameters that were considered for each method and
data set and give details on the processing of the data sets.

\subsection{121 UCI Machine Learning Repository data sets: Hyperparameters} \index{experiments!UCI}
For the UCI data sets, the best hyperparameter setting was determined by a grid-search over all
hyperparameter combinations using 15\% of the training data as validation set.
The early stopping parameter was determined on the smoothed learning curves of 100 epochs
of the validation set. Smoothing was done using moving averages of 10 consecutive 
values. We tested ``rectangular'' and ``conic'' layers -- rectangular layers have 
constant number of hidden units in each layer, conic layers start with the given 
number of hidden units in the first layer and then decrease the number of hidden units
to the size of the output layer according to the geometric progession.
If multiple hyperparameters provided identical performance on the validation 
set, we preferred settings with a higher number of layers, lower learning rates and higher dropout rates. 
All methods had the chance to adjust their hyperparameters to the data set at hand.

\index{experiments!UCI!hyperparameters}
\begin{table}[htp]
\begin{center}
\caption{Hyperparameters considered for self-normalizing networks in the UCI data sets.}

\begin{tabular}{ll}
\toprule
Hyperparameter  & Considered values \\ 
\midrule
  Number of hidden units & \{1024, 512, 256\} \\
  Number of hidden layers & \{2, 3, 4, 8, 16, 32\} \\
  Learning rate & \{0.01, 0.1, 1\} \\
  Dropout rate & \{0.05, 0\}\\
  Layer form & \{rectangular, conic\} \\
\bottomrule
\end{tabular}
\end{center}

\end{table}

\begin{table}[htp]
\begin{center}
\caption[Hyperparameters considered for ReLU networks in the UCI data sets.]{Hyperparameters considered for ReLU networks with MS initialization in the UCI data sets.}

\begin{tabular}{ll}
\toprule
Hyperparameter  & Considered values \\ 
\midrule
   Number of hidden units & \{1024, 512, 256\} \\
  Number of hidden layers & \{2,3,4,8,16,32\} \\
  Learning rate & \{0.01, 0.1, 1\} \\
  Dropout rate & \{0.5, 0\}\\
  Layer form & \{rectangular, conic\} \\
\bottomrule
\end{tabular}
\end{center}

\end{table}


\begin{table}[htp]
\begin{center}
\caption{Hyperparameters considered for batch normalized networks in the UCI data sets.}

\begin{tabular}{ll}
\toprule
Hyperparameter  & Considered values \\ 
\midrule
  Number of hidden units & \{1024, 512, 256\} \\
  Number of hidden layers & \{2, 3, 4, 8, 16, 32\} \\
  Learning rate & \{0.01, 0.1, 1\} \\
  Normalization & \{Batchnorm\} \\
  Layer form & \{rectangular, conic\} \\
\bottomrule
\end{tabular}
\end{center}

\end{table}


\begin{table}[htp]
\begin{center}
\caption{Hyperparameters considered for weight normalized networks in the UCI data sets.}

\begin{tabular}{ll}
\toprule
Hyperparameter  & Considered values \\ 
\midrule
  Number of hidden units & \{1024, 512, 256\} \\
  Number of hidden layers & \{2, 3, 4, 8, 16, 32\} \\
  Learning rate & \{0.01, 0.1, 1\} \\
  Normalization & \{Weightnorm\} \\
  Layer form & \{rectangular, conic\} \\
\bottomrule
\end{tabular}
\end{center}

\end{table}


\begin{table}[htp]
\begin{center}
\caption{Hyperparameters considered for layer normalized networks in the UCI data sets.}

\begin{tabular}{ll}
\toprule
Hyperparameter  & Considered values \\ 
\midrule
  Number of hidden units & \{1024, 512, 256\} \\
  Number of hidden layers & \{2, 3, 4, 8, 16, 32\} \\
  Learning rate & \{0.01, 0.1, 1\} \\
  Normalization & \{Layernorm\} \\
  Layer form & \{rectangular, conic\} \\

\bottomrule
\end{tabular}
\end{center}

\end{table}


\begin{table}[htp]
\begin{center}
\caption{Hyperparameters considered for Highway networks in the UCI data sets.}

\begin{tabular}{ll}
\toprule
Hyperparameter  & Considered values \\ 
\midrule
  Number of hidden layers & \{2, 3, 4, 8, 16, 32\} \\
  Learning rate & \{0.01, 0.1, 1\} \\
  Dropout rate & \{0, 0.5\} \\
\bottomrule
\end{tabular}
\end{center}

\end{table}


\begin{table}[htp]
\begin{center}
\caption{Hyperparameters considered for Residual networks in the UCI data sets.}

\begin{tabular}{ll}
\toprule
Hyperparameter  & Considered values \\ 
\midrule
  Number of blocks & \{2, 3, 4, 8, 16\} \\
  Number of neurons per blocks & \{1024, 512, 256\} \\
  Block form & \{rectangular, diavolo\} \\
  Bottleneck & \{25\%, 50\%\} \\
  Learning rate & \{0.01, 0.1, 1\} \\
\bottomrule
\end{tabular}
\end{center}

\end{table}


\clearpage

\subsection{121 UCI Machine Learning Repository data sets: detailed results}
\index{experiments!UCI!details}

\paragraph{Methods compared.} 
We used data sets and preprocessing scripts by \citet{bib:Fernandez2014} for data preparation and
defining training and test sets. With several flaws in the method comparison\citep{bib:Wainberg2016} that we avoided, 
the authors compared 179 machine learning methods of 17 groups in their experiments. 
The method groups were defined by \citet{bib:Fernandez2014} as follows:
Support Vector Machines, RandomForest, Multivariate adaptive regression splines (MARS), 
Boosting, Rule-based, logistic and multinomial regression,
Discriminant Analysis (DA), Bagging, 
Nearest Neighbour, DecisionTree, other Ensembles, Neural Networks, Bayesian, Other Methods,
generalized linear models (GLM), Partial least squares and principal component regression (PLSR), and Stacking.
However, many of methods assigned to those groups were merely different implementations of the 
same method. Therefore, we selected one representative of each of the 17 groups for method 
comparison. The representative method was chosen as the group's method with the median performance 
across all tasks. 
Finally, we included 17 other machine learning methods of \citet{bib:Fernandez2014},
and 6 FNNs, BatchNorm, WeightNorm, LayerNorm, 
Highway, Residual and MSRAinit networks, and self-normalizing neural networks (SNNs) giving a total of 24 compared methods. 

\paragraph{Results of FNN methods for all 121 data sets.}
The results of the compared FNN methods can be found in Table~\ref{tab:UCIfull}.

\begin{table}
\caption[Comparison of FNN methods on all 121 UCI data sets.]{Comparison of FNN methods on all 121 UCI data sets.. The table reports the accuracy of FNN methods at each individual 
task of the 121 UCI data sets. The first column gives the name of the data set, the second the number
of training data points $N$, the third the number of features $M$ and the consecutive columns the accuracy values of
self-normalizing networks (SNNs), ReLU networks without normalization and with MSRA initialization (MS),
Highway networks (HW), Residual Networks (ResNet), networks with batch normalization (BN), weight 
normalization (WN), and layer normalization (LN). \label{tab:UCIfull}}
\footnotesize
\begin{tabular}{lrrlllllll}
\toprule
dataset & $N$ & $M$ & SNN & MS & HW & ResNet & BN & WN & LN\tabularnewline
\midrule
abalone & 4177 & 9 & 0.6657 & 0.6284 & 0.6427 & 0.6466 & 0.6303 & 0.6351 & 0.6178\tabularnewline
acute-inflammation & 120 & 7 & 1.0000 & 1.0000 & 1.0000 & 1.0000 & 1.0000 & 1.0000 & 0.9000\tabularnewline
acute-nephritis & 120 & 7 & 1.0000 & 1.0000 & 1.0000 & 1.0000 & 1.0000 & 1.0000 & 1.0000\tabularnewline
adult & 48842 & 15 & 0.8476 & 0.8487 & 0.8453 & 0.8484 & 0.8499 & 0.8453 & 0.8517\tabularnewline
annealing & 898 & 32 & 0.7600 & 0.7300 & 0.3600 & 0.2600 & 0.1200 & 0.6500 & 0.5000\tabularnewline
arrhythmia & 452 & 263 & 0.6549 & 0.6372 & 0.6283 & 0.6460 & 0.5929 & 0.6018 & 0.5752\tabularnewline
audiology-std & 196 & 60 & 0.8000 & 0.6800 & 0.7200 & 0.8000 & 0.6400 & 0.7200 & 0.8000\tabularnewline
balance-scale & 625 & 5 & 0.9231 & 0.9231 & 0.9103 & 0.9167 & 0.9231 & 0.9551 & 0.9872\tabularnewline
balloons & 16 & 5 & 1.0000 & 0.5000 & 0.2500 & 1.0000 & 1.0000 & 0.0000 & 0.7500\tabularnewline
bank & 4521 & 17 & 0.8903 & 0.8876 & 0.8885 & 0.8796 & 0.8823 & 0.8850 & 0.8920\tabularnewline
blood & 748 & 5 & 0.7701 & 0.7754 & 0.7968 & 0.8021 & 0.7647 & 0.7594 & 0.7112\tabularnewline
breast-cancer & 286 & 10 & 0.7183 & 0.6901 & 0.7465 & 0.7465 & 0.7324 & 0.6197 & 0.6620\tabularnewline
breast-cancer-wisc & 699 & 10 & 0.9714 & 0.9714 & 0.9771 & 0.9714 & 0.9829 & 0.9657 & 0.9714\tabularnewline
breast-cancer-wisc-diag & 569 & 31 & 0.9789 & 0.9718 & 0.9789 & 0.9507 & 0.9789 & 0.9718 & 0.9648\tabularnewline
breast-cancer-wisc-prog & 198 & 34 & 0.6735 & 0.7347 & 0.8367 & 0.8163 & 0.7755 & 0.8367 & 0.7959\tabularnewline
breast-tissue & 106 & 10 & 0.7308 & 0.4615 & 0.6154 & 0.4231 & 0.4615 & 0.5385 & 0.5769\tabularnewline
car & 1728 & 7 & 0.9838 & 0.9861 & 0.9560 & 0.9282 & 0.9606 & 0.9769 & 0.9907\tabularnewline
cardiotocography-10clases & 2126 & 22 & 0.8399 & 0.8418 & 0.8456 & 0.8173 & 0.7910 & 0.8606 & 0.8362\tabularnewline
cardiotocography-3clases & 2126 & 22 & 0.9153 & 0.8964 & 0.9171 & 0.9021 & 0.9096 & 0.8945 & 0.9021\tabularnewline
chess-krvk & 28056 & 7 & 0.8805 & 0.8606 & 0.5255 & 0.8543 & 0.8781 & 0.7673 & 0.8938\tabularnewline
chess-krvkp & 3196 & 37 & 0.9912 & 0.9900 & 0.9900 & 0.9912 & 0.9862 & 0.9912 & 0.9875\tabularnewline
congressional-voting & 435 & 17 & 0.6147 & 0.6055 & 0.5872 & 0.5963 & 0.5872 & 0.5872 & 0.5780\tabularnewline
conn-bench-sonar-mines-rocks & 208 & 61 & 0.7885 & 0.8269 & 0.8462 & 0.8077 & 0.7115 & 0.8269 & 0.6731\tabularnewline
conn-bench-vowel-deterding & 990 & 12 & 0.9957 & 0.9935 & 0.9784 & 0.9935 & 0.9610 & 0.9524 & 0.9935\tabularnewline
connect-4 & 67557 & 43 & 0.8807 & 0.8831 & 0.8599 & 0.8716 & 0.8729 & 0.8833 & 0.8856\tabularnewline
contrac & 1473 & 10 & 0.5190 & 0.5136 & 0.5054 & 0.5136 & 0.4538 & 0.4755 & 0.4592\tabularnewline
credit-approval & 690 & 16 & 0.8430 & 0.8430 & 0.8547 & 0.8430 & 0.8721 & 0.9070 & 0.8547\tabularnewline
cylinder-bands & 512 & 36 & 0.7266 & 0.7656 & 0.7969 & 0.7734 & 0.7500 & 0.7578 & 0.7578\tabularnewline
dermatology & 366 & 35 & 0.9231 & 0.9121 & 0.9780 & 0.9231 & 0.9341 & 0.9451 & 0.9451\tabularnewline
echocardiogram & 131 & 11 & 0.8182 & 0.8485 & 0.6061 & 0.8485 & 0.8485 & 0.7879 & 0.8182\tabularnewline
ecoli & 336 & 8 & 0.8929 & 0.8333 & 0.8690 & 0.8214 & 0.8214 & 0.8452 & 0.8571\tabularnewline
energy-y1 & 768 & 9 & 0.9583 & 0.9583 & 0.8802 & 0.8177 & 0.8646 & 0.9010 & 0.9479\tabularnewline
energy-y2 & 768 & 9 & 0.9063 & 0.8958 & 0.9010 & 0.8750 & 0.8750 & 0.8906 & 0.8802\tabularnewline
fertility & 100 & 10 & 0.9200 & 0.8800 & 0.8800 & 0.8400 & 0.6800 & 0.6800 & 0.8800\tabularnewline
flags & 194 & 29 & 0.4583 & 0.4583 & 0.4375 & 0.3750 & 0.4167 & 0.4167 & 0.3542\tabularnewline
glass & 214 & 10 & 0.7358 & 0.6038 & 0.6415 & 0.6415 & 0.5849 & 0.6792 & 0.6981\tabularnewline
haberman-survival & 306 & 4 & 0.7368 & 0.7237 & 0.6447 & 0.6842 & 0.7368 & 0.7500 & 0.6842\tabularnewline
hayes-roth & 160 & 4 & 0.6786 & 0.4643 & 0.7857 & 0.7143 & 0.7500 & 0.5714 & 0.8929\tabularnewline
heart-cleveland & 303 & 14 & 0.6184 & 0.6053 & 0.6316 & 0.5658 & 0.5789 & 0.5658 & 0.5789\tabularnewline
heart-hungarian & 294 & 13 & 0.7945 & 0.8356 & 0.7945 & 0.8082 & 0.8493 & 0.7534 & 0.8493\tabularnewline
heart-switzerland & 123 & 13 & 0.3548 & 0.3871 & 0.5806 & 0.3226 & 0.3871 & 0.2581 & 0.5161\tabularnewline
heart-va & 200 & 13 & 0.3600 & 0.2600 & 0.4000 & 0.2600 & 0.2800 & 0.2200 & 0.2400\tabularnewline
hepatitis & 155 & 20 & 0.7692 & 0.7692 & 0.6667 & 0.7692 & 0.8718 & 0.8462 & 0.7436\tabularnewline
hill-valley & 1212 & 101 & 0.5248 & 0.5116 & 0.5000 & 0.5396 & 0.5050 & 0.4934 & 0.5050\tabularnewline
horse-colic & 368 & 26 & 0.8088 & 0.8529 & 0.7794 & 0.8088 & 0.8529 & 0.7059 & 0.7941\tabularnewline
ilpd-indian-liver & 583 & 10 & 0.6986 & 0.6644 & 0.6781 & 0.6712 & 0.5959 & 0.6918 & 0.6986\tabularnewline
\end{tabular}
\end{table}

\begin{table}
\footnotesize
\begin{tabular}{lrrlllllll}
image-segmentation & 2310 & 19 & 0.9114 & 0.9090 & 0.9024 & 0.8919 & 0.8481 & 0.8938 & 0.8838\tabularnewline
ionosphere & 351 & 34 & 0.8864 & 0.9091 & 0.9432 & 0.9545 & 0.9432 & 0.9318 & 0.9432\tabularnewline
iris & 150 & 5 & 0.9730 & 0.9189 & 0.8378 & 0.9730 & 0.9189 & 1.0000 & 0.9730\tabularnewline
led-display & 1000 & 8 & 0.7640 & 0.7200 & 0.7040 & 0.7160 & 0.6280 & 0.6920 & 0.6480\tabularnewline
lenses & 24 & 5 & 0.6667 & 1.0000 & 1.0000 & 0.6667 & 0.8333 & 0.8333 & 0.6667\tabularnewline
letter & 20000 & 17 & 0.9726 & 0.9712 & 0.8984 & 0.9762 & 0.9796 & 0.9580 & 0.9742\tabularnewline
libras & 360 & 91 & 0.7889 & 0.8667 & 0.8222 & 0.7111 & 0.7444 & 0.8000 & 0.8333\tabularnewline
low-res-spect & 531 & 101 & 0.8571 & 0.8496 & 0.9023 & 0.8647 & 0.8571 & 0.8872 & 0.8947\tabularnewline
lung-cancer & 32 & 57 & 0.6250 & 0.3750 & 0.1250 & 0.2500 & 0.5000 & 0.5000 & 0.2500\tabularnewline
lymphography & 148 & 19 & 0.9189 & 0.7297 & 0.7297 & 0.6757 & 0.7568 & 0.7568 & 0.7838\tabularnewline
magic & 19020 & 11 & 0.8692 & 0.8629 & 0.8673 & 0.8723 & 0.8713 & 0.8690 & 0.8620\tabularnewline
mammographic & 961 & 6 & 0.8250 & 0.8083 & 0.7917 & 0.7833 & 0.8167 & 0.8292 & 0.8208\tabularnewline
miniboone & 130064 & 51 & 0.9307 & 0.9250 & 0.9270 & 0.9254 & 0.9262 & 0.9272 & 0.9313\tabularnewline
molec-biol-promoter & 106 & 58 & 0.8462 & 0.7692 & 0.6923 & 0.7692 & 0.7692 & 0.6923 & 0.4615\tabularnewline
molec-biol-splice & 3190 & 61 & 0.9009 & 0.8482 & 0.8833 & 0.8557 & 0.8519 & 0.8494 & 0.8607\tabularnewline
monks-1 & 556 & 7 & 0.7523 & 0.6551 & 0.5833 & 0.7546 & 0.9074 & 0.5000 & 0.7014\tabularnewline
monks-2 & 601 & 7 & 0.5926 & 0.6343 & 0.6389 & 0.6273 & 0.3287 & 0.6644 & 0.5162\tabularnewline
monks-3 & 554 & 7 & 0.6042 & 0.7454 & 0.5880 & 0.5833 & 0.5278 & 0.5231 & 0.6991\tabularnewline
mushroom & 8124 & 22 & 1.0000 & 1.0000 & 1.0000 & 1.0000 & 0.9990 & 0.9995 & 0.9995\tabularnewline
musk-1 & 476 & 167 & 0.8739 & 0.8655 & 0.8992 & 0.8739 & 0.8235 & 0.8992 & 0.8992\tabularnewline
musk-2 & 6598 & 167 & 0.9891 & 0.9945 & 0.9915 & 0.9964 & 0.9982 & 0.9927 & 0.9951\tabularnewline
nursery & 12960 & 9 & 0.9978 & 0.9988 & 1.0000 & 0.9994 & 0.9994 & 0.9966 & 0.9966\tabularnewline
oocytes\_merluccius\_nucleus\_4d & 1022 & 42 & 0.8235 & 0.8196 & 0.7176 & 0.8000 & 0.8078 & 0.8078 & 0.7686\tabularnewline
oocytes\_merluccius\_states\_2f & 1022 & 26 & 0.9529 & 0.9490 & 0.9490 & 0.9373 & 0.9333 & 0.9020 & 0.9412\tabularnewline
oocytes\_trisopterus\_nucleus\_2f & 912 & 26 & 0.7982 & 0.8728 & 0.8289 & 0.7719 & 0.7456 & 0.7939 & 0.8202\tabularnewline
oocytes\_trisopterus\_states\_5b & 912 & 33 & 0.9342 & 0.9430 & 0.9342 & 0.8947 & 0.8947 & 0.9254 & 0.8991\tabularnewline
optical & 5620 & 63 & 0.9711 & 0.9666 & 0.9644 & 0.9627 & 0.9716 & 0.9638 & 0.9755\tabularnewline
ozone & 2536 & 73 & 0.9700 & 0.9732 & 0.9716 & 0.9669 & 0.9669 & 0.9748 & 0.9716\tabularnewline
page-blocks & 5473 & 11 & 0.9583 & 0.9708 & 0.9656 & 0.9605 & 0.9613 & 0.9730 & 0.9708\tabularnewline
parkinsons & 195 & 23 & 0.8980 & 0.9184 & 0.8367 & 0.9184 & 0.8571 & 0.8163 & 0.8571\tabularnewline
pendigits & 10992 & 17 & 0.9706 & 0.9714 & 0.9671 & 0.9708 & 0.9734 & 0.9620 & 0.9657\tabularnewline
pima & 768 & 9 & 0.7552 & 0.7656 & 0.7188 & 0.7135 & 0.7188 & 0.6979 & 0.6927\tabularnewline
pittsburg-bridges-MATERIAL & 106 & 8 & 0.8846 & 0.8462 & 0.9231 & 0.9231 & 0.8846 & 0.8077 & 0.9231\tabularnewline
pittsburg-bridges-REL-L & 103 & 8 & 0.6923 & 0.7692 & 0.6923 & 0.8462 & 0.7692 & 0.6538 & 0.7308\tabularnewline
pittsburg-bridges-SPAN & 92 & 8 & 0.6957 & 0.5217 & 0.5652 & 0.5652 & 0.5652 & 0.6522 & 0.6087\tabularnewline
pittsburg-bridges-T-OR-D & 102 & 8 & 0.8400 & 0.8800 & 0.8800 & 0.8800 & 0.8800 & 0.8800 & 0.8800\tabularnewline
pittsburg-bridges-TYPE & 105 & 8 & 0.6538 & 0.6538 & 0.5385 & 0.6538 & 0.1154 & 0.4615 & 0.6538\tabularnewline
planning & 182 & 13 & 0.6889 & 0.6667 & 0.6000 & 0.7111 & 0.6222 & 0.6444 & 0.6889\tabularnewline
plant-margin & 1600 & 65 & 0.8125 & 0.8125 & 0.8375 & 0.7975 & 0.7600 & 0.8175 & 0.8425\tabularnewline
plant-shape & 1600 & 65 & 0.7275 & 0.6350 & 0.6325 & 0.5150 & 0.2850 & 0.6575 & 0.6775\tabularnewline
plant-texture & 1599 & 65 & 0.8125 & 0.7900 & 0.7900 & 0.8000 & 0.8200 & 0.8175 & 0.8350\tabularnewline
post-operative & 90 & 9 & 0.7273 & 0.7273 & 0.5909 & 0.7273 & 0.5909 & 0.5455 & 0.7727\tabularnewline
primary-tumor & 330 & 18 & 0.5244 & 0.5000 & 0.4512 & 0.3902 & 0.5122 & 0.5000 & 0.4512\tabularnewline
ringnorm & 7400 & 21 & 0.9751 & 0.9843 & 0.9692 & 0.9811 & 0.9843 & 0.9719 & 0.9827\tabularnewline
seeds & 210 & 8 & 0.8846 & 0.8654 & 0.9423 & 0.8654 & 0.8654 & 0.8846 & 0.8846\tabularnewline
semeion & 1593 & 257 & 0.9196 & 0.9296 & 0.9447 & 0.9146 & 0.9372 & 0.9322 & 0.9447\tabularnewline
soybean & 683 & 36 & 0.8511 & 0.8723 & 0.8617 & 0.8670 & 0.8883 & 0.8537 & 0.8484\tabularnewline
spambase & 4601 & 58 & 0.9409 & 0.9461 & 0.9435 & 0.9461 & 0.9426 & 0.9504 & 0.9513\tabularnewline
spect & 265 & 23 & 0.6398 & 0.6183 & 0.6022 & 0.6667 & 0.6344 & 0.6398 & 0.6720\tabularnewline
spectf & 267 & 45 & 0.4973 & 0.6043 & 0.8930 & 0.7005 & 0.2299 & 0.4545 & 0.5561\tabularnewline
statlog-australian-credit & 690 & 15 & 0.5988 & 0.6802 & 0.6802 & 0.6395 & 0.6802 & 0.6860 & 0.6279\tabularnewline
statlog-german-credit & 1000 & 25 & 0.7560 & 0.7280 & 0.7760 & 0.7720 & 0.7520 & 0.7400 & 0.7400\tabularnewline
\end{tabular}
\end{table}

\begin{table}
\footnotesize
\begin{tabular}{lrrlllllll}
statlog-heart & 270 & 14 & 0.9254 & 0.8358 & 0.7761 & 0.8657 & 0.7910 & 0.8657 & 0.7910\tabularnewline
statlog-image & 2310 & 19 & 0.9549 & 0.9757 & 0.9584 & 0.9584 & 0.9671 & 0.9515 & 0.9757\tabularnewline
statlog-landsat & 6435 & 37 & 0.9100 & 0.9075 & 0.9110 & 0.9055 & 0.9040 & 0.8925 & 0.9040\tabularnewline
statlog-shuttle & 58000 & 10 & 0.9990 & 0.9983 & 0.9977 & 0.9992 & 0.9988 & 0.9988 & 0.9987\tabularnewline
statlog-vehicle & 846 & 19 & 0.8009 & 0.8294 & 0.7962 & 0.7583 & 0.7583 & 0.8009 & 0.7915\tabularnewline
steel-plates & 1941 & 28 & 0.7835 & 0.7567 & 0.7608 & 0.7629 & 0.7031 & 0.7856 & 0.7588\tabularnewline
synthetic-control & 600 & 61 & 0.9867 & 0.9800 & 0.9867 & 0.9600 & 0.9733 & 0.9867 & 0.9733\tabularnewline
teaching & 151 & 6 & 0.5000 & 0.6053 & 0.5263 & 0.5526 & 0.5000 & 0.3158 & 0.6316\tabularnewline
thyroid & 7200 & 22 & 0.9816 & 0.9770 & 0.9708 & 0.9799 & 0.9778 & 0.9807 & 0.9752\tabularnewline
tic-tac-toe & 958 & 10 & 0.9665 & 0.9833 & 0.9749 & 0.9623 & 0.9833 & 0.9707 & 0.9791\tabularnewline
titanic & 2201 & 4 & 0.7836 & 0.7909 & 0.7927 & 0.7727 & 0.7800 & 0.7818 & 0.7891\tabularnewline
trains & 10 & 30 & NA & NA & NA & NA & 0.5000 & 0.5000 & 1.0000\tabularnewline
twonorm & 7400 & 21 & 0.9805 & 0.9778 & 0.9708 & 0.9735 & 0.9757 & 0.9730 & 0.9724\tabularnewline
vertebral-column-2clases & 310 & 7 & 0.8312 & 0.8701 & 0.8571 & 0.8312 & 0.8312 & 0.6623 & 0.8442\tabularnewline
vertebral-column-3clases & 310 & 7 & 0.8312 & 0.8052 & 0.7922 & 0.7532 & 0.7792 & 0.7403 & 0.8312\tabularnewline
wall-following & 5456 & 25 & 0.9098 & 0.9076 & 0.9230 & 0.9223 & 0.9333 & 0.9274 & 0.9128\tabularnewline
waveform & 5000 & 22 & 0.8480 & 0.8312 & 0.8320 & 0.8360 & 0.8360 & 0.8376 & 0.8448\tabularnewline
waveform-noise & 5000 & 41 & 0.8608 & 0.8328 & 0.8696 & 0.8584 & 0.8480 & 0.8640 & 0.8504\tabularnewline
wine & 178 & 14 & 0.9773 & 0.9318 & 0.9091 & 0.9773 & 0.9773 & 0.9773 & 0.9773\tabularnewline
wine-quality-red & 1599 & 12 & 0.6300 & 0.6250 & 0.5625 & 0.6150 & 0.5450 & 0.5575 & 0.6100\tabularnewline
wine-quality-white & 4898 & 12 & 0.6373 & 0.6479 & 0.5564 & 0.6307 & 0.5335 & 0.5482 & 0.6544\tabularnewline
yeast & 1484 & 9 & 0.6307 & 0.6173 & 0.6065 & 0.5499 & 0.4906 & 0.5876 & 0.6092\tabularnewline
zoo & 101 & 17 & 0.9200 & 1.0000 & 0.8800 & 1.0000 & 0.7200 & 0.9600 & 0.9600\tabularnewline
\bottomrule
\end{tabular}
\end{table}

\paragraph{Small and large data sets.} 
We assigned each of the 121 UCI data sets into the group ``large datasets'' or 
``small datasets'' if the had more than 1,000 data points or less, respectively. 
We expected that Deep Learning methods require large data sets to competitive to other machine learning methods.
This resulted in 75 small and 46 large data sets. 

\paragraph{Results.}
The results of the method comparison are given in Tables~\ref{tab:uciS1} and \ref{tab:uciS2} for 
small and large data sets, respectively. On small data sets, SVMs performed best followed 
by RandomForest and SNNs. On large data sets, SNNs are the best method followed by SVMs and 
Random Forest.
\index{experiments!UCI!results}


\begin{table}[ht]
\caption[Method comparison on small UCI data sets]{UCI comparison reporting the average rank
of a method on 75 classification task of the 
UCI machine learning repository with  less than 1000 data points. 
For each dataset, the 24 compared methods, 
were ranked by their
accuracy and the ranks were averaged across the tasks. 
The first column gives the method group, the second the 
method, the third  
the average rank , and the last the $p$-value 
of a paired Wilcoxon test whether the difference to the best performing 
method is significant.
SNNs are ranked third having been outperformed by Random Forests and SVMs.  \label{tab:uciS1}}

\centering
\begin{tabular}{llrr}
  \toprule
 methodGroup & method & avg. rank & $p$-value \\ 
  \midrule
  SVM & LibSVM\_weka &  9.3 &\\ 
  RandomForest & RRFglobal\_caret &  9.6 & 2.5e-01 \\ 
  SNN & SNN &  9.6 & 3.8e-01 \\ 
  LMR & SimpleLogistic\_weka &  9.9 & 1.5e-01 \\ 
  NeuralNetworks & lvq\_caret & 10.1 & 1.0e-01 \\ 
  MARS & gcvEarth\_caret & 10.7 & 3.6e-02 \\ 
  MSRAinit & MSRAinit & 11.0 & 4.0e-02 \\ 
  LayerNorm & LayerNorm & 11.3 & 7.2e-02 \\ 
  Highway & Highway & 11.5 & 8.9e-03 \\ 
  DiscriminantAnalysis & mda\_R & 11.8 & 2.6e-03 \\ 
  Boosting & LogitBoost\_weka & 11.9 & 2.4e-02 \\ 
  Bagging & ctreeBag\_R & 12.1 & 1.8e-03 \\ 
  ResNet & ResNet & 12.3 & 3.5e-03 \\ 
  BatchNorm & BatchNorm & 12.6 & 4.9e-04 \\ 
  Rule-based & JRip\_caret & 12.9 & 1.7e-04 \\ 
  WeightNorm & WeightNorm & 13.0 & 8.3e-05 \\ 
  DecisionTree & rpart2\_caret & 13.6 & 7.0e-04 \\ 
  OtherEnsembles & Dagging\_weka & 13.9 & 3.0e-05 \\ 
  Nearest Neighbour & NNge\_weka & 14.0 & 7.7e-04 \\ 
  OtherMethods & pam\_caret & 14.2 & 1.5e-04 \\ 
  PLSR & simpls\_R & 14.3 & 4.6e-05 \\ 
  Bayesian & NaiveBayes\_weka & 14.6 & 1.2e-04 \\ 
  GLM & bayesglm\_caret & 15.0 & 1.6e-06 \\ 
  Stacking & Stacking\_weka & 20.9 & 2.2e-12 \\ 
   \bottomrule
\end{tabular}
\end{table}

\begin{table}[ht]
\caption[Method comparison on large UCI data sets]{UCI comparison reporting the average rank
of a method on 46 classification task of the 
UCI machine learning repository with more than 1000 data points. 
For each dataset, the 24 compared methods, 
were ranked by their
accuracy and the ranks were averaged across the tasks. 
The first column gives the method group, the second the 
method, the third  
the average rank , and the last the $p$-value 
of a paired Wilcoxon test whether the difference to the best performing 
method is significant.
SNNs are ranked first having outperformed diverse machine learning methods and
other FNNs.  \label{tab:uciS2}}

\centering
\begin{tabular}{llrr}
  \toprule
 methodGroup & method & avg. rank & $p$-value \\ 
  \midrule
  SNN & SNN &  5.8 & \\ 
  SVM & LibSVM\_weka &  6.1 & 5.8e-01 \\ 
  RandomForest & RRFglobal\_caret &  6.6 & 2.1e-01 \\ 
  MSRAinit & MSRAinit &  7.1 & 4.5e-03 \\ 
  LayerNorm & LayerNorm &  7.2 & 7.1e-02 \\ 
  Highway & Highway &  7.9 & 1.7e-03 \\ 
  ResNet & ResNet &  8.4 & 1.7e-04 \\ 
  WeightNorm & WeightNorm &  8.7 & 5.5e-04 \\ 
  BatchNorm & BatchNorm &  9.7 & 1.8e-04 \\ 
  MARS & gcvEarth\_caret &  9.9 & 8.2e-05 \\ 
  Boosting & LogitBoost\_weka & 12.1 & 2.2e-07 \\ 
  LMR & SimpleLogistic\_weka & 12.4 & 3.8e-09 \\ 
  Rule-based & JRip\_caret & 12.4 & 9.0e-08 \\ 
  Bagging & ctreeBag\_R & 13.5 & 1.6e-05 \\ 
  DiscriminantAnalysis & mda\_R & 13.9 & 1.4e-10 \\ 
  Nearest Neighbour & NNge\_weka & 14.1 & 1.6e-10 \\ 
  DecisionTree & rpart2\_caret & 15.5 & 2.3e-08 \\ 
  OtherEnsembles & Dagging\_weka & 16.1 & 4.4e-12 \\ 
  NeuralNetworks & lvq\_caret & 16.3 & 1.6e-12 \\ 
  Bayesian & NaiveBayes\_weka & 17.9 & 1.6e-12 \\ 
  OtherMethods & pam\_caret & 18.3 & 2.8e-14 \\ 
  GLM & bayesglm\_caret & 18.7 & 1.5e-11 \\ 
  PLSR & simpls\_R & 19.0 & 3.4e-11 \\ 
  Stacking & Stacking\_weka & 22.5 & 2.8e-14 \\ 
   \bottomrule
\end{tabular}
\end{table}

\clearpage

\subsection{Tox21 challenge data set: Hyperparameters} \index{experiments!Tox21}
For the Tox21 data set, the best hyperparameter setting was determined by a grid-search over all
hyperparameter combinations using the validation set defined by the challenge winners \citep{bib:Mayr2016}.
The hyperparameter space was chosen to be similar to the hyperparameters that were tested by \citet{bib:Mayr2016}.
The early stopping parameter was determined on the smoothed learning curves of 100 epochs 
of the validation set. Smoothing was done using moving averages of 10 consecutive 
values. We tested ``rectangular'' and ``conic'' layers -- rectangular layers have 
constant number of hidden units in each layer, conic layers start with the given 
number of hidden units in the first layer and then decrease the number of hidden units
to the size of the output layer according to the geometric progession.
All methods had the chance to adjust their hyperparameters to the data set at hand. 

\index{experiments!Tox21!hyperparameters}
\begin{table}[htp]
\begin{center}
\caption{Hyperparameters considered for self-normalizing networks in the Tox21 data set.}

\begin{tabular}{ll}
\toprule
Hyperparameter  & Considered values \\ 
\midrule
  Number of hidden units & \{1024, 2048\} \\
  Number of hidden layers & \{2,3,4,6,8,16,32\} \\
  Learning rate & \{0.01, 0.05, 0.1\} \\
  Dropout rate & \{0.05, 0.10\}\\
  Layer form & \{rectangular, conic\} \\
  L2 regularization parameter &  \{0.001,0.0001,0.00001\} \\
\bottomrule
\end{tabular}
\end{center}

\end{table}

\begin{table}[htp]
\begin{center}
\caption[Hyperparameters considered for ReLU networks in the Tox21 data set.]{Hyperparameters considered for ReLU networks with MS initialization in the Tox21 data set.}
\begin{tabular}{ll}
\toprule
Hyperparameter  & Considered values \\ 
\midrule
   Number of hidden units & \{1024, 2048\} \\
  Number of hidden layers & \{2,3,4,6,8,16,32\} \\
  Learning rate & \{0.01, 0.05, 0.1\} \\
  Dropout rate & \{0.5, 0\}\\
  Layer form & \{rectangular, conic\} \\
  L2 regularization parameter &  \{0.001,0.0001,0.00001\} \\
\bottomrule
\end{tabular}
\end{center}
\end{table}


\begin{table}[htp]
\begin{center}
\caption{Hyperparameters considered for batch normalized networks in the Tox21 data set.}

\begin{tabular}{ll}
\toprule
Hyperparameter  & Considered values \\ 
\midrule
  Number of hidden units & \{1024, 2048\} \\
  Number of hidden layers & \{2, 3, 4, 6, 8, 16, 32\} \\
  Learning rate & \{0.01, 0.05, 0.1\} \\
  Normalization & \{Batchnorm\} \\
  Layer form & \{rectangular, conic\} \\
  L2 regularization parameter &  \{0.001,0.0001,0.00001\} \\

\bottomrule
\end{tabular}
\end{center}

\end{table}


\begin{table}[htp]
\begin{center}
\caption{Hyperparameters considered for weight normalized networks in the Tox21 data set.}

\begin{tabular}{ll}
\toprule
Hyperparameter  & Considered values \\ 
\midrule
  Number of hidden units & \{1024, 2048\} \\
  Number of hidden layers & \{2, 3, 4, 6, 8, 16, 32\} \\
  Learning rate & \{0.01, 0.05, 0.1\} \\
  Normalization & \{Weightnorm\} \\
  Dropout rate & \{0, 0.5\} \\
  Layer form & \{rectangular, conic\} \\
  L2 regularization parameter &  \{0.001,0.0001,0.00001\} \\
\bottomrule
\end{tabular}
\end{center}

\end{table}


\begin{table}[htp]
\begin{center}
\caption{Hyperparameters considered for layer normalized networks in the Tox21 data set.}

\begin{tabular}{ll}
\toprule
Hyperparameter  & Considered values \\ 
\midrule
  Number of hidden units & \{1024, 2048\} \\
  Number of hidden layers & \{2, 3, 4, 6, 8, 16, 32\} \\
  Learning rate & \{0.01, 0.05, 0.1\} \\
  Normalization & \{Layernorm\} \\
  Dropout rate & \{0, 0.5\} \\
  Layer form & \{rectangular, conic\} \\
  L2 regularization parameter &  \{0.001,0.0001,0.00001\} \\

\bottomrule
\end{tabular}
\end{center}

\end{table}


\begin{table}[htp]
\begin{center}
\caption{Hyperparameters considered for Highway networks in the Tox21 data set.}

\begin{tabular}{ll}
\toprule
Hyperparameter  & Considered values \\ 
\midrule
  Number of hidden layers & \{2, 3, 4, 6, 8, 16, 32\} \\
  Learning rate & \{0.01, 0.05, 0.1\} \\
  Dropout rate & \{0, 0.5\} \\
  L2 regularization parameter &  \{0.001,0.0001,0.00001\} \\
\bottomrule
\end{tabular}
\end{center}

\end{table}


\begin{table}[htp]
\begin{center}
\caption{Hyperparameters considered for Residual networks in the Tox21 data set.}
\begin{tabular}{ll}
\toprule
Hyperparameter  & Considered values \\ 
\midrule
  Number of blocks & \{2, 3, 4, 6, 8, 16\} \\
  Number of neurons per blocks & \{1024, 2048\} \\
  Block form & \{rectangular, diavolo\} \\
  Bottleneck & \{25\%, 50\%\} \\
  Learning rate & \{0.01, 0.05, 0.1\} \\
  L2 regularization parameter &  \{0.001,0.0001,0.00001\} \\
\bottomrule
\end{tabular}
\end{center}
\end{table}


\clearpage

\paragraph{Distribution of network inputs.}
We empirically checked the assumption that the distribution of network inputs can 
well be approximated by a normal distribution. To this end, we investigated the 
density of the network inputs before and during learning and found that 
these density are close to normal distributions (see Figure~\ref{fig:clt}).

\begin{figure}
 \includegraphics[width=0.49\columnwidth]{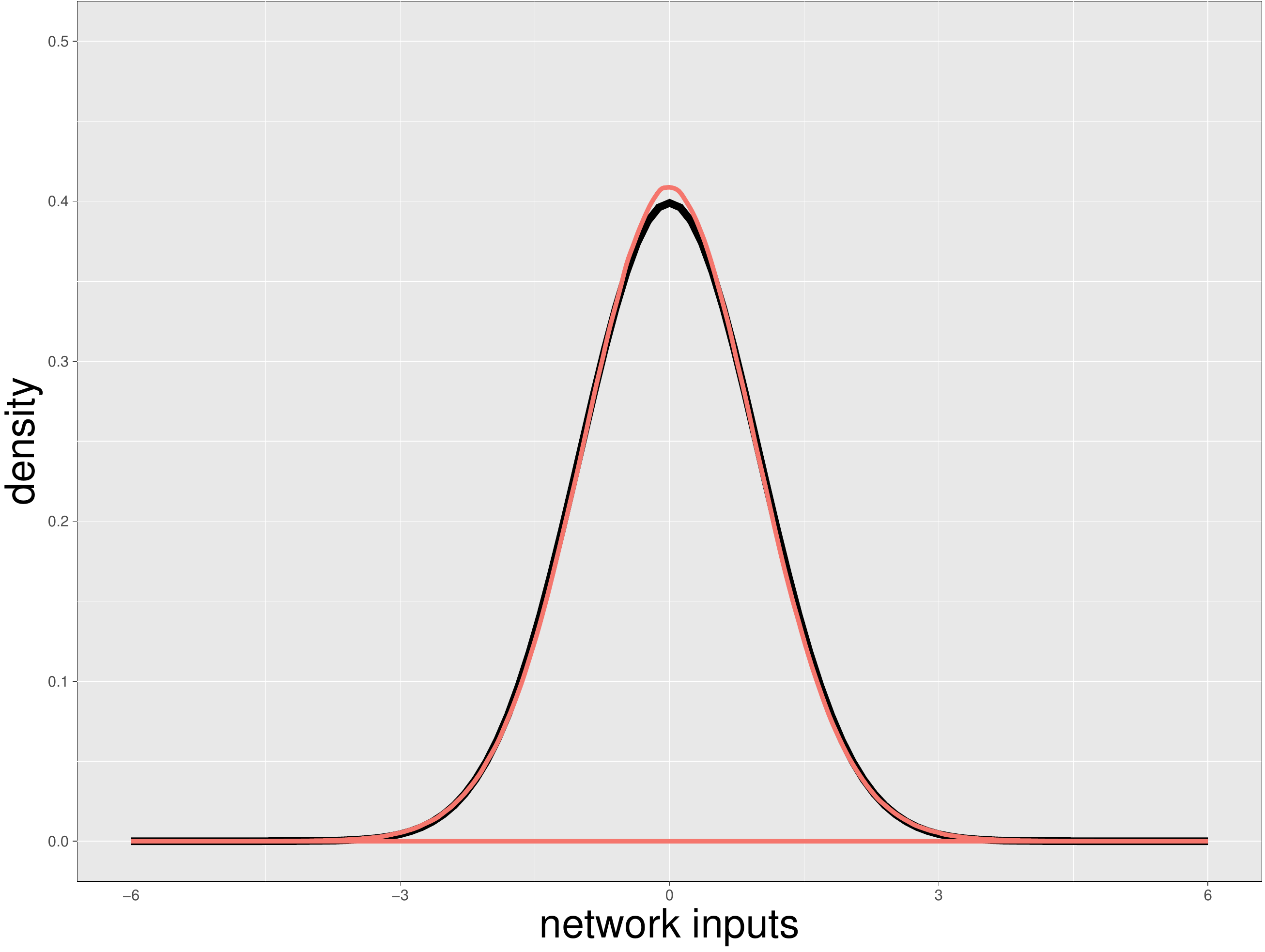}
 \includegraphics[width=0.49\columnwidth]{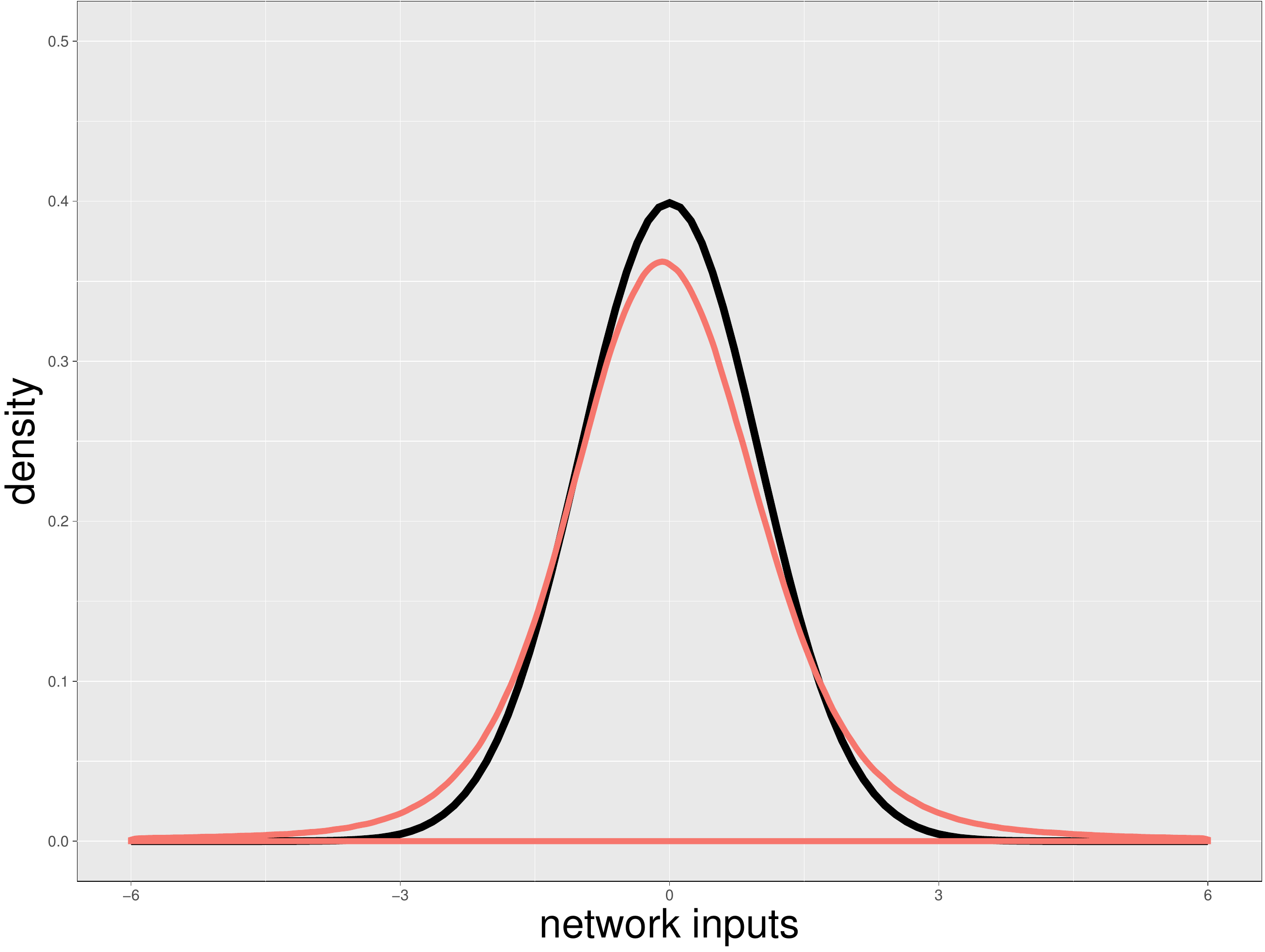}
 \caption[Distribution of network inputs in Tox21 SNNs.]{Distribution of network inputs of an SNN for the Tox21 data set. 
  The plots show the distribution of network inputs $z$ of the second layer of a typical Tox21 network.
  The red curves display a kernel density estimator of the network inputs and the black curve is the 
  density of a standard normal distribution. 
  {\bf Left panel:} At initialization time before learning. The distribution of network inputs is close to a standard 
  normal distribution.
  {\bf Right panel:} After 40 epochs of learning. The distributions of network inputs is close to a normal distribution.
  \label{fig:clt}
 }
\end{figure}

\clearpage

\subsection{HTRU2 data set: Hyperparameters}\index{experiments!HTRU2}
For the HTRU2 data set, the best hyperparameter setting was determined by a grid-search over all
hyperparameter combinations using one of the 9 non-testing folds as validation fold in a nested
cross-validation procedure. Concretely, 
if $M$ was the testing fold, we used $M-1$ as validation fold, and for $M=1$ we used fold $10$
for validation. The early stopping parameter was determined on the smoothed learning curves of 100 epochs 
of the validation set. Smoothing was done using moving averages of 10 consecutive 
values. We tested ``rectangular'' and ``conic'' layers -- rectangular layers have 
constant number of hidden units in each layer, conic layers start with the given 
number of hidden units in the first layer and then decrease the number of hidden units
to the size of the output layer according to the geometric progession.
All methods had the chance to adjust their hyperparameters to the data set at hand. \index{experiments!HTRU2!hyperparameters}

\begin{table}[htp]
\begin{center}
\caption{Hyperparameters considered for self-normalizing networks on the HTRU2 data set.}
\begin{tabular}{ll}
\toprule
Hyperparameter  & Considered values \\ 
\midrule
  Number of hidden units & \{256, 512, 1024\} \\
  Number of hidden layers & \{2, 4, 8, 16, 32\} \\
  Learning rate & \{0.1, 0.01, 1\} \\
  Dropout rate & \{ 0, 0.05\}\\
  Layer form & \{rectangular, conic\} \\
\bottomrule
\end{tabular}
\end{center}
\end{table}

\begin{table}[htp]
\begin{center}
\caption[Hyperparameters considered for ReLU networks on the HTRU2 data set.]{Hyperparameters considered for ReLU networks with Microsoft initialization on the HTRU2 data set.}
\begin{tabular}{ll}
\toprule
Hyperparameter  & Considered values \\ 
\midrule
  Number of hidden units & \{256, 512, 1024\} \\
  Number of hidden layers & \{2, 4, 8, 16, 32\} \\
  Learning rate & \{0.1, 0.01, 1\} \\
  Dropout rate & \{0, 0.5\}\\
  Layer form & \{rectangular, conic\} \\
\bottomrule
\end{tabular}
\end{center}
\end{table}

\begin{table}[htp]
\begin{center}
\caption{Hyperparameters considered for BatchNorm networks on the HTRU2 data set.}
\begin{tabular}{ll}
\toprule
Hyperparameter  & Considered values \\ 
\midrule
  Number of hidden units & \{256, 512, 1024\} \\
  Number of hidden layers & \{2, 4, 8, 16, 32\} \\
  Learning rate & \{0.1, 0.01, 1\} \\
  Normalization & \{Batchnorm\} \\
  Layer form & \{rectangular, conic\} \\
\bottomrule
\end{tabular}
\end{center}
\end{table}

\begin{table}[htp]
\begin{center}
\caption{Hyperparameters considered for WeightNorm networks on the HTRU2 data set.}
\begin{tabular}{ll}
\toprule
Hyperparameter  & Considered values \\ 
\midrule
  Number of hidden units & \{256, 512, 1024\} \\
  Number of hidden layers & \{2, 4, 8, 16, 32\} \\
  Learning rate & \{0.1, 0.01, 1\} \\
  Normalization & \{Weightnorm\} \\
  Layer form & \{rectangular, conic\} \\
\bottomrule
\end{tabular}
\end{center}
\end{table}

\begin{table}[htp]
\begin{center}
\caption{Hyperparameters considered for LayerNorm networks on the HTRU2 data set.}
\begin{tabular}{ll}
\toprule
Hyperparameter  & Considered values \\ 
\midrule
  Number of hidden units & \{256, 512, 1024\} \\
  Number of hidden layers & \{2, 4, 8, 16, 32\} \\
  Learning rate & \{0.1, 0.01, 1\} \\
  Normalization & \{Layernorm\} \\
  Layer form & \{rectangular, conic\} \\
\bottomrule
\end{tabular}
\end{center}
\end{table}

\begin{table}[htp]
\begin{center}
\caption{Hyperparameters considered for Highway networks on the HTRU2 data set.}
\begin{tabular}{ll}
\toprule
Hyperparameter  & Considered values \\ 
\midrule
  Number of hidden layers & \{2, 4, 8, 16, 32\} \\
  Learning rate & \{0.1, 0.01, 1\} \\
  Dropout rate & \{0, 0.5\}\\
\bottomrule
\end{tabular}
\end{center}
\end{table}

\begin{table}[htp]
\begin{center}
\caption{Hyperparameters considered for Residual networks on the HTRU2 data set.}
\begin{tabular}{ll}
\toprule
Hyperparameter  & Considered values \\ 
\midrule
  Number of hidden units & \{256, 512, 1024\} \\
  Number of residual blocks & \{2, 3, 4, 8, 16\} \\
  Learning rate & \{0.1, 0.01, 1\} \\
  Block form & \{rectangular, diavolo\} \\
  Bottleneck & \{0.25, 0.5\} \\
\bottomrule
\end{tabular}
\end{center}
\end{table}

\clearpage

\section{Other fixed points}
A similar analysis with corresponding function domains can be performed for other fixed points, for example for $\mu=\munn=0$ and $\nu=\nunn=2$, which leads
to a SELU activation function with parameters $\alpha_{\mathrm{02}}=1.97126$ and $\lambda_{\mathrm{02}}=1.06071$.

\section{Bounds determined by numerical methods}

In this section we report bounds on previously discussed expressions as determined by numerical methods (min and max have been
computed).

\begin{align}
0_ {(\mu = 0.06, \omega = 0, \nu = 1.35, \tau = 1.12)}\ &< \ \frac{\partial {\mathcal J}_{11}}{\partial \mu}  \ < \ .00182415_{(\mu = -0.1, \omega = 0.1, \nu = 1.47845, \tau = 0.883374)}\\ \nonumber
0.905413_{(\mu = 0.1, \omega = -0.1, \nu = 1.5, \tau = 1.25)}\ &< \  \frac{\partial {\mathcal J}_{11}}{\partial \omega}  \ < \ 1.04143_{(\mu = 0.1, \omega = 0.1, \nu = 0.8, \tau = 0.8)}\\ \nonumber
-0.0151177_{(\mu = -0.1, \omega = 0.1, \nu = 0.8, \tau = 1.25)}\ &< \  \frac{\partial {\mathcal J}_{11}}{\partial \nu}  \ < \ 0.0151177_{(\mu = 0.1, \omega = -0.1, \nu = 0.8, \tau = 1.25)}\\ \nonumber
-0.015194_{(\mu = -0.1, \omega = 0.1, \nu = 0.8, \tau = 1.25)}\ &< \  \frac{\partial {\mathcal J}_{11}}{\partial \tau}  \ < \ 0.015194_{(\mu = 0.1, \omega = -0.1, \nu = 0.8, \tau = 1.25)}\\ \nonumber
-0.0151177_{(\mu = -0.1, \omega = 0.1, \nu = 0.8, \tau = 1.25)}\ &< \  \frac{\partial {\mathcal J}_{12}}{\partial \mu}  \ < \ 0.0151177_{(\mu = 0.1, \omega = -0.1, \nu = 0.8, \tau = 1.25)}\\ \nonumber
-0.0151177_{(\mu = 0.1, \omega = -0.1, \nu = 0.8, \tau = 1.25)}\ &< \  \frac{\partial {\mathcal J}_{12}}{\partial \omega}  \ < \ 0.0151177_{(\mu = 0.1, \omega = -0.1, \nu = 0.8, \tau = 1.25)}\\ \nonumber
-0.00785613_{(\mu = 0.1, \omega = -0.1, \nu = 1.5, \tau = 1.25)}\ &< \  \frac{\partial {\mathcal J}_{12}}{\partial \nu}  \ < \ 0.0315805_{(\mu = 0.1, \omega = 0.1, \nu = 0.8, \tau = 0.8)}\\ \nonumber
0.0799824_{(\mu = 0.1, \omega = -0.1, \nu = 1.5, \tau = 1.25)}\ &< \ \frac{\partial {\mathcal J}_{12}}{\partial \tau}  \ < \ 0.110267_{(\mu = -0.1, \omega = 0.1, \nu = 0.8, \tau = 0.8)}\\ \nonumber
0_{(\mu = 0.06, \omega = 0, \nu = 1.35, \tau = 1.12)}\ &< \ \frac{\partial {\mathcal J}_{21}}{\partial \mu}  \ < \ 0.0174802_{(\mu = 0.1, \omega = 0.1, \nu = 0.8, \tau = 0.8)} \\ \nonumber
0.0849308_{(\mu = 0.1, \omega = -0.1, \nu = 0.8, \tau = 0.8)}\ &< \  \frac{\partial {\mathcal J}_{21}}{\partial \omega}  \ < \ 0.695766_{(\mu = 0.1, \omega = 0.1, \nu = 1.5, \tau = 1.25)}\\ \nonumber
-0.0600823_{(\mu = 0.1, \omega = -0.1, \nu = 0.8, \tau = 1.25)}\ &< \ \frac{\partial {\mathcal J}_{21}}{\partial \nu}  \ < \ 0.0600823_{(\mu = -0.1, \omega = 0.1, \nu = 0.8, \tau = 1.25)}\\ \nonumber
-0.0673083_{(\mu = 0.1, \omega = -0.1, \nu = 1.5, \tau = 0.8)}\ &< \ \frac{\partial {\mathcal J}_{21}}{\partial \tau}  \ < \ 0.0673083_{(\mu = -0.1, \omega = 0.1, \nu = 1.5, \tau = 0.8)}\\ \nonumber
-0.0600823_{(\mu = 0.1, \omega = -0.1, \nu = 0.8, \tau = 1.25)}\ &< \  \frac{\partial {\mathcal J}_{22}}{\partial \mu}  \ < \ 0.0600823_{(\mu = -0.1, \omega = 0.1, \nu = 0.8, \tau = 1.25)}\\ \nonumber
-0.0600823_{(\mu = 0.1, \omega = -0.1, \nu = 0.8, \tau = 1.25)}\ &< \  \frac{\partial {\mathcal J}_{22}}{\partial \omega}  \ < \ 0.0600823_{(\mu = -0.1, \omega = 0.1, \nu = 0.8, \tau = 1.25)}\\ \nonumber
-0.276862_{(\mu = -0.01, \omega = -0.01, \nu = 0.8, \tau = 1.25)}\ &< \  \frac{\partial {\mathcal J}_{22}}{\partial \nu}  \ < \ -0.084813_{(\mu = -0.1, \omega = 0.1, \nu = 1.5, \tau = 0.8)}\\ \nonumber
0.562302_{(\mu = 0.1, \omega = -0.1, \nu = 1.5, \tau = 1.25)}\ &< \ \frac{\partial {\mathcal J}_{22}}{\partial \tau}  \ < \ 0.664051_{(\mu = 0.1, \omega = 0.1, \nu = 0.8, \tau = 0.8)}
\end{align}

\begin{align}
\left| \frac{\partial {\mathcal J}_{11}}{\partial \mu} \right| \ &< \ 0.00182415 (0.0031049101995398316) \\ \nonumber
\left| \frac{\partial {\mathcal J}_{11}}{\partial \omega} \right| \ &< \ 1.04143 (1.055872374194189) \\ \nonumber
\left| \frac{\partial {\mathcal J}_{11}}{\partial \nu} \right| \ &< \ 0.0151177 (0.031242911235461816) \\ \nonumber
\left| \frac{\partial {\mathcal J}_{11}}{\partial \tau} \right| \ &< \ 0.015194 (0.03749149348255419) \\ \nonumber
\left| \frac{\partial {\mathcal J}_{12}}{\partial \mu} \right| \ &< \ 0.0151177 (0.031242911235461816) \\ \nonumber
\left| \frac{\partial {\mathcal J}_{12}}{\partial \omega} \right| \ &< \ 0.0151177 (0.031242911235461816) \\ \nonumber
\left| \frac{\partial {\mathcal J}_{12}}{\partial \nu} \right| \ &< \ 0.0315805 (0.21232788238624354) \\ \nonumber
\left| \frac{\partial {\mathcal J}_{12}}{\partial \tau} \right| \ &< \ 0.110267 (0.2124377655377270) \\ \nonumber
\left| \frac{\partial {\mathcal J}_{21}}{\partial \mu} \right| \ &< \ 0.0174802 (0.02220441024325437) \\ \nonumber
\left| \frac{\partial {\mathcal J}_{21}}{\partial \omega} \right| \ &< \ 0.695766 (1.146955401845684) \\ \nonumber
\left| \frac{\partial {\mathcal J}_{21}}{\partial \nu} \right| \ &< \ 0.0600823 (0.14983446469110305) \\ \nonumber
\left| \frac{\partial {\mathcal J}_{21}}{\partial \tau} \right| \ &< \ 0.0673083 (0.17980135762932363) \\ \nonumber
\left| \frac{\partial {\mathcal J}_{22}}{\partial \mu} \right| \ &< \ 0.0600823 (0.14983446469110305) \\ \nonumber
\left| \frac{\partial {\mathcal J}_{22}}{\partial \omega} \right| \ &< \ 0.0600823 (0.14983446469110305) \\ \nonumber
\left| \frac{\partial {\mathcal J}_{22}}{\partial \nu} \right| \ &< \ 0.562302 (1.805740052651535) \\ \nonumber
\left| \frac{\partial {\mathcal J}_{22}}{\partial \tau} \right| \ &< \ 0.664051 (2.396685907216327)
\end{align}

\section{References}
\label{sec:references}
\bibliographystyle{apalike} 
\bibliography{bibliography_fancy} 

\addcontentsline{toc}{section}{List of figures}
\listoffigures
\addcontentsline{toc}{section}{List of tables}
\listoftables

\newpage
\addcontentsline{toc}{section}{Brief index}
\printindex 

\end{document}